\documentclass[twoside,openright]{report}
\usepackage[utf8]{inputenc}
\usepackage[T1,T2A]{fontenc}

\usepackage{geometry}
\usepackage{csquotes}
\usepackage[russian,french,english]{babel}
\usepackage{amsmath,amsfonts,amssymb}
\usepackage{empheq}
\usepackage{dsfont}
\usepackage{mathrsfs}
\usepackage{siunitx} 
\usepackage[ruled,linesnumbered,vlined,algochapter]{algorithm2e}
\usepackage{multirow}
\usepackage{float}
\usepackage{array}
\usepackage{diagbox}
\usepackage{setspace}
\usepackage{catchfile}
\usepackage[normalem]{ulem}
\usepackage{subcaption}
\usepackage{nameref}
\usepackage[bookmarks,hidelinks,pdftitle={Low-Rank Matrix Factorizations with Volume-based Constraints and Regularizations},pdfauthor={Olivier Vu Thanh},pdfkeywords={low-rank matrix factorization, matrix factorization, identifiability, uniqueness, interpretability, algorithms, volume, blind source separation, hyperspectral unmixing, matrix completion}]{hyperref} 
\usepackage{todonotes}
\usepackage[c2]{optidef}

\usepackage{fancyhdr}
\usepackage{tocbibind}

\usepackage[title]{appendix}

\usepackage{tikz}
\usepackage{tikz-3dplot}
\usepackage{pgfplots}
\pgfplotsset{compat=newest}
\usepgfplotslibrary{statistics}
\usepgfplotslibrary{fillbetween}
\usepgfplotslibrary{groupplots}
\usetikzlibrary{shapes,calc,positioning,spy}

\newcommand\loaddata[1]{\CatchFileDef\loadeddata{#1}{\endlinechar=-1}} 

\def\addlegendimage{\csname pgfplots@addlegendimage\endcsname}

\usetikzlibrary{external}
\usepackage{xcolor}
\definecolor{mycolor}{rgb}{0.5961,0.3059,0.6392}%
\definecolor{mypurple}{rgb}{0.5961,0.3059,0.6392}%
\definecolor{mygreen}{rgb}{0.7020,0.8706,0.4118}%
\definecolor{myorange}{rgb}{0.9843,0.5020,0.4471}%
\definecolor{myblue}{rgb}{0.00000,0.44700,0.74100}%



\newenvironment{hiddenmusicsec}[2]{
  \begingroup
  \color{gray!10} 
  \par\noindent
  \raggedleft
  \raisebox{2.5\baselineskip}[0pt][0pt]{\href{#2}{#1}} 
  \par\vspace{-2.5\baselineskip}
}{
  \endgroup
}
\newenvironment{hiddenmusic}[2]{
  \begingroup
  \color{gray!10} 
  \par\noindent
  \raggedleft
  \raisebox{2\baselineskip}[0pt][0pt]{\href{#2}{#1}} 
  \par\vspace{-\baselineskip}
}{
  \endgroup
  \noindent
}

\newtheorem{lemma}{Lemma}

\newtheorem{theorem}{Theorem}
\newtheorem{corollary}{Corollary}
\newtheorem{definition}{Definition}
\newtheorem{example}{Example}

\newtheorem{assumption}{Assumption}
\newtheorem{remark}{Remark}

\newtheorem{proof}{Proof}

\newcommand{\overbar}[1]{\mkern 1.5mu\overline{\mkern-1.5mu#1\mkern-1.5mu}\mkern 1.5mu}
\numberwithin{conj}{chapter}
\numberwithin{lemma}{chapter}
\numberwithin{cor}{chapter}
\numberwithin{theorem}{chapter}
\numberwithin{corollary}{chapter}
\numberwithin{definition}{chapter}
\numberwithin{example}{chapter}
\numberwithin{question}{chapter}
\numberwithin{assumption}{chapter}
\numberwithin{remark}{chapter}
\numberwithin{problem}{chapter}
\numberwithin{proof}{chapter}

\DeclareMathOperator{\conv}{conv}
\DeclareMathOperator*{\argmax}{argmax}
\DeclareMathOperator*{\argmin}{argmin}
\DeclareMathOperator{\cone}{cone}
\DeclareMathOperator{\rank}{rank}

\DeclareMathOperator{\col}{col}

\DeclareMathOperator{\logdet}{logdet}
\DeclareMathOperator{\tr}{tr}
\DeclareMathOperator{\diag}{Diag}

\DeclareMathOperator*{\bd}{bd}
\DeclareMathOperator*{\ext}{ext}
\DeclareMathOperator*{\vol}{vol}
\DeclareMathOperator*{\dom}{dom}

\newcommand{\R}{\mathbb{R}}
\renewcommand{\P}{\mathcal{P}}
\renewcommand{\C}{\mathcal{C}}

\newcommand{\E}{\mathcal{E}}
\newcommand{\Wt}{{W^\top}}

\newcommand{\Wold}{W_o}
\newcommand{\Hold}{H_o}
\newcommand{\Wextra}{\overbar{W}}
\newcommand{\Hextra}{\overbar{H}}
\newcommand{\Wextraold}{\overbar{W}_o}
\newcommand{\Hextraold}{\overbar{H}_o}
\newcommand{\La}{\mathcal{L}}

\newcommand{\lzero}{$\ell_0$}
\newcommand{\lzeronorm}{\lzero{}-``norm''}
\newcommand{\lone}{$\ell_1$}
\newcommand{\lonenorm}{\lone{}-norm}
\newcommand{\ltwo}{$\ell_2$}
\newcommand{\ltwonorm}{\ltwo{}-norm}

\definecolor{DarkGreen}{rgb}{0.13,0.55,0.13}%

\usepackage[style=nature,sorting=nty,maxnames=99]{biblatex}
\def\MYTITLE{Low-Rank Matrix Factorizations with Volume-based Constraints and Regularizations}
\title{\MYTITLE{}}

\definecolor{brightpink}{rgb}{1.0, 0.0, 0.5}

\def\MYAUTHOR{Olivier Vu Thanh}
\author{\MYAUTHOR{}}

\date{\today}

\usepackage[capitalize]{cleveref}
\Crefname{line}{line}{lines}
\Crefname{assumption}{Assumption}{Assumptions}
\begin{document}

\begin{titlepage}
  \begin{center}

    \large
    Université de Mons\\
    Faculté Polytechnique\\
    Mathématique et Recherche Opérationnelle

    \vspace{3cm}

    \Huge
    \hrule
    \vspace{0.4cm}
    \MYTITLE{}
    \vspace{0.4cm}
    \hrule

    \vspace{1.5cm}

    \LARGE
    \MYAUTHOR{}

    \vfill

    \normalsize
    A thesis presented in partial fulfillment of the requirements for the degree of Docteur en Sciences de l'Ingénieur et Technologies

    \vspace{0.8cm}

    Dissertation committee:
    \begin{table}[h]
      \begin{tabular}{lll}
        Prof.\ Nicolas Gillis       & Université de Mons & Supervisor    \\
        Prof.\ Fabian Lecron        & Université de Mons & Co-supervisor \\
        Prof.\ Arnaud Vandaele      & Université de Mons & Chair \\
        Prof.\ Kejun Huang          & University of Florida & \\
        Prof.\ Clémence Prévost     & University of Lille & \\
        Prof.\ Matthieu Puigt       & Universit\'e du Littoral C\^ote d'Opale & \\
      \end{tabular}
    \end{table}

  \end{center}
\end{titlepage}

\chapter*{Abstract}
\addcontentsline{toc}{chapter}{Abstract}

Low-rank matrix factorizations (LRMFs) are a class of linear models widely used in various fields such as machine learning, signal processing, and data analysis. These models approximate a matrix as the product of two smaller matrices, where the left matrix captures latent features—the most important components of the data—while the right matrix linearly decomposes the data based on these features. There are many ways to define what makes a component "important." Standard LRMFs, such as the truncated singular value decomposition, focus on minimizing the distance between the original matrix and its low-rank approximation. In this thesis, the notion of "importance" is closely linked to interpretability and uniqueness, which are key to obtaining reliable and meaningful results.\\

This thesis thus focuses on volume-based constraints and regularizations designed to enhance interpretability and uniqueness. We first introduce two new volume-constrained LRMFs designed to enhance these properties. The first assumes that data points are naturally bounded (e.g., movie ratings between 1 and 5 stars) and can be explained by convex combinations of features within the same bounds, allowing them to be interpreted in the same way as the data. The second model is more general, constraining the factors to belong to convex polytopes. 
Then, two variants of volume-regularized LRMFs are proposed. The first minimizes the volume of the latent features, encouraging them to cluster closely together, while the second maximizes the volume of the decompositions, promoting sparse representations. Across all these models, uniqueness is achieved under the core principle that the factors must be "sufficiently scattered" within their respective feasible sets.\\

Motivated by applications such as blind source separation (e.g., hyperspectral unmixing) and missing data imputation (e.g., in recommender systems), this thesis also proposes efficient algorithms that make these models scalable and practical for real-world applications.

\chapter*{Résumé}

Les factorisations matricielles de faible rang (LRMFs) sont des modèles linéaires largement utilisés dans des domaines tels que l'apprentissage automatique, le traitement du signal et l'analyse de données. Ces modèles approchent une matrice en la décomposant en produit de deux matrices plus petites~: la première capture les caractéristiques latentes, c'est-à-dire les composantes les plus importantes des données, tandis que la seconde décompose linéairement les données à partir de ces caractéristiques. Il existe cependant de nombreuses manières de définir ce qui rend une composante "importante". Les LRMFs classiques, comme la décomposition en valeurs singulières tronquée, se concentrent sur la minimisation de la distance entre la matrice originale et son approximation de faible rang. Dans cette thèse, l'importance d'une composante est étroitement déterminée par l’interprétabilité et l’unicité, des notions clés pour obtenir des résultats fiables et pertinents.\\

Cette thèse explore donc des contraintes et régularisations volumiques visant à renforcer l’interprétabilité et l’unicité. Dans un premier temps, nous introduisons deux nouvelles variantes de LRMFs à contraintes volumiques. La première suppose que les points du jeu de données sont naturellement bornés (ex: des films notés entre 1 et 5 étoiles) et peuvent être expliqués par des combinaisons convexes de caractéristiques bornées de la même manière, permettant ainsi de les interpréter comme les données. Le second modèle est plus général et contraint les facteurs à appartenir à des polytopes convexes. Par ailleurs, nous proposons deux variantes de LRMF avec régularisation volumique~: la première minimise le volume des caractéristiques latentes, favorisant ainsi un rapprochement entre elles, tandis que la seconde maximise le volume des décompositions, encourageant des représentations parcimonieuses. Dans l’ensemble de ces modèles, l’unicité est assurée par le principe clé selon lequel les facteurs doivent être "suffisamment dispersés" dans leur ensemble de solutions possibles.\\

Motivée par des applications telles que la séparation de sources aveugles (par exemple, le démélange hyperspectral) et l’imputation de données manquantes (par exemple, dans les systèmes de recommandation), cette thèse propose également des algorithmes efficaces permettant à ces modèles d’être adaptés à des applications réelles.

\chapter*{Acknowledgements}
I would like first to thank my PhD supervisor Nicolas Gillis, from whom I learned a lot, humanly and scientifically.\\

\noindent I am thankful to the colleagues I met, chronologically: Andersen, Hien for welcoming me, Fabian my co-supervisor, Arnaud for our discussions, Pierre, Nadisic мой товарищ, Christos my hearty laughing neighbor, Atharva the flawless, Seraghiti the early bird coffee lover\footnote{addict}, Subhayan the connoisseur\footnote{in debauchery}, Barbarino the maidenless tarnished who should put these foolish ambitions to rest, Timothy the original Belgian, the goofy Florian and Amjad for his kindness.\\

\noindent Special thanks to Jule and Junior, the guardians of the Houdain park.\\

\noindent I also would like to thank the teachers from Grenoble INP who indirectly gave me the will to pursue a PhD.\\

\noindent I thank the jury members for agreeing to evaluate this thesis and for their useful comments.\\

Finally, thanks to Hélène for sailing with me.\\

\vfill \hfill I dedicate this thesis to Wallace.

\newpage
\phantom{a}

\newpage
\vspace{2em}
I acknowledge the support by the European Research Council (ERC Starting Grant, COLORAMAP, no 679515, and ERC consolidator Grant, eLinoR, no 101085607), by the Fonds de la Recherche Scientifique (F.R.S.) - FNRS and the Fonds Wetenschappelijk Onderzoek - Vlanderen (FWO) under EOS Project no O005318F-RG47, by the Francqui Foundation, by the F.R.S.-FNRS under the Research Project T.0097.2 and under a FRIA PhD grant.

\newpage
\thispagestyle{empty}

\phantom{a}
\vspace{0.3\textheight}

\tableofcontents

\newpage
\section*{Notation}
\addcontentsline{toc}{chapter}{Notation}

\begin{center}
\begin{tabular}{p{0.13\textwidth} p{0.8\textwidth}}
  $\mathbb{R}$ & set of real numbers \\
  $\mathbb{R}_*$ & set of real nonzero numbers \\
  $\mathbb{R}_{+}$ & set of real nonnegative numbers \\
  $\mathbb{R}^m$ & set of real column vectors of dimension $m$ \\
  $\mathbb{R}^{m \times n}$ & set of real matrices of dimension $m \times n$ \\
  $x_i$ or $x(i)$ & $i$-th entry of the vector $x$ \\
  $x(K)$ & subvector $x$ with indices in the set $K$ \\
  $A(i,:)$ & $i$-th row of the matrix $A$ \\
  $A(:,j)$ & $j$-th column of the matrix $A$ \\
  $A(i,j)$ & entry of the matrix $A$ indexed by $(i,j)$ \\
  $A(:,J)$ & submatrix of $A$ with column indices in the set $J$ \\
  $A^\top$ & transpose of the matrix $A$  \\
  $A^{-\top}$ & inverse of the transpose of the square matrix $A$\\
  $A\circ B$ & Hadamard product, that is $(A\circ B)(i,j)=A(i,j)B(i,j)$\\
  $e$ & vector of all ones of appropriate dimension\\
  $e_i$ & $i$-th canonical vector of appropriate dimension\\
  $J$ & matrix of all ones of appropriate dimension\\
  $E_{i,j}$ & matrix whose $(i,j)$-th element is equal to one and zero elsewhere, that is $e_i e_j^\top$ of appropriate dimension\\
  $x \geq 0$ & the vector $x$ is entry-wise nonnegative\\
  $A \geq 0$ & the matrix $A$ is entry-wise nonnegative\\
  $\Delta^r$ & probability simplex, ${\{x\in\R^r \ | \ x\geq0,e^\top x=1\}}$\\ 
  $\Delta^{r \times n}$ & set of matrices whose columns lies in $\Delta^r$\\
  $\mathbb{S}^n$ & set of symmetric $n\times n$ matrices, $\{X\in\R^{n\times n} \ | \ X=X^\top\}$\\
  $\mathbb{S}^n_+$ & set of symmetric positive semidefinite matrices, $\{X\in\mathbb{S}^n \ | \ X \succeq 0 \}$\\
  $\mathbb{S}^n_{++}$ & set of symmetric positive definite matrices, $\{X\in\mathbb{S}^n \ | \ X \succ 0 \}$\\
  $\cone(A)$ & conical hull of the columns of $A$, $\{ y \ | \ y = Az, z \geq 0\}$\\
  $\conv(A)$ & convex hull of the columns of matrix $A\in\R^{m\times n}$, ${\{ y \ | \ y = Az, z\in\Delta^n \}}$\\
  $\ext(\mathcal{X})$ & set of extreme points of the set $\mathcal{X}$\\
  $\bd(\mathcal{X})$ & boundary of the set $\mathcal{X}$ \\
  $\mathcal{X}^{*,g}$ & polar of the set $\mathcal{X}\subset\R^r$ with respect to $g$, that is, ${\{x\in\R^r \ | \ \langle x,y-g\rangle \geq 0, \text{ for all } y\in\mathcal{X}\}}$\\
  $\kappa(A)$ & condition number of the matrix $A$ \\
  $|S|$ & Cardinality of the set $S$, that is the number of elements in $S$\\
  $\|x\|_0$ & \lzeronorm{} of vector $x$, $ |\{i \ | \ x_i \neq 0 \}|$ \\
  $\|x\|_1$ & \lonenorm{} of vector $x \in \mathbb{R}^r$, $\sum_{i=1}^{r} | x_i |$ \\
  $\|x\|_2$ & \ltwonorm{} of vector $x \in \mathbb{R}^r$, $ \sqrt{\sum_{i=1}^r |x_i|^2}$ \\
  $\|A\|_F$ & Frobenius norm of matrix $A \in \mathbb{R}^{m \times r}$, $\sqrt{\sum_{i=1}^{m} \sum_{j=1}^{r} A{(i,j)}^2} $ \\
  $\|A\|$ & Spectral norm of matrix $A$, that is, its largest singular value\\
\end{tabular}
\end{center}

\section*{Acronyms}

\begin{tabular}{p{0.18\textwidth} p{0.62\textwidth} p{0.19\textwidth}}
  BSSMF & bounded simplex-structured matrix factorization & p.~\pageref{acro:BSSMF}\\
  CLRMF & constrained low-rank matrix factorization & p.~\pageref{acro:CLRMF}\\
  HU & hyperspectral unmixing & p.~\pageref{acro:HU}\\
  MinVol & minimum-volume & p.~\pageref{acro:MinVol}\\
  MinVol NMF & minimum volume matrix factorization & p.~\pageref{acro:MinVolNMF}\\
  MaxVol NMF & maximum volume matrix factorization & p.~\pageref{acro:MaxVolNMF}\\
  MVIE & maximum-volume inscribed ellipsoid & p.~\pageref{acro:MVIE}\\
  NMF & nonnegative matrix factorization & p.~\pageref{acro:NMF}\\
  ONMF & orthogonal matrix factorization & p.~\pageref{acro:ONMF}\\
  PMF & polytopic matrix factorization & p.~\pageref{acro:PMF}\\
  RandSPA & randomized successive matrix factorization & p.~\pageref{acro:RandSPA}\\
  SPA & successive projection algorithm & p.~\pageref{acro:SPA}\\
  SSC & sufficiently scattered conditions & p.~\pageref{preli:def:ssc}\\ 
  SSMF & simplex-structured matrix factorization & p.~\pageref{acro:SSMF}\\
  TITAN & inerTial block majorIzation minimization framework for non-smooth non-convex opTimizAtioN & p.~\pageref{acro:titan}\\
  VCA & vertex component analysis & p.~\pageref{acro:VCA}\\
\end{tabular}

\listoffigures

\listoftables

\chapter{Introduction}\label{chap:intro}
\begin{hiddenmusicsec}{ZEAL \& ARDOR - Sacrilegium III}{https://zealandardor.bandcamp.com/track/sacrilegium-iii-3}
\end{hiddenmusicsec}

\section*{Motivations}\addcontentsline{toc}{section}{Motivations}

The objective of machine learning is mainly to predict, 
classify or 
analyze data. This is usually done by using an algorithm that recognizes common and useful features in the data, according to a model. Compared to data-driven approaches, model-based approaches require more understanding of the data, but less amount of data during the learning. Particularly, linear models are interesting for their simplicity and interpretability. Consider some data stored in a matrix $X\in\R^{m\times n}$ where $m$ represents the dimension of a sample and $n$ the number the samples. A general linear model assumes that $X$ can be written as $X=WH+N$, where $W\in\R^{m\times r}$ can be interpreted as a basis matrix with each column of $W$ representing a feature, $H\in\R^{r\times n}$ can be interpreted as a decomposition of $X$ into the $W$ basis, and $N\in\R^{m\times n}$ is noise and model misfit. Take the $j$-th sample $X(:,j)$, it can be approximated by $$X(:,j)\approx WH(:,j)=\sum_{k=1}^{r}H(k,j)W(:,k).$$ In other words, each sample can be approximated by a weighted sum of features. The features are stored in $W$ and the weights are stored in $H$. This simple, yet powerful, data representation technique is applied in many domains, e.g., facial feature extraction~\cite{lee1999learning}, document clustering~\cite{fu2016robust}, blind source separation~\cite{ma2014signal, ozerov2009multichannel}, data fusion~\cite{prevost2023data}, demosaicing~\cite{abbas2024locally}, community detection~\cite{sorensen2022overlapping}, gene expression analysis~\cite{zhang2010binary}, in situ calibration of sensors~\cite{vuthanh2021insitu}, and recommender systems~\cite{rendle2022revisiting}. When $r<\rank(X)$, we refer to such models as low-rank matrix approximations. 
Low-rank matrix approximations/factorizations are linear dimension reduction techniques, that have recently emerged as very efficient models for unsupervised learning; see, e.g.,~\cite{Vaswani2018PCA, udell2019big} and the references therein. The most notable example is principal component analysis (PCA), which can be solved efficiently via the Singular Value Decomposition (SVD). In the last 20 years, many new more sophisticated models have been proposed, such as sparse PCA that requires one of the factors to be sparse to improve interpretability~\cite{dAspremont2007spca}, robust PCA to handle gross corruption and outliers~\cite{chandrasekaran2011rank, candes2011robust}, and low-rank matrix completion, also known as PCA with missing data, to handle missing entries in the input matrix~\cite{koren2009matrix}. 
The low-rank assumption supposes that there is redundancy in the data that can be explained linearly. Typically, the factors $W$ and $H$ are learned by minimizing an objective function. Different objective functions will promote different behaviors. The main objective function used in this thesis is the Frobenius norm, that is, $\|X-WH\|_F^2=\sum_{(i,j)}(X(i,j)-W(i,:)H(:,j))^2$. For the Frobenius norm, the best rank $r$ matrix approximation is given by the truncated SVD. This result is also known as the Eckart-Young theorem~\cite{EckartYoung36}. Depending on the application, the data and the goal at hand (e.g., clustering, denoising, feature extraction), additional structures/constraints on the factors $W$ and/or $H$, such as sparsity, nonnegativity and statistical independence to name a few, are more or less relevant in order to favor specific structures. We then talk of a Constrained Low-Rank Matrix Factorization (CLRMF)\label{acro:CLRMF}. In this thesis, we particularly focus on CLRMFs that encourage uniqueness, that is, a unique retrieval of $W$ and $H$. Uniqueness is also called identifiability. More details on identifiability are given in \Cref{preli:sec:identif}. Identifiability is useful in applications where the true underlying features and decomposition are desired, like in hyperspectral unmixing for instance where we aim at recovering the true materials present in the image along with their abundances in each pixel; see below for more details.\\



\section*{Applications}\addcontentsline{toc}{section}{Applications}

CLRMF is a very generic model and can be used in many applications. It can be used for, but it is not limited to, data imputation, noise reduction, data visualization and cluster analysis. Here, we mention two applications that will be often used in this thesis.

\begin{itemize}
    \item \textbf{Hyperspectral Unmixing (HU)}\label{acro:HU} 
    Light can be made of several electromagnetic waves that include radio waves, microwaves, infrared, visible light, ultraviolet, X-rays, and gamma rays. When light hits a material, this material absorbs an amount of the light, depending on the wavelengths of the electromagnetic spectrum. Some of the light is also reflected. When a white light hits a banana, we see the banana as being yellow because it absorbed the colors in the visible light spectrum except at the wavelengths corresponding to yellow. Even if we cannot see it, this phenomenon also happens outside of the visible light spectrum, providing very rich information.
    Each material has a unique spectral signature, which refers to how much light the material reflects at different wavelengths of the electromagnetic spectrum. In other words, a spectral signature is a pattern that shows how the reflectance of a material changes across various wavelengths, making it possible to identify different materials based on their reflectance behavior.
    A hyperspectral data cube of size $a \times b \times m $ contains the measured spectral reflectance in $m$ bandwidths of a $a\times b$ sized pixelated area. The spatial information can be vectorized by horizontally concatenating each pixel. Thus, a matrix $X\in\R_+^{m\times n}$ is obtained where $n=a \times b$ is the number of pixels. Due to physical constraints, satellites measuring the reflectance with a high spectral resolution have to compromise with the spatial resolution. Hence, a pixel can correspond to an area of several square meters. It is then possible that several materials are present in one pixel. HU consists in identifying the spectral signature of the materials present in the area, also called endmembers, as well as their abundance in each pixel. If we assume that the mixing process in a pixel is linear, HU can be performed with CLRMF. Properly doing so, the $k$-th column of $W$ should contain the spectral signature of the $k$-th endmember, and $H(k,j)$ should contain the abundance of the $k$-th endmember in the $j$-th pixel. As previously said, identifiability is then a key feature in HU. A practitioner wants to retrieve the true materials present in the area, as well as their true abundances. HU is discussed in \Cref{chap:randspa,chap:minvolnmf,chap:maxvolnmf}.

    \item \textbf{Matrix Completion for Recommender Systems} In some applications, either due to data corruption or simply due to missing measurements, it is possible that the data matrix $X$ is incomplete. This is the case in recommender systems for instance. Consider a movie-user rating data matrix $X\in\R_+^{m\times n}$, where the entry $X(i,j)$ is the rating that the $j$-th user gave to the $i$-th movie. Obviously, $X$ has some missing entries because all the users have not watched and rated all the movies. Let us assume that we have a way to estimate the missing values. It is then possible to recommend a movie to a user if, according to the estimation, this user should give a good rating to this movie. One of the most standard way to impute the missing entries is to assume that the hypothetical full matrix can be approximated by a low-rank matrix. Let us assume that we fix\footnote{Some CLRMFs for missing data completion do not need to fix the rank~\cite{cai2010singular}. We just make this assumption here for the sake of simplicity.} the rank of the estimation to $r$. Call $M\in\{0,1\}^{m\times n}$ the binary matrix\footnote{This is a particular case of $M\in[0,1]^{m\times n}$ being a weight matrix, where the weight $M(i,j)$ between $0$ and $1$ indicates how much the $(i,j)$-th entry can be trusted. $M(i,j)=0$ means that you do not trust the value $X(i,j)$ and $M(i,j)=1$ means that you trust the value $X(i,j)$.} of observed entries, where $M(i,j)=1$ if $X(i,j)$ is known, and $M(i,j)=0$ otherwise. Let us minimize the fitting error $$\|M\circ(X-WH)\|_F^2 = \sum_{(i,j),M(i,j)=1}(X(i,j)-W(i,:)H(:,j))^2$$ with respect to $W\in\R^{m\times r}$ and $H\in\R^{r\times n}$, where $\circ$ is the Hadamard product. If $X(i,j)$ is unknown, it can be estimated just by computing $W(i,:)H(:,j)$. Typically, other constraints are imposed on the factors $W$ and $H$ in order to avoid over-fitting and improve the imputation. Matrix completion is discussed in \Cref{chap:bssmf,chap:minvolnmf}.
\end{itemize}

\pagebreak
\section*{Thesis outline and related publications}\addcontentsline{toc}{section}{Thesis outline and related publications}

The aim of this thesis is fourfold:
\begin{enumerate}
    \item create new interpretable and identifiable matrix factorization models,
    \item improve existing matrix factorization models for some specific applications,
    \item develop fast algorithms for these models, and 
    \item apply these models and algorithms on data sets and compare to the state of the art.
\end{enumerate}

This thesis is structured as follows:\\

\vspace{1cm}\noindent\textbf{\Cref{chap:preli}. \nameref{chap:preli}}\\

In this chapter, we introduce some background on Nonnegative Matrix Factorization, Simplex-Structured Matrix Factorization and identifiability, often needed through the thesis.\\

\vspace{1cm}\noindent\textbf{\Cref{chap:bssmf}. \nameref{chap:bssmf}}\\

In this chapter, we propose a new low-rank matrix factorization model dubbed bounded simplex-structured matrix factorization (BSSMF). Given an input matrix $X$ and a factorization rank $r$, BSSMF looks for a matrix $W$ with $r$ columns and a matrix $H$ with $r$ rows such that $X \approx WH$ where the entries in each column of $W$ are bounded, that is, they belong to given intervals, and the columns of $H$ belong to the probability simplex, that is, $H$ is column stochastic. BSSMF generalizes nonnegative matrix factorization (NMF), and simplex-structured matrix factorization (SSMF). BSSMF is particularly well suited when the entries of the input matrix $X$ belong to a given interval; for example when the rows of $X$ represent images, or $X$ is a rating matrix such as in the Netflix and MovieLens datasets where the entries of $X$ belong to the interval $[1,5]$. The simplex-structured matrix $H$ not only leads to an easily understandable decomposition providing a soft clustering of the columns of $X$, but implies that the entries of each column of $WH$ belong to the same intervals as the columns of $W$. In this chapter, we first propose a fast algorithm for BSSMF, even in the presence of missing data in $X$. Then we provide identifiability conditions for BSSMF, that is, we provide conditions under which BSSMF admits a unique decomposition, up to trivial ambiguities. Finally, we illustrate the effectiveness of BSSMF on two applications: extraction of features in a set of images, and the matrix completion problem for recommender systems. \\

The content of this chapter is mainly extracted from

\noindent\cite{vuthanh2022bounded}~\fullcite{vuthanh2022bounded}\\
\noindent\cite{vuthanh2023bounded}~\fullcite{vuthanh2023bounded}.\\

\noindent\textbf{\Cref{chap:polytopicmf}. \nameref{chap:polytopicmf}}\\

Polytopic matrix factorization (PMF) decomposes a given matrix as the product of two factors where the rows of the first factor belong to a given convex polytope 
and the columns of the second factor belong to another given convex polytope. In this chapter we show that if the polytopes have certain invariant properties, and that if the rows of the first factor and the columns of the second factor are sufficiently scattered within their corresponding polytope, then this PMF is identifiable, that is, the factors are unique up to a signed permutation. The PMF framework is quite general, as it recovers other known structured matrix factorization models, and is highly customizable depending on the application. Hence, our result provides sufficient conditions that guarantee the identifiability of a large class of structured matrix factorization models.\\

The content of this chapter is mainly extracted from

\noindent\cite{vuthanh2023identifiability}~\fullcite{vuthanh2023identifiability}.\\

\vspace{1cm}\noindent\textbf{\Cref{chap:randspa}. \nameref{chap:randspa}}\\

The successive projection algorithm (SPA) is a widely used algorithm for nonnegative matrix factorization (NMF) under the separability assumption. Separability assumes that the cone of $W$ should be equal to the cone of the data $X$. In hyperspectral unmixing, that is, the extraction of materials in a hyperspectral image, separability is equivalent to the pure-pixel assumption and states that for each material present in the image there exists at least one pixel composed of only this material. SPA is fast and provably robust to noise, but is not robust to outliers. Also, it is deterministic, so for a given setting it always produces the same solution. Yet, it has been shown empirically that the non-deterministic algorithm vertex component analysis (VCA), when run sufficiently many times, often produces at least one solution that is better than the solution of SPA. In this chapter, we combine the best of both worlds and introduce a randomized version of SPA dubbed RandSPA, that produces potentially different results at each run. It can be run several times to keep the best solution, and it is still provably robust to noise. Experiments on the unmixing of hyperspectral images show that the best solution among several runs of RandSPA is generally better that the solution of vanilla SPA.\\

The content of this chapter is mainly extracted from

\noindent\cite{vuthanh2022randomized}~\fullcite{vuthanh2022randomized}.\\

\vspace{1cm}\noindent\textbf{\Cref{chap:minvolnmf}. \nameref{chap:minvolnmf}}\\

Nonnegative matrix factorization with the minimum volume criterion (MinVol NMF) guarantees that, under some mild and realistic conditions, the factorization has an  essentially unique solution. This result has been successfully leveraged in many applications, including topic modeling, hyperspectral image unmixing, and audio source separation. In this chapter, we propose a fast algorithm to solve MinVol NMF which is based on a recently introduced block majorization-minimization framework with extrapolation steps. We illustrate the effectiveness of our new algorithm compared to the state of the art on several real hyperspectral images and document datasets. We also focus on the use of the minimum volume criterion on the task of nonnegative data imputation, which, up to our knowledge, has never been explored before. The particular choice of the MinVol regularization is justified by its interesting identifiability property and by its link with the nuclear norm. We show experimentally that MinVol NMF is a relevant model for nonnegative data recovery, especially when the recovery of a unique embedding is desired. Additionally, we introduce a new version of MinVol NMF that outperforms vanilla MinVol for data recovery.\\

The content of this chapter is mainly extracted from

\noindent\cite{vuthanh2021inertial}~\fullcite{vuthanh2021inertial}\\
\noindent\cite{vuthanh2024minimum}~\fullcite{vuthanh2024minimum}.\\

\vspace{1cm}\noindent\textbf{\Cref{chap:maxvolnmf}. \nameref{chap:maxvolnmf}}\\

Nonnegative matrix factorization with a maximum volume criterion (MaxVol NMF) is an identifiable regularized low-rank model that has not been studied as much as its counterpart minimum-volume NMF (MinVol NMF). Given a matrix dataset $X$, MaxVol NMF consists in finding two nonnegative low-rank factors, $W$ and $H$, such that their product approximates $X$ while  the volume spanned by the origin and the rows of $H$ is as large as possible. This MaxVol criterion, combined with nonnegativity, will incite $H$ to be sparser. In MinVol NMF, the volume criterion is on $W$ and should be minimized. In the exact case, that is, $X = WH$, we show that MinVol NMF is equivalent to MaxVol NMF. Moreover, we show that MaxVol NMF behaves rather differently than MinVol NMF in the presence of noise, especially when the penalty on the volume criterion is increased. We also show how MaxVol NMF creates a continuum between NMF and orthogonal NMF with even clusters. We propose several algorithms to solve MaxVol NMF, which we apply on real datasets. Finally, we introduce the ``normalized'' variant of MaxVol NMF which exhibits better results than MinVol NMF and MaxVol NMF on Hyperspectral Unmixing (HU).

\section*{Open-source codes}\addcontentsline{toc}{section}{Open-source codes}
All the algorithms developed in this thesis are available online along with the data and (most of) the test scripts necessary to reproduce our experiments: \url{https://gitlab.com/vuthanho}


\chapter{Preliminaries}\label{chap:preli}
\begin{hiddenmusicsec}{Radiohead - Karma Police}{https://www.youtube.com/watch?v=4IJI6soiQhI}
\end{hiddenmusicsec}

\section{Nonnegative Matrix Factorization (NMF)}\label{preli:sec:nmf}

We say that a matrix is nonnegative if all its elements are larger or equal to zero. In the remaining of this thesis, $X\geq0$ means that  $X(i,j)\geq0$ for all $(i,j)$. Nonnegative matrix factorization (NMF)\label{acro:NMF}, popularized by Lee and Seung~\cite{lee1999learning}, is a linear dimensionality reduction technique that has become a standard tool to extract latent structures in nonnegative data. Given an input matrix $X \in \mathbb{R}^{m \times n}$ and a factorization rank $r < \min(m,n)$, NMF consists in finding two factors $W \in \mathbb{R}_+^{m \times r}$ and $H \in \mathbb{R}_+^{r \times n}$ such that $X \approx W H$. Columns of $X$ are called data points, and if $H$ is column-stochastic then the columns of $W$ can be seen as the vertices of a convex hull containing the data points; see \Cref{preli:sec:ssmf}. Applications of NMF include feature extraction in images, topic modeling, audio source separation, chemometrics, or blind hyperspectral unmixing (HU), see for example~\cite{cichocki2009nonnegative, xiao2019uniq, gillis2020,fu2019nonnegative} and the references therein. 

 Let us define the Exact NMF and NMF problems. 
\begin{definition}[Exact NMF]
  Given a nonnegative matrix $X\in\R_+^{m\times n}$, an exact NMF of size $r$ consists in finding two matrices $W\in\R_+^{m\times r}$ and $H\in\R_+^{r\times n}$ such that $X=WH$. The smallest $r$ such that you can find an exact NMF of $X$ is called the nonnegative rank of $X$ and is noted $\rank_+(X)$.
\end{definition}

\begin{definition}[NMF]
  Given a matrix $X\in\R^{m\times n}$, finding its NMF of size $r$ consists in solving
  \begin{mini}
    {W,H}{\|X-WH\|_F^2}{}{}
    \addConstraint{W\in\R_+^{m\times r}}
    \addConstraint{H\in\R_+^{r\times n}.}
  \end{mini}
  Note that we fixed our definition with the Frobenius norm because it is the only cost function used as a reconstruction error in this thesis. Nonetheless, other cost functions could be used, such as the beta-divergences \cite{fevotte2011algorithms}.
\end{definition}

When a data matrix is nonnegative, it makes sense to use NMF in order to decompose it with features that are also nonnegative, and in an additive way. The main advantage of NMF is that the nonnegativity constraints on the factors $W$ and $H$ lead to an easily interpretable part-based decomposition~\cite{lee1999learning}. 

\paragraph{Geometric interpretation of NMF} In the exact case, an NMF of size $r$ is equivalent to finding a cone with $r$ rays in the nonnegative orthant that contains all the data points $X(:,j)$. The columns of the matrix $W$ of the corresponding NMF are the rays that generated this polyhedral cone. Consider an NMF $X=WH$. For every $j$,  $X(:,j)=WH(:,j)$. Since $H(:,j)\geq0$ for all $j$, we have $\cone(X)\subseteq\cone(W)$ by definition of a cone\footnote{Equality is equivalent under the so-called \textit{separability assumption}; see \Cref{chap:randspa}}; see \Cref{preli:fig:geointerpNMF} for a 3D example.

\begin{figure}[hbtp!]
  \centering
  \begin{tikzpicture}
    \begin{axis}[view={45}{20},
      axis lines=center,
      ytick=\empty,xtick=\empty,ztick=\empty,
      xmin=-0.1, xmax=1.4,
      ymin=-0.1, ymax=1.4, 
      zmin=-0.1, zmax=1.4,
      legend style={at={(1.05,0.5)},anchor=center}
      ]
  
      \addplot3 [color=black,mark=*,mark options={solid},only marks,opacity=0.5]
      coordinates{
        (0.2,   0.36,  0.84)
        (0.2,   0.76,  0.44)
        (0.82,  0.18,  0.26)
        (0.15,  0.11,  0.51)
        (0.6,   0.6,   0.2)
        (0.208044,  0.47042,   0.691368)
        (0.356276,  0.371377,  0.489954)
        (0.409677,  0.571279,  0.387622)
        (0.548255,  0.440392,  0.331366)
        (0.335524,  0.255621,  0.46898)
        (0.254364,  0.167102,  0.533224)
        (0.267997,  0.375913,  0.722183)
        (0.493787,  0.533898,  0.338444)
        (0.575162,  0.441782,  0.325683)
        (0.550958,  0.384803,  0.310077)
      };
      \addlegendentry{Columns of $X$}
      \addplot3 [forget plot,color=green,fill=green!70,line join=round,line width=1pt,opacity=0.7]
      coordinates{
        (0,0,0) 
        (0.272727,  0.2,   0.927273) 
        (0.2,       0.36,  0.84) 
        (0,0,0) 
        (0.2,       0.36,  0.84) 
        (0.2,       0.76,  0.44) 
        (0,0,0) 
        (0.2,       0.76,  0.44) 
        (0.6,       0.6,   0.2) 
        (0,0,0) 
        (0.6,       0.6,   0.2) 
        (0.911111,  0.2,   0.288889) 
        (0,0,0) 
        };
      \addlegendimage{area legend, fill=green!70,opacity=0.3}
      \addlegendentry{$\cone(X)$}
      
      \addplot3 [forget plot,color=green,fill=green!70,line join=round,line width=1pt,opacity=0.7]
      coordinates{
        (0,0,0) 
        (0.911111,  0.2,   0.288889) 
        (0.272727,  0.2,   0.927273) 
        (0,0,0) 
        };
      \addplot3 [forget plot,color=green,fill=green!70,line join=round,line width=1pt,opacity=0.1]
        coordinates{
          (0,0,0) 
          (0.409091,  0.3,   1.39091) 
          (0.3,       0.54,  1.26) 
          (0,0,0) 
          (0.3,       0.54,  1.26) 
          (0.3,       1.14,  0.66) 
          (0,0,0) 
          (0.3,       1.14,  0.66) 
          (0.9,       0.9,   0.3) 
          (0,0,0) 
          (0.9,       0.9,   0.3) 
          (1.36667,   0.3,   0.433334) 
          (0,0,0)
          };
      \addplot3 [forget plot,color=green,fill=green!70,line join=round,line width=1pt,opacity=0.1]
      coordinates{
        (0,0,0) 
        (1.36667,   0.3,   0.433334) 
        (0.409091,  0.3,   1.39091) 
        (0,0,0) 
        };
      \addplot3 [color=red,mark=triangle*,mark options={solid},only marks,opacity=0.5]
      coordinates{
        (0.2,0.2,1)
        (0.2,1,0.2)
        (1,0.2,0.2)
      };
      \addlegendentry{Columns of $W$}
      
      \addplot3 [forget plot,color=blue,fill=blue!70,line join=round,line width=1pt,opacity=0.35]
      coordinates{
        (0,0,0)
        (0.2,0.2,1)
        (0.2,1,0.2)
        (0,0,0)
        (0.2,1,0.2)
        (1,0.2,0.2)
        (0,0,0)};
        \addlegendimage{area legend, fill=blue!70,opacity=0.3}
        \addlegendentry{$\cone(W)$}
      \addplot3 [forget plot,color=blue,fill=blue!70,line join=round,line width=1pt,opacity=0.35]
      coordinates{
        (0,0,0)
        (0.2,0.2,1)
        (1,0.2,0.2)
        (0,0,0)};
      \addplot3 [forget plot,color=blue,fill=blue!70,line join=round,line width=0pt,opacity=0.1]
      coordinates{
        (0,0,0)
        (0.3,0.3,1.5)
        (0.3,1.5,0.3)
        (0,0,0)
        (0.3,1.5,0.3)
        (1.5,0.3,0.3)
        (0,0,0)};
      \addplot3 [forget plot,color=blue,fill=blue!70,line join=round,line width=0pt,opacity=0.1]
      coordinates{
        (0,0,0)
        (0.3,0.3,1.5)
        (1.5,0.3,0.3)
        (0,0,0)};
    \end{axis}
  \end{tikzpicture}
  \caption[Geometric interpretation of Exact NMF]{Geometric interpretation of Exact NMF with $r=3$}
  \label{preli:fig:geointerpNMF}
\end{figure}
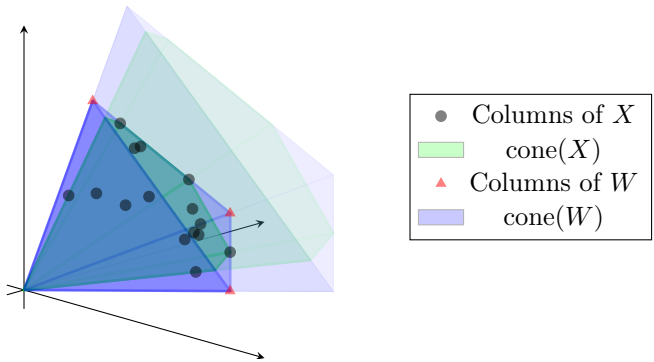

\section{Simplex-structured matrix factorization (SSMF)}\label{preli:sec:ssmf}

A key set that will be used through the thesis is the probability simplex, that will allow us to define the SSMF\label{acro:SSMF} problem. 

\begin{definition}[Probability simplex]
  We denote $\Delta^r$ the probability simplex, that is, the set
  $${\{x\in\R^r \ | \ x\geq0,e^\top x=1\}},$$
where $e$ is the vector of all ones of appropriate dimension.
\end{definition}
We then also denote $\Delta^{r\times n}$ the set of matrices of size $r\times n$ such that all their columns lie in $\Delta^r$, that is,
$${\{X\in\R^{r\times n} \ | \ X(:,j)\in\Delta^r \text{ for all }j\}}.$$

\begin{definition}[Exact SSMF]
  Given a matrix $X\in\R^{m\times n}$, an exact SSMF of size $r$ consists in finding two matrices $W\in\R^{m\times r}$ and $H\in\Delta^{r\times n}$ such that $X=WH$.
\end{definition}

\begin{definition}[SSMF]
  Given a matrix $X\in\R^{m\times n}$, finding its SSMF of size $r$ consists in solving
  \begin{mini}
    {W,H}{\|X-WH\|_F^2}{}{}
    \addConstraint{W\in\R^{m\times r}}
    \addConstraint{H\in\Delta^{r\times n}.}
  \end{mini}
\end{definition}
With SSMF, each data point $X(:,j)$ has to be explained through a convex combination of some features. Due to that, SSMF is quite useful for providing a soft clustering decomposition of the data. In recommender systems for instance, SSMF could provide this kind of interpretation: ``This user is behaving 80\% like this typical user and 20\% like this other typical user''.

\paragraph{Geometric interpretation of SSMF} In the exact case, an SSMF of size $r$ is equivalent to finding a convex hull with $r$ vertices that contains all the data points $X(:,j)$. The columns of the matrix $W$ of the corresponding SSMF are then the vertices of this convex hull. Consider an SSMF $X=WH$. For every $j$,  $X(:,j)=WH(:,j)$. Since $H(:,j)\in\Delta^r$ for all $j$, we have $\conv(X)\subseteq\conv(W)$ by definition of $\conv(W)$; see \Cref{preli:fig:geointerpSSMF} for a 3D example.

\begin{figure}[hbtp!]
  \centering
  \begin{tikzpicture}
    \begin{axis}[view={60}{30},
      axis lines=center,
      ytick=\empty,xtick=\empty,ztick=\empty,
      xmin=0, xmax=1.4,
      ymin=0, ymax=1.4, 
      zmin=0, zmax=1.4,
      legend style={at={(0.9,0.5)},anchor=center},
      axis line style={draw=none},
      ]
      \addplot3 [forget plot,color=black,dashed,line join=round,line width=0.5pt,opacity=0.2]
      coordinates{
        (0,0,0)
        (0,1,0)};
      \addplot3 [color=black,mark=*,mark options={solid},only marks,opacity=0.5]
      coordinates{
        (0.8, 0.2,  0.0) 
        (0.3, 0.3,  0.0) 
        (0.0, 0.8,  0.2) 
        (0.0, 0.3,  0.3) 
        (0.2, 0.0,  0.8) 
        (0.3, 0.0,  0.3) 
      };
      \addlegendentry{Columns of $X$}
      \addplot3 [forget plot,color=black,dashed,line join=round,line width=0.5pt,opacity=0.3]
      coordinates{
        (0.0, 0.3,  0.3) 
        (0.2, 0.0,  0.8) 
        (0.0, 0.3,  0.3) 
        (0.0, 0.8,  0.2) 
        (0.3, 0.3,  0.0) 
        (0.8, 0.2,  0.0) 
        (0.3, 0.3,  0.0) 
        (0.3, 0.0,  0.3) 
        (0.0, 0.3,  0.3) 
        (0.3, 0.3,  0.0) 
        };
      \addplot3 [forget plot,color=black,fill=green!70,line join=round,line width=0.5pt,opacity=0.3]
      coordinates{
        (0.8, 0.2,  0.0) 
        (0.0, 0.8,  0.2) 
        (0.2, 0.0,  0.8) 
        (0.3, 0.0,  0.3) 
        (0.8, 0.2,  0.0) 
        };
      \addplot3 [forget plot,color=black,line join=round,line width=0.5pt,opacity=0.5]
      coordinates{
        (0.8, 0.2,  0.0) 
        (0.0, 0.8,  0.2) 
        (0.2, 0.0,  0.8) 
        (0.8, 0.2,  0.0) 
        (0.3, 0.0,  0.3) 
        (0.2, 0.0,  0.8) 
        };
      \addlegendimage{area legend, fill=green!70,opacity=0.3}
      \addlegendentry{$\conv(X)$}
      
      \addplot3 [color=red,mark=triangle*,mark options={solid},only marks,opacity=0.5]
      coordinates{
        (0,0,0)
        (0,0,1)
        (0,1,0)
        (1,0,0)
      };
      \addlegendentry{Columns of $W$}
      
      \addplot3 [forget plot,color=blue!70,fill=blue!70,line join=round,line width=0pt,opacity=0.1]
      coordinates{
        (0,0,0)
        (0,0,1)
        (0,1,0)
        (1,0,0)
        (0,0,0)};
      \addplot3 [forget plot,color=black,line join=round,line width=0.5pt,opacity=0.7]
      coordinates{
        (0,0,0)
        (0,0,1)
        (0,1,0)
        (1,0,0)
        (0,0,0)
        (0,0,1)
        (1,0,0)};
      \addlegendimage{area legend, fill=blue!70,opacity=0.3}
      \addlegendentry{$\conv(W)$}
    \end{axis}
  \end{tikzpicture}
  \caption[Geometric interpretation of Exact SSMF]{Geometric interpretation of Exact SSMF with $r=4$ and $n=6$}
  \label{preli:fig:geointerpSSMF}
\end{figure}
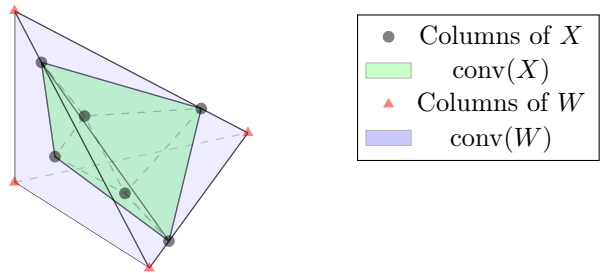

\pagebreak
\section{Identifiability} \label{bssmf:sec:identif}\label{preli:sec:identif}  

Let us first define a factorization model. 
\begin{definition}[Factorization model] 
Given a matrix $X \in \mathbb{R}^{m \times n}$, and an integer $r \leq \min(m,n)$, a  factorization model is an optimization model of the form 
\begin{equation} \label{preli:eq:factomodel}
\begin{split}
\min_{W \in \mathbb{R}^{m \times r}, H \in \mathbb{R}^{r \times n}} & 
g(W,H) \\
\text{ such that } & X = WH, \\
& W \in \Omega_W \text{ and } H \in \Omega_H,     
\end{split}
\end{equation}
where $g(W,H)$ is some criterion, and $\Omega_W$ and $\Omega_H$ are the feasible sets for $W$ and $H$, respectively. 
\end{definition}

Let us define the identifiability of a factorization model, and essential uniqueness of a pair $(W,H)$.  

\begin{definition}[Identifiability / Essential uniqueness]
\label{bssmf:def:identifiability}
\label{preli:def:identifiability}

\sloppy Let $X \in \mathbb{R}^{m \times n}$, and $r \leq \min(m,n)$ be an integer. 
Let $(W,H)$ be a solution to a given factorization model~\eqref{preli:eq:factomodel}. 
The pair $(W,H)$ is essentially unique 
for the factorization model~\eqref{preli:eq:factomodel} of matrix $X$ 
if and only if any other pair $(W',H') \in \mathbb{R}^{m \times r} \times \mathbb{R}^{r \times n}$ that solves the factorization model~\eqref{preli:eq:factomodel} satisfies, for all $k$, 
$$W'(:,k) = \alpha_k W(:,\pi(k)) $$
and 
$$H'(k,:) = \alpha_k^{-1} H(\pi(k),:), $$
where $\pi$ is a permutation of $\{1,2,\dots,r\}$, and $\alpha_k \neq 0$ for all $k$. 
In other terms, $(W',H')$ can only be obtained as a permutation and scaling of $(W,H)$. 
In that case, the factorization model is said to be identifiable for the matrix~$X$.  
\end{definition} 

A key question in theory and practice is to determine conditions on $X$, $g$, $ \Omega_W$ and $\Omega_H$ that lead to identifiable factorization models; see, e.g., \cite{xiao2019uniq, kueng2021binary} for discussions.
This will be a major topic of this thesis. 

\subsection{Identifiability of NMF} 
\label{bssmf:sec:identif_nmf}\label{preli:sec:identif_nmf}

NMF is not essentially unique in general. However, as opposed to SSMF (see \Cref{preli:sec:identif_ssmf}), NMF decompositions can be identifiable without the use of additional requirements. The first identifiability result was proposed in \cite{donoho2004does}. Their conditions, based on separability, are quite strong. In the context of nonnegative source separation,~\cite{moussaoui2005non} proposed some necessary conditions for the uniqueness of the solution.
One of the most relaxed sufficient condition for identifiability is based on the Sufficiently scattered condition (SSC).  
\begin{theorem}\cite[Theorem 4]{huang2013non} \label{bssmf:th:uniqNMFSSC}\label{preli:th:uniqNMFSSC}
If $W^\top \in \mathbb{R}^{r \times m}$ and $H \in \mathbb{R}^{r \times n}$ are sufficiently scattered,  
then the Exact NMF $(W,H)$ of $X=WH$ of size $r = \rank(X)$ is essentially unique. 
\end{theorem}

The SSC is defined as follows. 
\begin{definition}[Sufficiently scattered condition]  \label{bssmf:def:ssc}\label{preli:def:ssc}  
The matrix $H \in \mathbb{R}^{r \times n}_+$ is sufficiently scattered if the following two conditions are satisfied:\index{sufficiently scattered condition!definition} \\

[SSC1]  $\mathcal{C} = \{x \in \mathbb{R}^r_+ \ | \ e^\top x \geq \sqrt{r-1} \|x\|_2 \} \; \subseteq \; \cone(H)$.\\

[SSC2]  There does not exist any orthogonal matrix $Q$ such that $\cone(H) \subseteq \cone(Q)$, except for permutation matrices. 
\end{definition} 

\begin{lemma}
  \label{preli:th:dualC}
  The dual cone of $\C$ is given by $\C^*=\left\{y\in\R^r,~e^\top y \geq\|y\|_2\right\}$.
\end{lemma}
\noindent The proof for this lemma is provided in~\cite[Section 4.2.3.2]{gillis2020book}.\\

SSC1 requires the columns of $H$ to contain the cone $\mathcal{C}$, which is tangent to every facet of the nonnegative orthant; see \Cref{bssmf:fig:geoSSC}.
\begin{figure}[p] 
\centering
\begin{subfigure}{0.39\textwidth}
  \centering
  \caption{cone $\mathcal{C}$}
  \label{preli:fig:geoSSC:conC}
  \begin{tikzpicture}[scale=0.8]
	\definecolor{mygreen}{RGB}{152,191,98}
	\node [color=mygreen] at (4,0.8) {$\mathcal{C}$};
	\node [color=blue!70] at (2.2,0.7) {$\Delta^3$};
	\begin{axis}[view={60}{30},
		axis lines=center,
		xlabel={$e_1$}, ylabel={$e_2$}, zlabel={$e_3$},
		ytick=\empty,xtick=\empty,ztick=\empty,
		xmin=0, xmax=1.5,
		ymin=0, ymax=1.5, 
		zmin=0, zmax=1.5,
		colormap={mygreen}{rgb255=(152,191,98) rgb255=(0,0,0)},
		every axis x label/.style={at={(ticklabel cs:0.66)},anchor=north},
		every axis y label/.style={at={(ticklabel cs:0.66)},anchor=north},
		every axis z label/.style={at={(ticklabel cs:0.66)},anchor=east},
		]
		\addplot3[domain=0:0.57735, y domain=0:2*pi,samples=2,samples y=40,surf,shader=flat,point meta=0,opacity=0,fill opacity=0.5] 
		({0.788675*0.707106*x*cos(deg(y))-0.211325*0.707106*x*sin(deg(y))+0.57735*x},{-0.211325*0.707106*x*cos(deg(y))+0.788675*0.707106*x*sin(deg(y))+0.57735*x},{-0.57735*0.707106*x*cos(deg(y))-0.57735*0.707106*x*sin(deg(y))+0.57735*x});
		\addplot3 [color=blue!70,fill=blue!70,line width=0pt,opacity=0.8]
		coordinates{
			(1,0,0)
			(0,1,0)
			(0,0,1)};
		\addplot3[domain=0.57735:1, y domain=0:2*pi,samples=2,samples y=40,surf,shader=flat,point meta=0,opacity=0,fill opacity=0.5] 
		({0.788675*0.707106*x*cos(deg(y))-0.211325*0.707106*x*sin(deg(y))+0.57735*x},{-0.211325*0.707106*x*cos(deg(y))+0.788675*0.707106*x*sin(deg(y))+0.57735*x},{-0.57735*0.707106*x*cos(deg(y))-0.57735*0.707106*x*sin(deg(y))+0.57735*x});
	\end{axis}
\end{tikzpicture}
\end{subfigure}
\begin{subfigure}{0.6\textwidth}
  \centering
  \caption{SSC}
  \label{preli:fig:geoSSC:ssc}
  \begin{tikzpicture}[scale=0.8]
	\definecolor{mygreen}{RGB}{152,191,98}
	\begin{axis}[width=7cm,height=7cm,hide axis,axis equal,legend style={at={(1,0.55)},anchor=west},legend style={cells={align=left}}]
		\addplot [color=blue!70,line width=2pt]
		coordinates{
			( 0.0, 1.0)
			(-0.866025 ,-0.5)
			( 0.866025 ,-0.5)
			( 0.0, 1.0)
			(-0.866025 ,-0.5)};
		
		\addplot [color=black,
		mark=*,mark options={solid},only marks]
		coordinates{
			(-0.519615   , -0.5)
			( 0.69282    , -0.2)
			( 0.519615   , -0.5)
			(-0.69282    , -0.2)
			( 0.173205    , 0.7)
			(-0.173205    , 0.7)
			( -0.445139 ,  -0.33079)
			( 0.392984  ,  0.250864)
			(-0.444567  , -0.121731)
			( 0.0634332 ,  0.0152832)
			( -0.137925  ,  0.56291)
			(0.296064  , -0.295143)
		};
		\addlegendimage{area legend, fill=gray!50,opacity=0.3}
		\addlegendimage{area legend, fill=mygreen,opacity=0.5}
		\filldraw [fill=gray!50,opacity=0.3] (axis cs: -0.173205, 0.7) -- (axis cs: 0.173205, 0.7) -- (axis cs: 0.69282,-0.2) -- (axis cs: 0.519615,-0.5) -- (axis cs: -0.519615,-0.5)  --  
		(axis cs: -0.69282,-0.2) --  cycle;
		
		\filldraw[fill=mygreen,opacity=0.5] (axis cs: 0,0) circle (0.5);
		
		\legend{$\bd(\Delta^3)$,Columns of $H$\\scaled on $\Delta^3$,$\cone(H)\cap\Delta^3$,$\mathcal{C}\cap\Delta^3$};
	\end{axis}
\end{tikzpicture}
\end{subfigure}

\vspace{2cm}

\begin{subfigure}{0.49\textwidth}
  \centering
  \caption{\sout{SSC1}}
  \label{preli:fig:geoSSC:not-ssc1}
  \begin{tikzpicture}[scale=0.8]
	\definecolor{mygreen}{RGB}{152,191,98}
	\begin{axis}[width=7cm,height=7cm,hide axis,axis equal,legend style={at={(0.37,0.55)},anchor=south west},legend style={cells={align=left}}]
		
		\addplot [color=blue!70,line width=2pt]
		coordinates{
			( 2+0.0, 1.0)
			(2+-0.866025 ,-0.5)
			( 2+0.866025 ,-0.5)
			( 2+0.0, 1.0)
			(2+-0.866025 ,-0.5)};
		
		\addplot [color=black,
		mark=*,mark options={solid},only marks]
		coordinates{
			(2+-0.519615   , -0.5)
			(2+ 0.69282    , -0.2)
			(2+ 0.519615   , -0.5)
			(2+-0.69282    , -0.2)
			(2+ 0.173205    , 0.7)
			(2+ -0.445139 ,  -0.33079)
			(2+ 0.392984  ,  0.250864)
			(2+-0.444567  , -0.121731)
			(2+ 0.0634332 ,  0.0152832)
			(2+ -0.137925  ,  0.56291)
			(2+0.296064  , -0.295143)
		};
		
		\filldraw [fill=gray!50,opacity=0.3] (axis cs: 2+ -0.137925, 0.56291) -- (axis cs: 2+0.173205, 0.7) -- (axis cs: 2+0.69282,-0.2) -- (axis cs: 2+0.519615,-0.5) -- (axis cs: 2+-0.519615,-0.5)  --  
		(axis cs: 2+-0.69282,-0.2) --  cycle;
		
		\filldraw[fill=mygreen,opacity=0.5] (axis cs: 2+0,0) circle (0.5);
		
	\end{axis}
\end{tikzpicture}
\end{subfigure}
\begin{subfigure}{0.49\textwidth}
  \centering
  \caption{SSC1 \sout{SSC2}}
  \label{preli:fig:geoSSC:ssc1-not-ssc2}
  \begin{tikzpicture}[scale=0.8]
	\definecolor{mygreen}{RGB}{152,191,98}
	\begin{axis}[width=7cm,height=7cm,hide axis,axis equal,legend style={at={(0.37,0.55)},anchor=south west},legend style={cells={align=left}}]
		\addplot [color=blue!70,line width=2pt]
		coordinates{
			( 0.0, 1.0)
			(-0.866025 ,-0.5)
			( 0.866025 ,-0.5)
			( 0.0, 1.0)
			(-0.866025 ,-0.5)};
		
		\addplot [color=black,
		mark=*,mark options={solid},only marks]
		coordinates{
			( -0.345139 ,  -0.33079)
			( 0.392984  ,  0.250864)
			(-0.444567  , -0.121731)
			( 0.0634332 ,  0.0152832)
			( -0.137925  ,  0.4291)
			(0.296064  , -0.295143)
			(-0.57735  ,  0.0)
			(-0.288675 ,  0.5)
			( 0.57735  ,  0.0)
			( 0.288675 ,  0.5)
			( 0.288675 , -0.5)
			(-0.288675 , -0.5)		   
		};
		\addlegendimage{area legend, fill=gray!50,opacity=0.3}
		\addlegendimage{area legend, fill=mygreen,opacity=0.5}
		\filldraw [fill=gray!50,opacity=0.3] (axis cs: -0.57735  ,  0.0) -- (axis cs: -0.288675 ,  0.5) -- (axis cs: 0.288675 ,  0.5) -- (axis cs: 0.57735  ,  0.0) -- (axis cs: 0.288675 , -0.5)  --  
		(axis cs: -0.288675 , -0.5) --  cycle;
		
		\filldraw[fill=mygreen,opacity=0.5] (axis cs: 0,0) circle (0.5);
		
	\end{axis}
\end{tikzpicture}
\end{subfigure}

\caption[Illustration of the SSC in three dimensions]{Illustration of the SSC in three dimensions. On (\subref{preli:fig:geoSSC:conC}): the sets $\Delta^3$ and $\mathcal{C}$, they intersect at (0,0.5,0.5), (0.5,0,0.5), and (0.5,0.5,0). 
On (\subref{preli:fig:geoSSC:ssc}), (\subref{preli:fig:geoSSC:not-ssc1}) and (\subref{preli:fig:geoSSC:ssc1-not-ssc2}): examples of a matrix $H \in \mathbb{R}^{3 \times n}$ respectively satisfying the SSC, not satisfying SSC1 and satisfying SSC1 but not SSC2. 
}
\label{bssmf:fig:geoSSC}\label{preli:fig:geoSSC}
\end{figure}
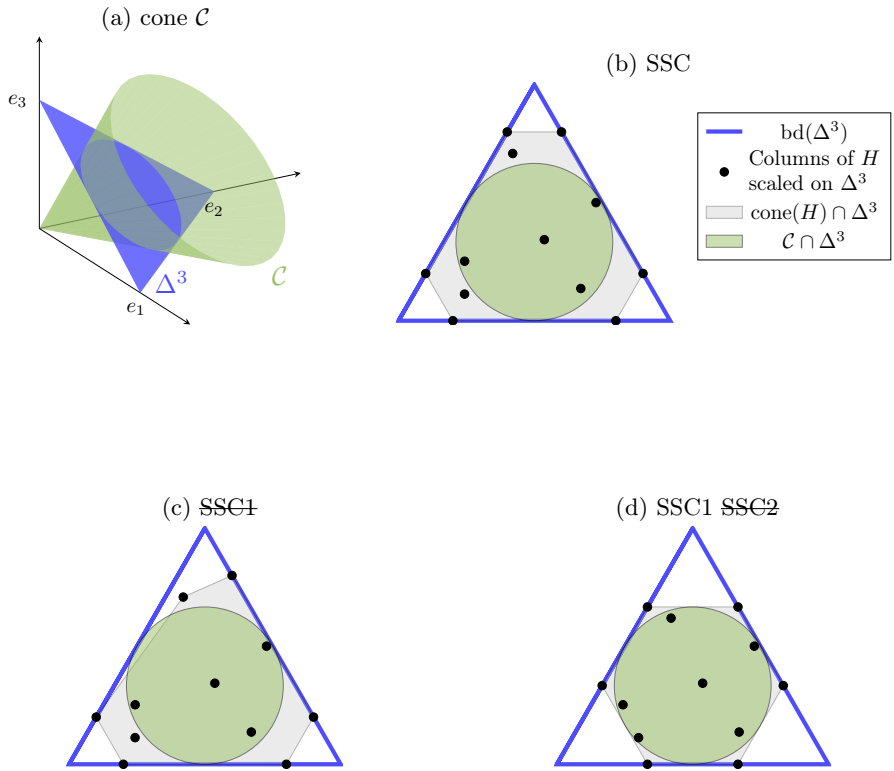 
Hence, satisfying SSC1 requires some degree of sparsity as $H$ needs to contain at least $r-1$ zeros per row~\cite[Th.~4.28]{gillis2020}. 
SSC2 is a mild regularity condition which is typically satisfied when SSC1 is satisfied. For more discussions on the SSC, we refer the interested reader to \cite{xiao2019uniq} and \cite[Chapter 4.2.3]{gillis2020}, and the references therein.

In practice, it is not likely for both $W^\top$ and $H$ to satisfy the SSC. Typically, $H$ will satisfy the SSC, as it is typically sparse. However, in many applications, $W^\top$ will not satisfy the SSC; in particular in applications where $W$ is not sparse, e.g., in hyperspectral unmixing, recommender systems, or imaging. This is why regularized NMF models have been introduced, including sparse and volume regularized NMF. 
We refer the interested reader to \Cref{chap:minvolnmf}, \Cref{chap:maxvolnmf} and \cite[Chapter 4]{gillis2020} for more details. 

\subsection{Identifiability of SSMF} 
\label{bssmf:sec:identif_ssmf}\label{preli:sec:identif_ssmf}
Without further requirements, SSMF is never identifiable; which follows from a result for semi-NMF which is a factorization model that requires only one factor, $H$, to be nonnegative~\cite{gillis2015exact}.  
Let $X = WH$ be an SSMF of $X$. We can obtain other SSMF of $X$ using the following transformation: for any $\alpha \geq 0$, let 
\[W(\alpha) := W \left( (1+\alpha) I - \frac{\alpha}{r}J \right),\]
and
\[\begin{split}
H(\alpha) :=& \left( \frac{1}{1+\alpha} H + \frac{\alpha}{(1+\alpha)r} J \right) \\
=& \left( \frac{1}{1+\alpha} I + \frac{\alpha}{(1+\alpha)r} J \right) H ,   
\end{split}\]
where $I$ is the identity matrix of appropriate dimension, $J$ is the matrix of all ones of appropriate dimension. The second equality follows from the fact that $e^\top H = e^\top$. 
The matrix $H(\alpha)$ is column stochastic since $H$ and $\frac{J}{r}$ are. One can check that $(W(\alpha), H(\alpha))$ is not a permutation and scaling of $(W,H)$ for $\alpha > 0$, while $WH = W(\alpha) H(\alpha)$ since\footnote{This is an invertible M-matrix, with positive diagonal elements and negative off-diagonal elements, whose inverse is nonnegative~\cite{berman1994nonnegative}.}  
\[\begin{split} A(\alpha) := & \left( (1+\alpha) I - \frac{\alpha}{r} J \right)^{-1} \\
= & \frac{1}{1+\alpha} I + \frac{\alpha}{(1+\alpha)r} J.\end{split}\] 
Geometrically, to obtain $W(\alpha)$, the columns of $W$ are moved towards the exterior of $\conv(W)$ and hence the convex hull of the column of $W(\alpha)$ contains the convex hull of the columns of $W$ and hence contains $\conv(X)$. This follows from the fact that  $W  = W(\alpha) A(\alpha)$, 
where $A$ is column stochastic. 

To obtain identifiability of SSMF, one needs  either to impose additional constraints on $W$ and/or $H$ such as sparsity~\cite{abdolali2021simplex}, 
or look for a solution minimizing a certain function~$g$. In particular, MinVol SSMF\label{acro:MinVol}, that is the solution $(W,H)$ that minimizes the volume of the convex hull of $W$ and the origin within its column space

\begin{mini}
  {\scriptstyle W\in\R^{m\times r}, H\in\R^{r\times n}}{\det(W^\top W) \label{preli:eq:exactminvolssmf}}{}{}
  \addConstraint{X=WH}
  \addConstraint{H\in\Delta^{r\times n},}
\end{mini}
is essentially unique given that $H$ satisfies the so-called sufficiently scattered condition (SSC). Note that the quantity $\det(W^\top W)$ is only relative to the aforementioned volume. The true volume is $\frac{1}{r!}\sqrt{\det(W^\top W)}$. 

For SSMF, we have the following identifiability result.  
\begin{theorem}\cite{fu2015blind, lin2015identifiability}  \label{bssmf:th:identifminvolSSMF}\label{preli:th:identifminvolSSMF}
If $H\in\R^{r\times n}$ is sufficiently scattered, then the Exact MinVol SSMF $(W,H)$, in the sense of~\eqref{preli:eq:exactminvolssmf}, of $X=WH$ of size $r=\rank{X}$ is essentially unique.

\end{theorem}
Note that this result has been generalized to the case where the columns of $H$ belong to a given polytope instead of the probability simplex; see~\cite{tatli2021generalized}.

In practice, because of noise and model 
misfit, 
SSMF optimization models need to balance the data fitting term which measures the discrepancy between $X$ and $WH$, and the volume regularization for $\conv(W)$. 
Typically, a problem with objective function of the form 
\[
\|X-WH\|_F^2+\lambda\det(W^\top W), 
\] 
is solved. This requires the tuning of the parameter $\lambda$, which is a nontrivial process~\cite{ang2019algorithms, Zhuang2019paramminvol, nguyen2024towards}.   

\pagebreak
\section{Brief summary of the thesis content} 


In the following chapters, we will study various CLRMFs; they are organized as follows: 
\begin{itemize}
    \item In \Cref{chap:bssmf}, we introduce a new model, dubbed BSSMF, for Bounded Simplex-Structured Matrix Factorization. BSSMF imposes that the columns of $W$ belong to a chosen hyperrectangle while the columns of $H$ have to belong to the probability simplex. The resulting factorization $WH$ naturally belongs to the same chosen hyperrectangle. This behavior is particularly meaningful for naturally bounded data. BSSMF is also identifiable under milder conditions than NMF.
    \item In \Cref{chap:polytopicmf}, we introduce a new model, dubbed PMF, for Polytopic Matrix Factorization. PMF imposes that the rows of $W$ belong to a chosen polytope and the columns of $H$ belong to another chosen polytope. PMF is identifiable with the same core idea than \Cref{preli:th:uniqNMFSSC}, that is, the rows of $W$ and the columns of $H$ should be ``sufficiently scattered'' within their respective feasible set. PMF is a very generic model, meaning that we provide identifiability for a wild class of matrix factorization models.
    \item In \Cref{chap:randspa}, we introduce RandSPA, a greedy algorithm to solve NMF under the separability assumption. RandSPA creates a continuum between SPA and VCA. Thus, it combines the robustness of SPA and the randomness of VCA, allowing to outperform them when taking the best run among several.
    \item In \Cref{chap:minvolnmf}, we study MinVol NMF, a model similar to MinVol SSMF \eqref{preli:eq:exactminvolssmf}. Thus, it inherits from its identifiability. As a reminder, the MinVol criterion penalizes the volume of the convex hull generated by the columns of $W$ and the origin. Compared to MinVol SSMF, $W$ has to be nonnegative. Also, there are variants of MinVol NMF where the simplex-structure can be either on rows of $H$ or the columns of $W$. We develop a fast algorithm for MinVol NMF based on an inertial block majorization-minimization framework for non-smooth non-convex optimization, that was also used for BSSMF. Then, we show that the MinVol criterion shows promising results for matrix completion.
    \item In \Cref{chap:maxvolnmf}, we introduce a new model, dubbed MaxVol NMF. The core idea revolves around volume penalization, like with MinVol NMF. The difference is that the volume of $H$ is maximized, instead of the volume of $W$ being minimized. In the exact case, MinVol NMF and MaxVol NMF are equivalent. Thus, MaxVol NMF is as identifiable as MinVol NMF. However, in the inexact case, MaxVol NMF offers more control on the sparsity of the decomposition factor $H$. This behavior is interesting in the context of HU. In fact, MaxVol NMF shows better results than MinVol NMF on real hyperspectral datasets.
\end{itemize}

\chapter{Bounded Simplex-Structured Matrix Factorization}\label{chap:bssmf}

\begin{hiddenmusic}{HOME - Resonance}{https://eyewitnessrecords.bandcamp.com/track/resonance}
\end{hiddenmusic}Recently, simplex-structured matrix factorization (SSMF), discussed in \Cref{preli:sec:ssmf}, was introduced as a generalization of NMF~\cite{wu2017stochastic}; see also~\cite{abdolali2021simplex} and the references therein. 
SSMF does not impose any constraint on $W$, while it requires $H$ to be column stochastic, that is, $H(:,j) \in \Delta^r$ for all $j$. As a reminder, ${\Delta^r = \{ x \in\mathbb{R}^r \ | \ x \geq 0, e^\top x = 1 \}}$ is the probability simplex  and $e$ is the vector of all ones of appropriate dimension. SSMF is closely related to various machine learning problems, such as latent Dirichlet allocation, clustering, 
and the mixed membership stochastic block model; see~\cite{bakshilearning} and the references therein. 
Let us recall why SSMF is a generalizarion of NMF by considering the exact NMF model, $X=WH$. Let us normalize the input matrix such that the entries in each column sum to one (w.l.o.g.\ we assume $X$, and $W$, do not have a zero column), that is, such that $X^\top e = e$, and let us impose w.l.o.g.\ that the entries in each column of $W$ also sum to one (by the scaling degree of freedom in the factorization $WH$), that is, $W^\top e = e$. Then, we have 
\begin{equation} \label{bssmf:eq:sumtoone}
 X^\top e = e = (WH)^\top e = H^\top W^\top e = H^\top e, 
\end{equation}
so that $H$ has to be column stochastic, since $H \geq 0$ and $H^\top e = e$ is equivalent to $H(:,j)\in\Delta^r$ for all $j$. 

In this chapter, we introduce bounded simplex-structured matrix factorization (BSSMF)\label{acro:BSSMF}. BSSMF imposes the columns of $W$ to belong to a hyperrectangle, namely $W(i,j) \in [a_i, b_i]$ for all $i,j$ for some parameters $a_i \leq b_i$ for all $i$. For simplicity, given $a \leq b \in \mathbb{R}^m$, we denote the hyperrectangle 
\[
 [a,b] = \{ y \in \mathbb{R}^m 
 \ | \ a_i \leq y_i \leq b_i \text{ for all } i \}, 
\] and refer to it as an interval. The hyperrectangle constraint on $W$ is denoted as $W(:,j)\in[a,b]$ for all $j$. Let us formally define BSSMF. 

\begin{definition}[BSSMF] 
Let $X \in \mathbb{R}^{m \times n}$, let \mbox{$r \leq \min(m,n)$} be an integer, and 
let $a, b \in \mathbb{R}^m$ with $a \leq b$.  
The pair 
$(W,H) 
\in 
\mathbb{R}^{m \times r} 
\times 
\mathbb{R}^{r \times n}$ 
is a BSSMF of $X$ of size $r$ for the interval $[a,b]$ if 
\[X=WH,\quad
W(:,k) \in [a,b] \text{ for all } k, 
\quad  
H(:,j)\in\Delta^r \text{ for all } j. 
\]
\end{definition}  
Since the columns of $H$ define convex combinations, the convex hull of the columns of  $X=WH$ is contained in the convex hull of the columns of $W$, which is itself contained in the hyperrectangle $[a,b]$. 
This implies that the hyperrectangle $[a,b]$ must contain the columns of the data matrix, $X=WH$. BSSMF reduces to SSMF when $a_i = -\infty$ and $b_i = +\infty$ for all $i$. When $X \geq 0$, BSSMF reduces to NMF when $a_i = 0$ and $b_i = +\infty$ for all $i$, after a proper normalization of $X$; see the discussion around Equation~\eqref{bssmf:eq:sumtoone}.

\paragraph{Outline and contribution of the chapter}

The chapter is organized as follows. 
In \Cref{bssmf:sec:motiv}, we explain the motivation of introducing BSSMF.
In \Cref{bssmf:sec:algo}, we propose an efficient algorithm for BSSMF.
In \Cref{bssmf:sec:identifiability_bssmf}, we provide an identifiability result for BSSMF, 
which follows from an identifiability result for NMF. 
In \Cref{bssmf:sec:numexp}, we illustrate the effectiveness of BSSMF on two applications: 
\begin{itemize}
\item Image feature extraction: the entries of $X$ are pixel intensities. For example, for a gray level image, the entries of $X$ belong to the interval $[0,255]$.

\item Recommender systems: 
the entries of $X$ are ratings of users for some items (e.g., movies). These ratings belong to an interval, e.g., [1,5] for the Netflix and MovieLens datasets.

\end{itemize}

\section{Motivation of BSSMF} \label{bssmf:sec:motiv}

The motivation to introduce BSSMF is mostly fourfold; this is described in the next four paragraphs. 

\paragraph{Bounded low-rank approximation}

When the data naturally belong to intervals, imposing the approximation to belong to the same interval allows to provide better approximations, taking into account this prior information. 
Imposing that the entries in $W$ belong to some interval and that $H$ is column stochastic resolves this issue. BSSMF implies that the columns of the  approximation $WH$ belong to the same interval as the columns of $W$. In fact,  for all $j$, 
\[
X(:,j)  \; \approx \;  W H(:,j)  \; \in \;  [a,b], 
\]
since $W(:,k) \in [a,b]^m$ for all $k$, and the entries of ${H(:,j)}$ are nonnegative and sum to one. 

Another closely related model was proposed in~\cite{liu2021factor} where the entries of the factors $W$ and $H$ are required to belong to bounded intervals. The authors showed that their model is suitable for clustering. Nonetheless, it is not clear how to choose the lower and upper bounds on the entries of $W$ and $H$ to obtain tight lower and upper bounds for their product $WH$. With BSSMF the choice for the lower and upper bounds is easier, e.g., choosing $a_i$ and $b_i$ to be the smallest and largest entry in $X(i,:)$, respectively, that is, bounding $W$ in the same way the data matrix is; see \Cref{bssmf:sec:identifiability_bssmf} for more details. 

\paragraph{Interpretability} 

BSSMF allows us to easily interpret both factors: the columns of $W$ can be interpreted in the same way as the columns of $X$ (e.g., as movie ratings, or pixel intensities), while the columns of $H$ provide a soft clustering of the columns of $X$ as they are column stochastic.
 BSSMF can be interpreted geometrically similarly as SSMF and NMF: the convex hull of the columns of $W$, $\conv(W)$, must contain $\conv(X)$, since $X(:,j) = WH(:,j)$ for all $j$ where $H$ is column stochastic, while it is contained in 
the hyperrectangle $[a,b]$:
\[
\conv(X) 
\quad \subseteq \quad 
\conv(W) 
\quad \subseteq \quad 
[a,b].
\] 

Imposing bounds on the approximation, via the element-wise constraints $a \leq WH \leq b$ for some $a, b \in \mathbb{R}$, was proposed in~\cite{kannan2014bounded} and applied successfully to recommender systems. 
However, this model does not allow to interpret the basis factor, $W$, in the same way as the data. Some elements in $W$ will probably be out of the rating range because $W$ is not directly constrained. Hence, the basis elements in $W$ can only be interpreted as ``eigen users'', while with BSSMF, the basis elements can be interpreted as virtual meaningful users. 
It is also difficult to interpret the factor $H$ as it could contain negative contributions. In fact, only imposing $a \leq WH \leq b$ typically leads to dense factors $W$ and $H$ (that is, factors that do not contain many zeros, as opposed to sparse factors), 
while in most applications interpretability usually comes with a certain sparsity degree in at least one of the factors.

A closely related model that tackles blind source separation is bounded component analysis (BCA) proposed in~\cite{cruces2010bounded, erdogan2013class}, where the sources are assumed to belong to compact sets (hyperrectangle being a special case), while no constraints is imposed on the mixing matrix. Again, without any constraints on the mixing matrix, BCA will generate dense factors with negative linear combinations which are difficult to interpret. 
Let us note that their motivation is different from ours, as their objective is to extract mixed sources, while ours is to extract interpretable features and decompose data through them. In \cite{mansour2002blind}, the authors also proposed a blind source separation algorithm for bounded sources based on geometrical concepts. The mixtures are assumed to belong to a parallelogram. The proposed separation technique is relies on mapping this parallelogram to a rectangle. Again, their objective is to extract mixed sources. Nonetheless, working with a domain different from a hyperrectangle could be of interest for future work.
 
\paragraph{Identifiability} 

Identifiability is key in practice as it allows to recover the ground truth that generated the data; see the discussion 
in \Cref{preli:sec:identif},~\cite{xiao2019uniq, kueng2021binary} and the references therein.
A drawback of SSMF is that it is never identifiable, see \Cref{bssmf:sec:identif_ssmf} for further details. On the counterpart NMF can be identifiable, which is discussed in \Cref{bssmf:sec:identif_nmf}. Nonetheless, the conditions are not mild. For NMF to be identifiable, it is necessary that the supports of the columns of $W$ (that is, the set of non-zero entries) are not contained in one another (this is called a Sperner family), and similarly for the supports of the rows of $H$; 
see, e.g.,~\cite{moussaoui2005non, laurberg2008theorems}. 
This requires the presence of zeros in each column of $W$ and row of $H$,  which can be a strong condition in some applications. For example, in hyperspectral unmixing, $W$ is typically not sparse because it recovers spectral signatures of constitutive materials which are typically positive. 
Although the conditions for NMF (and SSMF) to be identifiable can be weakened using additional constraints and regularization terms, it then requires hyperparameter tuning procedures. 
In~\cite{tatli2021generalized}, they propose a model where the columns of $H$ belong to a polytope. Using a maximum volume criterion on the convex hull of $H$, their model is identifiable under the condition that the convex hull of $H$ contains the ellipsoid of maximum volume inscribed in the constraining polytope. The use of the maximum volume criterion also requires hyperparameter tuning.
In~\cite{cruces2010bounded, erdogan2013class}, the sources are identifiable by optimizing some geometric criterion, 
respectively minimizing a perimeter,  
and maximizing the ratio between the volume of an ellipsoid and the volume of a hyperrectangle.  
These identifiability conditions are not relevant to our model. 
As we will see in \Cref{bssmf:sec:identifiability_bssmf}, BSSMF is identifiable under relatively mild conditions, 
while it does not require parameter tuning, as opposed to most regularized structured matrix factorization models that are identifiable. Let us note that it is also possible to formulate identifiable nonlinear matrix approximation models like the bilinear model of~\cite{deville2019separability}, but this is out of the scope of this chapter.

\paragraph{Robustness to overfitting} 

Another drawback of NMF and SSMF is that they are rather sensitive to the choice 
of~$r$. When $r$ is chosen too large, these two models are over-parameterized and will typically lead to overfitting. 
This is a well-known behaviour that can be addressed with additional regularization terms that need to be fine-tuned~\cite{rendle2021revisiting}. As we will see experimentally in \Cref{bssmf:sec:robust} for matrix completion, 
without any parameter tuning, BSSMF is much more robust to overfitting than NMF and unconstrained matrix factorization. The reason is that the additional bound constraints on $W$ and sum-to-one constraint on $H$ prevents columns of $W$ and of $WH$ from going outside the feasible range, $[a,b]$. In turn, BSSMF will be less sensitive to noise and an overestimation of $r$. For example, an outlier that falls outside the feasible set $[a,b]$ will not pose problems to BSSMF, while it may significantly impact the NMF and SSMF solutions.

\section{Inertial block-coordinate descent algorithm for BSSMF} \label{bssmf:sec:algo}

In this chapter, we consider the following BSSMF problem 
\begin{equation}
\label{bssmf:eq:BSSMF}
\begin{split}
\min_{W,H} g(W,H)&:=\frac{1}{2}\| X - WH \|_F^2 \\
\text{ such that } & W(:,k) \in [a,b] \text{ for all } k, \\
& H \geq 0, \text{ and } H^\top e=e, 
\end{split}
\end{equation} 
that uses the squared Frobenius norm to measure the error of the approximation. 

\subsection{Proposed algorithm}\label{bssmf:sec:proposedalgo}

Most NMF algorithms rely on block coordinate descent methods, that is, they update a subset of the variables at a time, such as the popular multiplicative updates of Lee and Seung~\cite{lee2001algorithms}, the hierarchical alternating least squares algorithm~\cite{Cichocki07HALS, GG12}, and a fast gradient based algorithm~\cite{guan2012nenmf}; see, e.g.,~\cite[Chapter 8]{gillis2020} and the references therein for more detail. More recently, an inerTial block majorIzation minimization framework for non-smooth non-convex opTimizAtioN (TITAN)\label{acro:titan} was introduced in~\cite{hien2023inertial} and has been shown to be particularly powerful to solve matrix and tensor factorization problems~\cite{hien2019extrapolNMF, man2021accelerating,vuthanh2021inertial}. 

To solve~\eqref{bssmf:eq:BSSMF}, we therefore apply TITAN which updates one block $W$ or $H$ at a time while fixing the value of the other block. In order to update $W$ (resp.\ $H$), TITAN chooses a block surrogate function for $W$ (resp.\ $H$), embeds an inertial term to this surrogate function and then minimizes the obtained inertial surrogate function. We have $\nabla_W g(W,H)=-(X-WH) H^\top$ which is Lipschitz continuous in $W$ with the Lipschitz constant $\|HH^\top\|$, where $\|.\|$ is the spectral norm. Similarly, $\nabla_H g(W,H)=-W^\top (X-WH)$ is Lipschitz continuous in $H$ with constant $\|W^\top W\|$. Hence, we choose the Lipschitz gradient surrogate for both $W$ and $H$ and choose the Nesterov-type acceleration as analyzed in \cite[Section 4.2.1]{hien2023inertial} and \cite[Remark 4.1]{hien2023inertial}, see also~\cite[Section 6.1]{hien2023inertial} and~\cite{vuthanh2021inertial} for similar applications. 

Recall that applying BSSMF to recommender systems is one of our motivations, meaning that our model should be able to handle missing entries in $X$. Let us consider the more general model 
\begin{equation}
\label{bssmf:eq:WBSSMF}
\begin{split}
\min_{W,H} g(W,H) & :=\frac{1}{2}\| M\circ(X - WH) \|_F^2 \\
\text{ such that } & W(:,k) \in [a,b] \text{ for all } k, \\
& H \geq 0, \text{ and } H^\top e=e,    
\end{split}
\end{equation} 
where $\circ$ corresponds to the Hadamard product, and $M$ is a weight matrix which can model missing entries using $M(i,j) = 0$ when  $X(i,j)$ is missing, and $M(i,j) = 1$ otherwise. It can also be used in other contexts;  
see, e.g., \cite{gabriel1979lower, SJ03, gillis2011low}.
\sloppy TITAN can also be used to solve~\eqref{bssmf:eq:WBSSMF}, where the gradients are equal to ${\nabla_W g(W,H)=-(M \circ (X-WH)) H^\top}$ and ${\nabla_H g(W,H)=-W^\top (M\circ  (X-WH))}$. We acknowledge that the identifiability result that will be presented in \Cref{bssmf:sec:identifiability_bssmf} does not hold for the case where some data are missing, this is an interesting direction of future research. 
\cref{bssmf:alg:BSSMF} describes TITAN for solving the general problem~\eqref{bssmf:eq:WBSSMF}, where $[.]^a_b$ is the column-wise projection on $[a,b]$ and $[.]_{\Delta^r}$ is the column wise projection on the simplex $\Delta^r$. Let us clarify that our implementation of TITAN, although looking similar to alternating fast projection gradient methods (AFPGMs), differs from them. Concretely, with TITAN, the inertial sequence is evolving at every iteration and is not restarted when the algorithm alternates between updating $W$ and updating $H$. Typically, AFPGMs would restart the inertial sequence when the algorithm alternates between the blocks, because their goal is to solve alternatively the sub-problems. TITAN considers the whole problem instead of considering several sub-problems. Hence, TITAN tries to accelerate the global convergence of the sequences rather than trying to accelerate the convergence for the sub-problems. For more details, see \Cref{minvol:sec:algoTITAN} where an implementation of TITAN for MinVol NMF is shown to be faster than an alternating projection gradient method with Nesterov extrapolation.

Due to our derived algorithm being a particular instance of TITAN with Lipschitz gradient surrogates~\cite[Section 4.2]{hien2023inertial}, \cref{bssmf:alg:BSSMF} guarantees a subsequential convergence, that is, every limit point of the generated sequence is a stationary point of Problem~\eqref{bssmf:eq:BSSMF}. The Julia code for \cref{bssmf:alg:BSSMF} is available on gitlab\footnote{\href{bssmf:https://gitlab.com/vuthanho/bssmf.jl}{https://gitlab.com/vuthanho/bssmf.jl}} (a MATLAB code is also available on gitlab\footnote{\href{bssmf:https://gitlab.com/vuthanho/bounded-simplex-structured-matrix-factorization}{https://gitlab.com/vuthanho/bounded-simplex-structured-matrix-factorization}} but it does not handle missing data). We omit the implementation details here, but let us mention that when data are missing, our Julia implementation does not compute the Hadamard product with $M$ explicitly but rather takes advantage of the sparsity of the data by using multithreading to improve the computational time. The projections $[.]^a_b$ and $[.]_{\Delta^r}$ are also computed using multithreading.

\begin{algorithm}
\caption{Proposed algorithm for BSSMF}
\label{bssmf:alg:BSSMF}
\SetKwInOut{Input}{input}
\SetKwInOut{Output}{output}
\DontPrintSemicolon
\Input{Input data matrix $X \in \mathbb{R}^{m \times n}$, bounds $a\leq b\in\mathbb{R}^m$, initial factors $W \in \mathbb{R}^{m \times r}$ s.t. $W(:,k) \in [a,b]$ for all $k$ and simplex structured $H \in \mathbb{R}^{r \times n}_+$, weights $M\in[0,1]^{m\times n}$}
\Output{$W$ and $H$}
$\alpha_1=1$, $\alpha_2=1$, $W_{old}=W, H_{old}= H$, $L_W^{prev}=L_W=\| H H^\top\|$,  $L_H^{prev}=L_H=\|W^\top W\|$\;
\Repeat{some stopping criteria is satisfied\nllabel{bssmf:alg:BSSMF:line:outerloop}}{
    \While{stopping criteria not satisfied\nllabel{bssmf:alg:BSSMF:line:Winnerloop}}{
        $\alpha_{0}=\alpha_1, \alpha_1=(1+\sqrt{1+4\alpha_0^2})/2$\;
        $\beta_{W}=\min\left[~(\alpha_0-1)/\alpha_{1},0.9999\sqrt{L_W^{prev}/L_W} ~\right] $\;
        $\overbar{W}\leftarrow W+\beta_{W}(W-W_{old}) $\;
        $W_{old}\leftarrow W$\;
        $W \leftarrow \left[\overbar{W}+\frac{(M\circ(X-\overbar{W}H))H^{\top}}{L_W}\right]_a^b$\;
        $ L_W^{prev} = L_W$\;
    }
    $ L_H \leftarrow \|W^\top W\| $\;
    \While{stopping criteria not satisfied\nllabel{bssmf:alg:BSSMF:line:Hinnerloop}}{
        $\alpha_{0}=\alpha_2, \alpha_2=(1+\sqrt{1+4\alpha_0^2})/2$\;
    	$ \beta_{H}=\min\left[~(\alpha_0-1)/\alpha_{2},0.9999\sqrt{L_H^{prev}/L_H} ~\right] $\;
    	$ \overbar{H} \leftarrow H+\beta_{H}(H-H_{old}) $\;
    	$H_{old}\leftarrow H$\;
    	$ H \leftarrow \left[\overbar{H}+\frac{W^{\top} (M\circ(X-W\overbar{H}))}{L_H} \right]_{\Delta^r}$\;\nllabel{bssmf:alg:BSSMF:line:proj}
    	$L_H^{prev} \leftarrow L_H$\;
    }
    $ L_W = \|HH^\top\| $\;
}
\end{algorithm}

\paragraph{Initialization} A simple choice to initialize the factors, $W$ and $H$, in \cref{bssmf:alg:BSSMF} is to randomly initialize them: for all $i$, each entry of $W(i,:)$ is generated using the uniform distribution in the interval $[a_i,b_i]$, while $H$ is generated using a uniform distribution in $[0,1]^{r\times n}$ whose columns are then projected on the simplex $\Delta^r$.

\paragraph{Choice of Lipschitz constant} When some data are missing, the Lipschitz constant of the gradients relatively to $W$ and $H$ could be smaller than $\|HH^\top\|$ and $\|W^\top W\|$, respectively.
Relatively to $H$ for instance, a smaller Lipschitz constant would be $\max_j\|W^\top\diag(M(:,j))W\|$. Indeed, 
\begin{align*}
    \|\nabla_H g(W,H_1) - \nabla_H g(W,H_2)\|_F &=\|W^\top(M\circ(W(H_1 - H_2)))\|_F\\
    \|W^\top(M\circ(W(H_1 - H_2)))\|_F^2 &= \sum_j \|W^\top(M(:,j)\circ(W(H_1(:,j)-H_2(:,j))))\|_F^2\\
    &= \sum_j \|W^\top\diag(M(:,j))W(H_1(:,j)-H_2(:,j))\|_F^2\\
    &\leq \max_j\|W^\top\diag(M(:,j))W\|^2\|H_1-H_2\|_F^2.
\end{align*}
Obviously, $\max_j\|W^\top\diag(M(:,j))W\|\leq\|W^\top W\|$ due to $M$ being binary. Equality is achieved if there exists a $j$ such that $M(:,j)=e$, that is, if at least one column is fully observed. Consequently, $\|W^\top W\|$ is clearly a Lipschitz constant of $\nabla_H g(W,H)$ relatively to $H$, but $\max_j\|W^\top\diag(M(:,j))W\|$ is a tighter one. By symmetry of the problem, this also applies to $\nabla_W g(W,H)$ relatively to $W$, where $\|HH^\top\|$ is a Lipschitz constant, but $\max_i\|H\diag(M(i,:))H^\top\|$ is a tighter one. Yet, we choose to keep $\|HH^\top\|$ and $\|W^\top W\|$ even when some data are missing since those values are faster to compute. This choice can be compensated by data centering; see \Cref{bssmf:sec:centering}. Note that when $M$ is a weight matrix, ${\nabla_W g(W,H)=-(M^{\circ 2} \circ (X-WH)) H^\top}$ and ${\nabla_H g(W,H)=-W^\top (M^{\circ 2}\circ  (X-WH))}$, where $M^{\circ 2}=M\circ M$. Similarly to the case where $M$ is binary, we then retrieve $\max_i\|H\diag(M(i,:))^2H^\top\|$ and $\max_j\|W^\top\diag(M(:,j))^2W\|$ as tighter Lipschitz constants than $\|HH^\top\|$ and $\|W^\top W\|$.

\subsection{Accelerating BSSMF algorithms via data centering}  \label{bssmf:sec:centering}

Not only the BSSMF model is invariant to translations of the input data (this is explained in details in \Cref{bssmf:sec:identifiability_bssmf}), but also the optimization, because of the simplex constraints. In particular, for any $\mu\in\R^m$, minimizing
\begin{equation}
\label{bssmf:eq:SSMF}
f(W,H):=\frac{1}{2}\| X - WH \|_F^2
\end{equation}
or 
\begin{equation}
\label{bssmf:eq:TSSMF}
f_\mu(W,H):=\frac{1}{2}\| X-\mu e^\top - (W-\mu e^\top)H \|_F^2
\end{equation} 
is equivalent in BSSMF, since $\mu e^\top H = \mu e^\top$ as $H$ is column stochastic.
However, \emph{outside the feasible set, $f$ and $f_\mu$ do not have the same topology}. Computing the gradients, we have $\nabla_H f(W,H)=W^\top(WH-X)$ which is Lipschitz continuous in $H$ with the Lipschitz constant $\|W^\top W\|$, and $\nabla_H f_\mu(W,H)=W_\mu^\top(W_\mu H-X_\mu)$ which is Lipschitz continuous in $H$ with the Lipschitz constant $\|W_\mu^\top W_\mu\|$, where $W_\mu=W-\mu e^\top$ and $X_\mu=X-\mu e^\top$. Particularly, for BSSMF, since $W$ can be interpreted in the same way as $X$, we expect $\operatorname{mean_{row}}(X)=\frac{1}{n}Xe\approx\operatorname{mean_{row}}(W)\in\R^m$, where $\operatorname{mean_{row}}(.)$ is the empirical mean of each row of the input. Let us in fact choose $\mu=\operatorname{mean_{row}}(X)$. From~\cite[Theorem 3]{honeine2014eigenanalysis}, we have $\|X_\mu^\top X_\mu\|\leq\|X^\top X\|$. Consequently, we expect the Lipschitz constant $\|W_\mu^\top W_\mu\|$ to be smaller than $\|W^\top W\|$. A smaller Lipschitz constant means that, when updating $H$, the gradient steps are allowed to be larger without losing any convergence guarantee. Hence, with the right translation on our data $X$, the optimization problem on $H$ is unchanged on the feasible set but \cref{bssmf:alg:BSSMF} can be accelerated. 

\begin{figure}[htbp!]
    \centering
    \begin{tikzpicture}
 
\begin{groupplot}[view={30}{70}, 
    group style={group size=1 by 2}, 
    colormap/cool, 
    point meta min=0, 
    point meta max = 0.62, 
    zmax=0.62, 
    xlabel={$H[1]$}, 
    ylabel={$H[2]$}, 
    legend columns = 3,
    legend style={at={(1,1.2)},/tikz/every even column/.append style={column sep=0.2cm}},
    width=0.8\textwidth,height=9cm-21pt
]

\nextgroupplot[colorbar]
\addplot3 [surf,mesh/ordering=y varies,opacity=1] table{bssmf/data/error_notcentered.txt};
\addlegendentry{$\|X-WH\|_F$}

\addplot3 [line width = 2pt,opacity=0.4,dashed] table{bssmf/data/feasibleset.txt};
\addlegendentry{feasible set}

\addplot3 [line width = 1pt,opacity=0.46,forget plot] table{bssmf/data/gradientstep_notcentered.txt};
\addplot3 [only marks, color = magenta,opacity=0.6,mark size = 1.5pt] table{bssmf/data/gradientonly_notcentered.txt};
\addlegendentry{gradient step}

\addplot3 [only marks, color = lime,opacity=0.6,mark size = 1.5pt] table{bssmf/data/projonly_notcentered.txt};
\addlegendentry{projected gradient}

\addplot3 [only marks, mark = x,mark size = 3pt,color = red] coordinates {(0.4,0.6,0)};
\addlegendentry{optimal solution}

\nextgroupplot[colorbar]
\addplot3 [surf,mesh/ordering=y varies] table{bssmf/data/error_centered.txt};
\addplot3 [line width = 2pt,opacity=0.4,dashed] table{bssmf/data/feasibleset.txt};
\addplot3 [line width = 1pt,opacity=0.6,forget plot] table{bssmf/data/gradientstep_centered.txt};
\addplot3 [only marks, color = magenta,opacity=0.6,mark size = 1.5pt] table{bssmf/data/gradientonly_centered.txt};
\addplot3 [only marks, color = lime,opacity=0.6,mark size = 1.5pt] table{bssmf/data/projonly_centered.txt};
\addplot3 [only marks, mark = x,mark size = 3pt,color = red] coordinates {(0.4,0.6,0)};
\end{groupplot}
 
\end{tikzpicture}
    \caption[Influence of centering the data on the cost function topology regarding $H$ via a small example]{Influence of centering the data on the cost function topology regarding $H$ via a small example ($m=2,r=2,n=1$). Top: without centering. Bottom: with centering. 
    Five projected gradient steps are shown, decomposed through one gradient descent step followed by its projection onto the feasible set.}
    \label{bssmf:fig:centering_topo}
\end{figure}
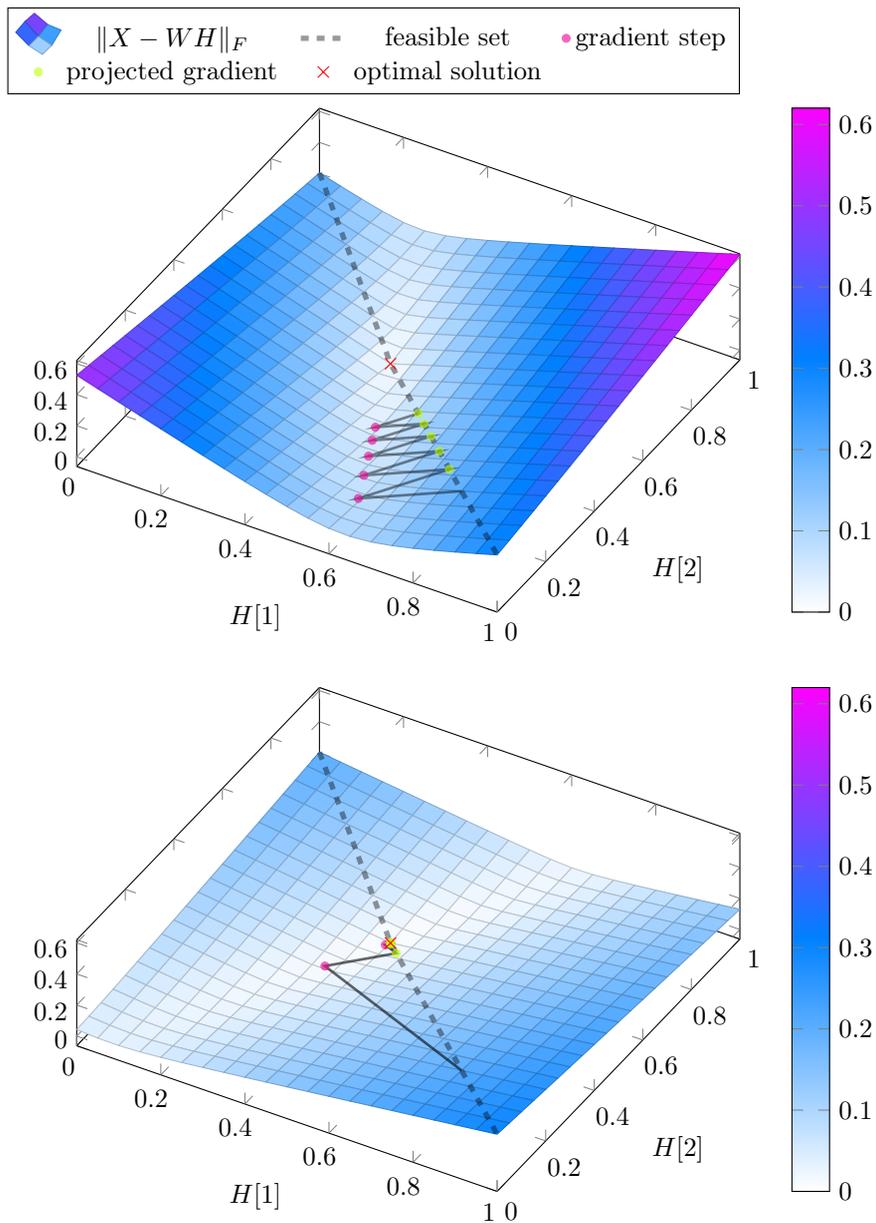


Let us illustrate this behavior on a small example with \mbox{$m=2$}, $n=1$, $r=2$. We choose $$
{X=\begin{pmatrix}0.4 & 0.3\\0.7 & 0.2\end{pmatrix}\begin{pmatrix}0.4\\0.6\end{pmatrix}}.
$$ 
We fix $$
W=\begin{pmatrix}0.4 & 0.3\\0.7 & 0.2\end{pmatrix}, 
$$ 
and try to solve, with respect to $H$, \cref{bssmf:eq:SSMF} and \cref{bssmf:eq:TSSMF}  with $\mu=\operatorname{mean_{row}}(X)$. 
We perform 5 projected gradient steps and display the results on \Cref{bssmf:fig:centering_topo}. On the top, 5 projected gradient steps are performed to update $H$ based on the original data $X$. On the bottom, 5 projected gradient steps are performed to update $H$ based on the centered data $X$. 
The feasible sets (in dash) are exactly the same, and therefore the optimal solutions are also the same. 
However, we observe that the landscape of the cost function outside the feasible region is smoother when the data are centered. This allows the solver to converge faster towards the optimal solution, as the gradients point better towards the optimal solution and the step sizes are larger. 
The improvement regarding the convergence speed by applying centering with real data will probably not be as drastic as in this small example. Still, minimizing a smoother function is always advantageous, and this will be shown empirically on real data in \Cref{bssmf:sec:convspeed}.

\subsection{Convergence speed and effect of acceleration strategies on real data} \label{bssmf:sec:convspeed}

In this subsection, the goal is twofold: 
(1)~show the effect of the extrapolation in TITAN by comparing \cref{bssmf:alg:BSSMF} to a non-extrapolated block coordinate descent, and 
(2)~show the acceleration effect of centering the data.

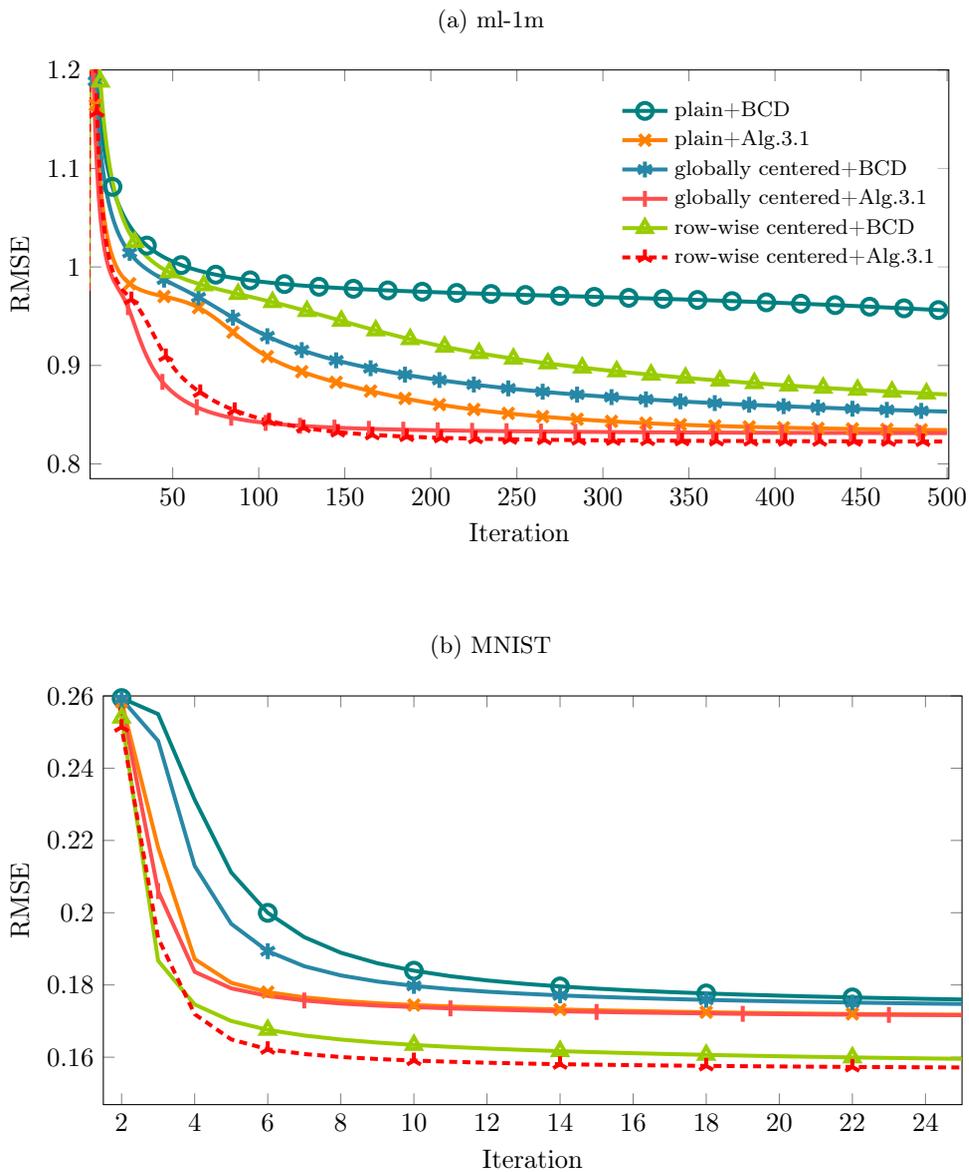
\begin{figure}
    \begin{subfigure}{\textwidth}
        \caption{ml-1m}
        \label{bssmf:fig:ml1m_extra_center}
        \tikzexternalexportnextfalse
        \begin{tikzpicture}
\begin{axis}[
            width=1\linewidth,
            height=7cm,
            ymax = 1.2,
            xmin = 2,
            xmax = 501,
            ylabel = {RMSE},
            xlabel = {Iteration},
            cycle list name=exotic,
            mark size = 3pt,
            mark repeat = 20,
            legend cell align={left},
            legend style={font=\footnotesize,at={(1,0.95)},anchor=north east,draw=none,fill opacity=0.5,text opacity=1}]

\addplot+[mark phase=10,mark = o,line width = 1.5pt] table[x index = 0, y index = 1]{bssmf/data/ml1m_extra_center.txt};
\addplot+[mark = x,line width = 1.5pt] table[x index = 0, y index = 2]{bssmf/data/ml1m_extra_center.txt};
\addplot+[mark = asterisk,line width = 1.5pt] table[x index = 0, y index = 3]{bssmf/data/ml1m_extra_center.txt};
\addplot+[mark = |,line width = 1.5pt] table[x index = 0, y index = 4]{bssmf/data/ml1m_extra_center.txt};
\addplot+[mark=triangle,line width = 1.5pt] table[x index = 0, y index = 5]{bssmf/data/ml1m_extra_center.txt};
\addplot+[mark=Mercedes star,line width = 1.5pt] table[x index = 0, y index = 6]{bssmf/data/ml1m_extra_center.txt};
\legend{plain+BCD,plain+Alg.\ref{bssmf:alg:BSSMF},globally centered+BCD,globally centered+Alg.\ref{bssmf:alg:BSSMF},row-wise centered+BCD,row-wise centered+Alg.\ref{bssmf:alg:BSSMF}}
\end{axis}
\end{tikzpicture}
    \end{subfigure}

    \begin{subfigure}{\textwidth}
        \caption{MNIST}
        \label{bssmf:fig:mnist_extra_center}
        \begin{tikzpicture}
\begin{axis}[
            width=1\linewidth,
            height=7cm,
            ymax = 0.26,
            xmin = 1.5,
            xmax = 25,
            ylabel = {RMSE},
            xlabel = {Iteration},
            cycle list name=exotic,
            mark size = 3pt,
            mark repeat = 4,
            legend cell align={left},
            legend style={font=\footnotesize,at={(1,0.85)},anchor=north east,draw=none,fill opacity=0.5,text opacity=1}]

\addplot+[mark = o,line width = 1.5pt] table[x index = 0, y index = 1]{bssmf/data/mnist_extra_center.txt};
\addplot+[mark = x,line width = 1.5pt] table[x index = 0, y index = 2]{bssmf/data/mnist_extra_center.txt};
\addplot+[ mark = asterisk,line width = 1.5pt] table[x index = 0, y index = 3]{bssmf/data/mnist_extra_center.txt};
\addplot+[mark phase=2, mark = |,line width = 1.5pt] table[x index = 0, y index = 4]{bssmf/data/mnist_extra_center.txt};
\addplot+[mark = triangle,line width = 1.5pt] table[x index = 0, y index = 5]{bssmf/data/mnist_extra_center.txt};
\addplot+[mark = Mercedes star,line width = 1.5pt] table[x index = 0, y index = 6]{bssmf/data/mnist_extra_center.txt};
\end{axis}
\end{tikzpicture}
    \end{subfigure}
    \caption[Evolution of the training error for ml-1m and MNIST]{Evolution of the training error for ml-1m and MNIST, averaged on 10 runs. For ml-1m, $r=5$, 1 inner iteration. For MNIST, $r=50$, 10 inner iterations.}
    \label{bssmf:fig:extra_center}
\end{figure} 
We will apply the BSSMF model on MNIST and ml-1m (these two datasets are properly introduced respectively in \Cref{bssmf:sec:inter} and \Cref{bssmf:sec:robust}) in six different scenarios: 3 data related scenarios $\times$ 2 algorithmic related scenarios. The data scenarios are raw data, globally centered data, and row-wise centered data (respectively called `plain', `globally centered' and `row-wise centered' in \Cref{bssmf:fig:extra_center}). Globally centered data are such that $\mu=\frac{e^\top X e}{mn}e$ and row-wise centered data are such that $\mu=\frac{1}{n}Xe$. Note that with a global centering, result from \cite[Theorem 3]{honeine2014eigenanalysis} does not hold anymore. Yet, we propose to see here how this centering strategy behaves.
For each data case, 2 algorithms are tested: 
(1)~\cref{bssmf:alg:BSSMF}, and (2)~a standard block coordinate descent (BCD) which is \cref{bssmf:alg:BSSMF} where the $\beta$'s are fixed to 0; this corresponds to the popular proximal alternating linearized minimization (PALM) algorithm~\cite{bolte2014proximal}. 
When the algorithms are compared on the same data scenario, \cref{bssmf:alg:BSSMF} always converges faster and to a better solution than BCD. We also observe that when the data are centered, globally or row-wise, applying the same algorithm always lead to faster convergence than on the plain case. Let us first comment the results on ml-1m (\Cref{bssmf:fig:ml1m_extra_center}). Applying BCD on the globally centered data is almost as fast as applying \cref{bssmf:alg:BSSMF} in the plain case, meaning that a good preprocessing is almost as important as a good acceleration strategy. With BCD only, global centering is faster than row-wise centering. With \cref{bssmf:alg:BSSMF}, global centering converges faster but row-wise centering converges to slightly better solutions.
\begin{table}[hbtp!]
    \centering
    \begin{tabular}{l|c|c|c}
        & plain & globally centered & row-wise centered \\ \hline
    BCD & 0.93  & 0.89  & 0.91 \\ \hline
    \cref{bssmf:alg:BSSMF} & 0.87 & 0.87 & 0.87
    \end{tabular}
    \caption[RMSE on the test set for ml-1m]{RMSE on the test set for ml-1m depending on the algorithm and the centering strategy}
    \label{bssmf:tab:rmseml1mcentering}
\end{table}
The root-mean-square errors (RMSEs) on the test set are available in \Cref{bssmf:tab:rmseml1mcentering}, highlighting the importance of a good acceleration strategy. This could be expected, since the centering only affects the convergence speed, but does not change the solutions.
Let us now comment the results on MNIST shown on \Cref{bssmf:fig:mnist_extra_center}.
Global centering does not really improve the convergence speed compared to the plain case, regardless of the used algorithm. Interestingly, row-wise centering provides a great improvement in convergence speed and better local minima, regardless of the used algorithm. In this case, a good centering strategy seems even more important than a good acceleration strategy.
Globally, regardless of the dataset, applying \cref{bssmf:alg:BSSMF} on centered data is the best strategy as compared with using plain data. As a consequence, it will be our default choice for the experiments in \Cref{bssmf:sec:numexp}. 

As mentioned above, when entries are missing, \cref{bssmf:alg:BSSMF} can take advantage of the sparsity of the data and uses multithreading. We report in~\Cref{bssmf:tab:time_ml-1m} the computation time of~\cref{bssmf:alg:BSSMF} in the experiment settings of~\Cref{bssmf:fig:ml1m_extra_center}, given by the macro \texttt{@btime} from the package \texttt{BenchmarkTools.jl}. The computation time in the settings of the experiment in~\Cref{bssmf:fig:mnist_extra_center} is reported in~\Cref{bssmf:tab:time_mnist}. When the dataset is full, like with MNIST, multithreading is only used for the projections $[.]^a_b$ and $[.]_{\Delta^r}$. Of course, multithreading is also employed for every matrix multiplication. However, the numbers of threads shown in \Cref{bssmf:tab:time_ml-1m,bssmf:tab:time_mnist} do not affect the computation time of matrix multiplications, as BLAS selects its own number of threads, independent of the number assigned to Julia. Note that there is no distinction between~\cref{bssmf:alg:BSSMF} and BCD in terms of computation time because the computation of the acceleration is negligible compared to the other computations.
\begin{table}[ht]
    \centering
    \begin{tabular}{l|c|c|c|c|c|c|c}
    \# threads & 1 & 2 & 4 & 6 & 8 & 10 & 12\\ \hline
    time (s) & 30.53 & 5.14 & 2.98 & 2.85 & 3.00 & 2.78 & 3.31 
    \end{tabular}
    \caption{Computation time of~\cref{bssmf:alg:BSSMF} in the experiment settings of~\Cref{bssmf:fig:ml1m_extra_center} depending on the number of used threads.}
    \label{bssmf:tab:time_ml-1m}
\end{table}
\begin{table}[ht]
    \centering
    \begin{tabular}{l|c|c|c|c|c|c|c}
    \# threads & 1 & 2 & 4 & 6 & 8 & 10 & 12\\ \hline
    time (s) & 27.79 & 21.92 & 16.67 & 15.22 & 15.73 & 16.01 & 16.65 
    \end{tabular}
    \caption{Computation time of~\cref{bssmf:alg:BSSMF} in the experiment settings of~\Cref{bssmf:fig:mnist_extra_center} depending on the number of used threads.}
    \label{bssmf:tab:time_mnist}
\end{table}

\section{Identifiability of BSSMF} \label{bssmf:sec:identifiability_bssmf}

A main motivation to introduce Bounded simplex-structured matrix factorization (BSSMF) is that it is identifiable under weaker conditions than NMF. 
We now state our main identifiability result for BSSMF, it is a consequence of the identifiability result of NMF and the following simple observation: $X = WH$ is a BSSMF for the interval $[a,b]$ implies that 
$be^\top - X =(be^\top - W)H$ and $X - ae^\top =(W - ae^\top) H$ are Exact NMF decompositions. 

\begin{theorem} \label{bssmf:th:uniqueBSSMF} 
Let ${W \in \mathbb{R}^{m \times r}}$ and ${H \in \mathbb{R}^{r \times n}}$ satisfy ${W(:,k) \in [a,b]}$ for all $k$ for some $a \leq b$,   
${H \geq 0}$, and ${H^\top e = e}$. 
If $\binom{W - a e^\top}{b e^\top - W}^\top \in \mathbb{R}^{r \times 2m}$ and $H \in \mathbb{R}^{r \times n}$ are sufficiently scattered,  
then the BSSMF $(W,H)$ of $X=WH$ of size $r = \rank(X)$ for the interval $[a,b]$ is essentially unique. 
\end{theorem}
\begin{proof}
Let $(W,H)$ be a BSSMF of $X$ for the interval $[a,b]$. 
As in the proof of \Cref{bssmf:lem:transfoBSSMF}, we have 
\[
X - a e^\top = WH - a e^\top = (W - a e^\top) H, 
\]
since $e^\top=e^\top H$. 
This implies that $(W- a e^\top,H)$ is an Exact NMF of $X - a e^\top$, since $W- a e^\top$ and $H$ are nonnegative. 
Similarly, we have 
\[
b e^\top - X 
= b e^\top - WH 
= (b e^\top - W)H,  
\]
which implies that $(b e^\top - W, H)$ is an Exact NMF of $b e^\top - X$, since  $b e^\top - W \geq 0$. 
Therefore, we have the Exact NMF  
\[
\left( 
\begin{array}{c}
X - a e^\top \\ 
b e^\top - X 
\end{array}
\right) = 
\left( 
\begin{array}{c}
W - a e^\top \\ 
b e^\top - W  
\end{array}
\right) H.  
\] 
By \Cref{bssmf:th:uniqNMFSSC}, this Exact NMF is unique if $\binom{W - a e^\top}{b e^\top - W}^\top$ and $H$ satisfy the SSC. 
This proves the result: in fact, the derivations above hold for any BSSMF of $X$. 
Hence, if $(W,H)$ was not an essentially unique BSSMF of $X$, there would exist another Exact NMF of $\binom{W - a e^\top}{b e^\top - W}^\top$, not obtained by permutation and scaling of $
\left( \binom{
W - a e^\top}{
b e^\top - W  
} , H \right)$, 
a contradiction. 
\end{proof}

Let us note that $W-ae^\top$ and $H$ being SSC, or $be^\top-W$ and $H$ being SSC, are also sufficient conditions for identifiability. These conditions are stronger, as $W-ae^\top$ being SSC or $be^\top-W$ being SSC implies that $\binom{W - a e^\top}{b e^\top - W}^\top$ is SSC. However, $\binom{W - a e^\top}{b e^\top - W}^\top$ does not imply that $W-ae^\top$ or $be^\top-W$ is SSC. The condition that $\binom{W - a e^\top}{b e^\top - W}^\top$ is SSC is much weaker than requiring $W^\top$ to be SSC in NMF. 
In fact, in NMF, $W^\top$ being SSC requires that it contains some
zero entries (at least $r-1$ per row \cite[Th.~4.28]{gillis2020}; this can also be seen on \Cref{preli:fig:geoSSC} in the case $r=3$). 
Since the SSC is only defined for nonnegative matrices and $W^\top$ contains zeros, $a$ has to be equal to the zero vector. 
In this case, $W^\top$ being SSC implies that $W^\top - e a^\top$ 
is SSC, and hence the corresponding BSSMF is identifiable. However, the reverse is not true. In fact, $\binom{W - a e^\top}{b e^\top - W}^\top$ being SSC means that sufficiently many values in $W$ are equal to its minimum and maximum bounds in $a$ and $b$. 
For example, in recommender systems, with $W(i,j) \in [1,5]$ for all $(i,j)$, many entries of $W$ are expected to be equal to 1 or to 5 (the minimum and maximum ratings), so that $\binom{W - a e^\top}{b e^\top - W}^\top$ will contain many zero entries, and hence likely to satisfy the SSC~\cite{fu2018identifiability}. On the other hand, $W$ is positive, and hence it cannot be part of an essentially unique Exact NMF.

Let us illustrate the difference between NMF and BSSMF on a simple example. 
\begin{example}[Non-unique NMF vs.\ unique BSSMF] 
\label{bssmf:example:1}
Let $\omega \in [0,1)$ and let 
\[
A_{\omega}  
= 
\left( \begin{array}{cccccc}
\omega & 1 & 1 & \omega & 0 & 0 \\ 
1 & \omega & 0 & 0 & \omega & 1 \\ 
0 & 0 & \omega & 1 & 1 & \omega \\ 
\end{array}
\right). 
\]
For $\omega < 0.5$, $A_{\omega}$ satisfies the SSC, while it does not for $\omega \leq 0.5$; see 
\cite[Example~3]{laurberg2008theorems}, 
\cite[Example~2]{huang2013non}, 
\cite[Example~4.16]{gillis2020}. 
Let us take 
\[
 H = 3 A_{1/3}  = 
\left( \begin{array}{cccccc}
1 & 3 & 3 & 1 & 0 & 0 \\ 
3 & 1 & 0 & 0 & 1 & 3 \\ 
0 & 0 & 1 & 3  & 3 & 1 \\ 
\end{array}
\right), 
\] 
which satisfies the SSC, and    
\[
W^\top = 3 A_{2/3} = 
\left( \begin{array}{cccccc} 
  2 &   3 &   3 &   2 &   0 &   0 \\ 
  3 &   2 &   0 &   0 &   2 &   3 \\ 
  0 &   0 &   2 &   3 &   3 &   2 \\ 
\end{array} \right), 
\]
which does not satisfy the SSC, but has some degree of sparsity. 
The NMF of 
\[
X 
= WH 
= \left( \begin{array}{cccccc} 
 11 &   9 &   6 &   2 &   3 &   9 \\ 
  9 &  11 &   9 &   3 &   2 &   6 \\ 
  3 &   9 &  11 &   9 &   6 &   2 \\ 
  2 &   6 &   9 &  11 &   9 &   3 \\ 
  6 &   2 &   3 &   9 &  11 &   9 \\ 
  9 &   3 &   2 &   6 &   9 &  11 \\ 
\end{array} \right)
\]
is not essentially unique. 
For example, 
\[
X = 
\left( \begin{array}{ccc} 
  0 &   3 &   1 \\ 
  1 &   3 &   0 \\ 
  3 &   1 &   0 \\ 
  3 &   0 &   1 \\ 
  1 &   0 &   3 \\ 
  0 &   1 &   3 \\ 
\end{array} \right) 
\left( \begin{array}{cccccc} 
  0 &   2 &   3 &   3 &   2 &   0 \\ 
  3 &   3 &   2 &   0 &   0 &   2 \\ 
  2 &   0 &   0 &   2 &   3 &   3 \\ 
\end{array} \right)
\]
is another decomposition which cannot be obtained as a scaling and permutation of $(W,H)$.  

However, the BSSMF of $X$ is unique, taking $a_i = 0$ and $b_i = 3$ for all $i$. In fact, $(3-W)^\top$ 
satisfies the SSC, as it is equal to  $3A_{1/3}$, up to permutation of its columns: 
\[
\begin{split}
3 - W^\top  
&= 
\left( \begin{array}{cccccc} 
  1 &   0 &   0 &   1 &   3 &   3 \\ 
  0 &   1 &   3 &   3 &   1 &   0 \\ 
  3 &   3 &   1 &   0 &   0 &   1 \\ 
\end{array} \right)\\
&= 
3 A_{1/3}(:, [4, 5, 6, 1, 2, 3] ).  
\end{split}
\]
Therefore, by \Cref{bssmf:th:uniqueBSSMF}, the BSSMF of $X$ is unique.  
\end{example}


\paragraph{Scaling ambiguity}
BSSMF is in fact more than essentially unique in the sense of~\Cref{bssmf:def:identifiability}. 
In fact, the scaling ambiguity can be removed because of $H$ being simplex structured, as shown in the following lemma. 
\begin{lemma}
	Let $H\in\mathbb{R}^{r\times n}$ such that $e^\top H=e^\top$ and $\rank(H)=r$. 
	Let $D \in \mathbb{R}^{r \times r}$ be a diagonal matrix, 
	and let $H'=DH$ be a scaling of the rows of $H$, and such that $e^\top H'=e^\top$. 
	Then $D$ must be the identity matrix, that is, $D=I$. 
\end{lemma}
\begin{proof} 
Let us denote $H^\dagger \in \mathbb{R}^{n\times r}$ the right inverse of $H$, which exists and is unique since $\rank(H) = r$, so that $H H^\dagger = I$. We have 
\begin{equation*}
\begin{split}
& e^\top H'= e^\top DH = e^\top \\
\Rightarrow \quad & e^\top DHH^\dagger=e^\top H^\dagger = e^\top \\
& \quad \quad \text{ since } e^\top H^\dagger=e^\top HH^\dagger=e^\top \\
\Rightarrow \quad & e^\top D = e^\top 
\quad \Rightarrow \quad D = I. 
\end{split}
\end{equation*}
Note that this lemma does not require $H$, $H'$ and $D$ to be nonnegative. 
\end{proof}

\paragraph{Geometric interpretation of BSSMF}

Solving BSSMF is equivalent to finding a polytope with $r$ vertices 
within the hyperrectangle defined by $[a,b]$ 
 that reconstructs as well as possible the data points. The fact that BSSMF is constrained within a hyperrectangle makes BSSMF more constrained than NMF, and hence more likely to be  essentially unique. 
 This will be illustrated empirically in \Cref{bssmf:sec:ident}. 
 Let us provide a toy example to better understand the distinction between NMF and BSSMF, namely let us use \Cref{bssmf:example:1} with $W=\frac{3}{10}A_{2/3}$ and $H=\frac{2}{3}A_{1/2}$ so that $X=WH$ is column stochastic. 
 \Cref{bssmf:fig:geobssmf} represents the feasible regions of NMF and BSSMF for the hypercube $[a,b] = [0,\frac{3}{10}]^3$ in a two-dimensional space within the affine hull of $W$;  see~\cite{gillis2020} for the details on how to construct such a representation.  
 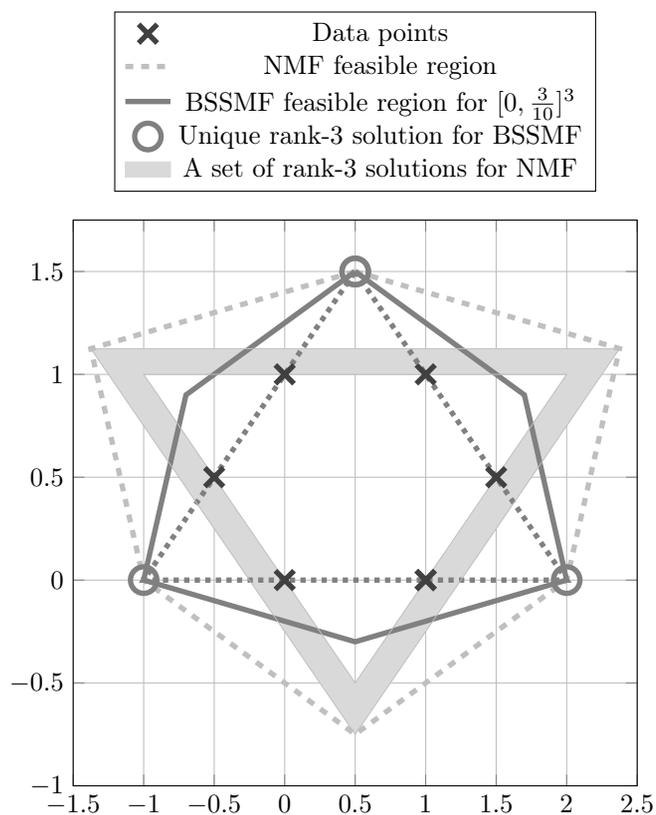
\begin{figure}
    \centering
%
%
\begin{tikzpicture}

\begin{axis}[%
width=7.5cm,
height=7.5cm,
at={(0cm,0cm)},
scale only axis,
xmin=-1.5,
xmax=2.5,
ymin=-1,
ymax=1.75,
xmajorgrids,
ymajorgrids,
legend style={at={(0.5,1.05)},anchor=south}
]
\addplot [color=darkgray,line width=2.0pt,only marks,mark size=5pt, mark=x]
  table[row sep=crcr]{%
1	0\\
1.5	0.5\\
1	1\\
0	1\\
-0.5	0.5\\
0   0\\
1	0\\
};
\addlegendentry{Data points}

\addplot [color=lightgray, dashed, line width=2.0pt, mark size=7pt]
  table[row sep=crcr]{%
0.5 	1.5\\
2.375	1.125\\
2	0\\
0.5	    -0.75\\
-1	    0\\
-1.375	1.125\\
0.5 	1.5\\
};
\addlegendentry{NMF feasible region}

\addplot [color=gray, line width=2.0pt, mark size=7.0pt]
  table[row sep=crcr]{%
0.5 	1.5\\
1.7	0.9\\
2	0\\
0.5	    -0.3\\
-1	    0\\
-0.7	0.9\\
0.5 	1.5\\
};
\addlegendentry{BSSMF feasible region for $[0,\frac{3}{10}]^3$}

\addplot [color=gray, dotted, line width=2.0pt, forget plot]
  table[row sep=crcr]{%
0.5 	1.5\\
2	0\\
-1	    0\\
0.5 	1.5\\
};
\addplot [only marks, color=gray, line width=2.0pt, mark size=5.0pt, mark=o]
  table[row sep=crcr]{%
0.5 	1.5\\
2	0\\
-1	    0\\
0.5 	1.5\\
};
\addlegendentry{Unique rank-3 solution for BSSMF}

\addplot [name path=A1, color=lightgray, line width=0pt, forget plot]
  table[row sep=crcr]{%
2.375	1.125\\
0.5	    -0.75\\
-1.375	1.125\\
2.375	1.125\\
};
\addplot [name path=A2, color=lightgray, line width=0pt, forget plot]
  table[row sep=crcr]{%
2   1\\
0.5 -0.5\\
-1  1\\
2   1\\
};
\addplot[gray!30] fill between[of=A1 and A2];
\addlegendentry{A set of rank-3 solutions for NMF}
\end{axis}
\end{tikzpicture}
    \caption[Geometric interpretation of BSSMF for \Cref{bssmf:example:1}]{Geometric interpretation of BSSMF for \Cref{bssmf:example:1}. Any triangle in the gray filled area containing the data points is a rank-3 solution for NMF. On the contrary, there is a unique rank-3 solution for BSSMF since there is a unique triangle containing the data points in the BSSMF feasible set.}
    \label{bssmf:fig:geobssmf}
\end{figure} 
 For this rank-3 factorization problem, solving NMF and BSSMF is equivalent to finding a triangle nested between the convex hull of the data points and the corresponding feasible region. 
 BSSMF has a unique solution, that is, there is a unique triangle between the data points and the BSSMF feasible region. On the other hand,  NMF is not identifiable: for example, any triangle within  the gray area containing the data points is a solution. 

In summary, for the BSSMF of $X = WH$ to be essentially unique, $W$ must contain sufficiently many entries equal to the lower and upper bounds, while $H$ must be sufficiently sparse. 

\paragraph{Choice of $a$ and $b$} \label{bssmf:para:choiceab} In practice, if $a$ and $b$ are unknown, it may be beneficial to choose them such that as many entries of $X$  are equal to the lower and upper bounds, and hence BSSMF is more likely to be identifiable. 
Let us denote $\tilde{a}_i = \min_j X(i,j)$ and $\tilde{b}_i = \max_j X(i,j)$ for all $i$, 
and let $X = WH$ be a BSSMF for the hyperrectangle $[a,b]$. 
We have $\tilde{a} \geq a$ and $\tilde{b} \leq b$ 
since $H(:,j) \in \Delta^r$ for all $j$. Hence, without any prior information, 
it makes sense to use a BSSMF with interval $[\tilde{a},\tilde{b}]$ which is contained in $[a,b]$. Note that such strategy will make BSSMF sensitive to outliers that are out of the ideal bounds $a$ and $b$.


\begin{remark}
    Interestingly, as shown in \Cref{bssmf:lem:transfoBSSMF} below, in the exact case, that is, when $X = WH$, we can assume w.l.o.g.\ that $[a_i, b_i] = [0,1]$ for all $i$ in BSSMF.   
\begin{lemma} \label{bssmf:lem:transfoBSSMF} 
Let $a \in \mathbb{R}^m$ and $b \in \mathbb{R}^m$ 
be such that $a_i < b_i$ for all $i$. 
The matrix $X = WH$ admits a BSSMF for the interval 
$[a, b]$ if and only if the matrix 
$\frac{[X - a e^\top]}{[(b-a) e^\top]}$ admits a BSSMF for the interval $[0, 1]^m$, 
where $\frac{[\cdot]}{[\cdot]}$ is the component-wise division of two matrices of the same size. 
\end{lemma}
\begin{proof} 
Let us show the direction $\Rightarrow$, the other is obtained exactly in the same way. 
Let the matrix  $X = WH$ admit a BSSMF for the interval 
$[a, b]$. We have 
\[
X - ae^\top 
= WH - a e^\top 
= (W - a e^\top) H, 
\]
since $e^\top H = e^\top$, as $H$ is column stochastic. This shows that $X' = X - ae^\top$ admits a BSSMF for the interval $[0,b-a]$ since $W' = (W - a e^\top) \in [0,b-a]$. For simplicity, let us denote $c = b-a > 0$. We have $X' = W' H$, while 
\[
\frac{[X-a e^\top]}{[(b-a) e^\top]} 
= 
\frac{[X']}{[c e^\top]} 
= 
\frac{[W' H]}{[c e^\top]} 
= 
\frac{[W']}{[c e^\top]}  H , 
\]
because $H$ is column stochastic. In fact, for all $i,j$, 
\[
\begin{split}
\frac{[W' H]_{i,j}}{[c e^\top]_{i,j}} 
& = 
 \frac{\sum_{k} W'(k,i) H(k,j)]_{i,j}}{c_i} \\
& = \sum_{k} \frac{W'(k,i)}{c_i} H(k,j) \\
& = \left( \frac{[W']}{[c e^\top]} H  \right)_{i,j}.     
\end{split}
\] 
Hence, $\frac{[X-a e^\top]}{[(b-a) e^\top]}$ admits a BSSMF for the interval $[0,1]^m$ since $H$ is column stochastic, 
and all columns of 
${\frac{[W']}{[c e^\top]} = \frac{[W - a e^\top]}{[(b-a) e^\top]}}$ belong to $[0,1]^m$. 
\end{proof}
\end{remark}

\begin{remark}[What if $a_i=b_i$ for some $i$?]  
\Cref{bssmf:lem:transfoBSSMF} does not cover the case $a_i=b_i$ for some $i$. In that case, we have $W(i,:) = a_i = b_i$ and therefore 
$X(i,:) = W(i,:) H = a_i e^\top = b_i e^\top$. 
This is not an interesting situation, and rows of $X$ with identical entries can be removed. In fact, after the transformation $X - ae^\top$, these rows are identically zero. 
\end{remark}

\Cref{bssmf:lem:transfoBSSMF} highlights another interesting property of BSSMF: as opposed to NMF, it is invariant to translations of the entries of the input matrix, given that $a$ and $b$ are translated accordingly. 
For example, in recommender systems datasets such as Netflix and MovieLens, $X(i,j) \in \{1,2,3,4,5\}$ for all $i,j$. 
Changing the scale, say to $\{0,1,2,3,4\}$, 
does not change the interpretation of the data, 
but will typically impact the NMF solution significantly\footnote{In fact, for NMF, it would make more sense to work on the datasets translated to $[0,4]$, 
as it would potentially allow it to be identifiable: zeros in $X$ imply zeros in $W$ and $H$, which are therefore more likely to satisfy the SSC.}, while the BSSMF solution will be unchanged, if the interval is translated from $[1,5]$ to $[0,4]$ since $H$ is invariant by translation on $X$. This property is in fact coming from SSMF.

\paragraph{Tightness of \Cref{bssmf:th:uniqueBSSMF}} Unfortunately, the conditions in \Cref{bssmf:th:uniqueBSSMF} are not necessary. This is due to SSC not being necessary for the uniqueness of NMF. Here is an example with
$$X=\begin{pmatrix}
    0.25 & 0.25 & 0.75 & 0.75 \\
    0.2  & 0.6  & 0.6  & 0.2  \\
    0.75 & 0.75 & 0.25 & 0.25 \\
    0.8  & 0.4  & 0.4  & 0.8  \\
\end{pmatrix}.$$
The unique BSSMF of $X$ of size 3 with the bounds 0 and 1 is given by
$$X=\underbrace{\begin{pmatrix}
    0 & 0.5 & 1 \\
    0.2 & 1 & 0.2 \\
    1 & 0.5 & 0 \\
    0.8 & 0 & 0.8
\end{pmatrix}}_W
\underbrace{\begin{pmatrix}
    0.75 & 0.5 & 0 & 0.25 \\
    0 & 0.5 & 0.5 & 0 \\
    0.25 & 0 & 0.5 & 0.75 \\
\end{pmatrix}\vphantom{\begin{pmatrix}0 \\ 0 \\ 0 \\ 0\end{pmatrix}}}_H.$$
However, $H$ cannot be SSC since there are not at least $r-1=2$ zeros per row. The matrix $[W^\top~J{-}W^\top]$ is sparse enough, yet, it cannot be SSC since its cone does not contain $e-e_i$ for $i$ in $1,\dots,3$.
See~\cite[Chapter 4.2.5]{gillis2020book} for the details on how the aforementioned factorization $X=WH$ is a unique NMF, and hence, a unique BSSMF for our chosen bounds. 
\section{Numerical experiments} \label{bssmf:sec:numexp}

The goal of this section is to  highlight the motivation points mentioned in \Cref{bssmf:sec:motiv} on real data sets. All experiments are run on a PC with an Intel(R) Core(TM) i7-9750H CPU @ 2.60GHz and 16GiB RAM. Let us recall that in order to retrieve NMF from \cref{bssmf:alg:BSSMF}, the bounds need to be set to $(a,b)=(0,+\infty)$ and the projection step on the probability simplex in \cref{bssmf:alg:BSSMF:line:proj} should be replaced by a projection on the nonnegative orthant. Hence, in our experiments, both NMF and BSSMF are solved with the same code implementation. 
\subsection{Interpretability} \label{bssmf:sec:inter}

When applied on a pixel-by-image matrix, NMF allows to automatically extract common features among a set of images. For example, if each row of $X$ is a vectorized facial image, the rows of $W$ will correspond to facial features~\cite{lee1999learning}. 

Let us compare NMF with BSSMF on the widely used MNIST handwritten digits dataset ($60,000$ images, $28\times28$ pixels) \cite{lecun1998gradient}. Each column of $X$ is a vectorized handwritten digit. For BSSMF to make more sense, we preprocess $X$ so that the intensities of the pixels in each digit belong to the interval $[0,1]$ (first remove from $X(:,j)$ its minimum entry, then divide by the maximum entry minus the minimum entry).

Let us take a toy example with $n=500$ randomly selected digits and $r=10$, in order to visualize the natural interpretability of BSSMF. The choice of $n$ is made solely for computational time considerations. For larger $n$, \Cref{bssmf:fig:Wb} might change but we will not lose interpretability. 
\begin{figure}
    \begin{subfigure}{\textwidth}
        \centering
        \includegraphics[width=\textwidth]{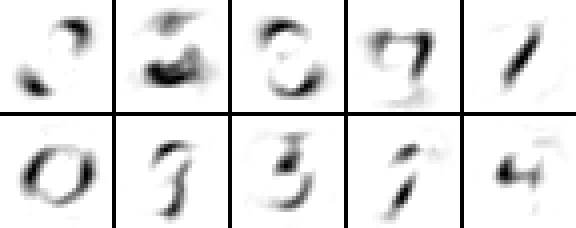}
        \caption{NMF}
        \label{bssmf:fig:Wf}
    \end{subfigure}

    \vspace{2cm}

    \begin{subfigure}{\textwidth}
        \centering
        \includegraphics[width=\textwidth]{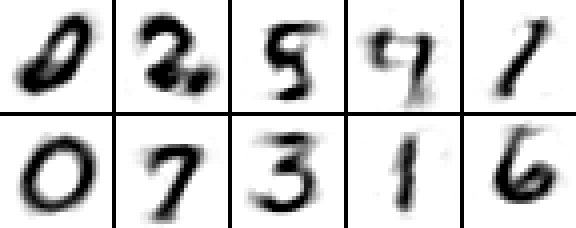}
        \caption{BSSMF}
        \label{bssmf:fig:Wb}
    \end{subfigure}
    \caption[Reshaped columns of the basis matrix $W$ for $r=10$ for MNIST]{Reshaped columns of the basis matrix $W$ for $r=10$ for MNIST with 500 digits.}
    \label{bssmf:fig:Ws}
\end{figure} 
\Cref{bssmf:fig:Wf} shows the features learned by NMF  which look like parts of digits. On the other hand, the features learned by BSSMF in \Cref{bssmf:fig:Wb} look mostly like real digits, because of the bound constraint and the simplex structure. 
%
In fact, as it is well known~\cite{lee1999learning} that NMF learns part-based representations, in this case, parts of digits. 
In other words, the columns of $W$ in NMF identify subset of pixels that are activated simultaneously in as many images as possible. 
Now, by the scaling degree of freedom, assume w.l.o.g.\ that $W(:,j) \in [0,1]^m$ for all $j$ in NMF. Since the columns of $W$ are parts of digits, each digit will have to use several of these parts, with an intensity close to one, so that $H$ will be far from being column stochastic. 
BSSMF, with the simplex constraint on $H$ and the bound constraints on $W$, therefore cannot  learn such a part-based representation. 
This is the reason why BSSMF learns more global features that, added on top of each other, reconstruct the digits. As it is shown in the MNIST experiment, these features look like digits themselves. 
Interestingly, if we progressively increase the upper bound, we would see that BSSMF progressively learns parts of digits, like NMF (using a lower bound of zero, that is, BSSMF with $[0,u]^m$ with $u \geq 1$). 
This is an indirect way of balancing the sparsity between $W$ and $H$. The larger the upper bound, the more relaxed is BSSMF and hence the sparser $W$ will be (given that the lower bound is 0).   
%
In~\Cref{bssmf:fig:Wb}, we distinguish numbers (like 7, 3 and 6). From a clustering point of view, this is of much interest because a column of $H$ which is near a ray of the probability simplex can directly be associated with the corresponding digit from~$W$. In this toy example, due to $r$ being small, 
an 8 cannot be seen. Nonetheless, an eight can be reconstructed as the weighted sum of the representations of a 5, a 3 and an italic 1; see \Cref{bssmf:fig:eight} for an example.   
\begin{figure}
    \centering
    \scalebox{1}{\input{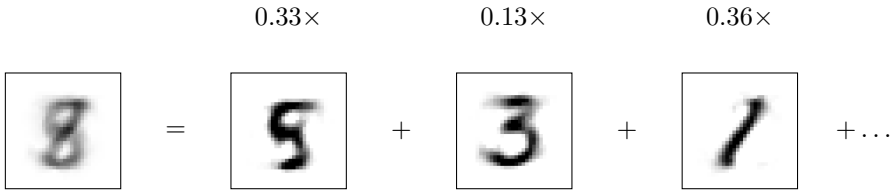}}
    \caption{Decomposition of an eight by BSSMF with $r=10$.}
    \label{bssmf:fig:eight}
\end{figure} 
Note that since BSSMF is more constrained than NMF, its reconstruction error might be larger than that of NMF. 
For our example ($r=10$), BSSMF has  relative error $\|X-WH\|_F/\|X\|_F$ of $61.56\%$, and  NMF of $59.04\%$. This is not always a drawback. In some applications, due to the presence of noise, although the reconstruction error of BSSMF is larger than that of NMF, the accuracy of the estimated factors $W$ and $H$ could be better, because it uses the prior information and is less prone to overfitting and less sensitive to outliers that might be outside the bounds. 
See also the discussion in \Cref{bssmf:sec:robust}   
 where NMF has a lower RMSE than BSSMF on the training set, but a larger RMSE than BSSMF on the test set. 
Note that we also compute NMFs using \cref{bssmf:alg:BSSMF} where the projections are performed on the nonnegative orthant, instead of on the bounded set for $W$ and on the probability simplex for~$H$. The stopping criteria in \cref{bssmf:alg:BSSMF:line:outerloop,bssmf:alg:BSSMF:line:Winnerloop,bssmf:alg:BSSMF:line:Hinnerloop} of \cref{bssmf:alg:BSSMF} are a maximum number of iterations equal to 500, 20 and 20, respectively, for both algorithms.

\subsection{Identifiability} \label{bssmf:sec:ident}
As it is NP-hard to check the SSC~\cite{huang2013non}, we perform experiments on MNIST and synthetic data where only a necessary condition for SSC1 is verified, 
namely \cite[Alg.~4.2]{gillis2020}. 

\paragraph{MNIST dataset} On MNIST, to see when $H$ satisfies this condition, we first vary $n$ from $100$ to $300$ for $m$ fixed (=$28$$\times$$28$). For $W^\top$, we fix $n$ to $300$, and downscale the resolution $m$ from $28$$\times$$28$ to $12$$\times$$12$ with a linear interpolation (\texttt{imresize3} in MATLAB), and the rank $r$ is varied from $12$ to $30$. 
Recall that both factors need to satisfy the SSC to correspond to an essentially unique factorization. In \Cref{bssmf:fig:ssc_nmf}, we see that $W^\top$ of NMF often satisfies the necessary condition. 
This is due to NMF learning ``parts'' of objects~\cite{lee1999learning}, which are sparse by nature, and sparse matrices are more likely to satisfy the SSC (\cref{bssmf:def:ssc}).  
On the contrary, even for a relatively large $n$, $H$ is too dense to satisfy the necessary condition. For $r\geq30$, the factor $H$ generated by NMF never satisfies the condition.
Meanwhile, in \Cref{bssmf:fig:ssc_bssmf} we see that $H$ of BSSMF always satisfies the condition when $n \geq225$ for $r=30$ and more generally, if $n$ and $m$ are large enough, both $H$ and $\binom{W}{J - W}^\top$ satisfy the necessary condition. This substantiates that BSSMF provides  essentially unique factorizations more often than NMF does. 

\begin{figure}
  \centering
  \begin{subfigure}{\textwidth}
    \centering
      \def\step{0.35}
\begin{tikzpicture}
\coordinate (grid_start) at (-5*\step , -4*\step);
\coordinate (grid_end) at (4*\step , 5*\step);
\node (titleH) [yshift=5*\step cm+0.5cm] {$H$};
\foreach \x in {1, ..., 5}
  \draw (-5*\step + 2*\x*\step -1.5*\step ,-4*\step) -- (-5*\step + 2*\x*\step-1.5*\step ,-4*\step cm -1pt) node[anchor=north] {$\pgfmathparse{10+(\x-1)*10}\pgfmathprintnumber{\pgfmathresult}$};
\node (xaxis) [yshift=-5*\step cm-0.3 cm] {rank $r$};
\foreach \x in {1, ..., 5}
  \draw (-5*\step , -4*\step + 2*\x*\step -1.5*\step) -- (-5*\step cm -1pt, -4*\step + 2*\x*\step-1.5*\step) node[anchor=east] {$\pgfmathparse{300-(\x-1)*50}\pgfmathprintnumber{\pgfmathresult}$};
\node (yaxis) [xshift=-6*\step cm-0.7 cm,rotate=90] {number of samples $n$};

\foreach \y [count=\n] in {
    {100,80,50,10,0,0,0,0,0},
    {100,100,70,0,0,0,0,0,0},
    {100,100,70,50,10,0,0,0,0},
    {100,100,90,50,0,0,0,0,0},
    {100,100,90,40,0,0,0,0,0},
    {100,100,90,30,10,0,0,0,0},
    {100,100,90,20,0,0,0,0,0},
    {100,100,100,20,10,0,0,0,0},
    {100,100,100,20,0,0,0,0,0},
} {
  \foreach \x [count=\m] in \y {
    \fill [white!\x!black] (-6*\step + \m*\step,6*\step - \n*\step) rectangle (4*\step cm,-4*\step cm);
  }
}
\draw[step=\step,gray,very thin] (grid_start) grid (grid_end);
\end{tikzpicture}
\hspace{0.3cm}
\begin{tikzpicture}
\coordinate (grid_start) at (-5*\step , -4*\step);
\coordinate (grid_end) at (4*\step , 5*\step);
\node (titleW) [yshift=5*\step cm+0.5cm] {$W^\top$};
\foreach \x in {1, ..., 5}
  \draw (-5*\step + 2*\x*\step -1.5*\step ,-4*\step) -- (-5*\step + 2*\x*\step-1.5*\step ,-4*\step cm -1pt) node[anchor=north] {$\pgfmathparse{14+(\x-1)*4}\pgfmathprintnumber{\pgfmathresult}$};
\node (xaxis) [yshift=-5*\step cm-0.3 cm] {rank $r$};
\foreach \x in {1, ..., 5}
  \draw (-5*\step , -4*\step + 2*\x*\step -1.5*\step) -- (-5*\step cm -1pt, -4*\step + 2*\x*\step-1.5*\step) node[anchor=east] {$\pgfmathparse{28-(\x-1)*4}\pgfmathprintnumber{\pgfmathresult}$};
\node (yaxis) [xshift=-6*\step cm-0.5 cm,rotate=90] {resolution $\sqrt{m}$};

\foreach \y [count=\n] in {
    {80,80,60,20,10,10,0,0,0},
    {100,60,70,40,20,10,0,0,0},
    {100,100,90,80,80,60,40,10,10},
    {100,100,80,80,100,90,50,30,20},
    {100,100,90,90,100,90,70,70,40},
    {100,100,100,100,100,100,90,70,90},
    {100,100,100,100,100,100,90,90,100},
    {100,100,100,100,100,100,90,100,90},
    {100,100,100,100,100,100,100,100,100},
} {
  \foreach \x [count=\m] in \y {
    \fill [white!\x!black] (-6*\step + \m*\step,6*\step - \n*\step) rectangle (4*\step cm,-4*\step cm);
  }
}
\draw[step=\step,gray,very thin] (grid_start) grid (grid_end);
\node (colormap) [yshift = 0.5*\step cm, xshift = 5*\step cm + 0.3 cm] {\pgfplotscolorbardrawstandalone[ 
    colormap={blackwhite}{color=(black) color=(white)},
    colorbar,
    point meta min=0,
    point meta max=1,
    colorbar style={
        width = \step cm,
        height=9*\step cm,
        ytick={0,0.2,0.4,0.6,0.8,1}}]};   
\end{tikzpicture}
      \caption{NMF}
      \label{bssmf:fig:ssc_nmf}
  \end{subfigure}
  
\vspace{1cm}

  \begin{subfigure}{\textwidth}
    \centering
      \def\step{0.35}
\begin{tikzpicture}
\coordinate (grid_start) at (-5*\step , -4*\step);
\coordinate (grid_end) at (4*\step , 5*\step);
\node (titleH) [yshift=5*\step cm+0.5cm] {$H$};
\foreach \x in {1, ..., 5}
  \draw (-5*\step + 2*\x*\step -1.5*\step ,-4*\step) -- (-5*\step + 2*\x*\step-1.5*\step ,-4*\step cm -1pt) node[anchor=north] {$\pgfmathparse{10+(\x-1)*10}\pgfmathprintnumber{\pgfmathresult}$};
\node (xaxis) [yshift=-5*\step cm-0.3 cm] {rank $r$};
\foreach \x in {1, ..., 5}
  \draw (-5*\step , -4*\step + 2*\x*\step -1.5*\step) -- (-5*\step cm -1pt, -4*\step + 2*\x*\step-1.5*\step) node[anchor=east] {$\pgfmathparse{300-(\x-1)*50}\pgfmathprintnumber{\pgfmathresult}$};
\node (yaxis) [xshift=-6*\step cm-0.7 cm,rotate=90] {number of samples $n$};

\foreach \y [count=\n] in {
    {90,100,90,50,10,0,0,0,0},
    {100,90,80,70,30,0,0,0,0},
    {100,100,90,90,80,30,10,0,0},
    {100,100,100,90,90,50,10,0,0},
    {100,100,100,100,90,70,40,20,0},
    {100,100,100,100,100,80,80,50,0},
    {100,100,100,100,100,100,90,50,30},
    {100,100,100,100,100,90,80,80,60},
    {100,100,100,100,100,100,100,80,70},
} {
  \foreach \x [count=\m] in \y {
    \fill [white!\x!black] (-6*\step + \m*\step,6*\step - \n*\step) rectangle (4*\step cm,-4*\step cm);
  }
}
\draw[step=\step,gray,very thin] (grid_start) grid (grid_end);
\end{tikzpicture}
\hspace{0.3cm}
\begin{tikzpicture}
\coordinate (grid_start) at (-5*\step , -4*\step);
\coordinate (grid_end) at (4*\step , 5*\step);
\node (titleW) [yshift=5*\step cm+0.5cm] {$\binom{W}{J - W}^\top$};
\foreach \x in {1, ..., 5}
  \draw (-5*\step + 2*\x*\step -1.5*\step ,-4*\step) -- (-5*\step + 2*\x*\step-1.5*\step ,-4*\step cm -1pt) node[anchor=north] {$\pgfmathparse{14+(\x-1)*4}\pgfmathprintnumber{\pgfmathresult}$};
\node (xaxis) [yshift=-5*\step cm-0.3 cm] {rank $r$};
\foreach \x in {1, ..., 5}
  \draw (-5*\step , -4*\step + 2*\x*\step -1.5*\step) -- (-5*\step cm -1pt, -4*\step + 2*\x*\step-1.5*\step) node[anchor=east] {$\pgfmathparse{28-(\x-1)*4}\pgfmathprintnumber{\pgfmathresult}$};
\node (yaxis) [xshift=-6*\step cm-0.5 cm,rotate=90] {resolution $\sqrt{m}$};

\foreach \y [count=\n] in {
    {70,20,0,0,0,0,0,0,0},
    {70,30,10,0,0,0,0,0,0},
    {100,80,70,20,20,0,0,0,0},
    {80,100,80,50,0,0,0,0,0},
    {70,100,90,50,70,30,0,0,0},
    {80,90,100,100,100,40,40,10,0},
    {90,100,100,100,90,70,40,30,0},
    {100,80,100,100,100,70,80,30,30},
    {100,90,100,100,100,100,100,100,80},
} {
  \foreach \x [count=\m] in \y {
    \fill [white!\x!black] (-6*\step + \m*\step,6*\step - \n*\step) rectangle (4*\step cm,-4*\step cm);
  }
}
\draw[step=\step,gray,very thin] (grid_start) grid (grid_end);
\node (colormap) [yshift = 0.5*\step cm, xshift = 5*\step cm + 0.3 cm] {\pgfplotscolorbardrawstandalone[ 
    colormap={blackwhite}{color=(black) color=(white)},
    colorbar,
    point meta min=0,
    point meta max=1,
    colorbar style={
        width = \step cm,
        height=9*\step cm,
        ytick={0,0.2,0.4,0.6,0.8,1}}]};   
\end{tikzpicture}
      \caption{BSSMF}
      \label{bssmf:fig:ssc_bssmf}
  \end{subfigure}
  \caption[Ratio on satisfying a necessary condition for SSC1]{Ratio, over 10 runs, of the factors generated by NMF in \Cref{bssmf:fig:ssc_nmf} and by BSSMF in \Cref{bssmf:fig:ssc_bssmf} that satisfy the necessary condition for SSC1  (white squares indicate that all matrices meet the necessary condition, black squares that none do). }
  \label{bssmf:fig:ssc_nec_cond}
\end{figure}
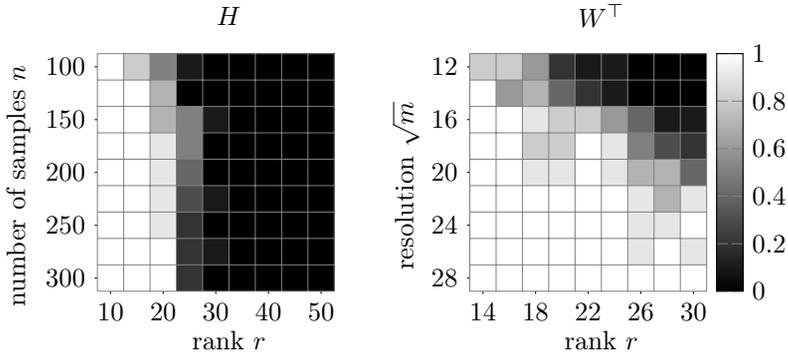
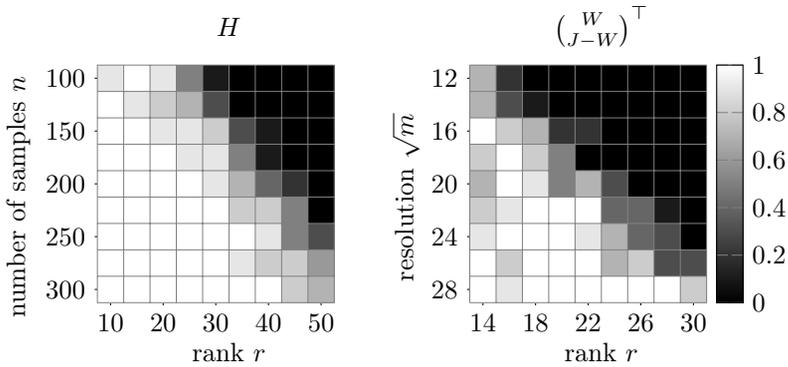

\paragraph{Synthetic datasets} 
Let us now perform an experiment to show how BSSMF is more likely than NMF to recover factors closer to the true ones, even when the sufficient conditions for identifiability are not satisfied. 
As there is no groundtruth for NMF and BSSMF on MNIST, we generate synthetic data as follows. Our synthetic datasets are of size $100\times100$, and their factorization rank is $10$. 
The matrix $H$ is generated randomly with values uniformly distributed between zero and one, and we randomly set $30\%$ of the values to zero. This allows us to ensure that $H$ satisfies the SSC. The reason behind ensuring that $H$ is SSC is that both 
NMF (\Cref{bssmf:th:uniqNMFSSC}) and BSSMF (\Cref{bssmf:th:uniqueBSSMF}) require that $H$ satisfies the SSC\footnote{In this experiment, because $n$ and $r$ are smaller, we could check that the SSC is satisfied (not a necessary condition), using Gurobi (\url{https://www.gurobi.com/}), a global optimization software.}. 
As we want to emphasize on how likely it is to retrieve the true factors for NMF and BSSMF, we make sure that their common conditions for identifiability are satisfied. 
The matrix $W$ is also generated randomly with values uniformly distributed between zero and one, and we then set a percentage of $p_{0,1}$ of the entries to zero and one, with the same probability to be equal to zero or one. Hence, $p_{0,1}$ percent of the values in $W$ touches the lower and upper bounds in BSSMF. 
Finally, we let $X=WH$ to get our synthetic data. We solve NMF and BSSMF on $X$ using \cref{bssmf:alg:BSSMF}. To assess the quality of the solutions, we report the average of the mean removed spectral angle (MRSA) and the subspace angle (see \Cref{minvol:def:angle}) between the columns of the true $W$ and the estimated $W$ (after an optimal permutation of the columns), as this is standard in the NMF literature. Given any two vectors $a$ and $b$, their MRSA is defined as 
\begin{equation*}
    \text{MRSA}(a,b) = \frac{100}{\pi} \text{arcos} 
    \left(\frac{ (a - \overline{a}e)^\top (b - \overline{b}e) }{\Vert a - \overline{a}e \Vert_{2}\Vert b - \overline{b} e\Vert_{2}}\right) \in \left[  0, 100\right], 
\end{equation*}
where $\overline{\cdot}$~is the average of the entries of a vector. 

\begin{figure}[htbp!]
\begin{subfigure}{\textwidth}
    \centering
\pgfplotsset{boxplot circle legend/.style={
    legend image code/.code={
        \draw[#1,line width=0.7pt] (0cm,-0.1cm) rectangle (0.6cm,0.1cm)
        (0.6cm,0cm) -- (0.7cm,0cm) (0cm,0cm) -- (-0.1cm,0cm)
        (0.7cm,0.1cm) -- (0.7cm,-0.1cm) (-0.1cm,0.1cm) -- (-0.1cm,-0.1cm)
        (0.3cm,-0.1cm) -- (0.3cm,0.1cm);
        \filldraw[#1] (0.9cm,0cm) circle (2pt);
    },},
    boxplot halfcircle legend/.style={
    legend image code/.code={
        \draw[#1,line width=0.7pt] (0cm,-0.1cm) rectangle (0.6cm,0.1cm)
        (0.6cm,0cm) -- (0.7cm,0cm) (0cm,0cm) -- (-0.1cm,0cm)
        (0.7cm,0.1cm) -- (0.7cm,-0.1cm) (-0.1cm,0.1cm) -- (-0.1cm,-0.1cm)
        (0.3cm,-0.1cm) -- (0.3cm,0.1cm);
        \begin{scope}
            \clip (0.9cm,0cm) circle (2pt);
            \fill[#1] (0.9cm+2pt,2pt) rectangle ++(-1.9pt,-4pt);
        \end{scope}
        \draw[#1,line width=0.7pt] (0.9cm,0cm) circle (2pt);
    },},
}
\begin{tikzpicture}
\begin{axis}[
    width=\textwidth,
    height=8cm,
    xticklabels={0\%,5\%,10\%,15\%,20\%,25\%,30\%},
    xmin=0,xmax=7,
    x tick label as interval,
    xtick={0,1,2,...,7},
    ylabel={Average MRSA},xlabel={\% of values equal to 0 and 1},
    ymode=log,
    ymin=1.0e-8,ymax=20,
    ymajorgrids=true,
    grid style=dashed,
    boxplot={
        draw position={1/(2+1) + floor(\plotnumofactualtype/2+1e-3) + 1/(2+1)*mod(\plotnumofactualtype,2)},
        box extend=1/(2+2)},
    legend style={
        at={(0.05,0.05)},anchor=south west,},]
\addlegendimage{boxplot halfcircle legend=gray}
\addlegendimage{boxplot circle legend=black}
\foreach \col in {0,...,6}{
\addplot[boxplot/draw direction=y,gray,mark=halfcircle*,line width=0.7pt] table[y index = \col]{bssmf/data/MRSA_NMF.txt};
\addplot[boxplot/draw direction=y,black,line width=0.7pt] table[y index = \col]{bssmf/data/MRSA_BSSMF.txt};
};
\legend{NMF,BSSMF}
\end{axis}
\end{tikzpicture}
    \label{bssmf:fig:mrsa_nmf_vs_bssmf}
\end{subfigure}  

\begin{subfigure}{\textwidth}
    \centering
\pgfplotsset{boxplot circle legend/.style={
    legend image code/.code={
        \draw[#1,line width=0.7pt] (0cm,-0.1cm) rectangle (0.6cm,0.1cm)
        (0.6cm,0cm) -- (0.7cm,0cm) (0cm,0cm) -- (-0.1cm,0cm)
        (0.7cm,0.1cm) -- (0.7cm,-0.1cm) (-0.1cm,0.1cm) -- (-0.1cm,-0.1cm)
        (0.3cm,-0.1cm) -- (0.3cm,0.1cm);
        \filldraw[#1] (0.9cm,0cm) circle (2pt);
    },},
    boxplot halfcircle legend/.style={
    legend image code/.code={
        \draw[#1,line width=0.7pt] (0cm,-0.1cm) rectangle (0.6cm,0.1cm)
        (0.6cm,0cm) -- (0.7cm,0cm) (0cm,0cm) -- (-0.1cm,0cm)
        (0.7cm,0.1cm) -- (0.7cm,-0.1cm) (-0.1cm,0.1cm) -- (-0.1cm,-0.1cm)
        (0.3cm,-0.1cm) -- (0.3cm,0.1cm);
        \begin{scope}
            \clip (0.9cm,0cm) circle (2pt);
            \fill[#1] (0.9cm+2pt,2pt) rectangle ++(-1.9pt,-4pt);
        \end{scope}
        \draw[#1,line width=0.7pt] (0.9cm,0cm) circle (2pt);
    },},
}
\begin{tikzpicture}
\begin{axis}[
    width=\textwidth,
    height=8cm,
    xticklabels={0\%,5\%,10\%,15\%,20\%,25\%,30\%},
    xmin=0,xmax=7,
    x tick label as interval,
    xtick={0,1,2,...,7},
    ytick={1e-9,1e-8,1e-7,1e-6,1e-5,1e-4,1e-3,1e-2,1e-1,1,10},
    ylabel={Angle},xlabel={\% of values equal to 0 and 1},
    ymode=log,
    ymin=1.0e-10,ymax=10,
    ymajorgrids=true,
    grid style=dashed,
    boxplot={
        draw position={1/(2+1) + floor(\plotnumofactualtype/2+1e-3) + 1/(2+1)*mod(\plotnumofactualtype,2)},
        box extend=1/(2+2)},
    legend style={
        at={(0.05,0.05)},anchor=south west,},]
\addlegendimage{boxplot halfcircle legend=gray}
\addlegendimage{boxplot circle legend=black}
\foreach \col in {0,...,6}{
\addplot[boxplot/draw direction=y,gray,mark=halfcircle*,line width=0.7pt] table[y index = \col]{bssmf/data/ANGLE_NMF.txt};
\addplot[boxplot/draw direction=y,black,line width=0.7pt] table[y index = \col]{bssmf/data/ANGLE_BSSMF.txt};
};
\legend{NMF,BSSMF}
\end{axis}
\end{tikzpicture}
    \label{bssmf:fig:angle_nmf_vs_bssmf}
\end{subfigure}  
\caption[Boxplots of the average MRSA and subspace angle between the true $W$ and the estimated $W$]{Boxplots of the average MRSA and subspace angle between the true $W$ and the estimated $W$ by NMF and BSSMF for the hypercube $[0,1]^{100}$ over 20 trials, depending on the percentage, $p_{0,1}$, of values equal to 0 and 1 in the true $W$.}
\end{figure}
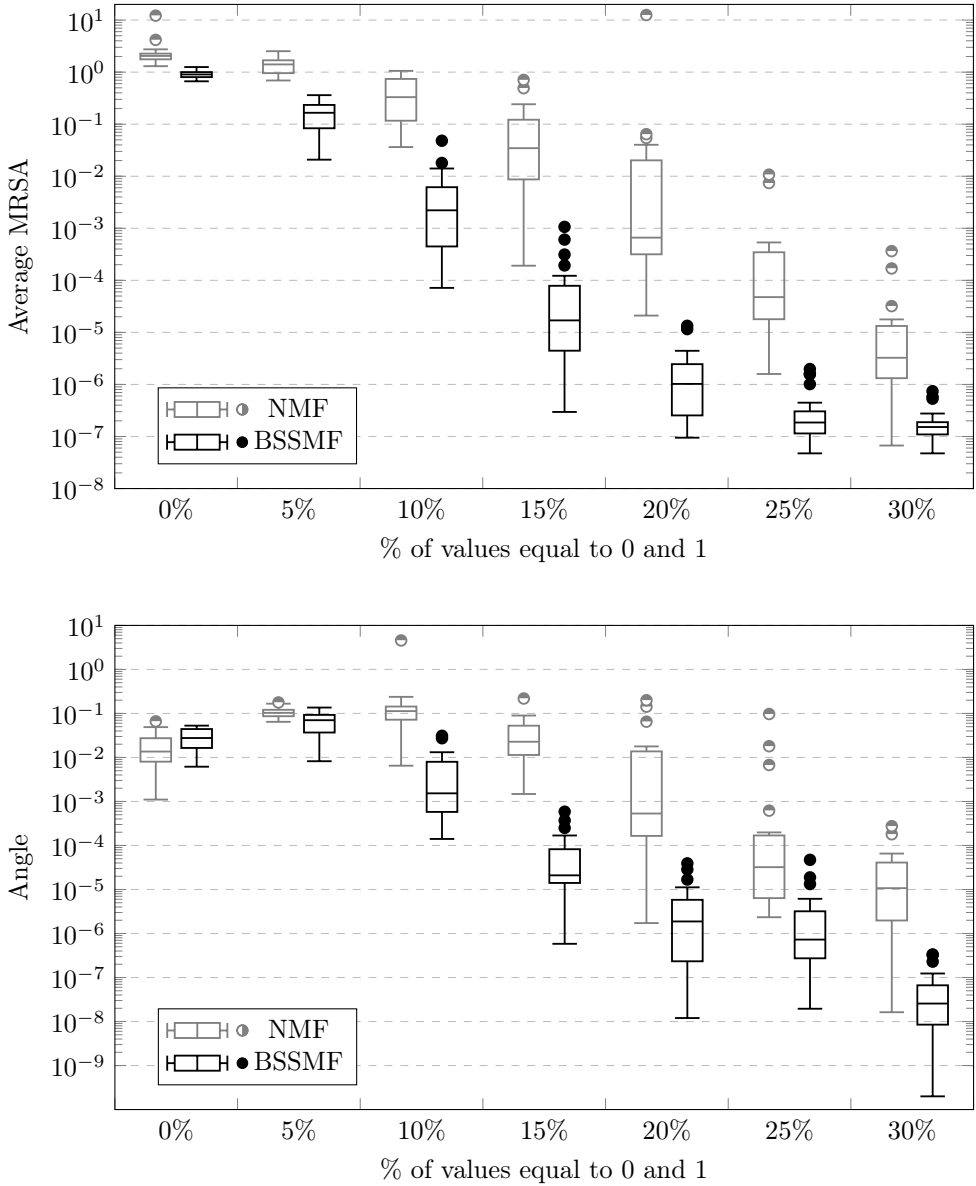

We vary the percentage $p_{0,1}$ of values touching the lower and uppper bounds in $W$ (namely, 0 and 1) from $0\%$ to $30\%$ with a $5\%$ increment. For each value of $p_{0,1}$, the test is performed 20 times. 
Let us note that among the generated true $W$'s, 
between $p_{0,1}=0\%$ and $p_{0,1}=15\%$, $\binom{W}{J-W}^\top$ never satisfies the necessary conditions for SSC1. 
For $p_{0,1}=20\%$, 3 out of the 20 generated $\binom{W}{J-W}^\top$ satisfies the necessary conditions for SSC1, 
10 out of 20 for $p_{0,1}=25\%$, 
and 17 out of 20 for $p_{0,1}=30\%$. 
Let us also note that for all values of $p_{0,1}$ within the considered range, $W$ never satisfies the necessary conditions for SSC1. 
The distribution of the average MRSAs and the subspace angle are respectively reported in \Cref{bssmf:fig:mrsa_nmf_vs_bssmf} and \Cref{bssmf:fig:angle_nmf_vs_bssmf}. Clearly, the MRSA is always smaller for BSSMF compared to NMF, even when the necessary conditions for SSC1 are not satisfied for $\binom{W}{J-W}^\top$; this is because the feasible set of BSsMF is contained in that of NMF, and hence the generated factors are more likely to be closer to the ground truth. This also illustrates that the conditions of \Cref{bssmf:th:uniqueBSSMF} for the identifiability of BSSMF are only sufficient, since BSSMF finds solutions with MRSA close to machine epsilon when these conditions are not fulfilled. 
\subsection{Robustness to overfitting} \label{bssmf:sec:robust}

In this section we compare unconstrained matrix factorization (MF), NMF and BSSMF on the matrix completion problem; more precisely, on rating datasets for recommendation systems. Let $X$ be an item-by-user matrix and suppose that user $j$ has rated item $i$, that rating would be stored in $X_{i,j}$. The matrix $X$ is then highly incomplete since a user has typically only rated a few of the items. 
In this context, NMF looks for nonnegative factors $W$ and $H$ such that $M\circ X\approx M\circ (WH)$, where $M_{i,j}$ is equal to $1$ when user $j$ rated item $i$ and is equal to $0$ otherwise. A missing rating $X_{i,j}$ is then estimated by computing $W(i,:)H(:,j)$. Features learned by NMF on rating datasets tend to be parts of typical users. Yet, the nonnegative constraint on the factors hardly makes the features interpretable by a practitioner. Suppose that the rating a user can give is an integer between $1$ and $5$ like in many rating systems, NMF can learn features whose values may fall under the minimum rating $1$ or may exceed the maximum rating $5$. Consequently, the features cannot directly be interpreted as typical users. On the contrary, with BSSMF, the extracted features will directly be interpretable if the lower and upper bounds are set to the minimum and maximum ratings. On top of that, BSSMF is expected to be less sensitive to overfitting than NMF since its feasible set is more constrained.

This last point will be highlighted in the following experiment on the ml-1m dataset\footnote{\href{bssmf:https://grouplens.org/datasets/movielens/1m/}{https://grouplens.org/datasets/movielens/1m/}}, which contains 1 million ratings from 6040 users on 3952 movies. As in~\cite{liang2018variational}, we split the data in two sets~: a training set and a test set. The test set contains 500 users. We also remove any movie that has been rated less than 5 times from both the training and test sets. For the test set, 80\% of a user's ratings are considered as known. The remaining 20\% are kept for evaluation. During the training, we learn $W$ only on the training set. During the testing, the learned $W$ is used to predict those 20\% kept ratings of the test set by solving the $H$ part only on the 80\% known ratings. This simulates new users that were not taken into account during the training, but for whom we would still want to predict the ratings. 
The reported RMSEs are computed on the 20\% kept ratings of the test set. In order to challenge the overfitting issue, we vary $r$ in $\{1,5,10,20,50,100\}$ for BSSMF, NMF and an unconstrained MF which are all computed using \cref{bssmf:alg:BSSMF}, where the projections onto the feasible sets are adapted accordingly (projection onto the nonnegative orthant for NMF, no projection for unconstrained MF). The stopping criteria in \cref{bssmf:alg:BSSMF:line:outerloop,bssmf:alg:BSSMF:line:Winnerloop,bssmf:alg:BSSMF:line:Hinnerloop} of \cref{bssmf:alg:BSSMF} are a maximum number of iterations equal to 200, 1 and 1, respectively, for all algorithms. The experiment is conducted on 10 random initializations and the average RMSEs are reported is \Cref{bssmf:tab:rmse_ml-1m}. As expected, BSSMF and NMF are more robust to overfitting than unconstrained MF. Additionnaly, BSSMF is also clearly more robust to overfitting than NMF. Its worse RMSE is $0.89$ with $r=100$ (and it is still equal to $0.89$ with $r=200$), while, for NMF, the RMSE is $0.92$ when $r=100$ (which is worse than a rank-one factorization giving a RMSE of $0.91$). 
\begin{table}
    \centering
    \begin{tabular}{r|lll}
        r & BSSMF & NMF & MF \\ \hline
        1 & $0.97 \pm 2\cdotp10^{-5}$ & $0.88 \pm 0.002$ & $0.91 \pm 5\cdotp10^{-6}$ \\
        5 & $0.87 \pm 0.001$ & $0.87 \pm 0.003$ & $0.87 \pm 0.003$ \\
        10 & $0.86 \pm 0.002$ & $0.87 \pm 0.001$ & $0.87 \pm 0.002$ \\
        20 & $0.87 \pm 0.002$ & $0.87 \pm 0.002$ & $0.88 \pm 0.002$ \\
        50 & $0.88 \pm 0.002$ & $0.90 \pm 0.004$ & $0.93 \pm 0.004$ \\
        100 & $0.89 \pm 0.003$ & $0.92 \pm 0.003$ & $0.99 \pm 0.004$
    \end{tabular}
    \caption{RMSE on the test set according to $r$, averaged \textpm{ standard deviation} on 10 runs on ml-1m}
    \label{bssmf:tab:rmse_ml-1m}
\end{table}


 The same experiment is conducted on the ml-100k dataset\footnote{\href{bssmf:https://grouplens.org/datasets/movielens/100k/}{https://grouplens.org/datasets/movielens/100k/}} which contains 100,000 ratings from 1,700 movies rated by 1,000 users. The test set contains 50 users. The results are reported in \Cref{bssmf:tab:rmse_ml-100k}, and the observations are similar: BSSMF is significantly more robust to overfitting than NMF and unconstrained MF. 
\begin{table}
    \centering
    \begin{tabular}{r|lll}
        r & BSSMF & NMF & MF \\ \hline
        1 & $0.98 \pm 1\cdotp10^{-4}$ & $0.91 \pm 3\cdotp10^{-5}$ & $0.91 \pm 5\cdotp10^{-5}$\\
        5 & $0.89 \pm 0.005$ & $0.89 \pm 0.01$ & $0.89 \pm 0.008$ \\
        10 & $0.90\pm 0.008$ & $0.90 \pm 0.009$ & $0.92 \pm 0.01$ \\
        20 & $0.91 \pm 0.01$ & $0.93 \pm 0.01$ & $0.97 \pm 0.02$ \\
        50 & $0.93 \pm 0.01$ & $0.97 \pm 0.01$ & $1.06 \pm 0.03$ \\
        100 & $0.94 \pm 0.01$ & $1.01 \pm 0.007$ & $1.13 \pm 0.02$
    \end{tabular}
    \caption{RMSE on the test set according to $r$, averaged \textpm{ standard deviation} on 10 runs on ml-100k}
    \label{bssmf:tab:rmse_ml-100k}
\end{table}

\section{Conclusion} 

In this chapter, we proposed a new factorization model, namely bounded simplex structured matrix factorization (BSSMF). 
Fitting this model retrieves interpretable factors: 
the learned basis features can be interpreted in the same way as the original data while the activations are nonnegative and sum to one, leading to a straightforward soft clustering interpretation. 
Instead of learning parts of objects as NMF, BSSMF learns objects that  can be used to explain the data through convex combinations. 
We have proposed a dedicated fast algorithm for BSSMF, and showed that, 
under mild conditions, BSSMF is essentially unique. 
 We also showed that the constraints in BSSMF make it robust to overfitting on rating datasets without adding any regularization term. 
Further work could include:  \begin{itemize}
     \item the use of BSSMF for other applications, 


     \item  the design of more efficient algorithms for BSSMF,  and 

     \item  the design of algorithms for other BSSMF models, e.g., with other data fitting terms such as the Kullback-Leibler divergence, as done recently in~\cite{leplat2021multiplicative} for SSMF with nonnegativity constraint on $W$. 
 \end{itemize}

\chapter{Identifiability of Polytopic Matrix Factorization}\label{chap:polytopicmf}
\begin{hiddenmusic}{Kino - Спокойная ночь}{https://www.youtube.com/watch?v=LRPmZ_ALkmY}
\end{hiddenmusic}NMF is not essentially unique in general. However, it has been proven to be identifiable under the sufficiently scattered conditions (SSC). A geometric interpretation of these sufficient conditions is the following: while making sure that $X=WH$ and that $W$ is nonnegative, it is not possible to decrease the ``volume'' of the cone of $W^\top$ without making the cone of $H$ get out of the nonnegative orthant, and vice versa; see \Cref{preli:sec:identif_nmf} for details. 

In this chapter, we focus on the identifiability of polytopic matrix factorization (PMF)\label{acro:PMF}. With NMF, the feasible domain is the nonnegative orthant. With PMF, the feasible domains are convex polytopes: the columns of $W^\top$ and $H$ belong to the polytopes $\P_W$ and $\P_H$, respectively. 
A variant of PMF has already been studied    in~\cite{tatli2021polytopic,tatli2021generalized} where the authors proposed a structured matrix factorization where: (i)~the matrix $W$ is unconstrained,  
(ii)~the columns of $H$ belong to a convex polytope, and 
(iii)~the goal is to find a factorization maximizing the volume of the convex hull of the columns of $H$. This model, proposed in~\cite{tatli2021polytopic,tatli2021generalized}, is also referred to as PMF, although it would have been more appropriate to refer to it as maximum-volume PMF. In fact, their proposed model could be viewed as a polytopic variant 
of minimum-volume semi-NMF; see~\cite{fu2016robust} and \Cref{preli:sec:identif_ssmf}, while our proposed model would rather be a polytopic variant of NMF. 

\paragraph{Outline and contribution of the chapter} 
Inspired by the identifiability conditions in~\cite{tatli2021polytopic} and similarly 
to~\Cref{preli:th:uniqNMFSSC}, our main contribution in this chapter is to  show that if the convex hull of $W^\top$ and $H$ are sufficiently scattered within their respective polytope, then the corresponding PMF is identifiable 
(\Cref{polytopicmf:th:uniquePMF}). 

In \Cref{polytopicmf:sec:pmf} we introduce PMF. \Cref{polytopicmf:sec:def} provides important definitions and properties. In \Cref{polytopicmf:sec:id} we prove our main result. \Cref{polytopicmf:sec:ex} presents known structured matrix factorization that are special cases of PMF, and how our theoretical finding relates to previous results.

\section{Polytopic Matrix Factorization}
\label{polytopicmf:sec:pmf}
In this chapter, we consider convex polytopes, 
that is, bounded polyhedra. 
A convex polytope  $\P$  can always be expressed in V-form, through a convex combination of its vertices:  
	\begin{equation}
		\P = \conv(V) = \{x \ | \ x = Vh, h \geq 0, \sum_i h_i = 1\},
		\label{polytopicmf:eq:polydefV} \vspace{-0.1cm} 
	\end{equation}
	where the columns of $V$ are the vertices, or the extremum points, of $\P$.
 We can now define PMF. Given a data matrix $X\in\R^{m\times n}$ and $r$, PMF computes $W$ and $H$ such that   
\begin{equation}
\label{polytopicmf:eq:PMF}
\begin{split}
 X = WH 	\text{ s.t. } & W(i,:) \in \P_W \text{ for all $i$ in } 1,\dots,m,  \\
	& H(:,j) \in \P_H \text{ for all $j$ in } 1,\dots,n, 
\end{split}
\end{equation} 
where $W\in\R^{m\times r}$ is the basis matrix, $H\in\R^{r\times n}$ is the coefficient matrix, $\P_W$ and $\P_H$
 are convex polytopes that respectively constrain the rows of $W$ and the columns of $H$. 
 This PMF is referred to as the quadruple $(W, H, \P_W, \P_H)$. 
 This framework is quite general: 
 it offers infinite varieties of structured matrix factorizations that promote different behaviors in the latent space, depending on the choice of $\P_W$ and $\P_H$. As we will show in~\Cref{polytopicmf:sec:ex}, PMF recovers factorizations that have been studied in the literature. 
 
\section{Definitions and Properties}
\label{polytopicmf:sec:def}

In this section, we provide important 
definitions and properties that are needed to achieve our main result on the identifiability of PMF (Theorem~\ref{polytopicmf:th:uniquePMF} in Section~\ref{polytopicmf:sec:id}). 

\paragraph{Identifiability.} 
Let us clarify what is meant by identifiability. A PMF $(W,H,\P_W,\P_H)$ is identifiable if for any other PMF $(W_*,H_*,\P_W,\P_H)$ of $X$, there exist a permutation matrix $\Pi$ and a diagonal matrix $D$ with diagonal values in $\{-1,1\}$ such that $W_*=W\Pi^\top D^{-1}$ and $H_*=D\Pi H$. We will refer to a matrix of the form $D\Pi$ as a signed permutation.  
Essential uniqueness of PMF is stronger than the NMF one, as it only allows a sign ambiguity, while NMF allows a scaling ambiguity. 

\paragraph{Maximum-volume ellipsoid and sufficient scatteredness.}  
Our sufficient scatteredness conditions that guarantee identifiability heavily rely on the notion of ellipsoids. 
Given a center, $\overbar{x}  \in \R^r$, and a positive definite matrix $E$, an ellipsoid is defined as 
$\E(E,\overbar{x}):=\left\{x\in\R^r  |   (x-\overbar{x})^\top E(x-\overbar{x})\leq r\right\}$. 
Its volume is given by $
	\vol(\E(E,\overbar{x}))=\frac{r^{r/2}\Omega_r}{\sqrt{\det(E)}}$ where $\Omega_r$ is the volume of a ball of radius $1$ in $\R^r$. 
 The axis of the ellipsoid are given by the eigenvectors of $E$, and their length is inversely proportional to the square root of corresponding eigenvalues; 
 see, e.g.,~\cite{todd2016minimum}. 
 Given an ellipsoid $\E(E,\overbar{x})$ and an invertible matrix $Q$, it can be shown that 
 $Q(\E(E,\overbar{x}))
 =\{Qx |   x\in\E\} 
 = 
\E(Q^{-\top} E Q^{-1},\overbar{y})$ where $\overbar{y} = Q\overbar{x}$, and hence the volume of $Q\E$ equals the volume of $\E$ times $|\det(Q)|$. 
This will be useful in our identifiability proof.

The Maximum-Volume Inscribed Ellipsoid (MVIE\label{acro:MVIE}) of a polytope $\P$, denoted $\E_{\P}$, is defined as the ellipsoid 
$\E_{\P} \subset \P$ with maximum volume $\vol(\E(E,\overbar{x}))$, that is, for which $\det(E)$ is minimized. It can be computed by solving a convex semidefinite program; see, e.g.,~\cite[Chap.~8.4.2]{boyd2004convex}.
A convex set is said to be sufficiently scattered relative 
to a polytope when it is contained in that polytope while  containing the MVIE 
of this polytope~\cite{tatli2021polytopic}. 

Our identifiability result will be based on the following sufficient scatteredness condition: \vspace{-0.1cm} 
\begin{definition}[Sufficiently Scattered Factor~\cite{tatli2021polytopic}]
	\label{polytopicmf:def:SSF}
	The matrix $H\in\R^{r\times n}$ is called a sufficiently scattered factor (SSF) corresponding to $\P$ if\vspace{0.1cm}
	
	\noindent [PMF.SSC1] $\P\supseteq \conv(H) \supset \E_{\P}$, and\vspace{0.1cm}
	
	\noindent [PMF.SSC2] $\conv(H)^{*,g_{\P}}\cap \bd(\E_\P^{*,g_\P})=\ext(\P^{*,g_\P})$, 
 
	\noindent where $\E_\P$ is the MVIE of $\P$ centered at $g_\P$. \vspace{-0.1cm} 
\end{definition}

The idea behind the condition [PMF.SSC1] is similar to [SSC1] in Theorem~\ref{preli:th:uniqNMFSSC}, as both conditions ensure that the considered factor is sufficiently scattered within its feasible set. The MVIE acts like the second order cone $\C$ in [SSC1] which is the largest cone contained in the nonnegative orthant. Here, [PMF.SSC1] ensures that the convex hull of a factor $H$ is contained in the polytope $\P$ and contains the MVIE of $\P$. The second condition [PMF.SSC2] makes sure that the MVIE is not contained too tightly. Let us illustrate why [PMF.SSC2] is important with the PMF $(H^\top,H,\Delta^3,\Delta^3)$ using 
Example~3 from~\cite{laurberg2008theorems}, see also~\cite[Example~2]{huang2013non}: 

\begin{equation}
\label{polytopicmf:eq:exampleH}
H=\frac{1}{3}\begin{pmatrix}
1 & 2 & 2 & 1 & 0 & 0\\
2 & 1 & 0 & 0 & 1 & 2\\
0 & 0 & 1 & 2 & 2 & 1
\end{pmatrix}.    
\end{equation}
As it can be seen on~\cref{polytopicmf:fig:exampleHpmfssc1}, $H$ satisfies [PMF.SSC1]. However,~\cref{polytopicmf:fig:exampleHpmfssc2} exposes why $H$ does not satisfy [PMF.SSC2], and it turns out that the PMF $(H^\top,H,\Delta^3,\Delta^3)$ is not identifiable:  
$$Q=\frac{1}{3}\begin{pmatrix}
-1 & 2 & 2\\
2 & -1 & 2\\
2 & 2 & -1
\end{pmatrix}$$
provides another PMF, $(H^\top Q^\top,QH,\Delta^3,\Delta^3)$, while $QH$ is not a signed permutation of the rows of $H$:
$$QH=\frac{1}{3}\begin{pmatrix}
    1 & 0 & 0 & 1 & 2 & 2\\
    0 & 1 & 2 & 2 & 1 & 0\\
    2 & 2 & 1 & 0 & 0 & 1
\end{pmatrix}.$$

\begin{figure}
    \centering
    \begin{subfigure}[b]{\textwidth}
        \centering
        \begin{tikzpicture}[scale=1]
	\begin{axis}[hide axis,legend style={at={(1.2,0.5)},anchor=center}]
 
		\addplot [color=red,line width=2pt]
		coordinates{
			( 0.0, 1.0)
			(-0.866025 ,-0.5)
			( 0.866025 ,-0.5)
			( 0.0, 1.0)
			(-0.866025 ,-0.5)};
        
		\addplot [color=teal,
		mark=*,mark options={solid},only marks]
		coordinates{
            (-0.57735  ,  0.0)
            (-0.288675 ,  0.5)
            ( 0.288675 ,  0.5)
            ( 0.57735  ,  0.0)
            ( 0.288675 , -0.5)
            (-0.288675 , -0.5)};
   
		\addlegendimage{area legend, fill=teal!50,opacity=0.6}
		\addlegendimage{area legend, fill=cyan,opacity=0.6}
		\filldraw [fill=teal!50,opacity=0.6] (axis cs: -0.57735, 0.0) -- (axis cs: -0.288675, 0.5) -- (axis cs: 0.288675, 0.5) -- (axis cs: 0.57735, 0.0) -- (axis cs: 0.288675, -0.5)  --  
		(axis cs: -0.288675, -0.5) --  cycle;
		
		\filldraw[fill=cyan,opacity=0.6] (axis cs: 0,0) circle (0.5);
		

		

      \legend{$\bd(\Delta^3)$,$H$,$\conv(H)$,$\E_{\Delta^3}$}
	\end{axis}
\end{tikzpicture}
        \caption{Visualization of why $H$ satisfies [PMF.SSC1].}
        \label{polytopicmf:fig:exampleHpmfssc1}
    \end{subfigure}

    \vspace{2cm}

    \begin{subfigure}[b]{\textwidth}
        \centering
        \begin{tikzpicture}[scale=1]
	\begin{axis}[hide axis,axis equal,legend style={at={(1.2,0.5)},anchor=center}]
        
		\filldraw [fill=teal!50,opacity=0.6] (axis cs: 0.0, 1.0) -- (axis cs: -0.866025, 0.5) -- (axis cs: -0.866025, -0.5) -- (axis cs: 0, -1.0) -- (axis cs: 0.866025, -0.5)  --  
		(axis cs: 0.866025, 0.5) --  cycle;
  
		\addplot [color=red,dashed,line width=2pt]
		coordinates{
            (-0.866025, 0.5)
            (-1.11022e-16, -1.0)
            (0.866025, 0.5)
            (-0.866025, 0.5)};
        
		\addplot [color=teal,
		mark=*,mark options={solid},only marks]
		coordinates{
			(0.0, 1.0)
            (-0.866025, 0.5)
            (-0.866025, -0.5)
            (0, -1.0)
            (0.866025, -0.5)
            (0.866025, 0.5)};
   
		\addlegendimage{area legend, fill=teal!50,opacity=0.6}
		\addlegendimage{cyan,line width = 2pt}
		
		\draw[cyan,line width = 2pt] (axis cs: 0,0) circle (1);
		

		

      \legend{$\bd({\Delta^3}^{*,g})$,$\ext(\conv(H)^{*,g})$,$\conv(H)^{*,g}$,$\bd(\E_{\Delta^3}^{*,g})$}
	\end{axis}
\end{tikzpicture}
        \caption{Visualization of why $H$ does not satisfy [PMF.SSC2].}
        \label{polytopicmf:fig:exampleHpmfssc2}
    \end{subfigure}
    \caption[Small example showing how {[PMF.SSC1]} can be satisfied without {[PMF.SSC2]} being satisfied]{A small example, with $H$ from~\cref{polytopicmf:eq:exampleH} and ${g=\begin{pmatrix}  1/3 & 1/3 & 1/3 
    \end{pmatrix}^\top}$, showing how [PMF.SSC1] can be satisfied without [PMF.SSC2] being satisfied.}
    \label{polytopicmf:fig:PMF.SSC}
\end{figure}

\paragraph{Permutation-and/or-sign-only invariant sets.}  
In addition to the sufficient scatteredness, the identifiability of PMF will rely on the following condition for the sets of vertices of $\P_W$ and $\P_H$.  
\begin{definition}
\label{polytopicmf:def:permsigninvariantset}
A set $\mathcal{X}$ is called a permutation-and/or-sign-only invariant (PSOI) set if, and only if, every linear transformation $A$ such that $A(\mathcal{X})=\mathcal{X}$ is a signed permutation, that is, 
 \mbox{$A=D\Pi$} where $\Pi$ is a permutation matrix and $D$ is a diagonal matrix with diagonal entries in $\{-1,1\}$.
\end{definition}

The set of vertices of full-dimensional polytopes will in most cases be PSOI sets. 
\begin{lemma} \label{polytopicmf:lem:rotation}
    Let the columns of $V \in \mathbb{R}^{r \times n}$ contain the vertices of the polytope $\mathcal{V} \subset \mathbb{R}^{r}$ and such that $\rank(V) = r$ (this holds for full-dimensional polytopes). 
    Let $A \in \mathbb{R}^{r \times r}$ be such that 
     $A V = V(:,\Pi)$ for some permutation $\Pi$. 
     Then $A$ is an orthogonal matrix, that is, a rotation of $\mathbb{R}^{r}$. 
\end{lemma}
\begin{proof}
    Since $A$ permutes the columns of $V$, and the set of permutations is finite, there exists $n$ such that $A^n V = V$. Since $V$ has rank $r$, it admits a right inverse, so that $A^n = I_r$, where $I_r$ is identity matrix of dimension $r$. 
    This implies that the eigenvalues of $A$ are roots of 1, and hence $A$ is orthogonal, that is, $A^\top A = I_r$. 
\end{proof}

In two dimensions, sets that are not PSOI are any regular polygon centered at the origin, except for the square (which is obtained by a rotation of 90 or 180 degrees in which case $A$ is a signed permutation). 
For example, the vertices of the regular triangle given by the columns of 
\[
V = \begin{pmatrix}
0 & \sqrt{3}/2 &  -\sqrt{3}/2 \\
1 & -1/2 & -1/2 
\end{pmatrix}
\]
are preserved by a rotation of 120 degrees, corresponding to $A = 
\begin{pmatrix}
-1/2& \sqrt{3}/2 \\
-\sqrt{3}/2 & -1/2 
\end{pmatrix}$, and $AV = 
\begin{pmatrix}
 \sqrt{3}/2 &  -\sqrt{3}/2 & 0  \\
 -1/2 & 1/2 & 1  
\end{pmatrix}$. 

In Section~\ref{polytopicmf:sec:ex}, we will use two polytopes:  $\Delta^r$ and $[a,b]^r$ for $b > a$. Let us show that their vertices are PSOI sets. 
For $\Delta^r$, this is trivial since $\Delta^r = \conv(I_r)$, hence any $A$ 
that satisfies   
$A I_r = I_r(:,\Pi)$ for some permutation $\Pi$ 
must be a permutation (note there is no sign ambiguity possible here).  
For the hypercube $[a,b]^r$, let us first prove the following lemma. 
\begin{lemma} \label{polytopicmf:lem:vectord}
    Let $a < b$ be scalars, and $d \in \mathbb{R}^r$ with $\|d\|_2=1$ be such that $d^\top x \in \{a,b\}$ for all $x \in \{a,b\}^r$. 
    Then $d$ is a unit vector, up to multiplication by -1. 
\end{lemma}
\begin{proof}
Let us prove the result by induction. 
    For $r=1$, the result is trivial, we must have $d=1$. 
    Assume the result holds for all $r' < r$, and let us denote $d = [d_{r-1}, d_r]$ with 
    $d_{r-1} \in \mathbb{R}^{r-1}$, and similarly for $x$. We have for all  $x \in \{a,b\}^r$ that \vspace{-0.1cm} 
    \[
d^\top x = 
d_{r-1}^\top x_{r-1} + d_r x_r \in \{a,b\}.  \vspace{-0.1cm}  
    \]
    If $d_r \in \{-1,0,1\}$, the result follows by induction since $\| d \|_2 = 1$. 
    Hence, assume $d_r \notin \{-1,0,1\}$. We have \vspace{-0.1cm}  
    \[
d_{r-1}^\top x_{r-1} + d_r a \in \{a,b\} 
\; \text{ and }  \;  
d_{r-1}^\top x_{r-1} + d_r b \in \{a,b\}.  \vspace{-0.1cm}  
    \]
    Let us denote $\alpha = d_{r-1}^\top x_{r-1}$, we have \vspace{-0.1cm}  
    \[
\alpha \in \{a- d_r a, b-d_r a\} 
\; \text{ and } \;  
\alpha \in \{a - d_r b, b-d_r b\}. \vspace{-0.1cm}  
    \]
    Since  $a \neq b$, 
  $a- d_ra \neq  a - d_rb$ and $b-d_r a \neq b - d_r b$ as $d_r \neq 0$, 
      $a- d_ra \neq  b-d_rb$ as $d_r \neq 1$, 
      and 
      $b-d_ra \neq b-d_rb$ as $d_r \neq -1$. 
      Hence, $\alpha$ cannot exist for $x_r \in \{a,b\}$, a contradiction. 
\end{proof}

\begin{corollary} \label{polytopicmf:cor:cubePSOI}
The set of vertices of $[a,b]^r$ is a PSOI set. 
\end{corollary} 
\begin{proof}
The set of vertices of $[a,b]^r$ are all vectors in $\{a,b\}^r$. Let the columns of $V \in \mathbb{R}^{r \times 2^r}$ contain the vertices of $[a,b]^r$, and the linear transformation $A$ satisfy  $AV = V(:,\Pi)$ for some permutation $\Pi$.  
By Lemma~\ref{polytopicmf:lem:rotation}, $A$ is orthogonal hence its rows have unit $\ell_2$ norm. This implies that every row of $A$ must satisfy the condition of Lemma~\ref{polytopicmf:lem:vectord} and hence are unit vectors. Since rows of $A$ are orthogonal, $A$ must be a signed permutation. 
\end{proof}

\section{Identifiability}
\label{polytopicmf:sec:id}

We can now state our main result: it fills a gap in the literature by combining the ideas of the identifiability of maximum-volume PMF in~\cite{tatli2021polytopic}, 
and of NMF in~\cite{huang2013non}. 

\begin{theorem}
	\label{polytopicmf:th:uniquePMF}
	Let $(W,H,\P_W,\P_H)$ be a PMF of $X$ 
 of size $r=\rank(X)$. 
 If $\Wt$ and $H$ are SSFs, and $\ext(\P_W)$ and $\ext(\P_H)$ are PSOI sets, then the PMF $(W,H,\P_W,\P_H)$ of $X=WH$ of size $r=\rank(X)$ is identifiable. 
\end{theorem}
\begin{proof}
	This proof follows that from~\cite[Th.~6]{tatli2021polytopic} where only $H$ is required to be sufficiently scattered while its volume is maximized. 
	Let $Q\in\R^{r\times r}$ be an invertible matrix such that $(WQ^{-1},QH)$ is a PMF of $X$ with 
	\begin{equation}
		\label{polytopicmf:eq:Qconpoly}
		\conv( Q^{-\top}\Wt)\subseteq\P_{W} 
  \; \text{ and  } \; 
  \conv(QH)\subseteq\P_H.
	\end{equation}
	Since $\Wt$ and $H$ are sufficiently scattered factors, their convex hull contains their corresponding MVIE: \begin{equation}
		\E_{\P_{W}} \subset \conv(\Wt) \; \text{ and } \;  \E_{\P_{H}} \subset \conv(H).
	\end{equation} 
 Then, \cref{polytopicmf:eq:Qconpoly} leads to 
	\begin{equation}
		\label{polytopicmf:eq:QMVIE}
		Q^{-\top}(\E_{\P_W}) \subseteq \P_{W} 
 \; \text{ and } \; Q(\E_{\P_{H}}) \subseteq \P_H.
	\end{equation} 

 The set $Q^{-\top}(\E_{\P_W})$ (resp.\ $Q(\E_{\P_{H}})$) is still an ellipsoid of volume $|\det(Q^{-1})|$ $\vol(\E_{\P_W})$ (resp.\ $|\det(Q)|\vol(\E_{\P_H})$). By definition of the MVIE, we have 
	\begin{align*}
		& |\det(Q^{-1})|\vol(\E_{\P_W}) \leq \vol(\E_{\P_W}) \text{ and } |\det(Q)|\vol(\E_{\P_H}) \leq \vol(\E_{\P_H}) \\
		\Leftrightarrow & |\det(Q^{-1})| \leq 1 \text{ and } |\det(Q)| \leq 1 \Leftrightarrow |\det(Q)| = 1.  
	\end{align*}
	This implies that $Q^{-\top}$ and $Q$ respectively map $\E_{\P_W}$ and $\E_{\P_H}$ onto themselves :
	\begin{equation}
        \label{polytopicmf:eq:Qselfmap}
		Q^{-\top}(\E_{\P_W}) = \E_{\P_W} \;  \text{ and }  \;  Q(\E_{\P_H}) = \E_{\P_H}.
	\end{equation}
	The remaining of the proof is exactly like in the remaining proof of~\cite[Th.~6]{tatli2021polytopic} by focusing on either $H$ or $\Wt$. 
 Focus on $H$ for example, and using [PMF.SSC2], the idea is to show that $Q(\ext(\P_H))=\ext(\P_H)$. Then, because $\ext(\P_H)$ is a PSOI set, $Q$ has to be a signed permutation.
\end{proof}

The last part of the proof of~\Cref{polytopicmf:th:uniquePMF} does not rely on both $\Wt$ and $H$ satisfying [PMF.SSC2], and on both $\ext(\P_W)$ and $\ext(\P_H)$ being PSOI sets. Actually,~\Cref{polytopicmf:th:uniquePMF} remains valid if only one the factors satisfies [PMF.SSC2] and if the vertices of its corresponding polytope form a PSOI set.
\begin{corollary}
Let $\Wt$ and $H$ satisfy [PMF.SSC1] and \\
\null\quad(i) $\Wt$ satisfy [PMF.SSC2] and $\ext(\P_W)$ be a PSOI set,\\
or\\
\null\quad(ii) $H$ satisfy [PMF.SSC2] and $\ext(\P_H)$ be a PSOI set,\\

\noindent then the PMF $(W,H,\P_W,\P_H)$ of $X=WH$ of size $r=\rank(X)$ is identifiable. 
\end{corollary}
\begin{proof}
    The same proof as~\Cref{polytopicmf:th:uniquePMF} applies. By symmetry, whether it is (i) or (ii) that is verified allows us to conclude that $Q$ is a signed permutation.
\end{proof}


The PSOI set condition can be relaxed to sets that are ``mutually'' PSOI, that is, there cannot exist a matrix $A$ which is not a signed permutation such that 
$A^{-\top}(\ext(\P_W)) = \ext(\P_W)$ and $A(\ext(\P_H)) = \ext(\P_H)$.
\begin{corollary}
\label{polytopicmf:th:relaxeduniquePMF}
	Let $(W,H,\P_W,\P_H)$ be a PMF of $X$ 
 of size $r=\rank(X)$. 
 If $\Wt$ and $H$ are SSFs, and $\ext(\P_W)$ and $\ext(\P_H)$ are mutually PSOI sets, then the PMF $(W,H,\P_W,\P_H)$ of $X=WH$ of size $r=\rank(X)$ is identifiable. 
\end{corollary}
\begin{proof}
    The same proof as~\Cref{polytopicmf:th:uniquePMF} applies up to~\cref{polytopicmf:eq:Qselfmap}. 
    Then, $W^\top$ and $H$ satisfying [PMF.SSC2] leads to $Q^{-\top}(\ext(\E_{\P_W}))=\ext(\E_{\P_W})$ and $Q(\ext(\E_{\P_H}))=\ext(\E_{\P_H})$. Then, because $\ext(\P_W)$ and $\ext(\P_H)$ are mutually PSOI sets, $Q$ has to be a signed permutation.
\end{proof}

\section{Examples of PMF}
\label{polytopicmf:sec:ex}

In this section, we show that some known constrained matrix factorizations are special instances of PMF, and explain how Theorem~\ref{polytopicmf:th:uniquePMF} relates to known identifiability results for these special cases. 

\subsection{Nonnegative Matrix Factorization (NMF)}
\label{polytopicmf:subsec:nmf}

An NMF, $X = WH$, requires $W$ and $H$ to be component-wise nonnegative. This is not a PMF since the nonnegative orthant is unbounded. However, 
if $W^\top$ and $H$ do not contain a column full of zeros (which can be assumed w.l.o.g.),  
then there exist two diagonal matrices, $D_l$ and $D_r$, such that $D_l W e = e$ and $e^\top H D_r = e^\top$. Hence, we can transform the NMF $X=WH$ into the PMF 
$(\Tilde{W},\Tilde{H},\Delta^r,\Delta^r)$ of $\Tilde{X}$ with $\Tilde{X}=D_l X D_r$, where $\Tilde{W}=D_l W$ and $\Tilde{H}=H D_r$.

Interestingly, the identifiability conditions for NMF 
in~\Cref{preli:th:uniqNMFSSC} and for PMF in~\Cref{polytopicmf:th:uniquePMF} are equivalent, because $\Tilde{H}$ satisfies the SSC in~\cref{preli:def:ssc} \textit{if and only if} $\Tilde{H}$ is an SSF according to~\cref{polytopicmf:def:SSF}, while $\ext(\Delta^r)$ is a PSOI set (see Section~\ref{polytopicmf:sec:def}). This is due to the fact that $\E_{\P_W}=\E_{\P_H}=\C\cap\Delta^r$, since the MVIE of $\Delta^r$ is an $(r-1)$-dimensional ball centered at $\frac{1}{r}e$ of radius $\frac{1}{\sqrt{r(r-1)}}$, within the affine subspace $\{x\in\R^r,~e^\top x=1\}$. Indeed, the diagonal matrices are just rescaling the rows of $W$ and the columns of $H$ such that they belong to $\Delta^r$. Hence, $\C\cap\Delta^r\subseteq\conv(\Tilde{H})$ \textit{if and only if} $\C\subseteq\cone(H)$, and by symmetry this also holds for $\Tilde{W}^\top$ and $W^\top$. 

\subsection{Factor-Bounded Matrix Factorization}
\label{polytopicmf:subsec:fbmf}

Factor-bounded matrix factorization (FBMF) requires the elements of each factor to be bounded. Given $a < b \in \R$, we write $a\leq W\leq b$ if $a \leq W(i,k) \leq b$ for all $(i,k)$. 
\begin{definition}[Factor-Bounded MF] \label{polytopicmf:def:factorboundedMF}
	Let $X \in \mathbb{R}^{m \times n}$, $r$ be an integer, $l_W < u_W\in\R$ and $l_H < u_H\in\R$. The pair $(W,H) \in \mathbb{R}^{m \times r} \times \mathbb{R}^{r \times n}$ is a FBMF of $X$ of size $r$ for the intervals $[l_W,u_W]$ and $[l_H,u_H]$ if 
	\begin{equation}
		\label{polytopicmf:eq:factorboundedMF}
X = WH 	\text{ such that }   l_W\leq W \leq u_W, l_H\leq H \leq u_H.  
	\end{equation}
\end{definition}  
This means that each row of $W$ then belongs to the hypercube $[l_W,u_W]^r$ and each column of $H$ belongs to the hypercube $[l_H,u_H]^r$. In~\cite{liu2021factor}, the authors 
propose a nonnegative FBMF (NFBMF), where $0\leq l_W$ and $0\leq l_H$ in~\cref{polytopicmf:eq:factorboundedMF}. 
They showed that NFBMF is particularly well suited for clustering tasks. To the best of our knowledge, FBMF has never been proven to be identifiable. Since~\cref{polytopicmf:eq:factorboundedMF} is a PMF with the choice $\P_W=[l_W,u_W]^r$ and $\P_H=[l_H,u_H]^r$, Theorem~\ref{polytopicmf:th:uniquePMF} applies to FBMF. 
The MVIE $\E_{\P_W}$ is an $r$-dimensional ball centered at $\frac{u_W+2l_W}{2}e$ of radius $\frac{u_W-l_W}{2}$, and similarly for $\E_{\P_H}$, while $\ext(\P_W)$ and $\ext(\P_H)$ are PSOI sets (Corollary~\ref{polytopicmf:cor:cubePSOI}).

\subsection{Bounded Simplex-Structured Matrix Factorization}
\label{polytopicmf:subsec:bssmf}

Bounded simplex-structured matrix factorization (BSSMF) was already presented in \Cref{chap:bssmf} as model useful to explain data that are convex combinations of vectors belonging to a hyperrectangle $[a,b]$, where $a \leq b \in \mathbb{R}^m$. The convex combinations are the columns of $H$ and the vectors belonging to $[a,b]$ are the columns of $W$. For more details on BSSMF, refer to \Cref{chap:bssmf}.
BSSMF does not belong to the class of PMFs. The hyperrectangle constraint on the columns of $W$ cannot in general be expressed as a polytopic constraint on the rows of $W$. 
However, when all entries of $a$, and of $b$, are equal to one another, the hyperrectangle constraint becomes a hypercube constraint that can be expressed by a polytopic row wise constraint. 
For example, when $X$ corresponds to a set of vectorized images, 
the intensity of a pixel belongs to $[0,1]$. 
If there is no specific pixel position that should be bounded differently than the others, every row of $W$ is bounded in the same way. 
In other words, the rows of $W$ belong to the hypercube $[0,1]^r$. Another example is when $X$ is a rating matrix whose entries are ordinal, e.g., the Netflix matrix with entries in $\{1,2,3,4,5\} \in [1,5]$. 
In these cases, 
BSSMF uses a hypercube $[a,b]^m$ and is equivalent to PMF since \Cref{bssmf:eq:BSSMF} is equivalent to~\cref{polytopicmf:eq:PMF} with $\P_W=[a,b]^r$ and $\P_{H}=\Delta^r$. 
BSSMF was shown to be identifiable under conditions described in \Cref{bssmf:th:uniqueBSSMF}, different from the ones in~\Cref{polytopicmf:th:uniquePMF}. 
When BSSMF and PMF are equivalent, which identifiability theorem is the strongest? Since BSSMF is invariant by translation along $e$, we can assume w.l.o.g.\ that $a=0$ for the sake of simplicity. Also, we do not need to focus on the conditions for $H$. Indeed, when $H(:,j)\in\Delta^r$ for all~$j$, [SSC1] is equivalent to [PMF.SSC1] because the MVIE of $\P_H=\Delta^r$ is equal to $\C\cap\Delta^r$. We then focus on the sufficient scatteredness of $W^\top$. The MVIE of $[0,b]^r$ is a ball $\E_{[0,b]^r}$ centered at $\frac{b}{2}e$ of radius $\frac{b}{2}$. This ball is tightly contained by $\C$, which means that for any convex set $A$ that contains $\E_{[0,b]^r}$, $\C\subseteq\cone(A)$. As a consequence, if $W^\top$ satisfies [PMF.SSC1], $W^\top$ satisfies [SSC1], which implies that $\binom{W}{b e^\top - W}^\top$ satisfies [SSC1]. However, it is possible that $\binom{W}{b e^\top - W}^\top$ satisfies [SSC1] while $W^\top$ does not satisfy [PMF.SSC1]. Here is an example with $\P_W=[0,1]^3$: \vspace{-0.1cm} 
\begin{equation}
\label{polytopicmf:eq:exampleW}
W^\top=\begin{pmatrix}
	0.8 & 0 & 0.2 & 0.2 & 0.8 & 1 \\
	0.2 & 0.8 & 0 & 1 & 0.2 & 0.8 \\
	0 & 0.2 & 0.8 & 0.8 & 1 & 0.2
\end{pmatrix}. \vspace{-0.1cm}  
\end{equation}
As it can be seen in~\cref{polytopicmf:fig:exampleSSC1}, the cone of $\binom{W}{1 - W}^\top$ contains $\C$ because $W$ reaches enough times the minimum and maximum bounds $0$ and $1$. However, in~\Cref{polytopicmf:fig:exampleSSF1} the convex hull of $W^\top$ does not contain the MVIE of $[0,1]^3$. Therefore, \Cref{polytopicmf:th:uniquePMF} is quite general but is not as strong as \Cref{bssmf:th:uniqueBSSMF} for BSSMF. 

\begin{figure}[htbp!]
\centering
\begin{subfigure}[b]{\textwidth}
    \centering
    \begin{tikzpicture}[scale=1]
	\begin{axis}[hide axis,legend style={at={(1.1,0.5)},anchor=west}]
 
		\addplot [color=red,line width=2pt]
		coordinates{
			( 0.0, 1.0)
			(-0.866025 ,-0.5)
			( 0.866025 ,-0.5)
			( 0.0, 1.0)
			(-0.866025 ,-0.5)};
        
		\addplot [color=teal,
		mark=*,mark options={solid},only marks]
		coordinates{
			(-0.173205, 0.7) 
			(-0.519615,-0.5) 
			( 0.69282,-0.2) 
			(-0.0866025,-0.35) 
			( 0.34641, 0.1) 
			(-0.259808, 0.25) 
			( 0.0866025,-0.35) 
			( 0.259808, 0.25) 
			(-0.34641, 0.1) 
			( 0.173205, 0.7) 
			(-0.69282,-0.2) 
			( 0.519615,-0.5)};
   
		\addlegendimage{area legend, fill=teal!50,opacity=0.6}
		\addlegendimage{area legend, fill=cyan,opacity=0.6}
		\filldraw [fill=teal!50,opacity=0.6] (axis cs: -0.173205, 0.7) -- (axis cs: 0.173205, 0.7) -- (axis cs: 0.69282,-0.2) -- (axis cs: 0.519615,-0.5) -- (axis cs: -0.519615,-0.5)  --  
		(axis cs: -0.69282,-0.2) --  cycle;
		
		\filldraw[fill=cyan,opacity=0.6] (axis cs: 0,0) circle (0.5);
		

		

      \legend{$\bd(\Delta^3)$,$\binom{W}{1 - W}^\top$,$\cone\left(\binom{W}{1 - W}^\top\right)\cap\Delta^3$,$\C\cap\Delta^3$}
	\end{axis}
\end{tikzpicture}
    \caption{[SSC1] being satisfied.}
    \label{polytopicmf:fig:exampleSSC1}
\end{subfigure}
                                                                                                                                                                                                                                                                                                                                                                     
\vspace{2cm}

\begin{subfigure}[b]{\textwidth}
    \centering
    \begin{tikzpicture}[scale=1]
\begin{axis}[view={10}{10},
legend style={at={(1.1,0.5)},
anchor=west},
axis equal, 
axis lines=none,
xmin = 0,xmax = 1,ymin = 0,ymax = 1,zmin = 0,zmax = 1,
legend columns=1,
legend style={/tikz/every even column/.append style={column sep=0.5cm}}
]

\addplot3 [color=orange,loosely dotted,line width=2pt,forget plot]
    coordinates{
(0,0,0)
(0,1,0)
(0,1,1)
(0,1,0)
(1,1,0)};

\addplot3 [color=red,line width=2pt,line join=round]
		coordinates{
			(0,0,1)
			(0,1,0)
			(1,0,0)
			(0,0,1)};

\addplot3 [color=black,
mark=*,mark options={solid},only marks]
coordinates{
    (0.8,0.2,0)
    (0,0.8,0.2)
    (0.2,0,0.8)
    (0.2,1,0.8)
    (0.8,0.2,1)
    (1,0.8,0.2)};

\addplot3 [color=black,fill=gray!50,opacity=0.2,
mark=none,forget plot]
coordinates{
    (0.8,0.2,0) 
    (0,0.8,0.2) 
    (0.2,0,0.8) 
    (0.2,1,0.8) 
    (0.8,0.2,1) 
    (0.2,0,0.8) 
    (0.8,0.2,0) 
    (0.8,0.2,1) 
    (1,0.8,0.2) 
    (0.8,0.2,0) 
    (0.8,0.2,1) 
    (0.2,0,0.8) 
    };
\addplot3 [color=black,line width=1.07pt,
mark=none,forget plot]
coordinates{
    (0.8,0.2,0) 
    (0,0.8,0.2) 
    (0.2,0,0.8) 
    (0.2,1,0.8) 
    (0.8,0.2,1) 
    (0.2,0,0.8) 
    (0.8,0.2,0) 
    (0.8,0.2,1) 
    (1,0.8,0.2) 
    (0.8,0.2,0) 
    (0.8,0.2,1) 
    (0.2,0,0.8) 
    };
\addlegendimage{area legend,line width = 0.5pt,fill=gray!15,opacity=1}

\addplot3 [color=black,loosely dashed,
mark=none,forget plot]
coordinates{
    (0,0.8,0.2) 
    (0.2,1,0.8) 
    (1,0.8,0.2) 
    (0,0.8,0.2) 
    };

\begin{scope}[canvas is yz plane at x=0.5]
\draw[blue, line width=2pt, dashed] let \p1=(0.5,0), \n1={veclen(\x1,\y1)} in (0.5,0.5) circle[radius=\n1];
\end{scope}
\begin{scope}[canvas is xy plane at z=0.5]
\draw[blue, line width=2pt, dashed] let \p1=(0.5,0), \n1={veclen(\x1,\y1)} in (0.5,0.5) circle[radius=\n1];
\end{scope}
\begin{scope}[canvas is xz plane at y=0.5]
\draw[blue, line width=2pt, dashed] let \p1=(0.5,0), \n1={veclen(\x1,\y1)} in (0.5,0.5) circle[radius=\n1];
\end{scope}
\addlegendimage{line width=2pt,blue,dashed}

\addplot3 [color=orange,line width=2pt,line join=round]
    coordinates{
(0,0,1)
(1,0,1)
(1,0,0)
(0,0,0)
(0,0,1)
(0,1,1)
(1,1,1)
(1,0,1)
(1,1,1)
(1,1,0)
(1,0,0)};

\addplot3 [color=red,line width=2pt]
coordinates{
    (0.1,0,0.9)
    (0.95,0,0.05)};

\legend{$\bd(\Delta^3)$,$W^\top$,$\conv(W^\top)$,$\bd(\E_{[0,1]^3})$,$\bd({[0,1]^3})$}
\end{axis}

\end{tikzpicture}
    \caption{[PMF.SSC1] not being satisfied.}
    \label{polytopicmf:fig:exampleSSF1}
\end{subfigure}
\caption[Visualization of how can $\binom{W}{1 - W}^\top$ satisfy {[SSC1]} while $W^\top$ does not satisfy {[PMF.SSC1]}]{Visualization of $\binom{W}{1 - W}^\top$ from~\cref{polytopicmf:eq:exampleW} satisfying [SSC1] while $W^\top$ does not satisfy [PMF.SSC1]. The cone of $\binom{W}{1 - W}^\top$ contains $\C$, while the convex hull of $W^\top$ does not contain the ball $\E_{[0,1]^3}$.}
\end{figure}

\section{Conclusion}
\label{polytopicmf:sec:conclu}

We presented PMF, a structured matrix factorization model where the latent space of the factors is constrained by given polytopes. 
The choice of the polytopes should depend on the data and the application at hand. When the polytopes have certain invariant properties, we derived some sufficient conditions under which the identifiability of a PMF is guaranteed. Geometrically, these conditions are based on the scatteredness of the factors within the constraining polytopes.     

\chapter{Randomized Successive Projection Algorithm for Separable NMF}\label{chap:randspa}




\begin{hiddenmusic}{ YĪN YĪN - One Inch Punch }{https://yinyin.bandcamp.com/album/one-inch-punch}
\end{hiddenmusic}In general, NMF is NP-hard~\cite{vavasis2010complexity} and not necessarily identifiable (\Cref{preli:sec:identif_nmf}), which are two main issues of NMF. 
However, under the \emph{separability assumption}, it is solvable in polynomial time and is identifiable~\cite{arora2012computing}.
This assumption states that for every vertex (column of $W$), there exists at least one data point (column of $X$) equal to this vertex.
In blind HU, which consists in identifying the materials present in a hyperspectral image as well as their distribution in the pixels of the image, this is known as the \emph{pure-pixel assumption} and means that for each material, there is at least one pixel composed almost purely of this material.
Many algorithms have been introduced that leverage this assumption, see for instance~\mbox{\cite[Chapter 7]{gillis2020book}} and the references therein.
Recently, algorithms for separable NMF that are provably robust to noise have been introduced~\cite{arora2012computing}.
One of the most widely used is the successive projection algorithm (SPA)~\cite{araujo2001successive}\label{acro:SPA}.

SPA is robust to noise and generally works well in practice.
However, it suffers from several drawbacks, notably it is sensitivity to outliers.
SPA is deterministic, that is for a given problem it gives the same result at every run. It is also greedy, in the sense that it extract vertices sequentially, so an error at a given iteration cannot be compensated in the following iterations.
In this chapter, we aim at addressing the sensitivity to outliers by designing a non-deterministic variant of SPA that could be run several times, in the hope that at least one run will not extract outliers.

Let us discuss an observation from~\cite{nadisic2021smoothed}.
The separable NMF algorithm called vertex component analysis (VCA) \cite{nascimento2005vertex}\label{acro:VCA} includes a random projection, therefore it is non-deterministic and at each run it produces potentially a different result.
VCA is simpler and its guarantees are weaker than those of SPA, and the experiments in~\cite{nadisic2021smoothed} show that VCA performs worse than SPA on average, but they also show that the best result of VCA over many runs is in most cases better that the result of SPA in terms of reconstruction error.
This observation is our main motivation to design a non-deterministic variant of SPA, that we coin as randomized SPA (RandSPA). 

\paragraph{Outline and contribution of the chapter}
In \cref{randspa:sec:spa} we introduce the general form of recursive algorithm for separable NMF analyzed in \cite{gillis2013fast} which generalizes SPA. 
In \cref{randspa:sec:randspa} we present the main contribution of this chapter, that is a randomized variant of SPA, called RandSPA\label{acro:RandSPA}.
We show the theoretical results on the robustness to noise of SPA still hold for RandSPA, while the randomization allows to better handle outliers by allowing a diversity in the solutions produced. 
In \cref{randspa:sec:xp} we illustrate the advantages of our method with experiments on both synthetic datasets and the unmixing of hyperspectral images.

\section{Successive Projection Algorithm}\label{randspa:sec:spa}

In this section, we discuss the successive projection algorithm (SPA).
It is based on the \emph{separability assumption}, detailed below.

\begin{assumption}[Separability]
\label{randspa:assumption:separability}
The $m$-by-$n$ matrix $X \in \mathbb{R}^{m\times n}$ is $r$-separable if there exist a nonnegative matrix $H$ such that \mbox{$X = X(:,\mathcal{J}) H$}, where $X(:,\mathcal{J})$ denotes the subset of columns of $X$ indexed by $\mathcal{J}$ and $|\mathcal{J}| = r$. 
\end{assumption}

The pseudocode for a general recursive algorithm for separable NMF is given in \Cref{randspa:algo:sepnmf}. Historically, the first variant of \Cref{randspa:algo:sepnmf} has been introduced by Araújo et al.~\cite{araujo2001successive} for spectroscopic component analysis with $f(x)=\|x\|_2^2=x^{\top}x$, which is the so-called SPA. In the noiseless case, that is, under \Cref{randspa:assumption:separability}, SPA is guaranteed to retrieve $\mathcal{J}$ and more generally, the vertices of the set of points which are the columns of $X$ \cite{ma2013signal}. 
This particular choice of $f$ is proved to be the most robust to noise given the bounds in \cite{gillis2013fast}. {See \Cref{randspa:corollary:precond} with $Q=I$ for the error bounds.} The algorithm is iterative and is composed of the following two main steps:
\begin{itemize}

    \item Selection step: the column that maximizes a given function $f$ is selected (\cref{randspa:alg:select}).
    
    \item Projection step: all the columns are projected onto the orthogonal complement of the current selected columns (\cref{randspa:alg:proj}).
    
\end{itemize}
These two steps are repeated $r$ times, $r$ being the target number of extracted columns. 
\begin{algorithm}[htb!]
\caption{Recursive algorithm for separable NMF~\cite{gillis2013fast}. {It coincides with SPA when $f(x) = \|x\|_2^2$.}  \label{randspa:algo:sepnmf}}

\KwIn{An $r$-separable matrix $X \in \mathbb{R}^{m \times n}$, a function $f$ to maximize.}

\KwOut{Index set $\mathcal{J}$ of cardinality $r$ such that $X \approx X(:,\mathcal{J}) H$ for some $H \geq 0$.}

Let $\mathcal{J} = \emptyset$, $P^\bot = I_m$, $V = [\;]$. \\

\For{$k=1$ : $r$}
{
    Let $j_k = \argmax_{1 \leq j \leq n} f(P^{\bot}X(:,j))$. (Break ties arbitrarily, if necessary.) \label{randspa:alg:select}\\
    Let $\mathcal{J} = \mathcal{J} \cup \{ j_k  \}$. \\
    Update the projector $P^\bot$ onto the orthogonal complement of $X(:,\mathcal{J})$:
    \begin{align*}
    & v_k = \frac{P^\bot X(:,j_k)}{\| P^\bot X(:,j_k) \|_2},\\
    & V = [V \, v_k],\\
    & P^\bot \leftarrow \left( I_m - VV^T \right).
    \end{align*} \label{randspa:alg:proj}
}
\end{algorithm}
The drawback with the $\ell_2$-norm is its sensitivity to outliers and the fact that it makes SPA deterministic.
If some outliers are selected, running SPA again would still retrieve the exact same outliers.


\section{Randomized SPA}\label{randspa:sec:randspa}

In this section, we introduce the main contribution of this work, that is a randomized variant of SPA called RandSPA.
Its key features are that it computes potentially different solutions at each run, thus allowing a multi-start strategy, and that the theoretical robustness results of SPA still hold.

RandSPA follows \Cref{randspa:algo:sepnmf} with $f(x)=x^\top QQ^\top x$, with $Q\in\mathbb{R}^{m\times\nu}$ being a randomly generated matrix with $\nu\geq r$.
To control the conditioning of $Q$, we generate the columns of $Q$ such that they are mutually orthogonal and such that $${\|Q(:,1)\|_2=1\geq\dots\geq\|Q(:,\nu)\|_2=1/\sqrt{\kappa}}$$ where $\kappa$ is the desired conditioning of $QQ^\top$. 
For the columns between the first and the last one, we make the arbitrary choice to fix them also to $1/\sqrt{\kappa}$.
If $Q^\top W$ has full column rank, which happens with probability one if $\nu \geq r$, RandSPA is robust to noise with the following bounds:

\begin{theorem}{\cite[Corollary 1]{gillis2015enhancing}} 
\label{randspa:corollary:precond}
Let $\Tilde{X}=X+N$, where $X$ satisfies \Cref{randspa:assumption:separability}, $W$ has full column rank, and $N$ is noise with $\max_j\|N(:,j)\|_2\leq\epsilon$; and let $Q\in\mathbb{R}^{m\times \nu}$ with $\nu\geq r$. If $Q^\top W$ has full column rank and 
$$ \epsilon\leq\mathcal{O}\left(\frac{\sigma_{\min}(W)}{\sqrt{r}\kappa^3(Q^\top W)}\right), $$
then SPA applied on matrix $Q^\top\Tilde{X}$ identifies a set of indices $\mathcal{J}$ corresponding to the columns of $W$ up to the error
$$\max_{1\leq j\leq r}\min_{k\in\mathcal{J}}\left\|W(:,j)-\Tilde{X}(:,k)\right\|_2\leq\mathcal{O}\left(\epsilon\kappa(W)\kappa(Q^\top W)^3\right).$$
\end{theorem}

\Cref{randspa:corollary:precond} is directly applicable to RandSPA since choosing $f(x)=x^\top QQ^\top x$ is equivalent to performing SPA on $Q^\top \Tilde{X}$. The only subtlety is that with RandSPA, a random $Q$ is drawn at each column extraction. The error bound for RandSPA is then the one with the highest drawn $\kappa(Q^\top W)$.

Let us note that choosing $\nu=1$ or
$\|Q(:,j)\|=1/\sqrt{\kappa}$ with $\kappa\rightarrow \infty$ for all $j>1$ retrieves VCA.
Choosing $\nu=m$ and $\kappa(Q)=1$ retrieves SPA. 
Hence, RandSPA creates a continuum between SPA, with more provable robustness, and VCA, with more solution diversity.

\section{Numerical experiments}\label{randspa:sec:xp}

In this section, we study empirically the performance of the proposed algorithm RandSPA on the unmixing of hyperspectral images.
The algorithms have been implemented in Julia~\cite{bezanson2017julia}.
The code for the algorithm is available as a Julia package in an online repository\footnote{\url{https://gitlab.com/vuthanho/randspa.jl}}. A different repository with the data and test scripts used in our experiments is also available\footnote{\url{https://gitlab.com/nnadisic/randspa}}. Our tests are performed on 5 real hyperspectral datasets\footnote{Originally downloaded from \url{http://lesun.weebly.com}, available at \url{https://gitlab.com/vuthanho/data}} described in \Cref{randspa:tab:datasets}.
\begin{table}[ht]
\centering 
\begin{tabular}{l|l|l|l}
Dataset & $m$  & $n$                      & $r$   \\ \hline
Jasper  & $198$  & $100 \times 100 = 10000$ & $4$  \\
Samson  & $156$  & $95 \times 95 = 9025$    & $3$  \\
Urban   & $162$  & $307 \times 307 = 94249$ & $5$  \\
Cuprite & $188$  & $250 \times 191 = 47750$ & $12$ \\
San Diego & $188$  & $400 \times 400 = 160000$ & $8$
\end{tabular}
\caption{Summary of the datasets, for which $X \in \mathbb{R}^{m\times n}$
.}\label{randspa:tab:datasets}
\end{table}

For all the tests, we choose $\nu=r+1$ and a relatively well conditioned $Q$ with $\kappa(Q)=1.5$. 
We then compute \mbox{$W = X(:,\mathcal{J})$} once with SPA and 30 times with RandSPA.
Next, we compute $H$ by solving the nonnegative least squares (NNLS) subproblem $\min_{H \geq 0} \| X - WH \|_F^2$ exactly with an active-set algorithm \cite{kim2008toward}, and we compute the relative reconstruction error ${\|X-WH\|_F}/{\|X\|_F}$.
For RandSPA, we show the best error and the median error among the 30 runs.
Note that in our setting we choose the best solution as the one with the lower reconstruction error, but other methods could be used to choose the best solution among all the computed ones.

The results of the experiments for SPA and RandSPA are presented in \Cref{randspa:tab:recerrhsu}.
The median error of RandSPA is on the same order than that of SPA, except for Cuprite where it is higher. This is probably because $r$ is greater than on other datasets. A good RandSPA run needs $r$ good successive $Q$'s, which is less probable when $r$ gets greater. This highlights that RandSPA could be improved. One possible improvement would be to draw several matrices $Q$'s at each iteration and select the best based on a criterion. The open question is then which criterion, and why. Going back to the median error of RandSPA, it is even slightly smaller than that of SPA for Samson and Urban.
On the other hand, the error from the best run of RandSPA is always smaller than that of SPA.
Particularly, the error is decreased respectively by 37\%, 32\% and 27\% for Samson, Urban and San Diego.
This improvement is quite noticeable.
\begin{table}[hbt!]
    \centering
    \begin{tabular}{l|ccc}
        Dataset &SPA     &Med. RandSPA  &Best RandSPA \\ \hline
        Jasper  &$8.6869$  &$8.7577$          &$8.0206$ \\
        Samson  &$6.4914$  &$6.3114$          &$3.9706$ \\
        Urban   &$10.9367$ &$9.6354$          &$6.5402$ \\
        Cuprite &$2.6975$  &$3.526$           &$2.2824$ \\
        San Diego &$12.6845$ &$12.8714$         &$9.2032$
    \end{tabular}
    \caption{Relative reconstruction error ${\|X-WH\|_F}/{\|X\|_F}$ in percent.}
    \label{randspa:tab:recerrhsu}
\end{table}

The resulting false color images for Jasper, Samson, Urban and Cuprite are shown on \Cref{randspa:fig:falsecolor}.
They represent the repartition of the materials identified by SPA and RandSPA in the image.
As we can see for Urban, SPA does not manage to separate well the grass and the trees (both the grass and trees are in green), while with RandSPA, it occurred that some random $Q$ amplified some directions that separate better the grass (in blue) and the trees (in green). Similarly, in the abundance maps from the unmixing of Samson in \Cref{randspa:fig:falsecolor}, RandSPA separates the soil (in red), the water (in blue) and the trees (in green) better than SPA where the soil (in blue) is extracted but the water is not clearly identified.

Let us discuss another experiment on the dataset Samson.
We add some Gaussian noise such that $SNR=20dB$, we fix $\kappa=1$ and vary $\nu$, and then show the average best error in 1,5,10 and 20 runs on \Cref{randspa:fig:avgrbestrun}.
As we can see, with a sufficient amount of runs that is 10 in this experiment, the relative error significantly improves for a $\nu$ near 10 in comparison to other choices of $\nu$. In particular, it is also better than both $\nu=1$ (VCA) and a high $\nu$ like $50$ that should behave like SPA. Without added noise, VCA would perform better than every $\nu$ higher than 1 starting from 10 runs.
However, when the data is noisy, this experiment 
highlights that VCA is not robust enough to noise and that the best run from a method between SPA and VCA is better than both SPA and VCA.



\begin{figure}[htb!]
    \centering
    \begin{tabular}{rcc}
    & (a) SPA & (b) RandSPA \\
    \rotatebox[origin=lc]{90}{\phantom{uuuuuuu}Jasper}
    & \includegraphics[width=3.6cm]{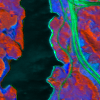}
    & \includegraphics[width=3.6cm]{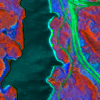} \\
    \rotatebox[origin=lc]{90}{\phantom{uuuuuuu}Samson}
    & \includegraphics[width=3.6cm]{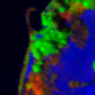}
    & \includegraphics[width=3.6cm]{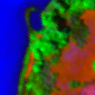} \\
    \rotatebox[origin=lc]{90}{\phantom{uuuuuuu}Urban}
    & \includegraphics[width=3.6cm]{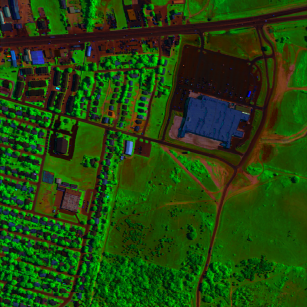}
    & \includegraphics[width=3.6cm]{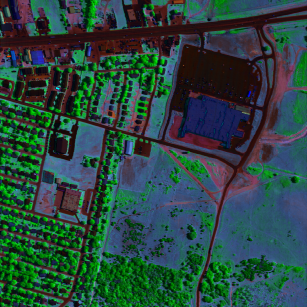} \\
    \rotatebox[origin=lc]{90}{\phantom{uuuuuuuuuu}Cuprite}
    & \includegraphics[width=3.6cm]{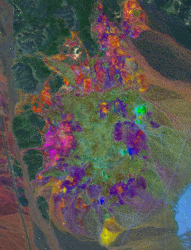}
    & \includegraphics[width=3.6cm]{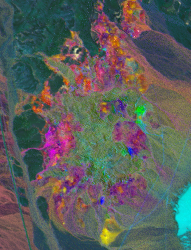}
    \end{tabular}
    \caption{Abundance maps in false color from the unmixing of hyperspectral images.}
    \label{randspa:fig:falsecolor}
\end{figure}

\begin{figure}[htb!]
    \centering
    \begin{tikzpicture}
\begin{axis}[ytick={1,...,6},
            yticklabels={1,3,6,10,25,50},
            ylabel = {Rank $\nu$ of $Q$},
            ymin=0,ymax=7,
            title = {Average best in 1 run},width=0.55\textwidth,
            height=5cm]
\foreach \col in {0,...,5}{
\addplot[boxplot,blue] table[y index = \col]{randspa/xp/res_nbrun1_snr20.txt};
\addplot[mark=none,line width = 0.5pt,dashed,red] coordinates {(12.41, 0) (12.41, 7)};
\legend{,,,,,,,SPA};
};
\end{axis}
\end{tikzpicture}
\begin{tikzpicture}
\begin{axis}[ytick={1,...,6},
            yticklabels={},
            ymin=0,ymax=7,
            title = {Average best in 5 runs},width=0.55\textwidth,
            height=5cm]
\foreach \col in {0,...,5}{
\addplot[boxplot,blue] table[y index = \col]{randspa/xp/res_nbrun5_snr20.txt};
\addplot[mark=none,line width = 0.5pt,dashed,red] coordinates {(12.41, 0) (12.41, 7)};
};
\end{axis}
\end{tikzpicture}

\vspace{2cm}

\begin{tikzpicture}
\begin{axis}[ytick={1,...,6},
            yticklabels={1,3,6,10,25,50},
            ylabel = {Rank $\nu$ of $Q$},
            ymin=0,ymax=7,
            xlabel = {Relative error ($\%$)},
            title = {Average best in 10 runs},width=0.55\textwidth,
            height=5cm]
\foreach \col in {0,...,5}{
\addplot[boxplot,blue] table[y index = \col]{randspa/xp/res_nbrun10_snr20.txt};
\addplot[mark=none,line width = 0.5pt,dashed,red] coordinates {(12.41, 0) (12.41, 7)};
};
\end{axis}
\end{tikzpicture}
\begin{tikzpicture}
\begin{axis}[ytick={1,...,6},
            yticklabels={},
            ymin=0,ymax=7,
            xlabel = {Relative error ($\%$)},
            title = {Average best in 20 runs},width=0.55\textwidth,
            height=5cm]
\foreach \col in {0,...,5}{
\addplot[boxplot,blue] table[y index = \col]{randspa/xp/res_nbrun20_snr20.txt};
\addplot[mark=none,line width = 0.5pt,dashed,red] coordinates {(12.41, 0) (12.41, 7)};
};
\end{axis}
\end{tikzpicture}
    \caption[Average best reconstruction error on Samson]{Average best reconstruction error on several runs, depending on $\nu$, with $\kappa=1$, on the hyperspectral image Samson with added noise such that $SNR=20dB$.}
    \label{randspa:fig:avgrbestrun}
\end{figure}
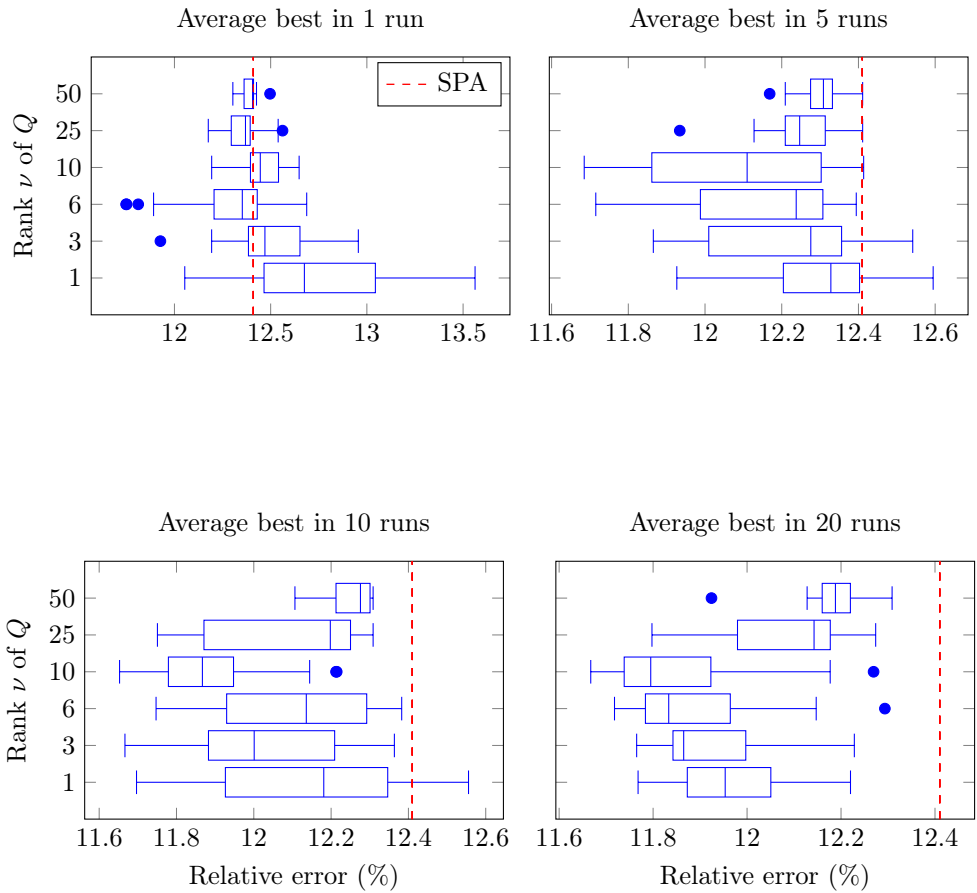


\section{Conclusion}\label{randspa:sec:conclu}

In this chapter, we introduced RandSPA, a variant of the separable NMF algorithm SPA that introduces randomness to allow a multi-start strategy.
The robustness results of SPA still hold for RandSPA, provided a bound on the noise that depends on the parameters used.
We showed empirically on the unmixing of hyperspectral images that, with sufficiently many runs, the best solution from RandSPA is generally better that the solution from SPA.
We also showed that RandSPA creates a continuum between the two algorithms SPA and VCA, as we can recover these algorithms by running RandSPA with some given parameter values.

\chapter{Minimum-Volume Nonnegative Matrix Factorization}\label{chap:minvolnmf}
\begin{hiddenmusic}{Thom Draft - Breathtaking}{https://www.youtube.com/watch?v=4vdToJ0x2ds}
\end{hiddenmusic}The minimum-volume criterion was originally thought by \cite{full1981extended} and the so called Craig's belief \cite{craig1994minimum}. It also appeared later in the chemometrics community. Up to our knowledge, the first implementation of the minimum-volume criterion coupled with NMF was proposed in~\cite{miao2007endmember}. The idea is that in the absence of pure pixels, given that all the data points are not strong mixtures, finding endmembers whose cone or convex hull tightly contains the data points retrieves the true endmembers. If the main motivation was spectral unmixing, the minimum-volume criterion has also been shown to be useful in other applications, like blind audio source separation for instance \cite{leplat2019blind,wang2021minimum}. Regardless of the application, the minimum-volume criterion encourages interpretability of the features since they are close to the data points. 
\paragraph{Outline and contribution of the chapter} In this chapter, we present known declinations of MinVol NMF\label{acro:MinVolNMF} in \Cref{minvol:sec:knownminvols}, we quickly recall on the identifiability of MinVol NMF in \Cref{minvol:sec:identifiability}, we propose a fast algorithm for MinVol NMF in \Cref{minvol:sec:titanizedminvol} and finally, we show how the MinVol criterion is promising for matrix completion in \Cref{minvol:sec:minvolnmc}.


\section{Existing variants of MinVol NMF}\label{minvol:sec:knownminvols}

Geometrically, the NMF $X=WH$ implies that $\cone(X)\subseteq\cone(W)$. 
Distinctively, with MinVol NMF, the convex hull of $W$ should enclose the convex hull of $X$ as tightly as possible\footnote{This interpretation only holds with a column-wise simplex-structured $H$.}, hence the expression ``minimum-volume''. In other words, MinVol NMF consists in finding a couple of factors $(W,H)\in\R_+^{m\times r}\times\R_+^{r\times n}$ such that $X=WH$ while minimizing the volume of the convex hull of 
the columns of $W$ and the origin, which is given by $\frac{1}{r!}\sqrt{\det(\Wt W)}$. 
This improves the interpretability of the features (the columns of $W$) while prioritizing a unique decomposition of the data under relatively mild assumptions, that are given in~\Cref{preli:th:identifminvolSSMF}. Additionally, one of the factors should be constrained such that the scaling ambiguity between $W$ and $H$ coupled with the minimized volume does not make $W$ tend to zero at optimality. Identifiable MinVol NMFs typically use simplex structuring constraints, namely $W\in\Delta^{m\times r}$~\cite{leplat2019blind} or $H\in\Delta^{r\times n}$~\cite{fu2015blind} or $H^\top\in\Delta^{n\times r}$~\cite{fu2018identifiability}, where $\Delta^{m\times r}=\{Y\in\R_+^{m\times r},e^\top Y=e^\top\}$ and $e$ is the all-one vector of appropriate dimension. 
See \Cref{minvol:sec:identifiability} for more details on the identifiability of MinVol NMF. The constraint $W\in\Delta^{m\times r}$ ensures that the columns of $W$ lie within the probability simplex. The constraint $H^\top\in\Delta^{n\times r}$ can be seen as a budget assignment constraint: each feature should be used in the decomposition as much as the others. Both aforementioned constraints are without loss of generality relatively to NMF because of the scaling ambiguity between $W$ and $H$, that is, any NMF $(W,H)$ can be scaled so that $W\in\Delta^{m\times r}$ or $H^\top\in\Delta^{n\times r}$ is satisfied. The constraint $H\in\Delta^{r\times n}$ is stronger than the other two as it is not without loss of generality relatively to NMF. For example, for the matrix  $X=\begin{pmatrix}
    1 & 0 & 1\\
    0 & 1 & 1
\end{pmatrix} $, 
there exists a rank-2 NMF of $X$ which is simply $X=IX$. However, any exact NMF of $X$ with the additional constraint that $H$ has to be column wise stochastic is of rank at least 3. Despite the loss of generality, this constraint remains useful in practice as it provides a soft clustering interpretation of the decomposition. It has been instrumental in hyperspectral imaging where each column of $H$ contains the abundances of the pure materials in a pixel which are nonnegative and sum to one~\cite{ma2014signal}. The constraint $H^\top\in\Delta^{n\times r}$ can be responsible for an ill conditioned $W$ that can lead to numerical issues~\cite{leplat2019blind}.
Consequently, among the three mentioned variants of MinVol, we will only consider the following exact formulation in the remainder of this chapter:
\begin{mini}
	{\scriptstyle W,H}{\det(\Wt W )}{\label{minvol:eq:exactminvol}}{}
	\addConstraint{X=WH}
	\addConstraint{W\in\Delta^{m \times r}}{,~H\in\R_+^{r \times n}.}{}
\end{mini}
Consequently, the inexact formulation is
\begin{mini}
	{\scriptstyle W,H}{\frac{1}{2}\|X-WH\|_F^2+\frac{\lambda}{2}\logdet(\Wt W +\delta I)}{\label{minvol:eq:nonexactminvol}}{}
	\addConstraint{W\in\Delta^{m \times r}}{,~H\in\R_+^{r \times n},}{}
\end{mini}
where $\delta$ is a parameter that prevents the $\logdet$ from going to $-\infty$ when $W$ is rank deficient, and $\lambda\geq0$ balances the two terms. Note that the true volume spanned by the columns of $W$ and the origin is equal to $\frac{1}{r!}\sqrt{\det(\Wt W)}$, but minimizing $\logdet(\Wt W)$ is equivalent in the exact case and makes the problem numerically easier to solve because the function $\logdet(\cdot)$ is concave and it is easier to design a ``nice'' majorizer for it~\cite{fu2016robust}. 

\section{Identifiability of MinVol NMF}\label{minvol:sec:identifiability}

The identifiability of MinVol NMF with $H\in\Delta^{r\times n}$ was indirectly mentioned through the identifiability of MinVol SSMF in \Cref{preli:th:identifminvolSSMF}. The additional nonnegative constraint on $W$ does not change this identifiability result. MinVol NMF with $H^\top\in\Delta^{n\times r}$ and $W\in\Delta^{m\times r}$ are also identifiable with very similar proofs. Since we only use MinVol NMF with $W\in\Delta^{m\times r}$, let us remind the proof of its identifiability. The proof was given in~\cite{leplat2019blind} and adapted from~\cite{fu2015blind}.
\begin{theorem}[\cite{leplat2019blind}]
	\label{minvol:th:uniqueminvol}
	Let $X=WH$ be a MinVol NMF of $X$ of size $r = \rank(X)$, in the sense of~\eqref{minvol:eq:exactminvol}. If $H$ satisfies SSC as in \Cref{preli:def:ssc}, then MinVol NMF $(W,H)$ of $X$ is essentially unique.
\end{theorem}
\begin{proof}
	Let $Q\in\R^{r\times r}$ be an invertible matrix such that $(WQ^{-1},QH)$ is another feasible solution of~\eqref{minvol:eq:exactminvol}. Since $W^\top e = e$ and $Q^{-\top}W^\top e=e$ because $(WQ^{-1},QH)$ is feasible, we have
	\begin{equation}
	 Q^{-\top}W^\top e=e 
		\quad \Leftrightarrow \quad  Q^{-\top} e = e.
	\end{equation}
	Multiplying on the left by $Q^\top$ leads to $e=Q^\top e$. Using again feasibility of $(WQ^{-1},QH)$,
	\begin{align}
		QH\geq0 \quad\Leftrightarrow\quad &  H^\top Q^\top\geq0 \label{minvol:eq:QHgeq0}\\
		\Leftrightarrow\quad & Q(i,:)^\top\in\cone^*(H) \\
		\Leftrightarrow\quad & \cone(Q^\top)\subseteq\cone^*(H).\label{minvol:eq:coneQ-in-coneH}
	\end{align}
	Since $H$ satisfies SSC1, $\C\subseteq\cone(H)$. By duality, $\cone^*(H)\subseteq\C^*$, where $\C^*$ is given in \Cref{preli:th:dualC}. With \eqref{minvol:eq:coneQ-in-coneH}, this implies that $\cone(Q^\top)\subseteq\C^*$. More explicitly,
	\begin{equation}\label{minvol:eq:explicit-coneQ-in-dualC}
		Q(i,:)e\geq\|Q(i,:)\|_2 \text{ for } i=1,\dots,r.
	\end{equation}
	Therefore,
	\begin{equation}
		\label{minvol:eq:detQineq}
		\begin{split}
			|\det(Q)| \leq & \prod_{i=1}^r\|Q(i,:)\|_2 \\
			\leq & \prod_{i=1}^r Q(i,:)e \\
			\leq & \left(\frac{\sum_{i=1}^{r}Q(i,:)e}{r}\right)^r = \left(\frac{e^\top Q^\top e}{r}\right)^r =1,
		\end{split}
	\end{equation}
	where the first inequality is coming from the Hadamard's inequality, the second from~\eqref{minvol:eq:explicit-coneQ-in-dualC}, and the last one from the arithmetic-geometric mean inequality and that $Q^\top e = e$. \\

	Suppose now that $(WQ^{-1},QH)$ is also an optimal solution to~\eqref{minvol:eq:exactminvol}. Then,
	\begin{align}
		&\det(Q^{-\top}W^\top W Q^{-1})  = \det(W^\top W) \\
		\Leftrightarrow \quad &|\det(Q)|^{-2}\det(W^\top W) = \det(W^\top W) \\
		\Leftrightarrow \quad &|\det(Q)|=1.
	\end{align}
	With $|\det(Q)|=1$, all inequalities in~\eqref{minvol:eq:detQineq} are equalities. Particularly, for all $i$,
	\begin{equation}
		Q(i,:)e = \|Q(i,:)\|_2 = 1
	\end{equation}
	and $|\det(Q)|=\prod_{i=1}^r\|Q(i,:)\|_2$, implying that $Q^\top$ is orthogonal. By duality of \eqref{minvol:eq:coneQ-in-coneH} and using that the cone of any orthogonal matrix is self dual, we have that $\cone(H)\subseteq\cone(Q^\top)$. Finally, since $H$ satisfies SSC2, $Q^\top$ can only be a permutation matrix.
\end{proof}

\section{Solving MinVol NMF with TITAN}\label{minvol:sec:titanizedminvol} 

As opposed to PCA/SVD, solving NMF is NP-hard in general~\cite{vavasis2009complexity}. Hence, most NMF algorithms rely on standard non-linear optimization schemes without global optimality guarantee. This also applies to MinVol NMF. In this section, we propose a fast method to solve MinVol NMF in \Cref{minvol:sec:algoTITAN}. Our method is an application of a recent inertial block majorization-minimization framework called TITAN~\cite{hien2023inertial}, that we already used in \Cref{chap:bssmf}. Experimental results on real datasets show that the proposed method performs better than the state of the art; see \Cref{minvol:sec:numexp}.

\subsection{TITANized MinVol NMF} \label{minvol:sec:algoTITAN} 

As far as we know, all algorithms for MinVol NMF rely on two-block coordinate descent methods that update each block ($ W $ or $ H $) by using some outer  optimization algorithm to solve the subproblems formed by restricting the MinVol NMF problem to each block.
For example, the state-of-the-art method from~\cite{leplat19} uses Nesterov fast gradient method to update each factor matrix, one at a time. 

Our proposed algorithm for~\eqref{minvol:eq:nonexactminvol} will be based on the TITAN framework from~\cite{hien2023inertial}. TITAN is an inertial block majorization minimization framework for nonsmooth nonconvex optimization. It updates one block at a time while fixing the values of the other blocks, as previous MinVol NMF algorithms. 
In order to update a block, TITAN chooses a block surrogate function for the corresponding objective function (a.k.a.\ a majorizer), embeds an inertial term to this surrogate function and then minimizes the obtained inertial surrogate function. When a Lipschitz gradient surrogate is used, TITAN reduces to the Nesterov-type accelerated gradient descent step for each block of variables \cite[Section 4.2]{hien2023inertial}.
The difference of TITAN compared to previous MinVol NMF algorithms is threefold: 
\begin{enumerate} 

\item The inertial force (also known as the extrapolation, or momentum) is used between block updates. This is a crucial aspect that will make our proposed algorithm faster: when we start the update of a block of variables (here, $ W $ or $ H $), we can use the inertial force (using the previous iterate)  although the other blocks have been updated in the meantime.  
\item TITAN allows to update the surrogate after each update of $ W $ and $ H $, which was not possible with the algorithm from~\cite{leplat19} because it applied fast gradient from convex optimization on a fixed surrogate.  

\item  It has subsequential convergence guarantee, that is, every limit point of the generated sequence is a stationary point of Problem~\eqref{minvol:eq:nonexactminvol}. Note that the state-of-the-art algorithm from~\cite{leplat19} does not have convergence guarantees.
\end{enumerate} 

\noindent \textbf{Remark.} 
The block prox-linear (BPL) method from~\cite{Xu2017} can be used to solve~\eqref{minvol:eq:nonexactminvol} since the block functions in $ W \mapsto \frac12\|X- W  H \|^2_F$ and in $ H \mapsto \frac12\|X- W  H \|^2_F$ 
have Lipschitz continuous gradients. 
However, BPL applies extrapolation to the Lipschitz gradient surrogate of these block functions and requires to compute the proximal point of the regularizer $\frac{\lambda}{2}\logdet( W^\top W +\delta I)$, which does not have a closed form. 
In contrast, TITAN applies extrapolation to the surrogate function of $ W \mapsto f( W , H )$ with a surrogate function for the regularizer $\frac{\lambda}{2}\logdet( W^\top W +\delta I)$  (see \Cref{minvol:sec:surrogateW}). This allows TITAN  to have closed-form solutions for the subproblems, an acceleration effect, and convergence guarantee.

\subsubsection{Surrogate functions} \label{minvol:sec:surrogates}
An important step of TITAN is to define a surrogate function for each block of variables. 
These surrogate functions are upper approximation of the objective function at the current iterate. Denote 
\[f( W ,  H ) = \frac{1}{2}\left\|X- W  H \right\|^2_F+\frac{\lambda}{2}\logdet( W^\top W +\delta I)\] and suppose we are cyclically updating $( W , H )$. Let us denote $u_{W_k}( W )$ the surrogate function of $ W \mapsto f( W ,H_k) $ to update $W_k$, that is, 
\begin{equation}
	f( W ,H_k) \leq u_{W_k}( W ) 
	\; 
	\text{ for all } 
	\; 
	 W  \in \mathcal{X}_W , 
\end{equation} 
where $ u_{W_k}(W_k) =f(W_k,H_k)$ and $\mathcal{X}_W $ is the feasible domain of $  W  $. Similarly, let us denote $u_{H_k}( H )$ the surrogate function of $ H \mapsto f(W_{k+1}, H )$ to update $H_k$, that is 
\begin{equation}
	f(W_{k+1}, H )\leq u_{H_k}( H ) 
	\; 
	\text{ for all } 
	\; 
	 H  \in \mathcal{X}_H ,  
\end{equation}
where $u_{H_k}(H_k) =f(W_{k+1},H_k)$ and $\mathcal{X}_H $ is the feasible domain of $  H  $. \\

\paragraph{Surrogate function and update of  $ W $} 
\label{minvol:sec:surrogateW}

Denote $ A =  W^\top W +\delta I,~B_k = W_k^\top W_k+\delta I $ and $ P_k = (B_k)^{-1}$. Since $ \logdet $ is concave, its first-order Taylor expansion around $B_k$ leads to $\logdet(A)\leq \logdet(B_k) + \langle (B_k)^{-1},A-B_k \rangle$. Hence,
\begin{equation}
f( W ,H_k)\leq \widetilde{f}_{W_k}( W )
:=\frac{1}{2}\left\|X- W H_k\right\|^2_F + \frac{\lambda}{2}\langle P_k, W^\top W  \rangle + C_1,
\label{minvol:eq:majorW}
\end{equation}
where $C_1$ is a constant independent of $ W $. Note that the gradient of $ W \mapsto\widetilde{f}_{W_k}( W )$, being equal to \[( W H_k-X)H_k^\top+\lambda  W P_k,\] is $ L^k_W  $-Lipschitz continuous with $ L^k_W =\|H_kH_k^\top+\lambda P_k\| $. Hence, from~\eqref{minvol:eq:majorW} and the descent lemma (see \cite[Section 2.1]{Nesterov2018}), 
\begin{equation}
f( W ,H_k) \leq u_{W_k}( W ) := \langle\nabla \widetilde{f}_{W_k}(W_k), W \rangle + \frac{L^k_W }{2}\| W -W_k\|^2_F + C_2, 
\label{minvol:eq:surrogateforW}
\end{equation}
where $C_2$ is a constant depending on $W_k$. We use the surrogate $u_{W_k}( W ) $ defined in~\eqref{minvol:eq:surrogateforW} to update $W_k$. As TITAN recovers Nesterov-type acceleration for the update of each block of variables~\cite[Section 4.2]{hien2023inertial}, 
we have the following update for $ W $:   
\begin{equation}
\label{minvol:eq:updateW}
\begin{split}
 W_{k+1}   & = \argmin_{ W \in\mathcal{X}_W } \langle\nabla \widetilde{f}_{W_k}(\overbar{W_k}), W \rangle + \frac{L^k_W }{2}\| W -\overbar{W_k}\|^2_F, \\
            & = \left[ \overbar{W_k} + \frac{(X-\overbar{W_k}H_k)H_k^\top-\lambda\overbar{W_k}P}{L_W^k} \right]_{\Delta^{m\times r}},
\end{split}
\end{equation}
where $[.]_{\Delta^{m\times r}}$ performs column wise projections onto the unit simplex as in~\cite{condat2016} in order to satisfy the constraint on $ W $ in \eqref{minvol:eq:nonexactminvol}, and where $\overbar{W_k}$ is an extrapolated point, that is, the current point $W_k$ plus some momentum,
\begin{equation}
    \overbar{W_k} = W_k + \beta_{ W }^{k}(W_k-W_{k-1}),
\end{equation}
where the extrapolation parameter $\beta_{ W }^{k}$ is chosen as follows 
\begin{equation}
    \label{minvol:eq:betaW}
\beta_{ W }^{k} = \min\left(~ \dfrac{\alpha_k-1}{\alpha_{k+1}} ,0.9999\sqrt{\dfrac{L_W^{k-1}}{L_W^{k}}} ~\right),  
\end{equation}
$\alpha_0=1$, $\alpha_k=(1+\sqrt{1+4 \alpha_{k-1}^2})/2$. This choice of parameter satisfies the conditions to have a subsequential convergence of TITAN, see \Cref{minvol:sec:convergence}.

\paragraph{Surrogate function and update of  $ H $} \label{minvol:sec:surrogateH}
Since
\[
\nabla_H  f(W_{k+1}, H )=W_{k+1}^\top(W_{k+1} H -X),
\]
the gradient of $ f $ according to $  H  $ is $L_H^{k}$-Lipschitz continuous with $L_H^{k}=\|W_{k+1}^\top W_{k+1}\|$. Hence, we use the following Lipschitz gradient surrogate to update $H_k$: 
\begin{equation}
	u_{H_k}( H ) = \langle\nabla_H  f(W_{k+1},H_k), H \rangle + \frac{L^k_H }{2}\| H -H_k\|^2_F +C_3,
\label{minvol:eq:surrogateH}
\end{equation}
where $ C_3 $ is a constant depending on $ H_k $. We derive our update rule for $ H $ by minimizing the surrogate function from \Cref{minvol:eq:surrogateH} embedded with extrapolation,
\begin{equation}
	\label{minvol:eq:minsurrogateH}
\begin{split}
	H_{k+1}    & = \argmin_{ H \in\mathcal{X}_H} \langle\nabla_H  f(W_{k+1},\overbar{H_k}), H \rangle + \frac{L^k_H }{2}\| H -\overbar{H_k}\|^2_F, \\
	& = \left[\overbar{ H }_k+\frac{1}{L_H^{k}}W_{k+1}^\top(X-W_{k+1}\overbar{ H }_k)  \right]_+,
\end{split}
\end{equation}
where $[\,.\,]_+$ denotes the projector setting all negative values to zero, and $\overbar{H_k}$ is the extrapolated $H_k$: 
\begin{equation}
    \overbar{H_k}=H_k+\beta^{k}_H (H_k - H_{k-1}),
\end{equation}
 where, as for the update of $ W $, 
 \begin{equation}
     \beta_{ H }^{k} = \min\left(~ \dfrac{\alpha_k-1}{\alpha_{k+1}} ,0.9999\sqrt{ \dfrac{L_H^{k-1}}{L_H^{k}}} ~\right).
 \end{equation}

\subsubsection{Algorithm} \label{minvol:sec:algodesc}
Note that the update of $ W $ in~\eqref{minvol:eq:updateW} and $ H $ in~\eqref{minvol:eq:minsurrogateH} was described when the cyclic update rule is applied. Since  TITAN also allows the essentially cyclic rule~\cite[Section 5]{hien2023inertial}, we can update $ W $ several times before switching updating $ H $, and vice versa. Doing so allows to pre-compute some matrix operations before updating the factors. For instance, $XH^\top$ can be computed before updating $W$. The result can then be used several times during the update, which will save some computation time. Note that this pre-computing trick only works when there are no missing entries, due to the Hadamard product with $M$.
This leads to our proposed method TITANized MinVol, see \Cref{minvol:algo:ourmethod} for the pseudocode.
The stopping criteria in \cref{minvol:algo:stopcritW,minvol:algo:stopcritH} are the same as in \cite{leplat19}. The way $\lambda$ and $\delta$ are computed is also identical to~\cite{leplat19}. 
Let us mention that technically the main difference with \cite{leplat19} resides in how the extrapolation is embedded. In \cite{leplat19} the Nesterov sequence is restarted and evolves in each inner loop to solve each subproblem corresponding to each block. In our algorithm, the extrapolation parameter $\beta_W $ (and $\beta_H $) for updating each block $ W $ (and $ H $) is updated continuously without restarting.
It means we are accelerating the global convergence of the sequences rather than trying to accelerate the convergence for the subproblems. Moreover, TITAN allows to update the surrogate function at each step, while the algorithm from \cite{leplat19} can only update it before each subproblem is solved, as it relies on Nesterov's acceleration for convex optimization.

\begin{algorithm}[htbp!]
    \caption{TITANized MinVol}
    \label{minvol:algo:ourmethod}
    \DontPrintSemicolon
    \KwIn{$W_0,H_0,\lambda,\delta$}
    $\alpha_1=1$, $\alpha_2=1$, $W_{old}=W_0,H_{old}=H_0$, $L_H^{prev}=\|W_{0}^\top W_{0}\|$, $L_W^{prev}=\|H_{0} H^\top_{0}+\lambda (W_{0}^\top W_{0}+\delta I)^{-1}\|$\;
    \KwOut{$W,H$}
    \While{stopping criteria not satisfied}{
        \While{stopping criteria not satisfied\label{minvol:algo:stopcritW}}{
            $\alpha_{0}=\alpha_1, \alpha_1=(1+\sqrt{1+4\alpha_0^2})/2$\;
            $ P \leftarrow ( W^\top W +\delta I)^{-1} $\;
            $L_W  \leftarrow\| H  H^\top+\lambda P\| $\;
            $\beta_{ W }=\min\left(~(\alpha_0-1)/\alpha_{1},0.9999\sqrt{L_W^{prev}/L_W } ~\right) $\;
            $\overbar{ W }\leftarrow W +\beta_{ W }( W -W_{old}) $\;
            $W_{old}\leftarrow  W $\;
            $ W  \leftarrow \left[\overbar{ W }+\frac{(X H^\top-\overbar{ W }( H  H^\top+\lambda P))}{L_W }\right]_{\Delta^{m\times r}}$\;
            $ L_W^{prev} \leftarrow L_W $\;
        }
        $ L_H  \leftarrow \| W^\top W \| $\;
        \While{stopping criteria not satisfied\label{minvol:algo:stopcritH}}{ 
            $\alpha_{0}=\alpha_2, \alpha_2=(1+\sqrt{1+4\alpha_0^2})/2$\;
            $ \beta_{ H }=\min\left(~(\alpha_0-1)/\alpha_{2},0.9999\sqrt{L_H^{prev}/L_H } ~\right) $\;
            $ \overbar{ H } \leftarrow  H +\beta_{ H }( H -H_{old}) $\;
            $H_{old}\leftarrow  H $\;
            $  H  \leftarrow \left[\overbar{ H }+\frac{ W^\top(X- W \overbar{ H }) }{L_H } \right]_+ $\;
            $L_H^{prev} \leftarrow L_H $\;
        }
    }
    \end{algorithm}

\subsubsection{Convergence guarantee} \label{minvol:sec:convergence}

In order to have a convergence guarantee, TITAN requires the update of each block to satisfy the nearly sufficiently decreasing property (NSDP), see \cite[Section 2]{hien2023inertial}.  
By \cite[Section 4.2.1]{hien2023inertial}, the update for $ H $ of TITANized MinVol satisfies the NSDP condition since it uses a Lipschitz gradient surrogate for $ H \mapsto f( W , H )$ combined with the Nesterov-type extrapolation; and the bounds of the extrapolation parameters in the update of $ H $ are derived similarly as in \cite[Section 6.1]{hien2023inertial}. However, it is important noting that the update for $ W $ of TITANized MinVol does not directly use a Lipschitz gradient surrogate for $ W \mapsto f( W , H )$. We thus need to verify NSDP condition for  the update of $ W $ by another method that is presented in the following.  

The function $u_{W_k}( W )$ is a Lipschitz gradient surrogate of $\tilde f_{W_k}( W ) $, and we apply the Nesterov-type extrapolation to obtain the update in \eqref{minvol:eq:updateW}. Note that the feasible set of $ W $ is convex.  Hence, it follows from \cite[Remark 4.1]{hien2023inertial} that 
\begin{equation}
\label{minvol:eq:NSDP1}
    \tilde f_{W_k}(W_k) + \frac{L_{ W }^k (\beta_{ W }^k)^2}{2}\|W_k-W_{k-1}\|^2_F 
    \geq
    \tilde f_{W_k}(W_{k+1}) + \frac{L_{ W }^k}{2}\|W_{k+1}-W_k\|^2_F.
\end{equation}
Furthermore, we note that $ \tilde f_{W_k}(W_k)=f(W_k,H_k)$, and $ \tilde f_{W_k}(W_{k+1}) \geq f(W_{k+1},H_k)$. Therefore, from~\eqref{minvol:eq:NSDP1} we have 
\begin{equation}
    f(W_k,H_k) + \frac{L_{ W }^k (\beta_{ W }^k)^2}{2}\|W_k-W_{k-1}\|^2_F 
    \geq
    f(W_{k+1},H_k) + \frac{L_{ W }^k}{2}\|W_{k+1}-W_k\|^2_F,
\end{equation}
which is the required NSDP condition of TITAN. Consequently, the choice of $\beta_{ W }^k$ in~\eqref{minvol:eq:betaW}  satisfy the required condition to guarantee subsequential convergence~\cite[Proposition 3.1]{hien2023inertial}.  

On the other hand, we note that the error function $ W  \mapsto err_1( W ):= u_{W_k}( W  )-f( W ,H_k )$ is continuously differentiable and  $\nabla_{ W } err_1(W_k )=0$; similarly for the error function $ H  \mapsto err_2( H ):= u_{H_k}( H  )-f(W_{k+1}, H  )$. Hence, it follows from \cite[Lemma 2.3]{hien2023inertial} that the Assumption 2.2 in \cite{hien2023inertial} is satisfied. Applying \cite[Theorem 3.2]{hien2023inertial}, we conclude that every limit point of the generated sequence is a stationary point of Problem~\eqref{minvol:eq:nonexactminvol}. It is worth noting that as TITANized MinVol does not apply restarting step, \cite[Theorem 3.5]{hien2023inertial} for a global convergence is not applicable.

\subsection{Numerical Experiments} \label{minvol:sec:numexp}

In this section we compare TITANized MinVol to~\cite{leplat19}, an accelerated version of the method from~\cite{fu16} (for $p=2$), on two NMF applications: hyperspectral unmixing and document clustering, which are dense and sparse datasets, respectively. All tests are performed on MATLAB R2018a, on a PC with an Intel® Core™ i7 6700HQ and 24 GB RAM. The code is available on an online repository\footnote{\url{https://gitlab.com/vuthanho/titanized-minvol}}.

The datasets used are shown in \Cref{minvol:tab:datasets}. For each data set, each algorithm is launched with the same random initializations, for the same amount of wall-clock time. In order to derive some statistics, for both hyperspectral unmixing and document clustering, 20 random initializations are used (each entry of $ W $ and $ H $ are drawn from the uniform distribution in [0,1]).  
The wall-clock time used for each data set is adjusted manually, and corresponds to the maximum displayed value on the respective time axes in \Cref{minvol:fig:avg}; see also \Cref{minvol:tab:leadtime}.  
    

\begin{table}[thb!]
    \centering
    \normalsize
    \begin{tabular}{l||c|c|c}
        Data set     & $m$ & $n$ & $r$ \\ \hline
        Urban       & 162 & 94249 & 6   \\
        Indian Pine & 200 & 21025 & 16 \\
        Pavia Univ. & 103 & 207400 & 9 \\
        San Diego   & 158 & 160000 & 7 \\
        Terrain     & 166 & 153500 & 5 \\
        20 News     & 61188 & 7505 & 20 \\
        Sports      & 14870 & 8580 & 7 \\
        Reviews     & 18483 & 4069 & 5 \\
    \end{tabular}
    
    \caption{Datasets used in our experiments and their respective dimensions}
    \label{minvol:tab:datasets}
\end{table}

For display purposes, for each data set, we compare the average of the scaled objective functions according to time, that is, the average of  $(f( W , H )-e_{\min})/\|X\|_F$ where $e_{\min}$ is the minimum obtained error among the 20 different runs and among both methods. The results are presented in \Cref{minvol:fig:avg}. On both hyperspectral and document datasets, TITANized MinVol converges on average faster than \cite{leplat19} except for the San Diego data set (although TITANized MinVol converges initially faster). For most tested datasets, MinVol~\cite{leplat19} cannot reach the same error as TITANized MinVol within the allocated time. 

\begin{figure}[htbp!]
\begin{subfigure}{0.49\linewidth}
\captionsetup{skip=0pt}
\caption{Urban}
\tikzexternalexportnextfalse 
\begin{tikzpicture}
	\begin{semilogyaxis}[
		ymax=1e-4,
		grid=major,width = \textwidth,height=4cm,cycle list name=exotic,no markers,
		legend entries={TITANized MinVol,MinVol from~\cite{leplat19}},
		]
		\addplot+[line width = 2pt]table {titanizedminvol/results_hyper/mean_urban1.dat};
		\addplot+[line width = 2pt]table {titanizedminvol/results_hyper/mean_urban2.dat};
	\end{semilogyaxis}
\end{tikzpicture}
\end{subfigure}
\begin{subfigure}{0.49\linewidth}
\captionsetup{skip=0pt}
\caption{Indian Pines}
\begin{tikzpicture}
	\begin{semilogyaxis}[
		ymax=1e-4,
		grid=major,width = \textwidth,height=4cm,cycle list name=exotic,no markers,
		]
		\addplot+[line width = 2pt]table {titanizedminvol/results_hyper/mean_indian1.dat};
		\addplot+[line width = 2pt]table {titanizedminvol/results_hyper/mean_indian2.dat};
	\end{semilogyaxis}
\end{tikzpicture}
\end{subfigure}
\vspace{0.3cm}

\begin{subfigure}{0.49\linewidth}
\captionsetup{skip=0pt}
\caption{Pavia Uni}
\begin{tikzpicture}
	\begin{semilogyaxis}[
		ymax=1e-4,
		xtick={0,30,60,90},
		grid=major,width = \textwidth,height=4cm,cycle list name=exotic,no markers,
		]
		\addplot+[line width = 2pt]table {titanizedminvol/results_hyper/mean_paviau1.dat};
		\addplot+[line width = 2pt]table {titanizedminvol/results_hyper/mean_paviau2.dat};
	\end{semilogyaxis}
\end{tikzpicture}
\end{subfigure}
\begin{subfigure}{0.49\linewidth}
\captionsetup{skip=0pt}
\caption{San Diego}
\begin{tikzpicture}
	\begin{semilogyaxis}[
		ymax=1e-4,
		ymin=1e-5,
		xtick={0,40,80,120},
		grid=major,width = \textwidth,height=4cm,cycle list name=exotic,no markers,
		]
		\addplot+[line width = 2pt]table {titanizedminvol/results_hyper/mean_sandiego1.dat};
		\addplot+[line width = 2pt]table {titanizedminvol/results_hyper/mean_sandiego2.dat};
	\end{semilogyaxis}
\end{tikzpicture}
\end{subfigure}
\vspace{0.3cm}

\begin{subfigure}{0.49\linewidth}
\captionsetup{skip=0pt}
\caption{Terrain}
\begin{tikzpicture}
	\begin{semilogyaxis}[
		ymax=1e-4,
		grid=major,width = \textwidth,height=4cm,cycle list name=exotic,no markers,
		]
		\addplot+[line width = 2pt]table {titanizedminvol/results_hyper/mean_terrain1.dat};
		\addplot+[line width = 2pt]table {titanizedminvol/results_hyper/mean_terrain2.dat};
	\end{semilogyaxis}
\end{tikzpicture}
\end{subfigure}
\begin{subfigure}{0.49\linewidth}
\captionsetup{skip=0pt}
\caption{20News}
\begin{tikzpicture}
	\begin{semilogyaxis}[
		ymax=1e-3,
		grid=major,width = \textwidth,height=4cm,cycle list name=exotic,no markers,
		]
		\addplot+[line width = 2pt]table {titanizedminvol/results_doc/mean_20news1.dat};
		\addplot+[line width = 2pt]table {titanizedminvol/results_doc/mean_20news2.dat};
	\end{semilogyaxis}
\end{tikzpicture}
\end{subfigure}
\vspace{0.3cm}

\begin{subfigure}{0.49\linewidth}
\captionsetup{skip=0pt}
\caption{Reviews}
\begin{tikzpicture}
	\begin{semilogyaxis}[
		ymax=1e-3,
		xlabel={time (s)},
		grid=major,width = \textwidth,height=4cm,cycle list name=exotic,no markers,
		]
		\addplot+[line width = 2pt]table {titanizedminvol/results_doc/mean_reviews1.dat};
		\addplot+[line width = 2pt]table {titanizedminvol/results_doc/mean_reviews2.dat};
	\end{semilogyaxis}
\end{tikzpicture}
\end{subfigure}
\begin{subfigure}{0.49\linewidth}
\captionsetup{skip=0pt}
\caption{Sports}
\begin{tikzpicture}
	\begin{semilogyaxis}[
		ymax=1e-3,
		xlabel={time (s)},
		grid=major,width = \textwidth,height=4cm,cycle list name=exotic,no markers,
		]
		\addplot+[line width = 2pt]table {titanizedminvol/results_doc/mean_sports1.dat};
		\addplot+[line width = 2pt]table {titanizedminvol/results_doc/mean_sports2.dat};
	\end{semilogyaxis}
\end{tikzpicture}
\end{subfigure}
\caption[Evolution w.r.t.\ time of the error for different datasets]{Evolution w.r.t.\ time of the average of  $(f( W , H )-e_{\min})/\|X\|_F$ for the  different datasets.} 
\label{minvol:fig:avg}
\end{figure}
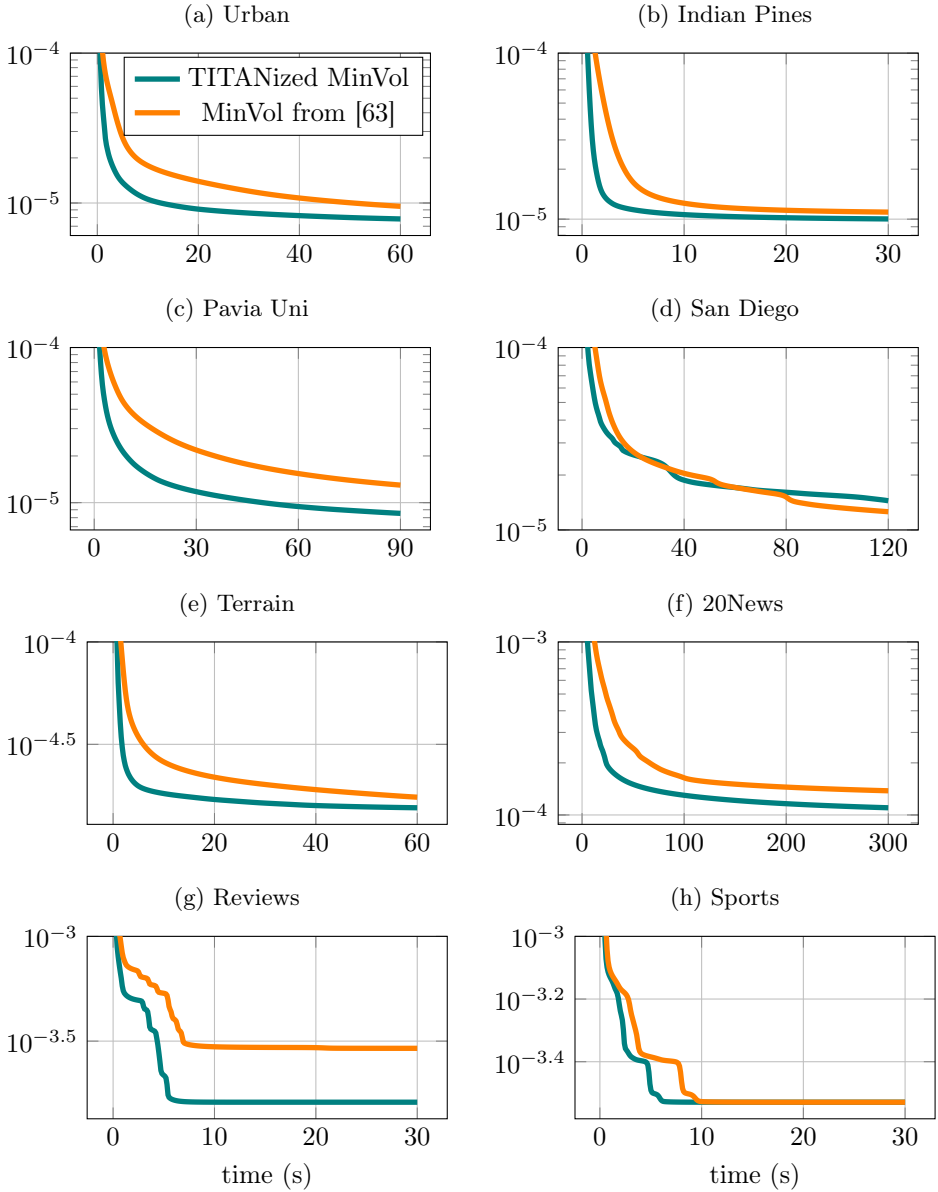

\begin{center}
\begin{table}[thb!] 
    \begin{center} 
    \normalsize
        \begin{tabular}{l||c|c|c} 
            Data set     & Our method's & wall-clock time 
            & 
            Saved \\
                        & lead time (s) & 
                for~\cite{leplat19}  &  wall-clock time \\ 
            \hline 
            Urban       &  $ 44$  & 60 & 73\% \\ 
            Indian Pines&  $ 25$  & 30 & 83\% \\
            Pavia Univ. &  $ 68$  & 90 & 76\% \\ 
            San Diego   &  NaN    & 120 & 0\% \\ 
            Terrain     &  $ 44$  & 60 & 73\% \\ 
            20News      &  $221$  & 300 & 74\% \\ 
            Reviews     &  $ 26$  & 30 & 80\% \\ 
            Sports      &  $ 15$  & 30 & 50\% \\ 
        \end{tabular}
    \end{center} 
\caption[TITANized MinVol's lead time]{TITANized MinVol's lead time over MinVol~\cite{leplat19} to obtain the same minimum error.}
\label{minvol:tab:leadtime}
\end{table} 
\end{center}

\begin{center}
\begin{table}[thb!]
    \centering
    \normalsize
    \begin{tabular}{c||c|c} 
        \multirow{2}{*}{Algorithm}  & \multicolumn{2}{c}{ranking} \\ \cline{2-3}
                                    & Hyperspectral unmixing    & Document clustering \\ \hline
        TITANized MinVol                  & $( 94, 6 ) $                & $( 55, 5 )$ \\ \hline
        MinVol \cite{leplat19}     & $(  6,94 )$                 & $(  5,55 )$ 
    \end{tabular}
    \caption[Ranking depending on the algorithm and the kind of data set]{Ranking among the different runs depending on the algorithm and the kind of data set}
    \label{minvol:tab:ranking}
\end{table}
\end{center}

The ranking among all the tests has been reported in~\Cref{minvol:tab:ranking}, where the $i$-th entry denotes how many times the corresponding algorithm was in the $i$-th place. We also reported in \Cref{minvol:tab:leadtime} TITANized MinVol's lead time over~\cite{leplat19} when the latter reaches its minimum error after the maximum allotted wall-clock time. The lead time is the time saved by TITANized MinVol to achieve the error of the method from~\cite{leplat19} using the maximum allotted wall-clock time. 
On average, TITANized MinVol is twice faster than \cite{leplat19}, with an average gain of wall-clock time above 50\%.

To summarize, our experimental results show that TITANized MinVol has a faster convergence speed and smaller final solutions than \cite{leplat19}. 

\section{Minimum-volume Nonnegative Matrix Completion}\label{minvol:sec:minvolnmc} 

Given a data matrix $X\in\R^{m\times n}$, there exist many scenarios where only a few entries of $X$ are observed, e.g., in recommender systems illustrated by the famous Netflix problem~\cite{koren2009matrix}. 
Recovering these missing entries is often tackled by assuming that the fully observed data follow a certain structure. 
If the structuring assumption is meaningful, by fitting a model that follows the same structure on the observed entries, it is possible to recover the missing entries; see, e.g., ~\cite{candes2010matrix,gross2011recovering,candes2012exact}. 
The low-rank assumption is meaningful in many scenarios~\cite{udell2019big}. If $X\in\R^{m\times n}$ is low-rank, we can express it as the product of two smaller matrices, $W\in\R^{m\times r}$ and $H\in\R^{r\times n}$, as $X=WH$ where $r\ll\min(m,n)$. 
Let us denote $\Omega\subseteq\{1,\dots,m\}\times\{1,\dots,n\}$ the set containing the indices of the observed entries in $X$. 
If the rank of $X$ is equal to $r$, we can look for $W\in\R^{m\times r}$ and $H\in\R^{r\times n}$ such that $X(i,j)=W(i,:)H(:,j)$ for all $(i,j)\in\Omega$. Then, for every missing entry at $(i,j)\in\overbar{\Omega}$, $X(i,j)$ can be estimated by computing $W(i,:)H(:,j)$. If $X$ is noisy and does not follow the low-rank assumption, it might still be relevant to approximate it through a low-rank structure, because low-rank matrix approximations can identify patterns in the data via the extraction of common features among data points.  

When the rank is unknown, a common tractable strategy is to minimize the nuclear norm, that is the sum of the singular values, of the estimation $\tilde{X}$ of $X$: 
\begin{equation*}
	\min_{\tilde{X}} \| \tilde{X} \|_*  
	\quad \text{ such that } \quad \P_\Omega(\tilde{X}) = \P_\Omega(X), 
\end{equation*}
where $\P_\Omega(Y)$ sets $Y(i,j)$ to zero if $(i,j)\notin\Omega$, or does not change it otherwise.

In this section, we consider the rank to be known, and our goal is not only to recover the missing entries in $X$, but also to recover the unique matrices $W$ and $H$ that generated the data $X=WH$. This could be useful in hyperspectral unmixing with missing data for instance, where the columns of $W$ are expected to be the spectral signatures of the underlying materials, and where the $j$-th column of $H$ contains the abundance in the $j$-th pixel of each extracted material. In this scenario, it is of course preferable to recover a unique set $(W,H)$. To perform this task, it is possible to first use a data completion algorithm, and then use a constrained matrix factorization algorithm to estimate the sought factors $W$ and $H$. Here, we focus on performing both tasks together, since estimating correctly $W$ and $H$ on $\Omega$ implies a correct recovery of the missing entries in $X = WH$.
We assume that the data and the factors are nonnegative, that is, $X\geq0$, $W\geq0$ and $H\geq0$, where $\geq$ is applied element wise. Hence, our goal is to perform NMF with missing data while recovering a unique decomposition. To do so, MinVol NMF is a relevant option, and its performances on matrix completion have never been explored before. In this chapter, we show that when correctly tuned, MinVol NMF performs well on the matrix completion task and is also able to retrieve the true underlying factors  using only a few observed entries. 

 
\subsection{Motivation}
\label{minvol:sec:NMC}

In this section, we justify the choice of the minimum-volume criterion for the task of nonnegative matrix completion. Matrix completion in general has been well studied, especially by the compressed sensing community~\cite{candes2012exact}. Among the techniques to perform matrix completion, the low-rank approach often arises, because the low-rank structure has been observed to be quite powerful in this setting, as it is able to identify hidden (linear) features in data. However, minimizing the rank of the estimation matrix while guaranteeing the equality constraints on the set of observed entries is NP-hard in general. 
A good convex relaxation that promotes low-rank structures is the nuclear norm minimization; see~\cite{recht2010guaranteed}. This is coming from the fact that the rank is the $\ell_0$ norm of the vector of the singular values, while the nuclear norm is the $\ell_1$ norm of this vector. Still, this requires to store the whole estimation $\tilde{X}$ of $X$, and it also becomes harder to impose additional structuring constraints. When the rank is known, we can fully exploit the low-rank structure by working with the low-rank factors $W$ and $H$ instead. It is then easier to add some structuring constraints on $W$ and $H$. Also, this allows one to deal with larger problems. 
Since 
$$
\|X\|_* = \min_{X=WH}\frac{1}{2}\left(\|W\|_F^2+\|H\|_F^2\right),
$$ 
\sloppy a good alternative to the nuclear norm regularization is then the regularizer $\frac{1}{2}\left(\|W\|_F^2+\|H\|_F^2\right)$~\cite{srebro2005maximum}. 
If the rank is unknown, an overestimated rank coupled with a proper penalization of $\frac{1}{2}\left(\|W\|_F^2+\|H\|_F^2\right)$ can yield state-of-the-art results. For example, in~\cite{rendle2022revisiting}, a properly tuned matrix factorization model using the above regularizer can outperform deep neural networks on recommendation systems. 
In~\cite{liu2017newtheory}, they showed that the sightly different regularizer $\|W\|_*+\frac{1}{2}\|H\|^2_F$ yields better results than $\frac{1}{2}\left(\|W\|_F^2+\|H\|_F^2\right)$, both with uniform or non-uniform samplings. Going back to our point of interest, it is interesting to observe that the MinVol regularizer provides more adaptability as a (non-convex) relaxation of the rank~\cite{leplat2019minimum}, since 
$\logdet(W^\top W+\delta I) = \sum_i\log(\sigma_i^2(W)+\delta).$ 
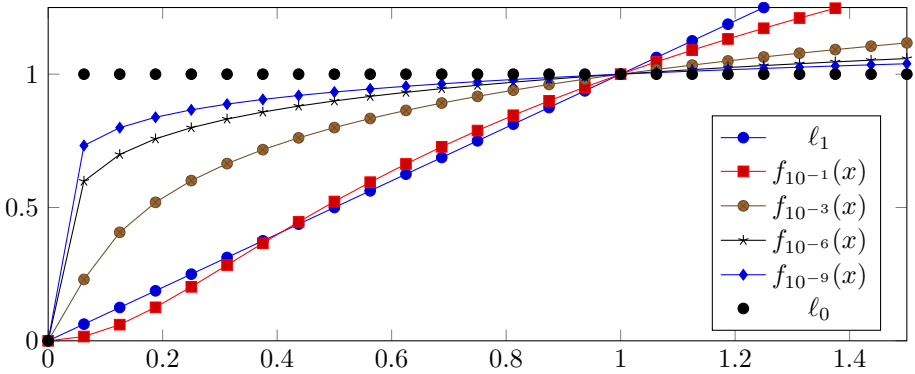
\begin{figure}[htb!]
\centering
\begin{tikzpicture}
	\begin{axis}[
		xmin=0, xmax=1.5, 
		ymin=0, ymax=1.25, 
		domain=0:1.5,  
		legend pos=south east,
		width=\textwidth,
		height=6cm,
		]
		\addplot    {x};
		\addlegendentry{$\ell_1$}
		\addplot    {(ln(x^2+0.1)-ln(0.1))/(ln(1+0.1)-ln(0.1))};
		\addlegendentry{$f_{10^{-1}}(x)$}
		\addplot    {(ln(x^2+1e-3)-ln(1e-3))/(ln(1+1e-3)-ln(1e-3))};
		\addlegendentry{$f_{10^{-3}}(x)$}
		\addplot    {(ln(x^2+1e-6)-ln(1e-6))/(ln(1+1e-6)-ln(1e-6))};
		\addlegendentry{$f_{10^{-6}}(x)$}
		\addplot    {(ln(x^2+1e-9)-ln(1e-9))/(ln(1+1e-9)-ln(1e-9))};
		\addlegendentry{$f_{10^{-9}}(x)$}
		\addplot[only marks] 	{min(x*1e9,1)};
		\addlegendentry{$\ell_0$}
	\end{axis}
\end{tikzpicture}
\caption[Approximation of the $\ell_0$ and $\ell_1$ norm via $f_\delta(x)$]{Function $f_\delta(x)=\frac{\ln(x^2+\delta)-\ln(\delta)}{\ln(1+\delta)-\ln(\delta)}$ for various values of $\delta$, along the $\ell_0$ and $\ell_1$ norm.}
\label{minvol:fig:minvolL0L1}
\end{figure}
As it can be seen in~\cref{minvol:fig:minvolL0L1}, $\logdet(W^\top W + \delta I)$ approximates a range of behaviors between the $\ell_0$ and the $\ell_1$ norms. In particular, as $\delta$ goes to zero, $\logdet(W^\top W + \delta I)$ converges to the $\ell_0$ norm of the vector of singular values of $X$, up to a constant factor.  Hence the MinVol criterion $\logdet(W^\top W+\delta I)$ is clearly a good candidate as a regularizer for NMC. 

Let us now propose two models to tackle NMC. 
The first one is to adapt~\eqref{minvol:eq:nonexactminvol} to the NMC problem, which yields
\begin{mini}
	{\scriptstyle W,H}{\frac{1}{2}\|\P_\Omega(X-WH)\|_F^2+\frac{\lambda}{2}\logdet(\Wt W +\delta I)}{\label{minvol:eq:nonexactminvolcompletion}}{}
	\addConstraint{W\in\Delta^{m \times r}}{,~H\in\R_+^{r \times n}.}{}
\end{mini} 
\Cref{minvol:th:uniqueminvol} does not extend to the case where some values are missing. If the matrix completion is not unique, then it is impossible to guarantee a unique recovery of the matrices $W$ and $H$. Hence, a trivial way to adapt \Cref{minvol:th:uniqueminvol} to missing values is to add the condition that matrix completion under MinVol NMF should be unique. However, better conditions than standard low-rank matrix completion theory under which solving~\eqref{minvol:eq:nonexactminvolcompletion} recovers a unique completion are, up to now, unknown.

The second one introduces a new variant of MinVol NMF which is not simplex structured. Inspired by the regularizer $\|W\|_*+\frac{1}{2}\|H\|^2_F$ and motivated by the link between the behavior of the nuclear norm and the MinVol criterion, here we consider $\logdet(W^\top W+\delta I)+\|H\|^2_F$ as a regularizer. The resulting new MinVol NMF adapted for NMC is 
\begin{mini}
	{\scriptstyle W,H}{\frac{1}{2}\|\P_\Omega(X\text{$-$}WH)\|_F^2+\frac{\lambda}{2}\logdet(\Wt W\text{$+$}\delta I)+\frac{\gamma}{2}\|H\|_F^2 \label{minvol:eq:newminvol}}{}{}
	\addConstraint{W\in\R_+^{m \times r} ,~H\in\R_+^{r \times n},}
\end{mini}
where $\lambda\geq0$ and $\gamma\geq0$ balance the regularizers.
Note that neither $W$ nor $H$ is simplex structured. The scaling ambiguity coupled with the volume penalization is counter balanced by the penalization of $\|H\|_F^2$. In fact, in the exact case and when $\delta=0$, every row of $H$ has the same norm at optimality. Consider a feasible $(W,H)$ for~\eqref{minvol:eq:newminvol} such that $X=WH$ and let ${f(D)=\frac{\lambda}{2}\logdet(D^{-1}W^\top WD^{-1})+\frac{\gamma}{2}\|DH\|_F^2}$ where $D=\diag(d_1,\dots,d_r)$ is a positive diagonal matrix that can be seen as the scaling ambiguity between $W$ and $H$. Nullifying the gradient of $f$ relatively to each $d_i$, we have that $d_i^2=\frac{\lambda}{\gamma\|H(i,:)\|_F^2}$, meaning that at optimality $\|H(i,:)\|_F^2=\frac{\lambda}{\gamma}$ for all $i$.



\subsection{Algorithms}
\label{minvol:sec:algorithms}

In~\Cref{minvol:sec:exp}, we compare NMF, MinVol~\eqref{minvol:eq:nonexactminvolcompletion} and new MinVol~\eqref{minvol:eq:newminvol}. For a fair comparison, these models are fit with the same algorithmic scheme, adapted from~\cite{vuthanh2021inertial}, which is an extrapolated alternating block majorization-minimization method already described in \Cref{minvol:sec:titanizedminvol}. Our adaptation is described in~\Cref{minvol:alg:generictitan}, where $\P_{\Delta^{m\times r}}$ (respectively $\P_{\R_+^{m\times r}}$)  projects a matrix of size $m\times r$ onto $\Delta^{m\times r}$ (respectively $\R_+^{m\times r}$). See~\cite{condat2016fast} for the details on the projection onto $\Delta^{m\times r}$. Essentially, the updates for $W$ and $H$ are several projected gradient descent steps, performed with a step size equal to the inverse of the Lipschitz constant. The updates for each model and each factor, as well as the corresponding Lipschitz constant, are given in~\Cref{minvol:tab:updatesW} and~\Cref{minvol:tab:updatesH}. The used Lipschitz constants are deliberately not tight. Consider the MinVol NMF update of $H$ for instance. 
Let $M\in\{0,1\}^{m\times n}$ be such that $M(i,j)=1$ if $(i,j)\in\Omega$, $M(i,j) = 0$ otherwise. A tighter Lipschitz constant is $\max_j \left\|W^\top\diag(M(:,j))W\right\|$; see the paragraph in \Cref{bssmf:sec:proposedalgo} on the choice of Lipschitz constant for the details.
We deliberately keep $\|W^\top W\|$ as it is less costly to compute and the additional cost might not be worth it. Moreover, if at least one column of $X$ is fully observed, then ${\max_j \left\|W^\top\diag(M(:,j))W\right\| = \|W^\top W\|} = \|W\|^2$. 

\begin{algorithm}[htb!]
\caption{Main algorithm scheme}
\label{minvol:alg:generictitan}
\SetKwInOut{Input}{input}
\DontPrintSemicolon
\Input{data matrix $X \in \mathbb{R}^{m \times n}$, initial factors $W \in \R_+^{m \times r}$  and $H\in \R_+^{r \times n}$}
$\alpha_1=\alpha_2=1,~\Wold=W,~\Hold=H$\;
\While{stopping criteria not satisfied\label{minvol:line:whileouterloop}}{
	\While{stopping criteria not satisfied\label{minvol:line:whileW}}{
		$\alpha_0=\alpha_1,\quad\alpha_1=\frac{1}{2}(1+\sqrt{1+4\alpha_0^2})$\;
		$\Wextra=W+\frac{\alpha_0-1}{\alpha_1}(W-\Wold)$\;
		$\Wold=W$\;
		Update $W$ according to~\Cref{minvol:tab:updatesW}\;
	}
	\While{stopping criteria not satisfied\label{minvol:line:whileH}}{
		$\alpha_0=\alpha_2,\quad\alpha_2=\frac{1}{2}(1+\sqrt{1+4\alpha_0^2})$\;
		$\Hextra=H+\frac{\alpha_0-1}{\alpha_2}(H-\Hold)$\;
		$\Hold=H$\;
		Update $H$ according to~\Cref{minvol:tab:updatesH}\;
	}
}
\end{algorithm}
\begin{table}[htb!]
	\centering
	\begin{tabular}{c||c}
		& Update \\ \hline &  \\
		MinVol & $\P_{\Delta^{m\times r}}\left(\Wextra-\frac{1}{L}\nabla_W\right)$ \\ &\\ \hline &  \\
		new MinVol~/~NMF (with $\lambda=0$) &  $\P_{\R_+^{m\times r}}\left(\Wextra-\frac{1}{L}\nabla_W\right)$
	\end{tabular}
\caption[Updates for $W$ according to the model]{Updates for $W$ according to the model, where $P=(\Wextra^\top\Wextra+\delta I)^{-1}$, $L=\|HH^\top+\lambda P\|$ and ${\nabla_W=\P_\Omega\left(\Wextra H-X\right)H^\top+\lambda WP}$.}
\label{minvol:tab:updatesW}
\end{table}
\begin{table}[hbt!]
	\centering
	\begin{tabular}{c||c|c}
		& $L$ & Update \\ \hline & & \\
		MinVol~/~NMF & $\|W^\top W\|$ & $\P_{\R_+^{r\times n}}\left(\Hextra-\frac{1}{L}\nabla_H\right)$ \\ &&\\ \hline & & \\
		new MinVol & $\|W^\top W+\gamma I\|$ & $\P_{\R_+^{r\times n}}\left(\frac{L-\gamma}{L}\Hextra-\frac{1}{L}\nabla_H\right)$
	\end{tabular}
	\caption[Updates for $H$ according to the model]{Updates for $H$ according to the model, where $\nabla_H=W^\top\P_\Omega\left(W\Hextra-X\right)$.}
	\label{minvol:tab:updatesH}
\end{table}


\subsection{Experiments}
\label{minvol:sec:exp}

The goal of this section is to highlight the performance of the MinVol criterion for NMC. 
All experiments are run with Julia on a PC with an Intel(R) Core(TM) i7-9750H CPU @ 2.60GHz and 16GB RAM. All displayed measurements are averaged out of 20 runs. The code is available at \url{https://gitlab.com/vuthanho/minvol-nmc}. 
The compared models are NMF (to provide a baseline of a non-regularized model), MinVol~\eqref{minvol:eq:nonexactminvolcompletion}, and the new proposed MinVol~\eqref{minvol:eq:newminvol}. For all models, the stopping criteria of the \textbf{while} loop in~\cref{minvol:line:whileouterloop} is just a number of outer iterations equal to 50, and the stopping criteria of the two \textbf{while} loops in~\cref{minvol:line:whileW,minvol:line:whileH} is a number of inner iterations equal to 20.
All models are also initialized with the same warm start $(W_0,H_0)$, which is the output of 500 iterations of NMF where the columns of $W$ are simplex-structured. In this setting, all methods converge. For both MinVols, $\lambda$ is first set to $\frac{\max(\|\P_\Omega(X-W_0H_0)\|_F^2,10^{-6})}{|\logdet(W_0^\top W_0+\delta I)|}$. For the new proposed MinVol, $\gamma$ is first set to $0.01\frac{\max(\|\P_\Omega(X-W_0H_0)\|_F^2,10^{-6})}{\|H_0\|_F^2}$. On the hyperparameters $\lambda$ and $\gamma$, we adapt the automatic tuning method developed in~\cite{nguyen2024towards}. The automatic tuning does not introduce a significant additional cost and is triggered when the difference between the current and the last objective values divided by $\|\P_\Omega(X)\|_F^2$ is less than $10^{-3}$. 

\paragraph{First experiment: noiseless synthetic data} 
The first experiment focuses on both data completion and recovery of the exact generating factors in a noiseless case. For this experiment, for a given rank $r$, we randomly generate two factors $(W,H)=[0,1]^{200\times r}\times[0,1]^{r\times200}$ following a uniform distribution. Then, $80\%$ random values of $H$ are set to zeros. This is a reasonable assumption in real scenarios such as hyperspectral unmixing. For the explored range of ranks, this will provide almost surely a sufficiently scattered $H$. Then, we generate the full data matrix $X$ simply by computing $WH$. The average of the elements of $X$ is always set to 1, dividing $X$ by its average. Finally, we create the observed data $\tilde{X}$ by removing a certain percentage of the entries in $X$. We vary the rank from 5 to 10, and the percentage of missing values from 80\% to 90\%. We report the root-mean-squared error (RMSE) of the missing values according to~\cref{minvol:def:rmse} and the maximum subspace angle between the factor $W$ that took part in generating the data $X$ and its estimation $\tilde{W}$ according to~\cref{minvol:def:angle}.

\begin{definition}[RMSE]
\label{minvol:def:rmse}
	The RMSE on the unobserved set $\overbar{\Omega}$ is defined as follows \\$${\text{RMSE}(\tilde{X},WH)=\sqrt{\frac{1}{|\overbar{\Omega}|} \| \P_{\overbar{\Omega}} (\tilde{X}-WH)\|^2_F}.}$$
\end{definition}

\begin{definition}[Subspace angle \cite{bjorck1973numerical,wedin2006angles}] 
\label{minvol:def:angle}
	Let $USV$ and $\tilde{U}\tilde{S}\tilde{V}$ respectively be the singular value decomposition of $W$ and $\tilde{W}$. Then the angle between the two subspaces specified by the columns of $W$ and $\tilde{W}$ is defined as follows
	$$\text{Angle}(W,\tilde{W})=\arcsin(\min(1,\|\tilde{U} - UU^\top\tilde{U}\|_ 2)).$$ 
\end{definition}

\def\step{0.35}
\def\nwidth{11}
\def\nheight{6}
\def\maxrmse{1}

\begin{figure}[htbp!]
\centering
	\begin{subfigure}{0.45\textwidth}
	\centering
	\caption{NMF}\vspace{-0.35cm}
		\loaddata{minvolnmc/data/mean_rmse_NMF.txt}
\begin{tikzpicture}[
]
\coordinate (grid_start) at (-{floor(0.5*\nwidth)*\step} , -{floor(0.5*\nheight)*\step});
\coordinate (grid_end) at ({ceil(0.5*\nwidth)*\step} , {ceil(0.5*\nheight)*\step});
\foreach \k in {1, ..., 6} \draw let   
    \p{a}=(grid_start), 
    \p{b}=(grid_end) in
    ( \x{a} + 2*\k*\step cm - 1.5*\step cm,\y{a}) -- (\x{a} + 2*\k*\step cm - 1.5*\step cm,\y{a} -1pt) node[anchor=north] {$\pgfmathparse{80+(\k-1)*2}\pgfmathprintnumber{\pgfmathresult}$};
\node (xaxis) [yshift=-floor(0.5*(\nheight))*\step cm -0.6 cm] {missing values (\%)};
\foreach \k in {1, ..., 6} \draw let   
    \p{a}=(grid_start), 
    \p{b}=(grid_end) in
    (\x{a} , \y{a} + \k*\step cm -0.5*\step cm) -- (\x{a} -1pt, \y{a} + \k*\step cm -0.5*\step cm) node[anchor=east] {$\pgfmathparse{4+\k}\pgfmathprintnumber{\pgfmathresult}$};
 \node (yaxis) [xshift=-floor(0.5*(\nwidth))*\step cm-0.6 cm,rotate=90] {rank};

\foreach \row [count=\n] in 
\loadeddata{
  \foreach \col [evaluate=\col as \scaledcol using {min(100,100/\maxrmse*\col)}] [count=\m] in \row {
    \fill [black!\scaledcol!white] (-{floor(0.5*\nwidth)*\step + (\m-1)*\step},{ceil(0.5*\nheight)*\step - (\n-1)*\step}) rectangle ++(\step,-\step);
  }
}
\draw[step=\step,gray,very thin] (grid_start) grid (grid_end);
\node[anchor=west] (colormap) [ xshift = ceil(0.5*\nwidth)*\step cm ] {\pgfplotscolorbardrawstandalone[ 
   colormap={blackwhite}{color=(white) color=(black)},
   colorbar,
   point meta min=0,
   point meta max=1,
   colorbar style={
   	ytick style ={color = white},
       width = \step cm,
       height= \nheight*\step cm,
       ytick={0,0.25,...,1},
   	yticklabels={0,,0.5,,1},
   }]};   
\end{tikzpicture}
	\end{subfigure}
	\hspace{0.5cm}
	\begin{subfigure}{0.45\textwidth}
	\centering
	\caption{MinVol}\vspace{-0.35cm}
		\loaddata{minvolnmc/data/mean_rmse_Minvol.txt}
\begin{tikzpicture}[
]
\coordinate (grid_start) at (-{floor(0.5*\nwidth)*\step} , -{floor(0.5*\nheight)*\step});
\coordinate (grid_end) at ({ceil(0.5*\nwidth)*\step} , {ceil(0.5*\nheight)*\step});
\foreach \k in {1, ..., 6} \draw let   
    \p{a}=(grid_start), 
    \p{b}=(grid_end) in
    ( \x{a} + 2*\k*\step cm - 1.5*\step cm,\y{a}) -- (\x{a} + 2*\k*\step cm - 1.5*\step cm,\y{a} -1pt) node[anchor=north] {$\pgfmathparse{80+(\k-1)*2}\pgfmathprintnumber{\pgfmathresult}$};
\node (xaxis) [yshift=-floor(0.5*(\nheight))*\step cm -0.6 cm] {missing values (\%)};
\foreach \k in {1, ..., 6} \draw let   
    \p{a}=(grid_start), 
    \p{b}=(grid_end) in
    (\x{a} , \y{a} + \k*\step cm -0.5*\step cm) -- (\x{a} -1pt, \y{a} + \k*\step cm -0.5*\step cm) node[anchor=east] {$\pgfmathparse{4+\k}\pgfmathprintnumber{\pgfmathresult}$};
\node (yaxis) [xshift=-floor(0.5*(\nwidth))*\step cm-0.6 cm,rotate=90] {rank};

\foreach \row [count=\n] in 
\loadeddata{
  \foreach \col [evaluate=\col as \scaledcol using {min(100,100/\maxrmse*\col)}] [count=\m] in \row {
    \fill [black!\scaledcol!white] (-{floor(0.5*\nwidth)*\step + (\m-1)*\step},{ceil(0.5*\nheight)*\step - (\n-1)*\step}) rectangle ++(\step,-\step);
  }
}
\draw[step=\step,gray,very thin] (grid_start) grid (grid_end);
\node[anchor=west] (colormap) [ xshift = ceil(0.5*\nwidth)*\step cm ] {\pgfplotscolorbardrawstandalone[ 
   colormap={blackwhite}{color=(white) color=(black)},
   colorbar,
   point meta min=0,
   point meta max=1,
   colorbar style={
   	ytick style ={color = white},
       width = \step cm,
       height= \nheight*\step cm,
       ytick={0,0.25,...,1},
   	yticklabels={0,,0.5,,1},
   }]};
\end{tikzpicture}
	\end{subfigure}

	\begin{subfigure}{0.45\textwidth}
	\centering
	\caption{auto MinVol}\vspace{-0.35cm}
		\loaddata{minvolnmc/data/mean_rmse_auto_Minvol.txt}
\begin{tikzpicture}[
]
\coordinate (grid_start) at (-{floor(0.5*\nwidth)*\step} , -{floor(0.5*\nheight)*\step});
\coordinate (grid_end) at ({ceil(0.5*\nwidth)*\step} , {ceil(0.5*\nheight)*\step});
\foreach \k in {1, ..., 6} \draw let   
    \p{a}=(grid_start), 
    \p{b}=(grid_end) in
    ( \x{a} + 2*\k*\step cm - 1.5*\step cm,\y{a}) -- (\x{a} + 2*\k*\step cm - 1.5*\step cm,\y{a} -1pt) node[anchor=north] {$\pgfmathparse{80+(\k-1)*2}\pgfmathprintnumber{\pgfmathresult}$};
\node (xaxis) [yshift=-floor(0.5*(\nheight))*\step cm -0.6 cm] {missing values (\%)};
\foreach \k in {1, ..., 6} \draw let   
    \p{a}=(grid_start), 
    \p{b}=(grid_end) in
    (\x{a} , \y{a} + \k*\step cm -0.5*\step cm) -- (\x{a} -1pt, \y{a} + \k*\step cm -0.5*\step cm) node[anchor=east] {$\pgfmathparse{4+\k}\pgfmathprintnumber{\pgfmathresult}$};
\node (yaxis) [xshift=-floor(0.5*(\nwidth))*\step cm-0.6 cm,rotate=90] {rank};

\foreach \row [count=\n] in 
\loadeddata{
  \foreach \col [evaluate=\col as \scaledcol using {min(100,100/\maxrmse*\col)}] [count=\m] in \row {
    \fill [black!\scaledcol!white] (-{floor(0.5*\nwidth)*\step + (\m-1)*\step},{ceil(0.5*\nheight)*\step - (\n-1)*\step}) rectangle ++(\step,-\step);
  }
}
\draw[step=\step,gray,very thin] (grid_start) grid (grid_end);
\node[anchor=west] (colormap) [ xshift = ceil(0.5*\nwidth)*\step cm ] {\pgfplotscolorbardrawstandalone[ 
   colormap={blackwhite}{color=(white) color=(black)},
   colorbar,
   point meta min=0,
   point meta max=1,
   colorbar style={
   	ytick style ={color = white},
       width = \step cm,
       height= \nheight*\step cm,
       ytick={0,0.25,...,1},
   	yticklabels={0,,0.5,,1},
   }]};
\end{tikzpicture}
	\end{subfigure}
	\hspace{0.5cm}
	\begin{subfigure}{0.45\textwidth}
	\centering
	\caption{new MinVol}\vspace{-0.35cm}
		\loaddata{minvolnmc/data/mean_rmse_auto_Minvol_penFro.txt}
\begin{tikzpicture}[
]
\coordinate (grid_start) at (-{floor(0.5*\nwidth)*\step} , -{floor(0.5*\nheight)*\step});
\coordinate (grid_end) at ({ceil(0.5*\nwidth)*\step} , {ceil(0.5*\nheight)*\step});
\foreach \k in {1, ..., 6} \draw let   
    \p{a}=(grid_start), 
    \p{b}=(grid_end) in
    ( \x{a} + 2*\k*\step cm - 1.5*\step cm,\y{a}) -- (\x{a} + 2*\k*\step cm - 1.5*\step cm,\y{a} -1pt) node[anchor=north] {$\pgfmathparse{80+(\k-1)*2}\pgfmathprintnumber{\pgfmathresult}$};
\node (xaxis) [yshift=-floor(0.5*(\nheight))*\step cm -0.6 cm] {missing values (\%)};
\foreach \k in {1, ..., 6} \draw let   
    \p{a}=(grid_start), 
    \p{b}=(grid_end) in
    (\x{a} , \y{a} + \k*\step cm -0.5*\step cm) -- (\x{a} -1pt, \y{a} + \k*\step cm -0.5*\step cm) node[anchor=east] {$\pgfmathparse{4+\k}\pgfmathprintnumber{\pgfmathresult}$};
\node (yaxis) [xshift=-floor(0.5*(\nwidth))*\step cm-0.6 cm,rotate=90] {rank};

\foreach \row [count=\n] in 
\loadeddata{
  \foreach \col [evaluate=\col as \scaledcol using {min(100,100/\maxrmse*\col)}] [count=\m] in \row {
    \fill [black!\scaledcol!white] (-{floor(0.5*\nwidth)*\step + (\m-1)*\step},{ceil(0.5*\nheight)*\step - (\n-1)*\step}) rectangle ++(\step,-\step);
  }
}
\draw[step=\step,gray,very thin] (grid_start) grid (grid_end);
\node[anchor=west] (colormap) [ xshift = ceil(0.5*\nwidth)*\step cm ] {\pgfplotscolorbardrawstandalone[ 
   colormap={blackwhite}{color=(white) color=(black)},
   colorbar,
   point meta min=0,
   point meta max=1,
   colorbar style={
   	ytick style ={color = white},
       width = \step cm,
       height= \nheight*\step cm,
       ytick={0,0.25,...,1},
   	yticklabels={0,,0.5,,1},
   }]}; 
\end{tikzpicture}
	\end{subfigure}
\caption[Average RMSE according to the rank $r$ and to the percentage of missing values]{Average RMSE according to the rank $r$ and to the percentage of missing values over 20 runs.}
\label{minvol:fig:noiseless_rmse}
\end{figure}

\begin{figure}[hbp!]
	\centering
    \begin{subfigure}{0.45\textwidth}
	\centering
    \caption{NMF}\vspace{-0.35cm}
        \loaddata{minvolnmc/data/mean_angle_NMF.txt}
\begin{tikzpicture}[
]
\coordinate (grid_start) at (-{floor(0.5*\nwidth)*\step} , -{floor(0.5*\nheight)*\step});
\coordinate (grid_end) at ({ceil(0.5*\nwidth)*\step} , {ceil(0.5*\nheight)*\step});
\foreach \k in {1, ..., 6} \draw let   
    \p{a}=(grid_start), 
    \p{b}=(grid_end) in
    ( \x{a} + 2*\k*\step cm - 1.5*\step cm,\y{a}) -- (\x{a} + 2*\k*\step cm - 1.5*\step cm,\y{a} -1pt) node[anchor=north] {$\pgfmathparse{80+(\k-1)*2}\pgfmathprintnumber{\pgfmathresult}$};
\node (xaxis) [yshift=-floor(0.5*(\nheight))*\step cm -0.6 cm] {missing values (\%)};
\foreach \k in {1, ..., 6} \draw let   
    \p{a}=(grid_start), 
    \p{b}=(grid_end) in
    (\x{a} , \y{a} + \k*\step cm -0.5*\step cm) -- (\x{a} -1pt, \y{a} + \k*\step cm -0.5*\step cm) node[anchor=east] {$\pgfmathparse{4+\k}\pgfmathprintnumber{\pgfmathresult}$};
 \node (yaxis) [xshift=-floor(0.5*(\nwidth))*\step cm-0.6 cm,rotate=90] {rank};

\foreach \row [count=\n] in 
\loadeddata{
  \foreach \col [evaluate=\col as \scaledcol using {100/90*\col}] [count=\m] in \row {
    \fill [black!\scaledcol!white] (-{floor(0.5*\nwidth)*\step + (\m-1)*\step},{ceil(0.5*\nheight)*\step - (\n-1)*\step}) rectangle ++(\step,-\step);
  }
}
\draw[step=\step,gray,very thin] (grid_start) grid (grid_end);
\node[anchor=west] (colormap) [ xshift = ceil(0.5*\nwidth)*\step cm ] {\pgfplotscolorbardrawstandalone[ 
	colormap={blackwhite}{color=(white) color=(black)},
	colorbar,
	point meta min=0,
	point meta max=90,
	colorbar style={
		ytick style ={color = white},
		width = \step cm,
		height= \nheight*\step cm,
		ytick={0,30,60,90}}]};   
\end{tikzpicture}
    \end{subfigure}
	\hspace{0.5cm}
    \begin{subfigure}{0.45\textwidth}
	\centering
    \caption{MinVol}\vspace{-0.35cm}
        \loaddata{minvolnmc/data/mean_angle_Minvol.txt}
\begin{tikzpicture}[
]
\coordinate (grid_start) at (-{floor(0.5*\nwidth)*\step} , -{floor(0.5*\nheight)*\step});
\coordinate (grid_end) at ({ceil(0.5*\nwidth)*\step} , {ceil(0.5*\nheight)*\step});
\foreach \k in {1, ..., 6} \draw let   
    \p{a}=(grid_start), 
    \p{b}=(grid_end) in
    ( \x{a} + 2*\k*\step cm - 1.5*\step cm,\y{a}) -- (\x{a} + 2*\k*\step cm - 1.5*\step cm,\y{a} -1pt) node[anchor=north] {$\pgfmathparse{80+(\k-1)*2}\pgfmathprintnumber{\pgfmathresult}$};
\node (xaxis) [yshift=-floor(0.5*(\nheight))*\step cm -0.6 cm] {missing values (\%)};
\foreach \k in {1, ..., 6} \draw let   
    \p{a}=(grid_start), 
    \p{b}=(grid_end) in
    (\x{a} , \y{a} + \k*\step cm -0.5*\step cm) -- (\x{a} -1pt, \y{a} + \k*\step cm -0.5*\step cm) node[anchor=east] {$\pgfmathparse{4+\k}\pgfmathprintnumber{\pgfmathresult}$};
\node (yaxis) [xshift=-floor(0.5*(\nwidth))*\step cm-0.6 cm,rotate=90] {rank};

\foreach \row [count=\n] in 
\loadeddata{
  \foreach \col [evaluate=\col as \scaledcol using {100/90*\col}] [count=\m] in \row {
    \fill [black!\scaledcol!white] (-{floor(0.5*\nwidth)*\step + (\m-1)*\step},{ceil(0.5*\nheight)*\step - (\n-1)*\step}) rectangle ++(\step,-\step);
  }
}
\draw[step=\step,gray,very thin] (grid_start) grid (grid_end);
\node[anchor=west] (colormap) [ xshift = ceil(0.5*\nwidth)*\step cm ] {\pgfplotscolorbardrawstandalone[ 
	colormap={blackwhite}{color=(white) color=(black)},
	colorbar,
	point meta min=0,
	point meta max=90,
	colorbar style={
		ytick style ={color = white},
		width = \step cm,
		height= \nheight*\step cm,
		ytick={0,30,60,90}}]};   
\end{tikzpicture}
    \end{subfigure}

    \begin{subfigure}{0.45\textwidth}
	\centering
    \caption{auto MinVol}\vspace{-0.35cm}
        \loaddata{minvolnmc/data/mean_angle_auto_Minvol.txt}
\begin{tikzpicture}[
]
\coordinate (grid_start) at (-{floor(0.5*\nwidth)*\step} , -{floor(0.5*\nheight)*\step});
\coordinate (grid_end) at ({ceil(0.5*\nwidth)*\step} , {ceil(0.5*\nheight)*\step});
\foreach \k in {1, ..., 6} \draw let   
    \p{a}=(grid_start), 
    \p{b}=(grid_end) in
    ( \x{a} + 2*\k*\step cm - 1.5*\step cm,\y{a}) -- (\x{a} + 2*\k*\step cm - 1.5*\step cm,\y{a} -1pt) node[anchor=north] {$\pgfmathparse{80+(\k-1)*2}\pgfmathprintnumber{\pgfmathresult}$};
\node (xaxis) [yshift=-floor(0.5*(\nheight))*\step cm -0.6 cm] {missieng values (\%)};
\foreach \k in {1, ..., 6} \draw let   
    \p{a}=(grid_start), 
    \p{b}=(grid_end) in
    (\x{a} , \y{a} + \k*\step cm -0.5*\step cm) -- (\x{a} -1pt, \y{a} + \k*\step cm -0.5*\step cm) node[anchor=east] {$\pgfmathparse{4+\k}\pgfmathprintnumber{\pgfmathresult}$};
\node (yaxis) [xshift=-floor(0.5*(\nwidth))*\step cm-0.6 cm,rotate=90] {rank};

\foreach \row [count=\n] in 
\loadeddata{
  \foreach \col [evaluate=\col as \scaledcol using {100/90*\col}] [count=\m] in \row {
    \fill [black!\scaledcol!white] (-{floor(0.5*\nwidth)*\step + (\m-1)*\step},{ceil(0.5*\nheight)*\step - (\n-1)*\step}) rectangle ++(\step,-\step);
  }
}
\draw[step=\step,gray,very thin] (grid_start) grid (grid_end);
\node[anchor=west] (colormap) [ xshift = ceil(0.5*\nwidth)*\step cm ] {\pgfplotscolorbardrawstandalone[ 
   colormap={blackwhite}{color=(white) color=(black)},
   colorbar,
   point meta min=0,
   point meta max=90,
   colorbar style={
   	ytick style ={color = white},
       width = \step cm,
       height= \nheight*\step cm,
       ytick={0,30,60,90}}]};  
\end{tikzpicture}
    \end{subfigure}
	\hspace{0.5cm}
    \begin{subfigure}{0.45\textwidth}
	\centering
    \caption{new MinVol}\vspace{-0.35cm}
        \loaddata{minvolnmc/data/mean_angle_auto_Minvol_penFro.txt}
\begin{tikzpicture}[
]
\coordinate (grid_start) at (-{floor(0.5*\nwidth)*\step} , -{floor(0.5*\nheight)*\step});
\coordinate (grid_end) at ({ceil(0.5*\nwidth)*\step} , {ceil(0.5*\nheight)*\step});
\foreach \k in {1, ..., 6} \draw let   
    \p{a}=(grid_start), 
    \p{b}=(grid_end) in
    ( \x{a} + 2*\k*\step cm - 1.5*\step cm,\y{a}) -- (\x{a} + 2*\k*\step cm - 1.5*\step cm,\y{a} -1pt) node[anchor=north] {$\pgfmathparse{80+(\k-1)*2}\pgfmathprintnumber{\pgfmathresult}$};
\node (xaxis) [yshift=-floor(0.5*(\nheight))*\step cm -0.6 cm] {missing values (\%)};
\foreach \k in {1, ..., 6} \draw let   
    \p{a}=(grid_start), 
    \p{b}=(grid_end) in
    (\x{a} , \y{a} + \k*\step cm -0.5*\step cm) -- (\x{a} -1pt, \y{a} + \k*\step cm -0.5*\step cm) node[anchor=east] {$\pgfmathparse{4+\k}\pgfmathprintnumber{\pgfmathresult}$};
\node (yaxis) [xshift=-floor(0.5*(\nwidth))*\step cm-0.6 cm,rotate=90] {rank};

\foreach \row [count=\n] in 
\loadeddata{
  \foreach \col [evaluate=\col as \scaledcol using {100/90*\col}] [count=\m] in \row {
    \fill [black!\scaledcol!white] (-{floor(0.5*\nwidth)*\step + (\m-1)*\step},{ceil(0.5*\nheight)*\step - (\n-1)*\step}) rectangle ++(\step,-\step);
  }
}
\draw[step=\step,gray,very thin] (grid_start) grid (grid_end);
\node[anchor=west] (colormap) [ xshift = ceil(0.5*\nwidth)*\step cm ] {\pgfplotscolorbardrawstandalone[ 
	colormap={blackwhite}{color=(white) color=(black)},
	colorbar,
	point meta min=0,
	point meta max=90,
	colorbar style={
		ytick style ={color = white},
		width = \step cm,
		height= \nheight*\step cm,
		ytick={0,30,60,90}}]};    
\end{tikzpicture}
    \end{subfigure}
\caption[Average angle according to the rank $r$ and to the percentage of missing values]{Average angle according to the rank $r$ and to the percentage of missing values over 20 runs.}
\label{minvol:fig:noiseless_angle}
\end{figure}
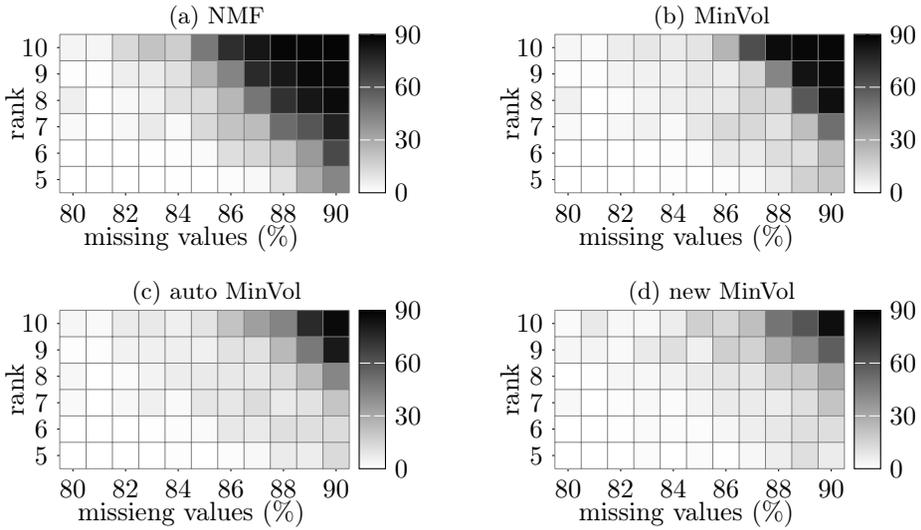

The RMSEs are reported in~\cref{minvol:fig:noiseless_rmse} and the subspace angles in~\cref{minvol:fig:noiseless_angle}. MinVol NMF coupled with the proposed auto-tuning proposed in~\cite{nguyen2024towards} clearly outperforms the vanilla MinVol NMF with a fixed $\lambda$. The auto-tuned MinVol NMF is itself outperformed by our new proposed variant of MinVol NMF. For $90\%$ missing values and a rank equal to 10 for instance, the average RMSE of the auto-tuned Minvol is 0.52 while it is 0.41 for the new MinVol.

\paragraph{Second experiment: noisy synthetic data} We keep the same settings as in the first experiment, while fixing the rank to 10, and adding some uniformly distributed noise. The noise level corresponds to the RMSE between the clean data and the noisy data. We vary the noise level from 0 to 1 and the percentage of missing values from $80\%$ to $90\%$. We report the RMSE in~\cref{minvol:fig:noisy_rmse}. It is not necessary to report the subspace angle since it is degrading too fast. Perfect matrix completion is a necessary condition to retrieve a low subspace angle, which is already not possible starting from a noise level equal to 0.2. Results in~\cref{minvol:fig:noisy_rmse} show that our proposed variant of MinVol NMF is more consistent relatively to the percentage of missing values and more precise than vanilla MinVol NMF in the presence of noise.

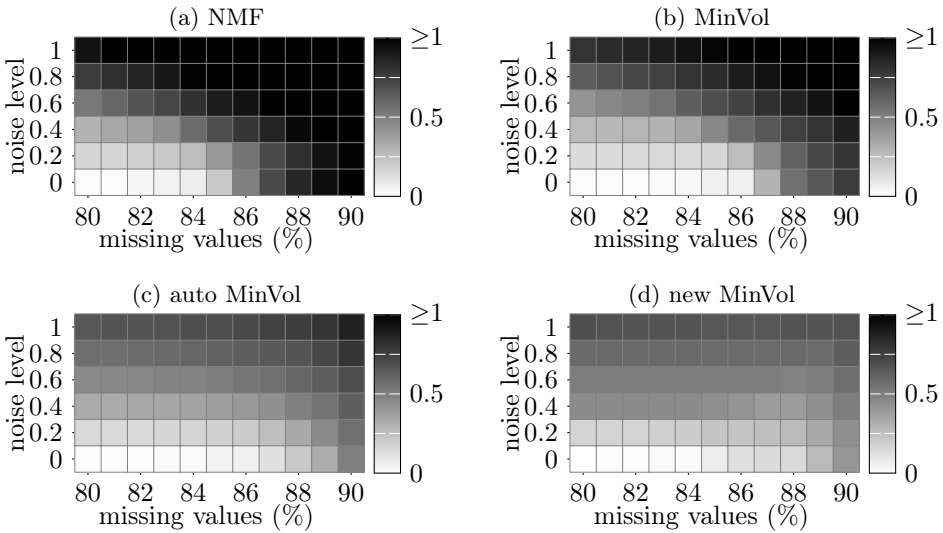
\begin{figure}[htbp!]
	\centering
	\begin{subfigure}{0.45\textwidth}
		\centering
		\caption{NMF}\vspace{-0.35cm}
		\loaddata{minvolnmc/data/noisy_mean_rmse_NMF.txt}
\begin{tikzpicture}[
]
\coordinate (grid_start) at (-{floor(0.5*\nwidth)*\step} , -{floor(0.5*\nheight)*\step});
\coordinate (grid_end) at ({ceil(0.5*\nwidth)*\step} , {ceil(0.5*\nheight)*\step});
\foreach \k in {1, ..., 6} \draw let   
    \p{a}=(grid_start), 
    \p{b}=(grid_end) in
    ( \x{a} + 2*\k*\step cm - 1.5*\step cm,\y{a}) -- (\x{a} + 2*\k*\step cm - 1.5*\step cm,\y{a} -1pt) node[anchor=north] {$\pgfmathparse{80+(\k-1)*2}\pgfmathprintnumber{\pgfmathresult}$};
\node (xaxis) [yshift=-floor(0.5*(\nheight))*\step cm -0.6 cm] {missing values (\%)};
\foreach \k in {1, ..., 6} \draw let   
    \p{a}=(grid_start), 
    \p{b}=(grid_end) in
    (\x{a} , \y{a} + \k*\step cm -0.5*\step cm) -- (\x{a} -1pt, \y{a} + \k*\step cm -0.5*\step cm) node[anchor=east] {$\pgfmathparse{0+(\k-1)*0.2}\pgfmathprintnumber{\pgfmathresult}$};
\node (yaxis) [xshift=-floor(0.5*(\nwidth))*\step cm-0.8 cm,rotate=90] {noise level};

\foreach \row [count=\n] in 
\loadeddata{
  \foreach \col [evaluate=\col as \scaledcol using {min(100,100/\maxrmse*\col)}] [count=\m] in \row {
    \fill [black!\scaledcol!white] (-{floor(0.5*\nwidth)*\step + (\m-1)*\step},{ceil(0.5*\nheight)*\step - (\n-1)*\step}) rectangle ++(\step,-\step);
  }
}
\draw[step=\step,gray,very thin] (grid_start) grid (grid_end);
\node[anchor=west] (colormap) [ xshift = ceil(0.5*\nwidth)*\step cm ] {\pgfplotscolorbardrawstandalone[ 
	colormap={blackwhite}{color=(white) color=(black)},
	colorbar,
	point meta min=0,
	point meta max=1,
	colorbar style={
		ytick style ={color = white},
		width = \step cm,
		height= \nheight*\step cm,
		ytick={0,0.25,...,1},
		yticklabels={0,,0.5,,$\geq$1},
	}]};  
\end{tikzpicture}
	\end{subfigure}
	\hspace{0.5cm}
	\begin{subfigure}{0.45\textwidth}
		\centering
		\caption{MinVol}\vspace{-0.35cm}
		\loaddata{minvolnmc/data/noisy_mean_rmse_Minvol.txt}
\begin{tikzpicture}[
]
\coordinate (grid_start) at (-{floor(0.5*\nwidth)*\step} , -{floor(0.5*\nheight)*\step});
\coordinate (grid_end) at ({ceil(0.5*\nwidth)*\step} , {ceil(0.5*\nheight)*\step});
\foreach \k in {1, ..., 6} \draw let   
    \p{a}=(grid_start), 
    \p{b}=(grid_end) in
    ( \x{a} + 2*\k*\step cm - 1.5*\step cm,\y{a}) -- (\x{a} + 2*\k*\step cm - 1.5*\step cm,\y{a} -1pt) node[anchor=north] {$\pgfmathparse{80+(\k-1)*2}\pgfmathprintnumber{\pgfmathresult}$};
\node (xaxis) [yshift=-floor(0.5*(\nheight))*\step cm -0.6 cm] {missing values (\%)};
\foreach \k in {1, ..., 6} \draw let   
    \p{a}=(grid_start), 
    \p{b}=(grid_end) in
    (\x{a} , \y{a} + \k*\step cm -0.5*\step cm) -- (\x{a} -1pt, \y{a} + \k*\step cm -0.5*\step cm) node[anchor=east] {$\pgfmathparse{0+(\k-1)*0.2}\pgfmathprintnumber{\pgfmathresult}$};
    \node (yaxis) [xshift=-floor(0.5*(\nwidth))*\step cm-0.8 cm,rotate=90] {noise level};

\foreach \row [count=\n] in 
\loadeddata{
  \foreach \col [evaluate=\col as \scaledcol using {min(100,100/\maxrmse*\col)}] [count=\m] in \row {
    \fill [black!\scaledcol!white] (-{floor(0.5*\nwidth)*\step + (\m-1)*\step},{ceil(0.5*\nheight)*\step - (\n-1)*\step}) rectangle ++(\step,-\step);
  }
}
\draw[step=\step,gray,very thin] (grid_start) grid (grid_end);
\node[anchor=west] (colormap) [ xshift = ceil(0.5*\nwidth)*\step cm ] {\pgfplotscolorbardrawstandalone[ 
	colormap={blackwhite}{color=(white) color=(black)},
	colorbar,
	point meta min=0,
	point meta max=1,
	colorbar style={
		ytick style ={color = white},
		width = \step cm,
		height= \nheight*\step cm,
		ytick={0,0.25,...,1},
		yticklabels={0,,0.5,,$\geq$1},
	}]};  
\end{tikzpicture}
	\end{subfigure}

	\begin{subfigure}{0.45\textwidth}
		\centering
		\caption{auto MinVol}\vspace{-0.35cm}
		\loaddata{minvolnmc/data/noisy_mean_rmse_auto_Minvol.txt}
\begin{tikzpicture}[
]
\coordinate (grid_start) at (-{floor(0.5*\nwidth)*\step} , -{floor(0.5*\nheight)*\step});
\coordinate (grid_end) at ({ceil(0.5*\nwidth)*\step} , {ceil(0.5*\nheight)*\step});
\foreach \k in {1, ..., 6} \draw let   
    \p{a}=(grid_start), 
    \p{b}=(grid_end) in
    ( \x{a} + 2*\k*\step cm - 1.5*\step cm,\y{a}) -- (\x{a} + 2*\k*\step cm - 1.5*\step cm,\y{a} -1pt) node[anchor=north] {$\pgfmathparse{80+(\k-1)*2}\pgfmathprintnumber{\pgfmathresult}$};
\node (xaxis) [yshift=-floor(0.5*(\nheight))*\step cm -0.6 cm] {missing values (\%)};
\foreach \k in {1, ..., 6} \draw let   
    \p{a}=(grid_start), 
    \p{b}=(grid_end) in
    (\x{a} , \y{a} + \k*\step cm -0.5*\step cm) -- (\x{a} -1pt, \y{a} + \k*\step cm -0.5*\step cm) node[anchor=east] {$\pgfmathparse{0+(\k-1)*0.2}\pgfmathprintnumber{\pgfmathresult}$};
    \node (yaxis) [xshift=-floor(0.5*(\nwidth))*\step cm-0.8 cm,rotate=90] {noise level};

\foreach \row [count=\n] in 
\loadeddata{
  \foreach \col [evaluate=\col as \scaledcol using {min(100,100/\maxrmse*\col)}] [count=\m] in \row {
    \fill [black!\scaledcol!white] (-{floor(0.5*\nwidth)*\step + (\m-1)*\step},{ceil(0.5*\nheight)*\step - (\n-1)*\step}) rectangle ++(\step,-\step);
  }
}
\draw[step=\step,gray,very thin] (grid_start) grid (grid_end);
\node[anchor=west] (colormap) [ xshift = ceil(0.5*\nwidth)*\step cm ] {\pgfplotscolorbardrawstandalone[ 
   colormap={blackwhite}{color=(white) color=(black)},
   colorbar,
   point meta min=0,
   point meta max=1,
   colorbar style={
   	ytick style ={color = white},
       width = \step cm,
       height= \nheight*\step cm,
       ytick={0,0.25,...,1},
   	yticklabels={0,,0.5,,$\geq$1},
   }]};  
\end{tikzpicture}
	\end{subfigure}
	\hspace{0.5cm}
	\begin{subfigure}{0.45\textwidth}
        \centering
        \caption{new MinVol}\vspace{-0.35cm}
        \loaddata{minvolnmc/data/noisy_mean_rmse_auto_Minvol_penFro.txt}
\begin{tikzpicture}[
]
\coordinate (grid_start) at (-{floor(0.5*\nwidth)*\step} , -{floor(0.5*\nheight)*\step});
\coordinate (grid_end) at ({ceil(0.5*\nwidth)*\step} , {ceil(0.5*\nheight)*\step});
\foreach \k in {1, ..., 6} \draw let   
    \p{a}=(grid_start), 
    \p{b}=(grid_end) in
    ( \x{a} + 2*\k*\step cm - 1.5*\step cm,\y{a}) -- (\x{a} + 2*\k*\step cm - 1.5*\step cm,\y{a} -1pt) node[anchor=north] {$\pgfmathparse{80+(\k-1)*2}\pgfmathprintnumber{\pgfmathresult}$};
\node (xaxis) [yshift=-floor(0.5*(\nheight))*\step cm -0.6 cm] {missing values (\%)};
\foreach \k in {1, ..., 6} \draw let   
    \p{a}=(grid_start), 
    \p{b}=(grid_end) in
    (\x{a} , \y{a} + \k*\step cm -0.5*\step cm) -- (\x{a} -1pt, \y{a} + \k*\step cm -0.5*\step cm) node[anchor=east] {$\pgfmathparse{0+(\k-1)*0.2}\pgfmathprintnumber{\pgfmathresult}$};
\node (yaxis) [xshift=-floor(0.5*(\nwidth))*\step cm-0.8 cm,rotate=90] {noise level};

\foreach \row [count=\n] in 
\loadeddata{
  \foreach \col [evaluate=\col as \scaledcol using {min(100,100/\maxrmse*\col)}] [count=\m] in \row {
    \fill [black!\scaledcol!white] (-{floor(0.5*\nwidth)*\step + (\m-1)*\step},{ceil(0.5*\nheight)*\step - (\n-1)*\step}) rectangle ++(\step,-\step);
  }
}
\draw[step=\step,gray,very thin] (grid_start) grid (grid_end);
\node[anchor=west] (colormap) [ xshift = ceil(0.5*\nwidth)*\step cm ] {\pgfplotscolorbardrawstandalone[ 
	colormap={blackwhite}{color=(white) color=(black)},
	colorbar,
	point meta min=0,
	point meta max=1,
	colorbar style={
		ytick style ={color = white},
		width = \step cm,
		height= \nheight*\step cm,
		ytick={0,0.25,...,1},
		yticklabels={0,,0.5,,$\geq$1},
	}]};  
\end{tikzpicture}
	\end{subfigure}
\caption[Average RMSE according to the noise level and to the percentage of missing values]{Average RMSE according to the noise level and to the percentage of missing values over 20 runs.}
\label{minvol:fig:noisy_rmse}
\end{figure}

\vfill

\pagebreak
\section{Identifiability of MinVol NMF with \texorpdfstring{$\ell_1$}{l1} penalty}

In the previous section, we studied the model 
\begin{mini}
	{\scriptstyle W,H}{\frac{1}{2}\logdet(\Wt W)+\frac{1}{2}\|H\|_F^2\label{minvol:eq:fullexactfrominvol}}{}{}
	\addConstraint{X=WH}
	\addConstraint{W\in\R_+^{m \times r} ,~H\in\R_+^{r \times n}}
\end{mini}
for missing data. Additionally, we mentioned that the identifiability of \eqref{minvol:eq:fullexactfrominvol} remains unknown with conditions milder than \Cref{preli:th:uniqNMFSSC}. For the sake of completion, let us mention in this section that the model
\begin{mini}
	{\scriptstyle W,H}{\frac{1}{2}\logdet(\Wt W)+\|H\|_1\label{minvol:eq:exactL1minvol}}{}{}
	\addConstraint{X=WH}
	\addConstraint{W\in\R_+^{m \times r} ,~H\in\R_+^{r \times n}}
\end{mini}
is just as identifiable as vanilla MinVol NMF. 
\begin{theorem}
	\label{minvol:th:uniqueexactL1minvol}
	Let $X=WH$ be an $\ell_1$-MinVol NMF of $X$ of size $r = \rank(X)$, in the sense of~\eqref{minvol:eq:exactL1minvol}. If $H$ satisfies SSC as in \Cref{preli:def:ssc}, then $\ell_1$-MinVol NMF $(W,H)$ of $X$ is essentially unique.
\end{theorem}
This follows from two key points:
\begin{itemize}
	\item Consider a feasible $(W,H)$ for~\eqref{minvol:eq:exactL1minvol} such that $X=WH$ and let ${f(D)=\frac{1}{2}\logdet(D^{-1}W^\top WD^{-1})+\|DH\|_1}$ where $D=\diag(d_1,\dots,d_r)$ is a positive diagonal matrix that can be seen as the scaling ambiguity between $W$ and $H$. Nullifying the gradient of $f$ relatively to each $d_i$, we have that $d_i=\frac{1}{\|H(i,:)\|_1}$, meaning that at optimality $\|H(i,:)\|_1=1$ for all $i$, since $d_i=1$ at optimality (otherwise one can improve the solution by scaling, which would therefore not be globally optimal). Or more compactly, $He=e$.
	\item If $H$ is SSC, MinVol NMF with $He=e$ is identifiable~\cite{fu2018identifiability}.
\end{itemize}

\section{Conclusion}

In this chapter, we developed a new algorithm to solve MinVol NMF based on the inertial block majorization-minimization framework of~\cite{hien2023inertial}. This framework, under some conditions that hold for our method, guarantees subsequential convergence. Experimental results show that this acceleration strategy performs better than the state-of-the-art accelerated MinVol NMF algorithm from~\cite{leplat19}. Then, we argued on the favor of using more the MinVol criterion in the domain of matrix completion, which has never been explored before. Not only the MinVol criterion can emulate a broad of behaviors going from the rank minimization to the nuclear minimization, but it also acts in favor of recovering the unique decomposition of a low-rank matrix if it exists. This paper also introduced a new variant of MinVol NMF which is not simplex-structured. Experiments show that a properly tuned MinVol NMF provides encouraging results, both on the task of matrix completion and unique factors recovery. Last but not least, experiments show that our new proposed variant of MinVol NMF outperforms vanilla MinVol NMF. Future work should focus on the potential identifiability of this new variant and on comparing with other matrix completion algorithms.



\chapter{Maximum-Volume Nonnegative Matrix Factorization}\label{chap:maxvolnmf}
\begin{hiddenmusic}{Hélène Vogelsinger - Reminiscence}{https://helenevogelsinger.bandcamp.com/track/reminiscence} 
\end{hiddenmusic}In this chapter, we present a new volume regularized NMF, dubbed MaxVol NMF\label{acro:MaxVolNMF} for Maximum-Volume Nonnegative Matrix Factorization. Compare to MinVol NMF, MaxVol NMF maximizes the volume of $H$ instead of minimizing the volume of $W$. To the best of our knowledge, MaxVol MF (without nonnegativity) has only been briefly discussed in~\cite{tatli2021polytopic} as the sparse nonnegative case of their framework. Its behavior on HU has not been explored, and their proposed algorithm is in fact using an algorithm designed to solve MinVol MF coming from~\cite{fu2016robust}. However, we will see that in the inexact case MinVol NMF and MaxVol NMF behaves differently. In particular MaxVol NMF is much more effective to extract sparse factors and does not generate rank-deficient solutions; see below for more details. 
\paragraph{Outline and contribution of the chapter} In \Cref{maxvolnmf:sec:motivation} and \Cref{maxvolnmf:sec:maxvolnmf}, we motivate, introduce and analyze MaxVol NMF. In \Cref{maxvolnmf:sec:solvemaxvol}, we propose two algorithms to solve MaxVol NMF. In \Cref{maxvolnmf:sec:normmaxvolnmf}, we present a normalized variant of MaxVol NMF that exhibits better performance than MinVol NMF and MaxVol NMF in the context of HU. Finally, we conclude and discuss future works in \Cref{maxvolnmf:sec:conclusion}.

\section{Motivation}\label{maxvolnmf:sec:motivation}
In the previous chapter, we highlighted the strengths of MinVol NMF. Let us also highlight two of its main weaknesses in the context of Hyperspectral Unmixing. In the remaining of this chapter, the MinVol NMF we are referring to is the one where the simplex structure is imposed on the columns of $H$, that is, $H\in\Delta^{r\times n}$.

First, the MinVol criterion introduces a bias that can reduce the quality of the unmixing. Let us illustrate this with the Samson dataset. The three main endmembers present in Samson are water, soil and tree. Due to the spectral signature of water having a low magnitude relatively to the spectral signature of soil and tree, a bad estimation of the water spectral signature does not increase significantly the reconstruction error. Consider the MinVol penalty on top of that, decreasing the norm of the spectral signature of water is an easy way to decrease to volume of $W$, and it can be done at a very small ``reconstruction price''. This can be seen on \Cref{maxvolnmf:fig:lambda1minvol}, where the spectral signature of water (in red) for MinVol NMF with $\lambda=1$ contains $156-132=24$ zeros (reported in \Cref{maxvolnmf:table:normswatersamsonminvol}), while there should not be any zeros because there is not a wavelength at which water absorbs completely electromagnetic energy. Here, increasing $\lambda$ will only worsen this behavior, as it can be seen with $\lambda=50$ on \Cref{maxvolnmf:fig:samsonseverallambdaminvol} and with the $l_0$ norms reported in \Cref{maxvolnmf:table:normswatersamsonminvol}.

\begin{table}
    \centering
    \begin{tabular}{c|c|c|c|c|c|c}
       $\lambda$ & 0 (NMF) & 1 & 5 & 10 & 50 & 1 with autotuning \\ \hline
       $l_2$ norm of water & 0.73 & 0.35 & 0.36 & 0.35 & 0.29 & 0.31 \\ \hline
       $l_0$ norm of water & 152 & 132 & 130 & 128 & 120 & 123
    \end{tabular}
    \caption[Norms of the spectral signature of the water retrieved by MinVol for the Samson dataset]{$l_2$ and $l_0$ of the spectral signature of the water retrieved by MinVol for the Samson dataset, which is of size $156\times 9025$.}
    \label{maxvolnmf:table:normswatersamsonminvol}
\end{table}
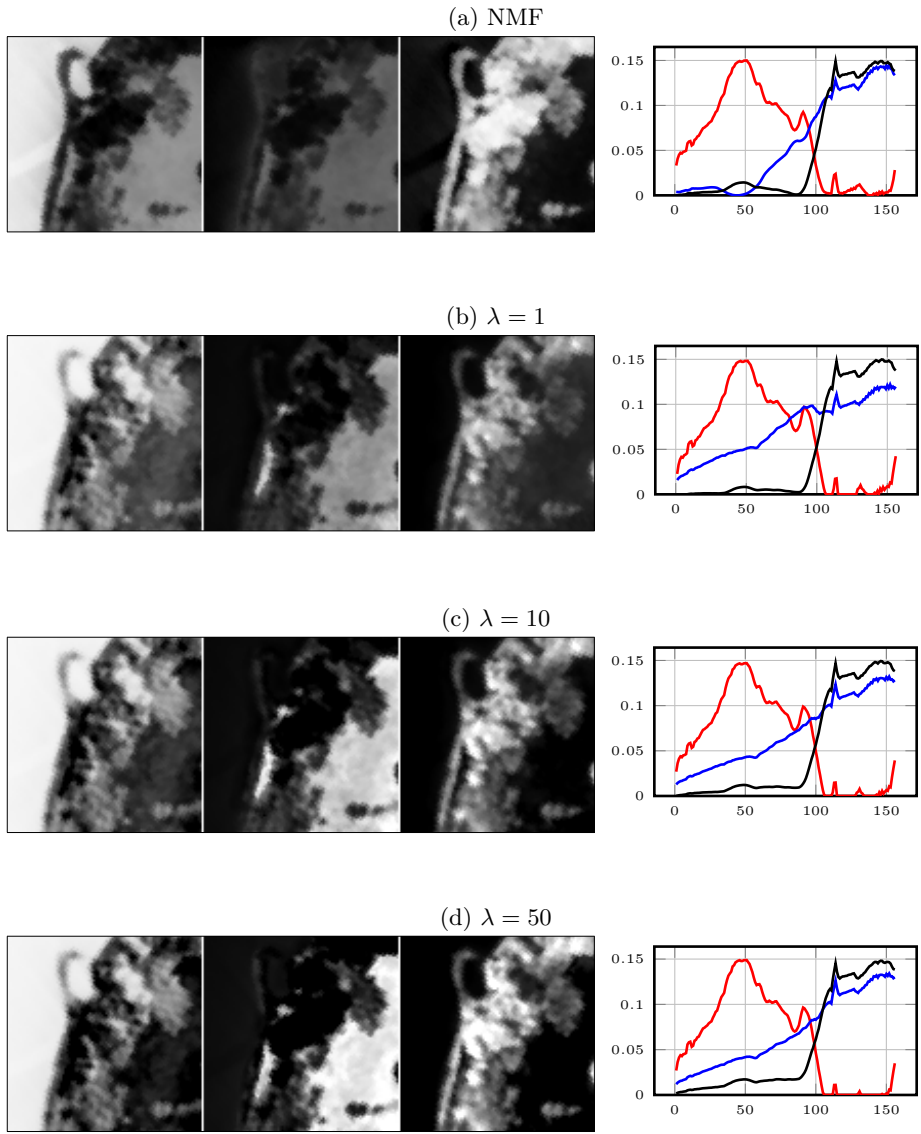
\begin{figure}
\begin{subfigure}{\linewidth}
    \caption{NMF}
    \begin{minipage}{0.6\textwidth}
            \fbox{\includegraphics[width=\linewidth]{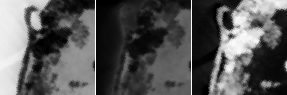}}
    \end{minipage}
    \begin{minipage}{0.39\textwidth}
        	\begin{tikzpicture}
        		\begin{axis}[cycle list name=color list,width=\linewidth,height=0.7\linewidth,line width = 1pt,no markers,grid=major,
                    tick label style={/pgf/number format/fixed,font=\tiny},ymin=0]
        		\foreach \col in {0,...,2} {
        			\addplot table [x expr=\coordindex+1, y index=\col] {maxvolnmf/figures/samson_minvol_0.txt};
        		}
        		\end{axis}
        	\end{tikzpicture}
    \end{minipage}
\end{subfigure}

\vspace{0.7cm}

\begin{subfigure}{\linewidth}
    \caption{$\lambda=1$}
    \begin{minipage}{0.6\textwidth}
            \fbox{\includegraphics[width=\linewidth]{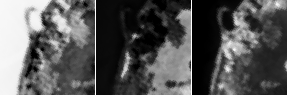}}
    \end{minipage}
    \begin{minipage}{0.39\textwidth}
        	\begin{tikzpicture}
        		\begin{axis}[cycle list name=color list,width=\linewidth,height=0.7\linewidth,line width = 1pt,no markers,grid=major,
                    tick label style={/pgf/number format/fixed,font=\tiny},ymin=0]
        		\foreach \col in {0,...,2} {
        			\addplot table [x expr=\coordindex+1, y index=\col] {maxvolnmf/figures/samson_minvol_1.txt};
        		}
        		\end{axis}
        	\end{tikzpicture}
    \end{minipage}
    \label{maxvolnmf:fig:lambda1minvol}
\end{subfigure}

\vspace{0.7cm}

\begin{subfigure}{\linewidth}
    \caption{$\lambda=10$}
    \begin{minipage}{0.6\textwidth}
            \fbox{\includegraphics[width=\linewidth]{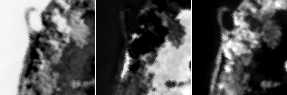}}
    \end{minipage}
    \begin{minipage}{0.39\textwidth}
        \begin{tikzpicture}
            \begin{axis}[cycle list name=color list,width=\linewidth,height=0.7\linewidth,line width = 1pt,no markers,grid=major,
                tick label style={/pgf/number format/fixed,font=\tiny},ymin=0]
            \foreach \col in {0,...,2} {
                \addplot table [x expr=\coordindex+1, y index=\col] {maxvolnmf/figures/samson_minvol_10.txt};
            }
            \end{axis}
        \end{tikzpicture}
    \end{minipage}
\end{subfigure}

\vspace{0.7cm}

\begin{subfigure}{\linewidth}
    \caption{$\lambda=50$}
    \begin{minipage}{0.6\textwidth}
            \fbox{\includegraphics[width=\linewidth]{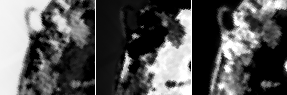}}
    \end{minipage}
    \begin{minipage}{0.39\textwidth}
        \begin{tikzpicture}
            \begin{axis}[cycle list name=color list,width=\linewidth,height=0.7\linewidth,line width = 1pt,no markers,grid=major,
                tick label style={/pgf/number format/fixed,font=\tiny},ymin=0]
            \foreach \col in {0,...,2} {
                \addplot table [x expr=\coordindex+1, y index=\col] {maxvolnmf/figures/samson_minvol_50.txt};
            }
            \end{axis}
        \end{tikzpicture}
    \end{minipage}
\end{subfigure}
\caption[Abundance maps and endmembers for MinVol on Samson]{Abundance maps and normalized endmembers (from the left to the right: {\color{red}water}, {\color{blue}soil} and tree) for MinVol on the Samson dataset with $\delta=1$.}
\label{maxvolnmf:fig:samsonseverallambdaminvol}
\end{figure}

Second, the sparsity of the decomposition is implicit and depends on the quality of the data. In the presence of noise, at some point, increasing the weight $\lambda$ of the volume criterion will not particularly increase the sparsity of $H$ and improve the decomposition. See \Cref{maxvolnmf:fig:samsonseverallambdaminvol}. With $\lambda$ increasing, the corresponding abundance map gets a little bit crispier. Still, the improvement in terms of sparsity is not that significant, and at the price of a worse spectral signature for the water. Now consider some data of better quality, like the Moffett dataset for instance. We can see on \Cref{maxvolnmf:fig:moffettseverallambdaminvol} that the abundance map for NMF is not perfect, but it is already a better decomposition than what NMF could provide for Samson on \Cref{maxvolnmf:fig:samsonseverallambdaminvol}. Adding the MinVol criterion with $\lambda=1$ improves the decomposition and, as a consequence, the sparsity. Still, the water and tree extraction are not right, as there are some detected water within the lands where it should in fact be trees. Increasing $\lambda$ to 10 slightly improves this, but the wrong water artifacts are still here. Then, increasing again $\lambda$ does not improve the unmixing. With $\lambda=50$, one of the columns of $W$ just collapses to zero. If a practitioner has some a priori knowledge on the sparsity of the decomposition, MinVol NMF cannot explicitly control sparsity, though sparsity is often desired in unmixing. 

In this chapter, we will see how MaxVol NMF preserves the spirit of MinVol NMF without the aforementioned weaknesses.

\begin{figure}
\begin{subfigure}{\linewidth}
    \caption{NMF}
    \begin{minipage}{0.6\textwidth}
            \fbox{\includegraphics[width=\linewidth]{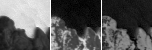}}
    \end{minipage}
    \begin{minipage}{0.39\textwidth}
        \begin{tikzpicture}
            \begin{axis}[cycle list name=color list,width=\linewidth,height=0.7\linewidth,line width = 1pt,no markers,grid=major,
                tick label style={/pgf/number format/fixed,font=\tiny},ymin=0]
            \foreach \col in {0,...,2} {
                \addplot table [x expr=\coordindex+1, y index=\col] {maxvolnmf/figures/moffett_minvol_0.txt};
            }
            \end{axis}
        \end{tikzpicture}
    \end{minipage}
\end{subfigure}

\vspace{0.7cm}

\begin{subfigure}{\linewidth}
    \caption{$\lambda=1$}
    \begin{minipage}{0.6\textwidth}
            \fbox{\includegraphics[width=\linewidth]{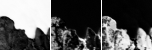}}
    \end{minipage}
    \begin{minipage}{0.39\textwidth}
        \begin{tikzpicture}
            \begin{axis}[cycle list name=color list,width=\linewidth,height=0.7\linewidth,line width = 1pt,no markers,grid=major,
                tick label style={/pgf/number format/fixed,font=\tiny},ymin=0]
            \foreach \col in {0,...,2} {
                \addplot table [x expr=\coordindex+1, y index=\col] {maxvolnmf/figures/moffett_minvol_1.txt};
            }
            \end{axis}
        \end{tikzpicture}
    \end{minipage}
\end{subfigure}

\vspace{0.7cm}

\begin{subfigure}{\linewidth}
    \caption{$\lambda=10$}
    \begin{minipage}{0.6\textwidth}
            \fbox{\includegraphics[width=\linewidth]{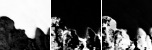}}
    \end{minipage}
    \begin{minipage}{0.39\textwidth}
        \begin{tikzpicture}
            \begin{axis}[cycle list name=color list,width=\linewidth,height=0.7\linewidth,line width = 1pt,no markers,grid=major,
                tick label style={/pgf/number format/fixed,font=\tiny},ymin=0]
            \foreach \col in {0,...,2} {
                \addplot table [x expr=\coordindex+1, y index=\col] {maxvolnmf/figures/moffett_minvol_10.txt};
            }
            \end{axis}
        \end{tikzpicture}
    \end{minipage}
\end{subfigure}

\vspace{0.7cm}

\begin{subfigure}{\linewidth}
    \caption{$\lambda=50$}
    \begin{minipage}{0.6\textwidth}
            \fbox{\includegraphics[width=\linewidth]{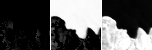}}
    \end{minipage}
    \begin{minipage}{0.39\textwidth}
        \begin{tikzpicture}
            \begin{axis}[cycle list name=color list,width=\linewidth,height=0.7\linewidth,line width = 1pt,no markers,grid=major,
                tick label style={/pgf/number format/fixed,font=\tiny},ymin=0]
            \foreach \col in {0,...,2} {
                \addplot table [x expr=\coordindex+1, y index=\col] {maxvolnmf/figures/moffett_minvol_50.txt};
            }
            \end{axis}
        \end{tikzpicture}
    \end{minipage}
\end{subfigure}
\caption[Abundance maps and endmembers for MinVol on Moffett]{Abundance maps and normalized endmembers (from the left to the right: {\color{red}water}, {\color{blue}tree} and soil, except for $\lambda=50$) for MinVol on the Moffett dataset with $\delta=0.1$.}
\label{maxvolnmf:fig:moffettseverallambdaminvol}
\end{figure}
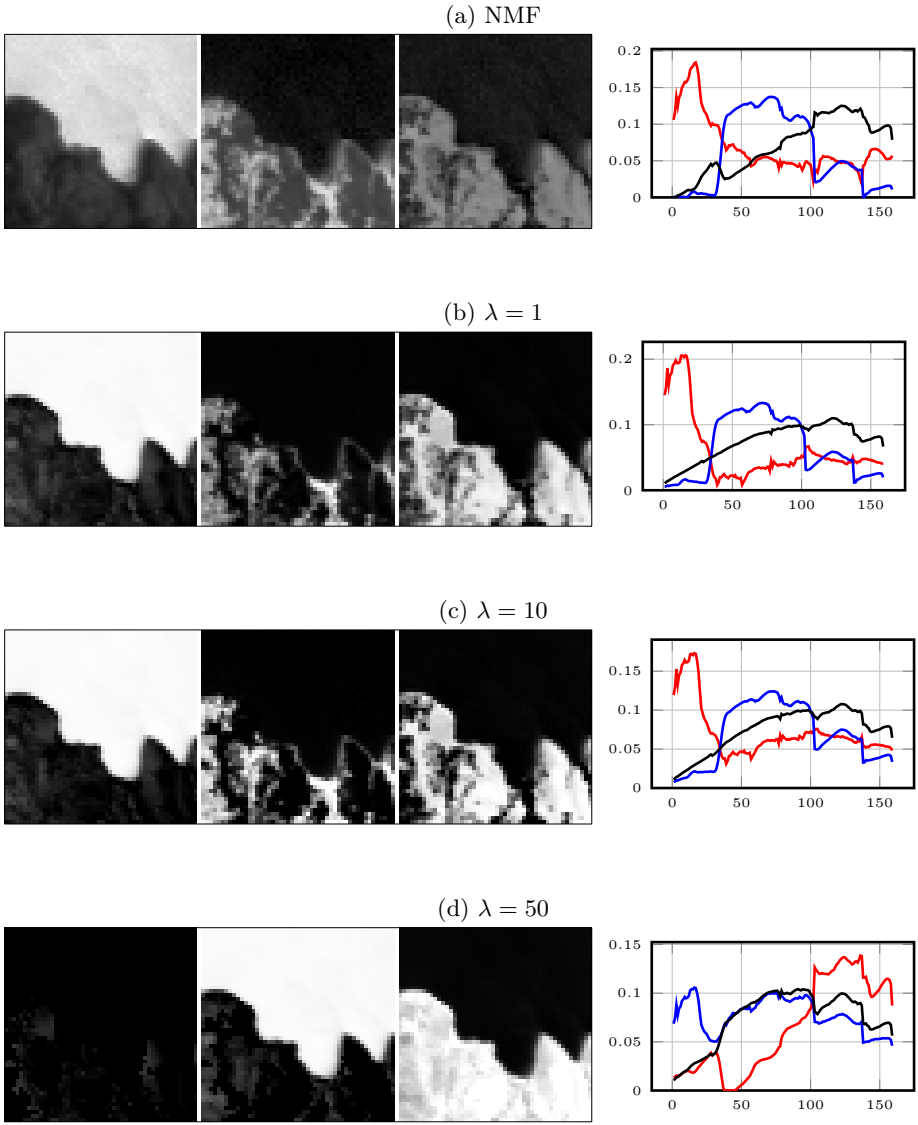
\vfill
\pagebreak
\section{MaxVol NMF}\label{maxvolnmf:sec:maxvolnmf}

Let us introduce MaxVol NMF through its equivalence with MinVol NMF in the exact case. Consider the full rank SSNMF $X=\overbar{W}\overbar{H}$. For any full column rank matrix $W$ of the same size as $\overbar{W}$, there exists an invertible matrix $Q$ such that $W=\overbar{W}Q$. Then, $$\det(W^\top W)=\det(Q^\top \overbar{W}^\top \overbar{W}Q)=\det(Q)^2\det(\overbar{W}^\top \overbar{W}).$$ Minimizing $\det(W^\top W)$ is equivalent to minimizing $\det(Q)^2\det(\overbar{W}^\top \overbar{W})$ relatively to $Q$. Hence, computing the exact MinVol NMF of $X$ is equivalent to solving
\begin{mini}|s|
    {\scriptstyle Q}{\det(Q)^2}
    {\label{maxvolnmf:eq:exactQminvol}}{}
    \addConstraint{\overbar{W}Q\geq0,Q^{-1}\overbar{H}\in\Delta^{r\times n}.}
\end{mini}

\noindent An obvious MinVol NMF of $X$ is then $(\overbar{W}Q,Q^{-1}\overbar{H})$. Minimizing the quantity $\det(Q)^2$ is equivalent to maximizing the quantity $\det(Q^{-2})$. To sum up, in the exact case, minimizing the volume of $W$ is equivalent to maximizing the volume of $H$. Here is the exact MaxVol NMF formulation:

\begin{maxi}|s|
    {\scriptstyle W,H}{\det(HH^\top)}
    {\label{maxvolnmf:eq:exactmaxvol}}{}
    \addConstraint{X=WH}
    \addConstraint{W\geq0,H\in\Delta^{r\times n}.}
\end{maxi}

\subsection{Identifiability of MaxVol NMF}\label{maxvolnmf:sec:identifiability}

MaxVol NMF is just as identifiable as MinVol NMF. Actually, the proof is almost exactly the same as the one for MinVol NMF.
\begin{theorem}
	\label{maxvolnmf:th:uniquemaxvol}
	Let $X=WH$ be a MaxVol NMF of $X$ of size $r = \rank(X)$, in the sense of~\eqref{maxvolnmf:eq:exactmaxvol}. If $H$ satisfies SSC as in \Cref{preli:def:ssc}, then MaxVol NMF $(W,H)$ of $X$ is essentially unique.
\end{theorem}
\begin{proof}
	Let $Q\in\R^{r\times r}$ be an invertible matrix such that $(WQ^{-1},QH)$ is another feasible solution of~\eqref{maxvolnmf:eq:exactmaxvol}. There exists a right inverse $H^\dagger$ such that $HH^\dagger=I$ because $\rank(H)=r$. Since $e^\top H = e^\top$ and $e^\top QH= e^\top$ because $(WQ^{-1},QH)$ is feasible, we have
	\begin{equation}
		e^\top Q = e^\top QHH^\dagger=e^\top H^\dagger = e^\top H H^\dagger= e^\top.
	\end{equation}
    For the same reasons as in the proof of \Cref{minvol:th:uniqueminvol}, that is, from \eqref{minvol:eq:QHgeq0} to \eqref{minvol:eq:explicit-coneQ-in-dualC}, we have
	\begin{equation}
		\label{maxvolnmf:eq:detQineq}
		\begin{split}
			|\det(Q)| \leq & \prod_{i=1}^r\|Q(i,:)\|_2 \\
			\leq & \prod_{i=1}^r Q(i,:)e \\
			\leq & \left(\frac{\sum_{i=1}^{r}Q(i,:)e}{r}\right)^r = \left(\frac{e^\top Q e}{r}\right)^r = 1,
		\end{split}
	\end{equation}
	where the first inequality is coming from the Hadamard's inequality, the second from~\eqref{minvol:eq:explicit-coneQ-in-dualC}, and the last one from the arithmetic-geometric mean inequality and that $e^\top Q = e^\top$. \\

	Suppose now that $(WQ^{-1},QH)$ is also an optimal solution to~\eqref{maxvolnmf:eq:exactmaxvol}. Then,
	\begin{align}
		&\det(QHH^\top Q^\top)  = \det(HH^\top) \\
		\Leftrightarrow \quad &|\det(Q)|^{2}\det(HH^\top) = \det(HH^\top) \\
		\Leftrightarrow \quad &|\det(Q)|=1.
	\end{align}
	The remainder of the proof is exactly like in \Cref{minvol:th:uniqueminvol}.
\end{proof}

\subsection{Behavior of MaxVol NMF}\label{maxvolnmf:sec:behaviormaxvol}

In the inexact case, we consider the following MaxVol NMF formulation:

\begin{mini}|s|
{\scriptstyle W,H}
{f(W,H):=\frac{1}{2}\| X - WH \|_F^2 -\lambda\logdet(HH^\top+\delta I)\label{maxvolnmf:eq:nonexactmaxvolmf}}{}{}
\addConstraint{W\geq0,H\in\Delta^{r \times n}.}
\end{mini} 

It should be noted that, unlike MinVol NMF, from an optimization perspective, the $\delta$ term in the $\logdet$ is not needed anymore. Maximizing the $\logdet$ will prevent $H$ from being rank deficient. Still, we keep $\delta$ in our model as it has some physical meaning. This is discussed in \Cref{maxvolnmf:sec:normmaxvolnmf}.

To understand the main difference between MinVol NMF and MaxVol NMF, consider the asymptotic case when $\lambda$ goes to infinity. For MinVol NMF, $W$ will just converge to $0$. For MaxVol NMF, $H$ will converge to a matrix whose rows are mutually orthogonal and such that the $l_2$ norm of each row are as close to each other as possible. Let us justify this intuition by considering the problem

\begin{mini}
    {X\in \mathbb{S}^r}{f_0(X)=\log\det X^{-1}}{\label{maxvolnmf:eq:maxvol}}{}
    \addConstraint{e^\top X e\leq a}
    \addConstraint{X\geq0,}
\end{mini}
where $a>0$ and $\dom f_0=\mathbb{S}^r_{++}$. We want to prove that $X=\frac{a}{r}I$ is the unique minimizer of~\eqref{maxvolnmf:eq:maxvol}. We solve this problem through its dual using the conjugate of $f_0$, like in \cite[Section 5.1.6]{boyd2004convex}.

\begin{definition}
    The conjugate $f^*$ of a function $f:\R^r\rightarrow\R$ is given by
    $$f^*(y)=\sup_{x\in\dom f}\left(y^\top x-f(x)\right).$$
\end{definition}

\noindent Considering the optimization problem with linear inequality and equality constraints
\begin{mini}
    {}{f_0(x)}{\label{maxvolnmf:eq:linconst}}{}
    \addConstraint{Ax\leq b}
    \addConstraint{Cx=d,} 
\end{mini}
the conjugate of $f_0$ can be used to write the dual function for~\eqref{maxvolnmf:eq:linconst} as
\begin{align}
    g(\lambda,\nu)&= \inf_x\left(f_0(x)+\lambda^\top(Ax-b)+\nu^\top(Cx-d)\right)\\
    &= -b^\top\lambda -d^\top\nu +\inf_x\left(f_0(x)+(A^\top\lambda+C^\top\nu)^\top x\right)\\
    &= -b^\top\lambda-d^\top\nu-f_0^*(-A^\top\lambda-C^\top\nu)\label{maxvolnmf:eq:conjindual}.
\end{align}
The domain of $g$ follows from the domain of $f_0^*$:
\[\dom g=\{(\lambda,\nu)|-A^\top\lambda-C^\top\nu\in\dom f_0^*\}.\]

Let us go back to the conjugate function of $f_0$, which is defined as
\[f_0^*(Y)=\sup_{X\succ0}\left(\langle Y,X\rangle+\log\det X\right).\]

We first show that $\langle Y,X\rangle+\log\det X$ is unbounded above unless $Y\prec0$. If $Y\nprec0$, then $Y$ has an eigenvector $v$, with $\|v\|_2=1$, and eigenvalue $\lambda\geq0$. Taking $X=I+tvv^\top$ we find that 
\[\langle Y,X\rangle+\log\det X=\tr Y+t\lambda+\log\det(I+tvv^\top)=\tr Y+t\lambda+\log(1+t),\]
which is unbounded above as $t\rightarrow\infty$.
\noindent Now consider the case $Y\prec0$. We can find the maximizing $X$ by setting the gradient with respect to $X$ equal to zero:
\[\nabla_X(\langle Y,X\rangle+\log\det X)=Y+X^{-1}=0,\]
which leads to $X=-Y^{-1}$. Therefore, we have
\begin{equation}
    f_0^*(Y)=\log\det(-Y)^{-1}-r
    \label{maxvolnmf:eq:conjugatef0}
\end{equation}
with $\dom f_0^*=-\mathbb{S}^r_{++}$.\\

Applying the result in~\eqref{maxvolnmf:eq:conjindual}, the dual function for problem~\eqref{maxvolnmf:eq:maxvol} is given by

\begin{equation}
    g(\lambda,\nu)=\left\{\begin{array}{ll}
        \log\det\bigl(\lambda J-\sum\limits_{i,j}\nu_{i,j}E_{i,j}\bigr) +r-\lambda a & \text{ if }\lambda J-\sum\limits_{i,j}\nu_{i,j}E_{i,j}\succ0,  \\
        \infty&\text{ otherwise,} 
    \end{array}\right.
\end{equation}
with $\lambda\in\R_+$ and $\nu\in\R_+^{r\times r}$. Let $\lambda^*=\frac{r}{a}$, $\nu^*_{i,j}=\frac{r}{a}$ if $i\neq j$, $\nu^*_{i,j}=0$ if $i=j$ and $X^*=\frac{a}{r}I$. We have $f_0(X^*)=g(\lambda^*,\nu^*)$, meaning that there is no duality gap and that $X^*$ is a solution of~\eqref{maxvolnmf:eq:maxvol}. Finally, $X^*$ is the unique solution because $f_0$ is strongly convex. 

Due to this result and to the fact that $e^\top HH^\top e=n$, if $n=dr$ where $d\in\mathbb{N}^*$, then increasing $\lambda$ will make $HH^\top$ converge to a diagonal whose elements are all equal to $d$. In other words, the rows of $H$ will be mutually orthogonal. The simplex constraint on the columns of $H$ and the fact that the rows are mutually orthogonal will impose that $H(i,j)\in\{0,1\}$. The norm of each row is then just the square root of the number of non-zero elements in the corresponding row. From the HU point of view, one pixel will be assigned to only one material. This is equivalent to a hard clustering where every cluster should be of the same size. On a side note, if $n$ is not a multiple of $r$, then the norm of each row will have to be different. The clustering behavior of MaxVol NMF is interesting and offers more control over the sparsity of the decomposition than MinVol NMF. Also, maximizing the volume of $H$ indirectly minimizes the volume of $W$ without the drawback of potentially setting a useful endmember to zero due to its low reflectance. However, the fact that increasing $\lambda$ tends to an even clustering is a clear weakness. See the experiment on \Cref{maxvolnmf:fig:samsonmaxvollambdas} on Samson. Increasing $\lambda$ intensifies the clustering, until a hard clustering is achieved with $\lambda=50$. Increasing $\lambda$ removes some of the false positives for water, but not all of them. This is probably because there are more pixels containing trees than water or stone in this dataset. Correctly assigning the water false positives to tree will unbalance even more the size of the clusters, though MaxVol NMF favors clusters of the same size. Also, the improvement of the abundance map of the water is at the cost of a hard clustering, while a soft clustering would be preferable to properly unmix soil and tree.

\begin{figure}[htbp!]
    \begin{subfigure}{\textwidth}
        \centering
        \caption{$\lambda=0.5$}
        \fbox{\includegraphics[width=0.8\textwidth]{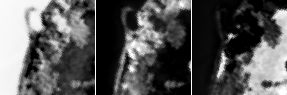}}
    \end{subfigure}

    \vspace{0.5cm}

    \begin{subfigure}{\textwidth}
        \centering
        \caption{$\lambda=5$}
        \fbox{\includegraphics[width=0.8\textwidth]{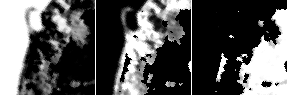}}
    \end{subfigure}

    \vspace{0.5cm}

    \begin{subfigure}{\textwidth}
        \centering
        \caption{$\lambda=10$}
        \fbox{\includegraphics[width=0.8\textwidth]{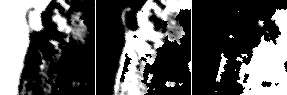}}
    \end{subfigure}

    \vspace{0.5cm}

    \begin{subfigure}{\textwidth}
        \centering
        \caption{$\lambda=50$}
        \fbox{\includegraphics[width=0.8\textwidth]{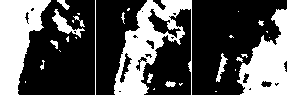}}
    \end{subfigure}
\caption{Abundance maps of MaxVol NMF on Samson, depending on $\lambda$.}
\label{maxvolnmf:fig:samsonmaxvollambdas}
\end{figure}

In \Cref{maxvolnmf:sec:normmaxvolnmf}, we present an improved variant of MaxVol NMF where the volume of the row wise normalized $H$ is maximized instead. This variant is an improvement as it is not biased towards clusters of the same size.

\begin{remark}
    About the results in \Cref{maxvolnmf:fig:samsonmaxvollambdas}:
    \begin{itemize}
        \item $\lambda$ is tuned using \cite{nguyen2024towards}, in the same way we used it in \Cref{minvol:sec:exp}. 
        \item In \Cref{maxvolnmf:sec:solvemaxvol} we show two different algorithms to solve MaxVol NMF. The abundance maps displayed on \Cref{maxvolnmf:fig:samsonmaxvollambdas} are the same regardless of the used algorithm, except for $\lambda=50$ where the adaptive gradient method crashes, probably due to some numerical issues. The ADMM based algorithm still works well with $\lambda=50$.
    \end{itemize}
\end{remark}

\section{Solving MaxVol NMF} \label{maxvolnmf:sec:solvemaxvol}

The most common strategy to solve problems like~\cref{maxvolnmf:eq:nonexactmaxvolmf} is to use an alternated block optimization scheme. Consider blocks of variables, while updating one block, fix the others. When the update is finished, repeat the same process for the next block. Here, we only consider two blocks: $W$ and $H$. The main difficulty in solving~\cref{maxvolnmf:eq:nonexactmaxvolmf} holds in the $-\lambda\logdet$ term. Since $X\rightarrow \logdet(X)$ is concave, it is easy to derive a surrogate for~\cref{minvol:eq:nonexactminvol} relatively to $W$ whose gradient is Lipschitz continuous. The first-order Taylor approximation at the current iterate $W^k$ is enough, as it has been seen in \Cref{minvol:sec:surrogateW}. It is then possible to update $W$ by minimizing the obtained Lipschitz surrogate. This is exactly equivalent to performing a projected gradient step with a step size equal to the inverse of the Lipschitz constant of the gradient of the surrogate, that is $\frac{1}{\|HH^\top+\lambda({W^k}^\top W^k+\delta I)^{-1}\|}$. Since $-\logdet(.)$ is not concave, it prevents us from using for~\cref{maxvolnmf:eq:nonexactmaxvolmf} the same update strategy that has been derived for~\cref{minvol:eq:nonexactminvol}. In this section, we propose several algorithms to solve~\cref{maxvolnmf:eq:nonexactmaxvolmf}. The first algorithm in~\cref{maxvolnmf:sec:adgrad} is adapted from~\cite{malitsky2020adaptive}. Its core idea is to approximate the local Lipschitzness by using the previous iterate and to compute the corresponding Lipschitz gradient descent. The second algorithm is based on the Alternating Direction Method of Multipliers (ADMM).

Let us note that relatively to $W$, another choice could be to consider each of its columns as a block, also known as HALS~\cite{GG12}. As the update of $W$ is not the main concern while solving~\cref{maxvolnmf:eq:nonexactmaxvolmf} with an alternated block scheme, we will not explore this possibility. Note also that HALS could not be used to update $H$. It would alternatively update the rows of $H$, although they depend on each other because of the probability simplex constraint $H\in\Delta^{r\times n}$.

\subsection{Adaptive accelerated gradient descent}
\label{maxvolnmf:sec:adgrad}
Our first proposed algorithm for~\cref{maxvolnmf:eq:nonexactmaxvolmf} relies on~\cite[Alg. 2]{malitsky2020adaptive}. This algorithm uses the previous iterate to approximate the local Lipschitzness and derive an appropriate step size. The previous iterate is also used to induce some extrapolation. The only knowledge that is needed from $f$ is its gradient. It should be noted that~\cite[Alg. 2]{malitsky2020adaptive} is only designed for a one block variable, that is, all variables are updated at the same time. In our case, it would mean that $[\Wt,H]\in\R^{r\times(m+n)}$ should be updated all at once. Most gradient based algorithm for constrained matrix factorization are using a two block alternating strategy, performing several updates on $W$ and then several updates on $H$. By doing so, a gradient based two block alternating algorithm can save computation time by precomputing some matrix operations that remain unchanged during the update of one block. We will follow this common two block strategy and, as it has been said before, all we need are the gradients:
\begin{align}
\begin{split}
    \nabla f(W) &= \frac{\partial f}{\partial W}(W) \\
                &= (WH-X)H^\top,
\end{split}\\
\begin{split}
    \nabla f(H) &= \frac{\partial f}{\partial H}(H) \\
                &= W^\top(WH-X) - 2\lambda(HH^\top+\delta I)^{-1}H.
\end{split}
\end{align}

Our adaptation of~\cite[Alg. 2]{malitsky2020adaptive} with a two block strategy is given in~\cref{maxvolnmf:alg:Adgrad2}.

\begin{algorithm}[htbp!]
\caption{Adgrad2}
\label{maxvolnmf:alg:Adgrad2}
\DontPrintSemicolon
\KwIn{data matrix $X \in \mathbb{R}^{m \times n}$, initial factors $\Wold \in \mathbb{R}_+^{m \times r}$  and $\Hold \in \Delta^{r \times n}$}
$\Gamma_{\Wold}=\|\Hold \Hold^\top\|,\gamma_{\Wold}={\Gamma_{\Wold}}^{-1},\theta_W=\Theta_W=10^9,\Wextraold=\Wold,\overbar{W}=W=[\Wold-10^{-6}\nabla f(\Wold)]_+$\;
$\Gamma_{Ho}=\|\Wold^\top \Wold\|,\gamma_{Ho}={\Gamma_{Ho}}^{-1},\theta_H=\Theta_H=10^9,\Hextraold=\Hold,\overbar{H}=H=[\Hold-10^{-6}\nabla f(\Hold)]_{\Delta^{r \times n}}$\;
\For{$k=1,2,\dots$\nllabel{maxvolnmf:alg:Adgrad2:line:outerloop}}{
    \While{stopping criteria not satisfied\nllabel{maxvolnmf:alg:Adgrad2:line:Winnerloop}}{
        $\gamma_{W}=\min\left(\gamma_{\Wold}\sqrt{1+\frac{\theta_W}{2}},\frac{\|\Wextra-\Wextraold\|_F}{2\|\nabla f(\Wextra)-\nabla f(\Wextraold)\|_F}\right)$\;
        $\Gamma_{W}=\min\left(\Gamma_{\Wold}\sqrt{1+\frac{\Theta_W}{2}},\frac{\|\nabla f(\Wextra)-\nabla f(\Wextraold)\|_F}{2\|\Wextra-\Wextraold\|_F}\right)$\;
        $W=[\Wextra-\gamma_{W}\nabla f(\Wextra)]_+$\;
        $\theta_W=\gamma_{W}/\gamma_{\Wold},\Theta_W=\Gamma_{W}/\Gamma_{\Wold}$\;
        $\Wextraold=\Wextra$\;
        $\Wextra=W+\frac{1-\sqrt{\gamma_{W}\Gamma_{W}}}{1+\sqrt{\gamma_{W}\Gamma_{W}}}(W-\Wold)$\;
        $\Wold=W$\;
        $\gamma_{\Wold}=\gamma_{W},\Gamma_{\Wold}=\Gamma_{W}$\;
    }
    \While{stopping criteria not satisfied\nllabel{maxvolnmf:alg:Adgrad2:line:Hinnerloop}}{
        $\gamma_{H}=\min\left(\gamma_{\Hold}\sqrt{1+\frac{\theta_H}{2}},\frac{\|\Hextra-\Hextraold\|_F}{2\|\nabla f(\Hextra)-\nabla f(\Hextraold)\|_F}\right)$\;
        $\Gamma_{H}=\min\left(\Gamma_{\Hold}\sqrt{1+\frac{\Theta_H}{2}},\frac{\|\nabla f(\Hextra)-\nabla f(\Hextraold)\|_F}{2\|\Hextra-\Hextraold\|_F}\right)$\;
        $H=[\Hextra-\gamma_{H}\nabla f(\Hextra)]_{\Delta^{r \times n}}$\label{maxvolnmf:alg:lineupdateH}\;
        $\theta_H=\gamma_{H}/\gamma_{\Hold},\Theta_H=\Gamma_{H}/\Gamma_{\Hold}$\;
        $\Hextraold=\Hextra$\;
        $\Hextra=H+\frac{1-\sqrt{\gamma_{H}\Gamma_{H}}}{1+\sqrt{\gamma_{H}\Gamma_{H}}}(H-\Hold)$\;
        $\Hold=H$\;
        $\gamma_{\Hold}=\gamma_{H},\Gamma_{\Hold}=\Gamma_{H}$\;
    }
}
\end{algorithm}

\begin{remark}
    The adaptive part is mostly useful for the update of $H$. In order to update $W$ any other algorithm could be used instead of the \textbf{while} loop in~\cref{maxvolnmf:alg:Adgrad2:line:Winnerloop}.
\end{remark}

\subsection{Alternating direction method of multipliers (ADMM) for the MaxvolMF problem}
\label{maxvolnmf:sec:admm}

Let us consider the following ADMM reformulation of~\cref{maxvolnmf:eq:nonexactmaxvolmf}:
\begin{mini}|s|
{\scriptstyle W,H,Y,\Lambda}
{\La(W,H,Y,\Lambda):=\frac{1}{2}\| X - WH \|_F^2 -\lambda\logdet(Y+\delta I)+\langle Y-HH^\top,\Lambda\rangle\label{maxvolnmf:eq:admmmaxvol}}{}{}
\breakObjective{\qquad\qquad\qquad\qquad+\frac{\rho}{2}\|Y-HH^\top\|_F^2}
\addConstraint{W\geq0,H\in\Delta^{r \times n}.}
\end{mini} 
According to~\cite{bertsekas2016nonlinear}, the ADMM algorithm consists of the following updates:

\begin{align}
    W^{k+1} &= \argmin_{W\geq0}\La(W,H^k,Y^k,\Lambda^k) \label{maxvolnmf:eq:admm_W}\\
    H^{k+1} &= \argmin_{H\geq\Delta^{r \times n}}\La(W^{k+1},H,Y^k,\Lambda^k) \label{maxvolnmf:eq:admm_H}\\
    Y^{k+1} &= \argmin_{Y}\La(W^{k+1},H^{k+1},Y,\Lambda^k) \label{maxvolnmf:eq:admm_Y}\\
    \Lambda^{k+1} &= \Lambda^{k} + \rho(Y^{k+1}-H^{k+1}{H^{k+1}}^\top) \label{maxvolnmf:eq:admm_Lambda}
\end{align}

\paragraph{Updating $W$}

Like in~\cref{maxvolnmf:sec:adgrad}, the update for $W$ can be computed through any algorithm for constrained convex problems, as~\cref{maxvolnmf:eq:admm_W} is equivalent to $$W^{k+1} = \argmin_{W\geq0} \frac{1}{2}\|X-WH^{k}\|^2_F,$$ where $W\rightarrow\frac{1}{2}\|X-WH^{k}\|^2_F$ is convex. Here we propose to use TITAN~\cite{hien2023inertial} with a Lipschitz surrogate, like in \Cref{minvol:sec:titanizedminvol}. The resulting update for $W^{k+1}$ is detailed in~\cref{maxvolnmf:alg:admm_titan_W}

\begin{algorithm}[htb!]
\caption{Update of $W$ with TITAN}
\label{maxvolnmf:alg:admm_titan_W}
\DontPrintSemicolon
\KwIn{$\alpha_1,X,H^k,W,\Wold$}
\KwOut{$W$}
$L_W=\|H^k{H^k}^\top\|$\;
\While{stopping criteria not satisfied}{
    $\alpha_0=\alpha_1$\;
    $\alpha_1=\frac{1}{2}(1+\sqrt{1+4\alpha_0^2})$\;
    $\beta=\frac{\alpha_0-1}{\alpha_1}$\;
    $\Wextra=W+\beta(W-\Wold)$\;
    $\Wold=W$\;
    $W=[\Wextra+\frac{1}{L_W}(XH^\top-\Wextra H^k{H^k}^\top)]_+$
}
\end{algorithm}

\paragraph{Updating $H$}

We propose two ways of updating $H$. The first one consists of solving directly~\cref{maxvolnmf:eq:admm_H} with the adaptive accelerated gradient descent algorithm described in~\cref{maxvolnmf:sec:adgrad}. The second one, that we will describe here, consists of deriving a non-Euclidean gradient method. Basically, we find a Bregman surrogate of $H\rightarrow\La(W^{k+1},H,Y^k,\Lambda^k):=\La(H)$ and update $H$ by minimizing this surrogate instead. The main motivation to use such a surrogate is that there does not exist a Lipschitz surrogate of $\La$ relatively to $H$. The gradient of $H\rightarrow\La(W^{k+1},H,Y^k,\Lambda^k)$ is clearly not Lipschitz continuous because the gradient of $\|Y-HH^\top-\delta I\|_F^2$ relatively to $H$ is cubic. Although $H\rightarrow\La(H)$ is not $L$-smooth, using the framework of~\cite{bauschke2017descent}, we can show that it is smooth relatively to the quartic norm kernel proposed in~\cite{dragomir2021quartic}.
\begin{definition}[Bregman distance]
    $$D_h(x,y)=h(x)-h(y)-\langle\nabla h(y),x-y\rangle$$
with $h$ a properly chosen convex function, dubbed a distance kernel.
\end{definition}
Note that $D_h$ is not a proper distance as it is asymmetric.
\begin{definition}[Relative smoothness~\cite{bauschke2017descent}]
    We say that a differentiable function $f:\R^{r\times n}\rightarrow\R$ is $L$-smooth relatively to the distance kernel $h$ if there exists $L>0$ such that for every $X,Y\in\R^{r\times n}$,
    $$f(X)\leq f(Y) + \langle\nabla f(Y),X-Y\rangle+L D_h(X,Y).$$
    If $f$ is twice differentiable, $L$-smoothness relatively to $h$ is equivalent to 
    $$\nabla^2f(X)[U,U]\leq L\nabla^2 h(X)[U,U]\quad\forall X,U\in\R^{r\times n},$$
    where $\nabla^2f(X)[U,U]$ denotes the second derivative of $f$ at $X$ in the direction $U$.
\end{definition}
First, we focus on the relative smoothness of the quartic term. According to~\cite{dragomir2021quartic} we have
\begin{equation}
	\label{maxvolnmf:eq:bregsurroquarticpart}
	\frac{1}{2}\|Y-HH^\top\|_F^2:=g(H)\leq g(H^k)+\langle\nabla g(H^k),H-H^k\rangle + D_h(H,H^k)
\end{equation}
where $\nabla g(H^k)=2(H^k{H^k}^\top-Y)H^k$,
and $h(H)=\frac{\alpha}{4}\|H\|_F^4+\frac{\sigma}{2}\|H\|_F^2$ with $\alpha=6$ and $\sigma=2\|Y\|_2$. Substituting~\eqref{maxvolnmf:eq:bregsurroquarticpart} in~\eqref{maxvolnmf:eq:admmmaxvol},
\begin{multline}
	\label{maxvolnmf:eq:lagrangianquarticsurrogate}
	\La(H)\leq u_{H^k}(H):=\frac{1}{2}\|X-WH\|_F^2-\langle HH^\top,\Lambda\rangle+\rho\langle\nabla g(H^k),H\rangle+\rho h(H)\\-\rho\langle\nabla h(H^k),H \rangle+C_H
\end{multline}
where $C_H$ is a constant relatively to $H$. Compute the second directional derivative of $u_{H^k}$
\begin{align}
\nabla^2 u_{H^k}(H)[U,U]=&\langle(W^\top W-2\Lambda^\top)U,U\rangle+\rho\sigma\|U\|_F^2+\rho\alpha(\|H\|_F^2\|U\|_F^2+2\langle H,U\rangle^2),\nonumber\\
\leq&\rho\alpha(\|H\|_F^2\|U\|_F^2+2\langle H,U\rangle^2)+(\|W^\top W-2\Lambda^\top\|_2+\rho\sigma)\|U\|_F^2,\nonumber\\
\label{maxvolnmf:eq:kernel1smooth}
=&\nabla^2\left(\frac{\tilde{\alpha}}{4}\|H\|_F^4+\frac{\tilde{\sigma}}{2}\|H\|_F^2\right)[U,U],
\end{align}
where $\tilde{\alpha}=\rho\alpha$ and $\tilde{\sigma}=\rho\sigma+\|W^\top W-2\Lambda^\top\|_2$.
From~\eqref{maxvolnmf:eq:kernel1smooth} and~\eqref{maxvolnmf:eq:lagrangianquarticsurrogate}, $H\rightarrow\La(H)$ is 1-smooth relatively to the kernel $\tilde{h}:H\rightarrow\frac{\tilde{\alpha}}{4}\|H\|_F^4+\frac{\tilde{\sigma}}{2}\|H\|_F^2$. More explicitly, 
\begin{equation}
	\La(H)\leq u_{H^k}(H^k)+\langle\nabla u_{H^k}(H^k),H-H^k\rangle+D_{\tilde{h}}(H,H^k).
\end{equation}

The update for $H$ is then obtained by minimizing the aforementioned surrogate 
\begin{align}
	H^{k+1}&=\argmin_{H\in\Delta^{r \times n}}\left\{ \langle\nabla u_{H^k}(H^k),H\rangle +\tilde{h}(H) -\langle\nabla \tilde{h}(H^k),H\rangle \right\}, \nonumber\\
	\label{maxvolnmf:eq:Hminbregsurro}
	&=\argmin_{H\in\Delta^{r \times n}}\left\{t_k(H):=\tilde{h}(H)-\langle Q^k,H\rangle \right\},
\end{align}
where $Q^k=\nabla \tilde{h}(H^k)-\nabla u_{H^k}(H^k)$. This is equivalent to the Bregman proximal iteration map described in~\cite{dragomir2021quartic} with a step size equal to 1.

\begin{corollary}
	\label{maxvolnmf:cor:bregupdateHform}
	The solution of~\eqref{maxvolnmf:eq:Hminbregsurro} is of the form $$H^{k+1}=\frac{1}{\tilde{\alpha}\|H^{k+1}\|_F^2+\tilde{\sigma}}[Q^k-e\nu^\top]_+,$$
	where $\nu\in\R^{n}$.
\end{corollary}
\begin{proof}
Consider the Lagrangian of~\eqref{maxvolnmf:eq:Hminbregsurro}
$$\La_{t_k}(H,\Lambda,\nu)=t_k(H)-\langle H,\Lambda\rangle+\langle H^\top e-e,\nu\rangle$$
where $\Lambda\in\R^{r\times n}_+$ and $\nu\in\R^n$. According to the KKT optimality conditions:

\begin{empheq}[left=\empheqlbrace]{align}
	H^{k+1} &\in \Delta^{r \times n}, \\
	\langle\Lambda^*,H^{k+1}\rangle &= 0, \label{maxvolnmf:eq:kkt12}\\
	\nabla t_k(H^{k+1})-\Lambda^*+e{\nu^*}^\top &= 0, \label{maxvolnmf:eq:kkt13}  
\end{empheq}

\begin{empheq}[left=\Leftrightarrow\empheqlbrace]{align}
	H^{k+1} &\in \Delta^{r \times n}, \\
	\langle\nabla\tilde{h}(H^{k+1})-Q^k+e{\nu^*}^\top,H^{k+1}\rangle &= 0, \label{maxvolnmf:eq:kkt22}\\
	\nabla\tilde{h}(H^{k+1})-Q^k+e{\nu^*}^\top &\geq 0, \label{maxvolnmf:eq:kkt23}
\end{empheq}
where~\eqref{maxvolnmf:eq:kkt22} is coming from substituting~\eqref{maxvolnmf:eq:kkt13} in~\eqref{maxvolnmf:eq:kkt12}, and~\eqref{maxvolnmf:eq:kkt23} is coming from the fact that $\Lambda^*\geq0$. First, combining~\eqref{maxvolnmf:eq:kkt22} and~\eqref{maxvolnmf:eq:kkt23}, we have
\begin{equation}
	\label{maxvolnmf:eq:kkt_hadamard}
	(\nabla\tilde{h}(H^{k+1})-Q^k+e{\nu^*}^\top) \circ H^{k+1} = 0
\end{equation}
where $\circ$ is the Hadamard product. For all $p$ in $1,\dots,r$, for all $j$ in $1,\dots,n$,
\begin{enumerate}
	\item if $Q^k(p,j)-\nu^*_j<0$, $\nabla\tilde{h}(H^{k+1})(p,j)-(Q^k(p,j)-\nu^*_j)>0$ because $\nabla\tilde{h}(H)=(\tilde{\alpha}\|H\|_F^2+\tilde{\sigma})H\geq0$, then \eqref{maxvolnmf:eq:kkt_hadamard} $\Rightarrow H^{k+1}(p,j)=0$,
	
	\item if $Q^k(p,j)-\nu^*_j=0$, $\nabla\tilde{h}(H^{k+1})(p,j)=(\tilde{\alpha}\|H^{k+1}\|_F^2+\tilde{\sigma})H^{k+1}(p,j)$ so~\eqref{maxvolnmf:eq:kkt_hadamard} $\Rightarrow H^{k+1}(p,j)=0$,
	
	\item if $Q^k(p,j)-\nu^*_j>0$, $\nabla\tilde{h}(H^{k+1})(p,j)=(\tilde{\alpha}\|H^{k+1}\|_F^2+\tilde{\sigma})H^{k+1}>0$ by~\eqref{maxvolnmf:eq:kkt23}, then~\eqref{maxvolnmf:eq:kkt_hadamard} $\Rightarrow \nabla\tilde{h}(H^{k+1})(p,j)-(Q^k(p,j)-\nu^*_j)=0\Leftrightarrow H^{k+1}(p,j)=\frac{Q^k(p,j)-\nu^*_j}{\tilde{\alpha}\|H^{k+1}\|_F^2+\tilde{\sigma}}$.
\end{enumerate}

In the end, $H^{k+1}=\frac{1}{\tilde{\alpha}\|H^{k+1}\|_F^2+\tilde{\sigma}}[Q^k-e{\nu^*}^\top]_+$.
\end{proof}
In particular, $\nu$ in~\Cref{maxvolnmf:cor:bregupdateHform} is such that $e^\top[Q^k-e{\nu}^\top]_+=(\tilde{\alpha}\|H^{k+1}\|_F^2+\tilde{\sigma})e^\top\in\R^n$ since $e^\top H^{k+1}=e^\top$. In other words, $[Q^k-e{\nu}^\top]_+$ projects $Q$ on a scaled probability simplex where the scaling is equal to $\tilde{\alpha}\|H^{k+1}\|_F^2+\tilde{\sigma}$. How do we find $\nu$ since it depends on $H^{k+1}$? We propose to solve this inexactly with a simple fixed point algorithm where $\|H^{k+1}\|_F^2$ is the variable to optimize. The idea is that when $\|H^{k+1}\|_F^2$ is fixed, $\nu$ has a closed form solution. So for a fixed $\|H^{k+1}\|_F^2$ we compute $\nu$, then we update $\|H^{k+1}\|_F^2$ according to the new $\nu$ and repeat this process. The algorithm is described in~\Cref{maxvolnmf:alg:Hminbregsurro}. When $\|H^{k+1}\|_F^2$ is fixed, there are several algorithms that can compute exactly $\nu$. In~\cite{held1974validation}, the proposed algorithm requires to sort the entries of each column of $Q^k$. The complexity of this algorithm is mainly due to this sorting. Once the sorting is completed, $\nu$ is found just by computing $n$ times the $\max$ between $r$ entries, which is linear. There exist faster algorithms like~\cite{condat2016fast} that do not rely on sorting. However, note that $Q^k$ is not changing in~\Cref{maxvolnmf:alg:Hminbregsurro}. Hence, using~\cite{held1974validation} to compute $\nu$ in \cref{maxvolnmf:alg:Hminbregsurro:nuline} only has a linear complexity if $Q^k$ is sorted only once before the \textbf{while} loop. In our code, $\epsilon$ is fixed to $10^{-6}$ and the \textbf{while} loop cannot exceed 100 iterations.

\begin{algorithm}[htb!]
	\caption{Algorithm for~\eqref{maxvolnmf:eq:Hminbregsurro}}
	\label{maxvolnmf:alg:Hminbregsurro}
	\SetKwInOut{Init}{init}
	\DontPrintSemicolon
	\KwIn{$Q^k,\tilde{\alpha},\tilde{\sigma}$}
	\Init{$\|H^{k+1}\|_F^2,\nu$}
	\KwOut{$H^{k+1}$}
	\While{$\frac{\|H^{k+1}\|_F^2-\left\|\frac{1}{\tilde{\alpha}\|H^{k+1}\|_F^2+\tilde{\sigma}}[Q^k-e\nu^\top]_+\right\|_F^2}{\|H^{k+1}\|_F^2}>\epsilon$}{
		compute $\nu$ such that $\frac{1}{\tilde{\alpha}\|H^{k+1}\|_F^2+\tilde{\sigma}}[Q^k-e\nu^\top]_+\in\Delta^{r\times n}$\label{maxvolnmf:alg:Hminbregsurro:nuline}\;
		update $\|H^{k+1}\|_F^2$ to $\left\|\frac{1}{\tilde{\alpha}\|H^{k+1}\|_F^2+\tilde{\sigma}}[Q^k-e\nu^\top]_+\right\|_F^2$\;
	}
	$H^{k+1}=\frac{1}{\tilde{\alpha}\|H^{k+1}\|_F^2+\tilde{\sigma}}[Q^k-e\nu^\top]_+$\;
\end{algorithm}

\paragraph{Updating $Y$}

Recall the ADMM update of $Y^{k+1}$, that is
\begin{equation}
Y^{k+1}=\argmin_{Y\succ-\delta I}-\lambda\logdet(Y+\delta I)+\langle Y,\Lambda\rangle+\frac{\rho}{2}\|Y-HH^\top\|_F^2.
\end{equation}
Consider the change of variable $Z=Y+\delta I$,
\begin{equation}
Y^{k+1}+\delta I=\argmin_{Z\succ0}-\lambda\logdet(Z)+\frac{\rho}{2}\left\|Z-\left(HH^\top+\delta I -\frac{1}{\rho}\Lambda\right)\right\|_F^2. \label{maxvolnmf:eq:admmYafterchangevar}
\end{equation}
According to~\cite[Lemma 2.1]{wang2010solving},~\eqref{maxvolnmf:eq:admmYafterchangevar} has a closed form solution which is $$\Phi_{\frac{\lambda}{\rho}}^+\left(HH^\top+\delta I -\frac{1}{\rho}\Lambda\right)$$ where $\Phi_\gamma^+(x)=\frac{1}{2}(\sqrt{x^2+4\gamma}+x)$ and for a symmetric $A$ with an eigen value decomposition $A=PDP^\top$ and $D=\diag(d)$, $\Phi_\gamma^+(A)=P\diag(\Phi_\gamma^+(d))P^\top$ where $\Phi_\gamma^+(d)$ is applied element-wise. In the end,
$$Y^{k+1}=\Phi_{\frac{\lambda}{\rho}}^+\left(HH^\top+\delta I -\frac{1}{\rho}\Lambda\right)-\delta I.$$


\subsection{Comparison of the two algorithms}

Here, we compare the different proposed algorithms for MaxVol NMF, both on synthetic datasets and real datasets. The results are averaged over 10 runs and are presented on \Cref{maxvolnmf:fig:algos}. For the synthetic dataset, $W$ is drawn following a uniform distribution in $[0,1]$ and $H$ is such that each of its column is drawn following a Dirichlet distribution where the concentration parameters are all equal to $0.2$. The input matrix is then just $X=WH$. A different $X$ is drawn at each run. The compared algorithms are Adgrad2 (\Cref{maxvolnmf:sec:adgrad}), ADMM (\Cref{maxvolnmf:sec:admm}) and ADMM+Adgrad. ADMM+Adgrad has the same formulation as in~\eqref{maxvolnmf:eq:admmmaxvol}, but the update for $H$~\eqref{maxvolnmf:eq:admm_H} is performed using the adaptive gradient descent method instead of minimizing the proposed Bregman surrogate. Regardless of the dataset and of the algorithm, the number of iterations is fixed to 500, the number of inner iterations\footnote{This value represents how many times $H$ is updated in a row before updating $W$, and vice-versa.} is fixed to 20, $\lambda$ and $\delta$ are fixed to $1$, the automatic tuning of $\lambda$ is switched off because it changes the cost function. In \Cref{maxvolnmf:fig:algos}, on both synthetic data and Moffett, ADMM with $\rho=0.01$ has the best convergence speed and the lowest error. Still on synthetic data and Moffett, we can see how the proposed Bregman surrogate provides a nice approximation of the original ADMM formulation~\eqref{maxvolnmf:eq:admmmaxvol}. For equal $\rho$'s, ADMM always converges faster and to a lower error than ADMM+Adgrad. This experimentally justifies our choice for the use of a Bregman surrogate to update $H$ in the ADMM formulation of MaxVol NMF. However, this is at the cost of a higher computation time, due to \Cref{maxvolnmf:alg:Hminbregsurro}, as it can be seen in the reported average times in \Cref{maxvolnmf:tab:averagetimesynth}. One can always increase the tolerance threshold $\epsilon$ in \Cref{maxvolnmf:alg:Hminbregsurro}, but should remain careful. Let us increase $\epsilon$ to $10^{-3}$. The computation time of ADMM with $\rho=0.01$ is greatly improved, as a run on the synthetic dataset only lasts 2.44s in average. However, for $\rho=0.1$ the algorithm diverges, as it can be seen on \Cref{maxvolnmf:fig:admm_bigger_epsi}, and the computation time is increased to 6.90s in average. Finally, ADMM is not always better than Adgrad2, like with Samson on \Cref{maxvolnmf:fig:algos_samson}. Reasons as to why one algorithm would be better than the other are, up to now, unknown.

\pagebreak

\phantom{e}

\vfill

\begin{table}[h]
    \begin{tabular}{c||c|c|c|c|c}
        \multirow{2}*{Alg.} & \multirow{2}*{Adgrad2} & ADMM+Adgrad& ADMM+Adgrad& ADMM & ADMM \\ 
        & & $\rho=0.01$ & $\rho=0.1$ & $\rho=0.01$ & $\rho=0.1$ \\ \hline
        Time (s) & 3.67 & 2.88 & 2.39 & 5.33 & 23.5
    \end{tabular}
    \caption{Average time per run on synthetic datasets}
    \label{maxvolnmf:tab:averagetimesynth}
\end{table}

\vfill

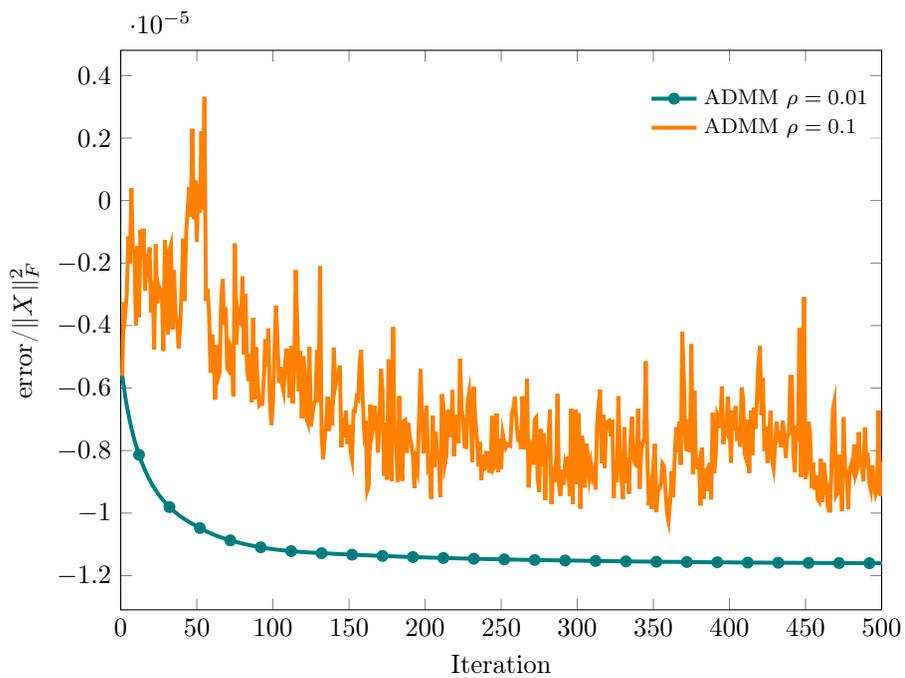
\begin{figure}[h]
    \centering
    \begin{tikzpicture}
        \begin{axis}[
                    width=0.9\linewidth,
                    height=9cm,
                    xmin = 0,
                    xmax = 500,
                    ylabel = error/$\|X\|_F^2$,
                    xlabel = {Iteration},
                    cycle list name=exotic,
                    mark size = 1.5pt,
                    mark repeat = 20,
                    legend cell align={left},
                    legend style={font=\footnotesize,at={(1,0.95)},anchor=north east,draw=none,fill opacity=0.5,text opacity=1}]
        \addplot+[mark phase=12,line width = 1.5pt] table[x expr=\coordindex+1, y index = 3]{maxvolnmf/figures/div_conv_mean_50_500_5.txt};
        \addplot+[mark = none,line width = 1.5pt] table[x expr=\coordindex+1, y index = 4]{maxvolnmf/figures/div_conv_mean_50_500_5.txt};
        \legend{ADMM $\rho=0.01$,ADMM $\rho=0.1$}
        \end{axis}
    \end{tikzpicture}
    \caption{ADMM on synthetic dataset with $\epsilon=10^{-3}$}
    \label{maxvolnmf:fig:admm_bigger_epsi}
\end{figure}

\vfill

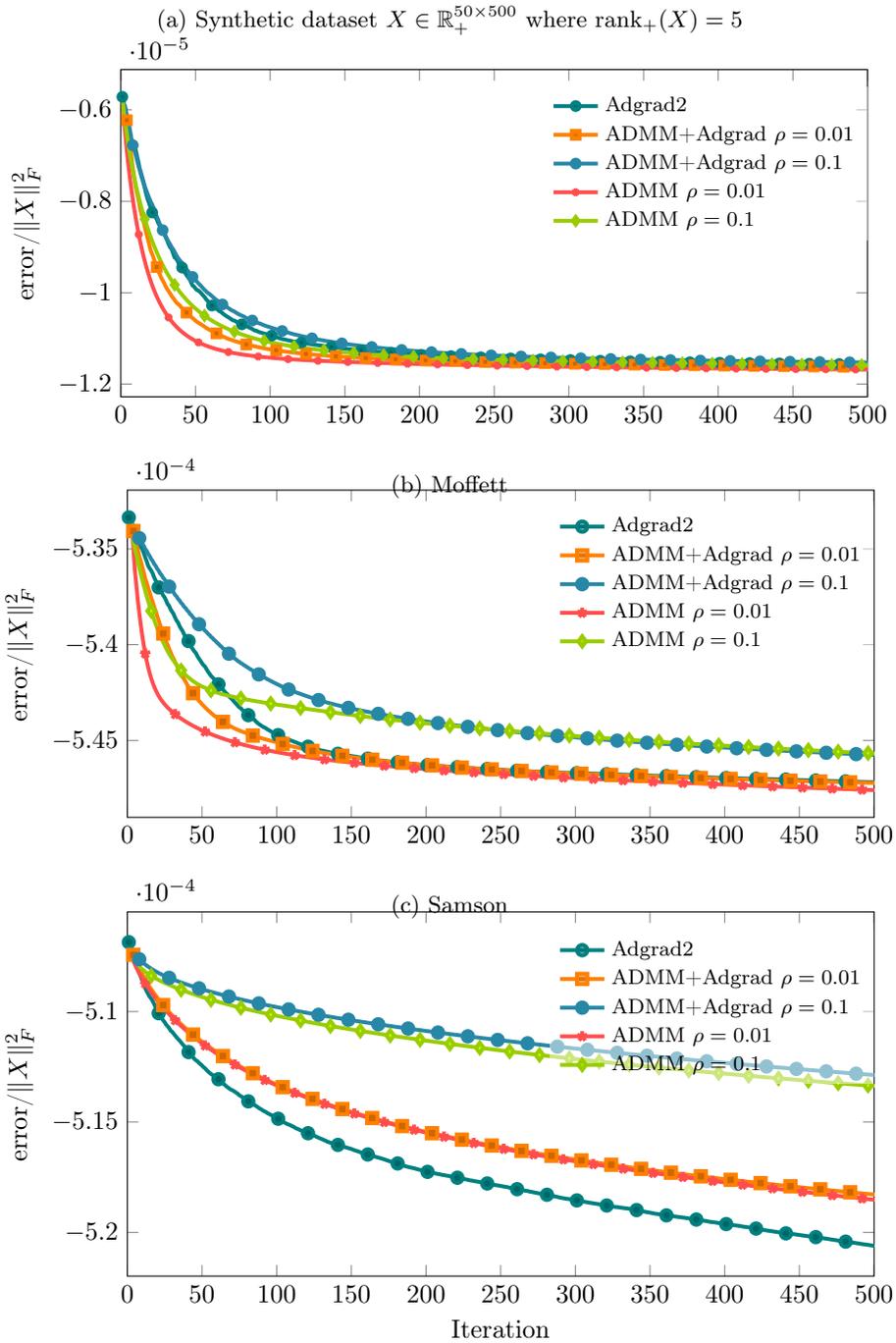
\begin{figure}[htbp!]
    \centering
    \begin{subfigure}{\linewidth}
        \centering
        \caption{Synthetic dataset $X\in\R^{50\times 500}_+$ where $\rank_+(X)=5$}
        \label{maxvolnmf:fig:algos_synth}
        \begin{tikzpicture}
            \begin{axis}[
                        width=0.9\linewidth,
                        height=6cm,
                        xmin = 0,
                        xmax = 500,
                        ylabel = error/$\|X\|_F^2$,
                        cycle list name=exotic,
                        mark size = 1.5pt,
                        mark repeat = 20,
                        legend cell align={left},
                        legend style={font=\footnotesize,at={(1,0.95)},anchor=north east,draw=none,fill opacity=0.5,text opacity=1}]
            \addplot+[mark phase= 0,line width = 1.5pt] table[x expr=\coordindex+1, y index = 0]{maxvolnmf/figures/conv_mean_50_500_5.txt};
            \addplot+[mark phase= 4,line width = 1.5pt] table[x expr=\coordindex+1, y index = 1]{maxvolnmf/figures/conv_mean_50_500_5.txt};
            \addplot+[mark phase= 8,line width = 1.5pt] table[x expr=\coordindex+1, y index = 2]{maxvolnmf/figures/conv_mean_50_500_5.txt};
            \addplot+[mark phase=12,line width = 1.5pt] table[x expr=\coordindex+1, y index = 3]{maxvolnmf/figures/conv_mean_50_500_5.txt};
            \addplot+[mark phase=16,line width = 1.5pt] table[x expr=\coordindex+1, y index = 4]{maxvolnmf/figures/conv_mean_50_500_5.txt};
            \legend{Adgrad2,ADMM+Adgrad $\rho=0.01$,ADMM+Adgrad $\rho=0.1$,ADMM $\rho=0.01$,ADMM $\rho=0.1$}
            \end{axis}
        \end{tikzpicture}
    \end{subfigure}

    \vspace{0.3cm}

    \begin{subfigure}{\linewidth}
        \centering
        \caption{Moffett}
        \label{maxvolnmf:fig:algos_moffett}
        \vspace{-0.6cm}
        \begin{tikzpicture}
            \begin{axis}[
                        width=0.9\linewidth,
                        height=6cm,
                        xmin = 0,
                        xmax = 500,
                        ylabel = error/$\|X\|_F^2$,
                        cycle list name=exotic,
                        mark size = 2pt,
                        mark repeat = 20,
                        legend cell align={left},
                        legend style={font=\footnotesize,at={(1,0.95)},anchor=north east,draw=none,fill opacity=0.5,text opacity=1}]
            \addplot+[mark phase= 0,line width = 1.5pt] table[x expr=\coordindex+1, y index = 0]{maxvolnmf/figures/conv_moffett.txt};
            \addplot+[mark phase= 4,line width = 1.5pt] table[x expr=\coordindex+1, y index = 1]{maxvolnmf/figures/conv_moffett.txt};
            \addplot+[mark phase= 8,line width = 1.5pt] table[x expr=\coordindex+1, y index = 2]{maxvolnmf/figures/conv_moffett.txt};
            \addplot+[mark phase=12,line width = 1.5pt] table[x expr=\coordindex+1, y index = 3]{maxvolnmf/figures/conv_moffett.txt};
            \addplot+[mark phase=16,line width = 1.5pt] table[x expr=\coordindex+1, y index = 4]{maxvolnmf/figures/conv_moffett.txt};
            \legend{Adgrad2,ADMM+Adgrad $\rho=0.01$,ADMM+Adgrad $\rho=0.1$,ADMM $\rho=0.01$,ADMM $\rho=0.1$}
            \end{axis}
        \end{tikzpicture}
    \end{subfigure}

    \vspace{0.3cm}

    \begin{subfigure}{\linewidth}
        \centering
        \caption{Samson}
        \label{maxvolnmf:fig:algos_samson}
        \vspace{-0.6cm}
        \begin{tikzpicture}
            \begin{axis}[
                        width=0.9\linewidth,
                        height=6.5cm,
                        xmin = 0,
                        xmax = 500,
                        ylabel = error/$\|X\|_F^2$,
                        xlabel = {Iteration},
                        cycle list name=exotic,
                        mark size = 2pt,
                        mark repeat = 20,
                        legend cell align={left},
                        legend style={font=\footnotesize,at={(1,0.95)},anchor=north east,draw=none,fill opacity=0.5,text opacity=1}]
            \addplot+[mark phase= 0,line width = 1.5pt] table[x expr=\coordindex+1, y index = 0]{maxvolnmf/figures/conv_samson.txt};
            \addplot+[mark phase= 4,line width = 1.5pt] table[x expr=\coordindex+1, y index = 1]{maxvolnmf/figures/conv_samson.txt};
            \addplot+[mark phase= 8,line width = 1.5pt] table[x expr=\coordindex+1, y index = 2]{maxvolnmf/figures/conv_samson.txt};
            \addplot+[mark phase=12,line width = 1.5pt] table[x expr=\coordindex+1, y index = 3]{maxvolnmf/figures/conv_samson.txt};
            \addplot+[mark phase=16,line width = 1.5pt] table[x expr=\coordindex+1, y index = 4]{maxvolnmf/figures/conv_samson.txt};
            \legend{Adgrad2,ADMM+Adgrad $\rho=0.01$,ADMM+Adgrad $\rho=0.1$,ADMM $\rho=0.01$,ADMM $\rho=0.1$}
            \end{axis}
        \end{tikzpicture}
    \end{subfigure}
    \caption{Comparison of algorithms for MaxVol NMF on various datasets}
    \label{maxvolnmf:fig:algos}
\end{figure}

\pagebreak

\section{Normalized MaxVol NMF}\label{maxvolnmf:sec:normmaxvolnmf}

In \Cref{maxvolnmf:sec:maxvolnmf}, we mentioned that a drawback of MaxVol \eqref{maxvolnmf:eq:nonexactmaxvolmf} is its bias towards an even clustering. Here, we introduce a normalized variant of MaxVol NMF, where the volume of the row wise normalized $H$ is maximized instead of the standard volume:

\begin{mini}|s|
    {\scriptstyle W,H}
    {f(W,H):=\frac{1}{2}\| X - WH \|_F^2 -\lambda\logdet(\widetilde{H}\widetilde{H}^\top+\delta I)\label{maxvolnmf:eq:nonexactnormalizedmaxvolmf}}{}{}
    \addConstraint{W\geq0,H\geq0}
    \addConstraint{\widetilde{H}=S^{-1}H \text{ where } S=\diag(\|H(1,:)\|_2,\dots,\|H(r,:)\|_2).}
\end{mini} 

This model is interesting for several reasons. \\

When $\lambda$ is increasing, $\widetilde{H}\widetilde{H}^\top$ converges to the identity. In other words, increasing $\lambda$ acts in favor of mutually orthogonal rows of $H$. Unlike MaxVol NMF, the norm of the rows of $H$ can be anything since it is $\widetilde{H}\widetilde{H}^\top$ that converges to the identity and not $HH^\top$. In fact, Normalized MaxVol NMF can be viewed as a continuum between NMF and Orthogonal NMF (ONMF)\label{acro:ONMF}. With $\lambda=0$, NMF is retrieved. Increasing $\lambda$ progressively retrieves ONMF. Let us prove this asymptotic behavior of Normalized MaxVol. To do so, we show that the problem

\begin{mini}
    {X\in \mathbb{S}^r}{f_0(X)=\log\det X^{-1}}{\label{maxvolnmf:eq:normmaxvol}}{}
    \addConstraint{\diag(X)=e}
    \addConstraint{0\leq X \leq 1,}
\end{mini}
where $\dom f_0=\mathbb{S}^r_{++}$ has $X=I$ as a unique minimizer. Again, we solve this problem through its dual using the conjugate of $f_0$, which has already been computed in \eqref{maxvolnmf:eq:conjugatef0}. First, \eqref{maxvolnmf:eq:normmaxvol} can be reformulated as 
\begin{mini}
    {X\in \mathbb{S}^r}{f_0(X)=\log\det X^{-1}}{}{}
    \addConstraint{\langle E_{ii},X\rangle=1 \text{ for all } i}
    \addConstraint{\langle -E_{ij},X\rangle\leq0 \text{ for all } i,j}
    \addConstraint{\langle E_{ij},X\rangle\leq1 \text{ for all } i,j.}
\end{mini}
Using again~\eqref{maxvolnmf:eq:conjindual}, we can write the dual of~\eqref{maxvolnmf:eq:normmaxvol} with the conjugate of $f_0$:

\begin{equation}
    g(\lambda,\gamma,\nu)=\left\{\begin{array}{ll}
        \log\det\bigl(\diag(\nu)+\gamma-\lambda\bigr)+r-\langle J,\gamma \rangle - e^\top\nu & \text{ if } \diag(\nu)+\gamma-\lambda \succ0,  \\
        \infty&\text{ otherwise,} 
    \end{array}\right.
\end{equation}
where $\lambda\in\R^{r \times r}_+$, $\gamma\in\R^{r \times r}_+$ and $\nu\in\R^r$. Let $\lambda^*=0,\gamma^*=0,\nu^*=e$ and $X^*=I$. We have $f_0(X^*)=g(\lambda^*,\gamma^*,\nu^*)=0$, meaning that there is no duality gap and that $X^*$ is a solution of \eqref{maxvolnmf:eq:normmaxvol}. Finally, $X^*$ is the unique solution because $f_0$ is strongly convex.\\

It is possible to control the range of the volume criterion via $\delta$. When $\delta>0$, we have $$\logdet(\widetilde{H}\widetilde{H}^\top+\delta I)\in\left[\log(1+r\delta^{-1})+r\log\delta,r\log(1+\delta)\right]$$ where the minimum and maximum are respectfully reached when $\widetilde{H}\widetilde{H}^\top=J$ and $\widetilde{H}\widetilde{H}^\top=I$. The parameter $\delta$ then controls how larger can the volume of $\widetilde{H}$ be, while $\lambda$ still balances the reconstruction error and the volume criterion. By increasing $\delta$, the dynamical range is reduced, as it can be seen on \Cref{maxvolnmf:fig:range_normvol}. With respect to $\lambda$ and the reconstruction error, it is then harder to increase the volume of $\widetilde{H}$. In the context of HU, $\delta$ can be seen as a mixture tolerance parameter, while $\lambda$ is more like a noise level estimation parameter.\\

\begin{figure}[htbp!]
    \centering
    \caption[Range of the logdet depending on $r$ and $\delta$]{Value $r\log(1+\delta)-\log(1+r\delta^{-1})-r\log\delta$ depending on $\delta$ for various $r$'s.}
    \label{maxvolnmf:fig:range_normvol}
    \begin{tikzpicture}
        \begin{axis}[
                    width=0.9\linewidth,
                    height=6.5cm,
                    ymax = 1e2,
                    grid=major,
                    ymode=log,
                    xmin = 0,
                    xmax = 1,
                    ylabel = {},
                    xlabel = {$\delta$},
                    cycle list name=exotic,
                    legend columns = 2,
                    legend cell align={left},
                    legend style={font=\footnotesize,at={(1,0.95)},anchor=north east,draw=none,fill opacity=0.5,text opacity=1}]
        \foreach \j in {1,...,8}{
            \pgfmathtruncatemacro{\myresult}{\j+2}%
            \addplot+[mark=none,line width = 1.5pt] table[x index=0, y index = \j]{maxvolnmf/figures/range_normvol.txt};
            \addlegendentryexpanded{$r=\myresult$};
        }
        \end{axis}
    \end{tikzpicture}
\end{figure}
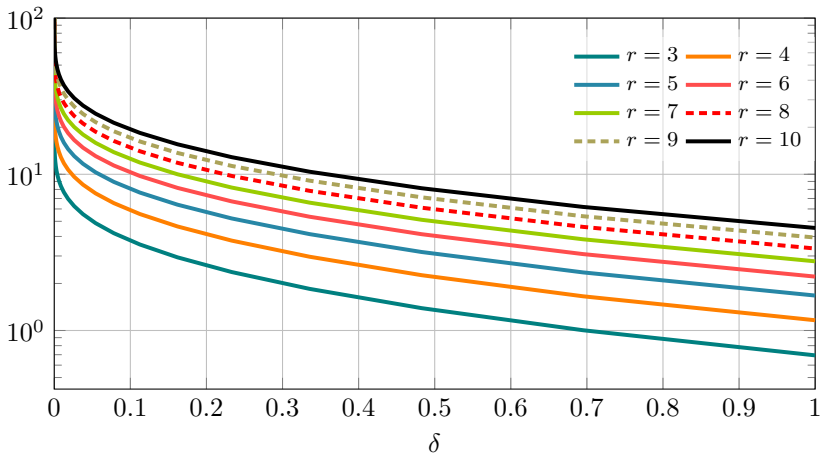

Another advantage of the normalized formulation of MaxVol NMF is the removal of the simplex structure on $H$. Let us remind that this simplex structure is not without loss of generality. In HU for instance, if there are two pure pixels of tree but one of them receives more light than the other, then a perfect unmixing would require a different grass endmember for each one. In other words, the simplex structure might require a larger rank. Also, the projection on the probability simplex is costly.\\

In spite of the benefits the normalized variant brings, we ``lose'' two aspects of the vanilla MaxVol NMF. The most notable one is the identifiability. It remains unknown if Normalized MaxVol NMF is identifiable or not. We also lose the possibility to solve Normalized MaxVol NMF with the same ADMM formulation that we used for MaxVol NMF. We could not find a kernel that would provide us with a Bregman surrogate. Even if we did, the considered Bregman surrogate would then need to be nice enough to be easily solved. This is not a big issue since we can still solve it with the adaptive accelerated gradient descent method, which is described in \Cref{maxvolnmf:sec:solvenormmaxvolnmf}.

\section{Solving Normalized MaxVol NMF}\label{maxvolnmf:sec:solvenormmaxvolnmf}

Here, we propose to solve Normalized MaxVol NMF \eqref{maxvolnmf:eq:nonexactnormalizedmaxvolmf} with the adaptive accelerated gradient descent method, already introduced in \Cref{maxvolnmf:sec:adgrad}. Like it has been said in \Cref{maxvolnmf:sec:adgrad}, we only need to know the gradient. Hence, in this section, we only describe the computation of the gradient. The algorithm is exactly the same as \Cref{maxvolnmf:alg:Adgrad2}, except for the projected gradient step in \cref{maxvolnmf:alg:lineupdateH} that should be replaced by $H=[\Hextra-\gamma_{H}\nabla f(\Hextra)]_+$ because there is no simplex structure in the normalized variant. Let us now compute the gradient of $f$ in \eqref{maxvolnmf:eq:nonexactnormalizedmaxvolmf} relatively to $H$.\\

Knowing that 
\begin{align}
    \frac{\partial \widetilde{H}(k,:)}{\partial H(k,j)}&=\begin{pmatrix}
        -\frac{H(k,1)H(k,j)}{\|H(k,:)\|^3} & \cdots & \frac{\|H(k,:)\|^2-H(k,j)^2}{\|H(k,:)\|^3} & \cdots & -\frac{H(k,n)H(k,j)}{\|H(k,:)\|^3}
    \end{pmatrix}\\
    &=\frac{1}{\|H(k,:)\|^3}\left(\|H(k,:)\|^2 e_j^\top - H(k,j)H(k,:)\right),
\end{align}
and using the chain rule, we have that

\begin{align}
    \MoveEqLeft[3]\frac{\partial \logdet(\widetilde{H}\widetilde{H}^\top+\delta I)}{\partial H(k,j)} = \left\langle\frac{\partial \logdet(\widetilde{H}\widetilde{H}^\top+\delta I)}{\partial \widetilde{H}},\frac{\partial \widetilde{H}}{\partial H(k,j)}\right\rangle \\
    = & \left\langle 2(\widetilde{H}\widetilde{H}^\top+\delta I)^{-1}\widetilde{H}, \frac{1}{\|H(k,:)\|}E_{k,j} - \frac{H(k,j)}{\|H(k,:)\|^3}e_k H(k,:) \right\rangle \\
    = &\begin{multlined}[t]
        \frac{1}{\|H(k,:)\|}\langle 2(\widetilde{H}\widetilde{H}^\top+\delta I)^{-1}\widetilde{H},E_{k,j} \rangle \\ - \frac{1}{\|H(k,:)\|}\langle 2(\widetilde{H}\widetilde{H}^\top+\delta I)^{-1}, e_k \widetilde{H}(k,:)\widetilde{H}^\top \rangle \widetilde{H}(k,j).
    \end{multlined}
\end{align}
In the end, 
\begin{equation}
    \frac{\partial \logdet(\widetilde{H}\widetilde{H}^\top +\delta I)}{\partial H} = 2S^{-1}\left[(\widetilde{H}\widetilde{H}^\top +\delta I)^{-1}-\diag\left((\widetilde{H}\widetilde{H}^\top +\delta I)^{-1}\widetilde{H}\widetilde{H}^\top\right)\right]\widetilde{H}
\end{equation}
and
\begin{equation}
    \frac{\partial f}{\partial H}=W^\top(WH-X) - 2\lambda S^{-1}\left[(\widetilde{H}\widetilde{H}^\top +\delta I)^{-1}-\diag\left((\widetilde{H}\widetilde{H}^\top +\delta I)^{-1}\widetilde{H}\widetilde{H}^\top\right)\right]\widetilde{H}.
\end{equation}

\pagebreak

\section{Performance of Normalized MaxVol NMF on Hyperspectral Unmixing}\label{maxvolnmf:sec:normmaxvolnmfexp}

In this section, we evaluate the performance of Normalized MaxVol NMF on famous hyperspectral datasets. Results can be compared with~\cite{zhu2017hyperspectral} where some ground-truths for a variety of known hyperspectral datasets are proposed. Even if these are called ground-truths, hyperspectral ground-truths do not exist except if the measurements are performed in a controlled environment. Consider the proposed abundance maps for Urban with four endmembers in~\cite{zhu2017hyperspectral}. Clearly, some trees are detected where in fact it should be a mixture of grass and soil. Some rooftops are also detected where it should be soil. Still, the author used as many a priori knowledge as possible to provide these abundance maps and endmembers that are probably close to reality. Our message here is that ground-truths for these hyperspectral datasets should be interpreted with caution. On Moffett and on Samson, our model clearly outperforms MinVol NMF and MaxVol NMF, see \Cref{maxvolnmf:fig:moffett_nmaxvol_1_05,maxvolnmf:fig:samson_nmaxvol}. Water, soil and tree are correctly separated and their spectral signatures are very close to the expected ones in \cite{zhu2017hyperspectral}. 
\begin{figure}[htbp!]
	\hfill\fbox{\includegraphics[width=0.95\linewidth]{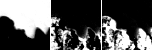}}\hfill

	\hfill\begin{tikzpicture}
        \begin{groupplot}[
            group style={
                group name=my plots,
                group size= 3 by 1,
                x descriptions at=edge bottom,
                y descriptions at=edge left,
                horizontal sep=0.07cm,
                vertical sep=0.cm,
            },
			width=0.433\linewidth,no markers,grid=major,xlabel={},ylabel={},tick label style={/pgf/number format/fixed,font=\tiny},ymin=0,ymax=0.3
        ]
        \nextgroupplot
            \addplot [line width = 1pt,blue] table [x expr=\coordindex+1, y index=0] {maxvolnmf/figures/moffett_nmaxvol_1_05.txt};
        \nextgroupplot
            \addplot [line width = 1pt,red] table [x expr=\coordindex+1, y index=1] {maxvolnmf/figures/moffett_nmaxvol_1_05.txt};
        \nextgroupplot
            \addplot [line width = 1pt] table [x expr=\coordindex+1, y index=2] {maxvolnmf/figures/moffett_nmaxvol_1_05.txt};
        \end{groupplot}
	\end{tikzpicture}\hfill
    \caption[Abundance maps and endmembers by Normalized MaxVol NMF on Moffett]{Abundance maps and endmembers ({\color{blue}water}, {\color{red}tree} and soil) by Normalized MaxVol NMF on Moffett, with $\lambda=1$ and $\delta=0.5$.}
    \label{maxvolnmf:fig:moffett_nmaxvol_1_05}
\end{figure}
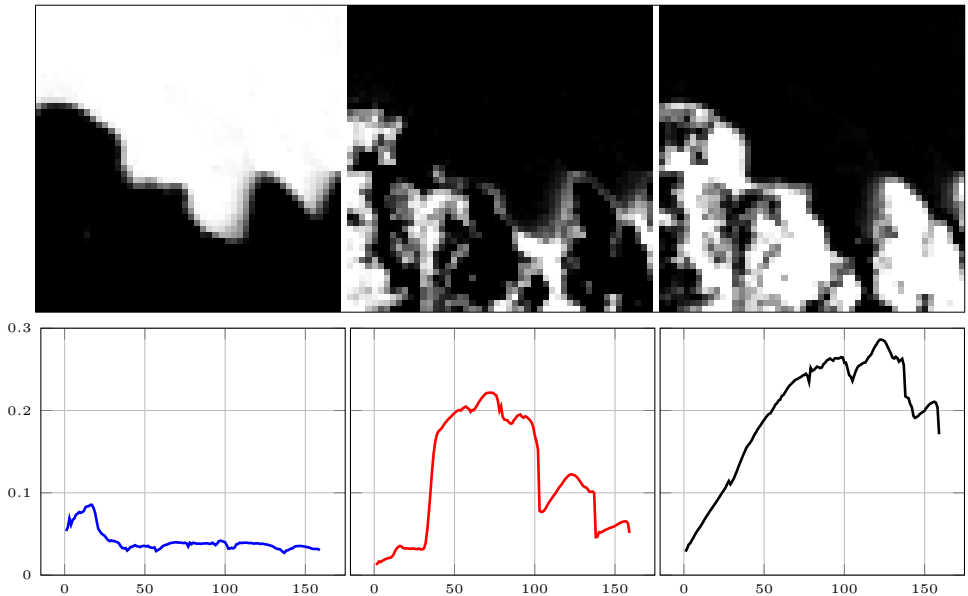

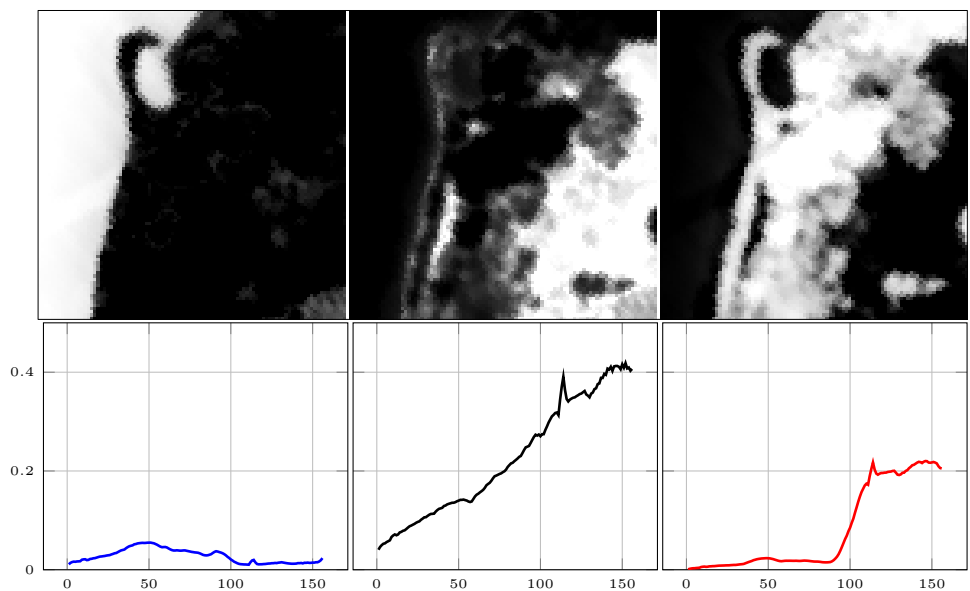
\begin{figure}[htbp!]
	\hfill\fbox{\includegraphics[width=0.95\linewidth]{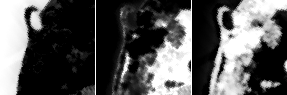}}\hfill

	\hfill\begin{tikzpicture}
        \begin{groupplot}[
            group style={
                group name=my plots,
                group size= 3 by 1,
                x descriptions at=edge bottom,
                y descriptions at=edge left,
                horizontal sep=0.07cm,
                vertical sep=0.cm,
            },
			width=0.433\linewidth,no markers,grid=major,xlabel={},ylabel={},tick label style={/pgf/number format/fixed,font=\tiny},ymin=0,ymax=0.5
        ]
        \nextgroupplot
            \addplot [line width = 1pt,blue] table [x expr=\coordindex+1, y index=0] {maxvolnmf/figures/samson_nmaxvol.txt};
        \nextgroupplot
            \addplot [line width = 1pt] table [x expr=\coordindex+1, y index=1] {maxvolnmf/figures/samson_nmaxvol.txt};
        \nextgroupplot
            \addplot [line width = 1pt,red] table [x expr=\coordindex+1, y index=2] {maxvolnmf/figures/samson_nmaxvol.txt};
        \end{groupplot}
	\end{tikzpicture}\hfill
    \caption[Abundance maps and endmembers by Normalized MaxVol NMF on Samson]{Abundance maps and endmembers ({\color{blue}water}, soil and {\color{red}tree}) by Normalized MaxVol NMF on Samson, with $\lambda=1$ and $\delta=0.5$.}
    \label{maxvolnmf:fig:samson_nmaxvol}
\end{figure}

Now that we very briefly assured that our model works on simple datasets, let us comment on one of its interesting features. Consider the Samson dataset again, but with $r=6$. With MinVol NMF, the excessive endmembers should be brought to zero by tuning $\lambda$ and $\delta$. Otherwise, MinVol NMF would be over-parameterized and would just learn the noise, which would output a non-interpretable matrix factorization. On \Cref{maxvolnmf:fig:samson_r6_smaxvol_05_05}, we can see that increasing the rank above the standard $r=3$ for Samson allowed MaxVol NMF to learn more spectral varieties. We can see three different kinds of tree and two different kinds of soil. We can then add together the rows of $H$ that corresponds to varieties of the same endmembers. The resulting merged abundance maps are available on \Cref{maxvolnmf:fig:grouped_samson}. One can notice how close are the abundance maps on \Cref{maxvolnmf:fig:grouped_samson} and \Cref{maxvolnmf:fig:samson_nmaxvol} to each other. Actually, results with $r=6$ are more satisfying for the water unmixing. On \Cref{maxvolnmf:fig:samson_nmaxvol}, some small artifacts of false-positives can be seen, especially in the bottom right corner of the abundance map corresponding to water. These artifacts are not visible on \Cref{maxvolnmf:fig:grouped_samson}. It would seem that MaxVol NMF allows to increase the number of parameters in order to improve the results in a controlled manner, at least in the context of HU. 

Let us confirm this behavior on the Urban dataset. This dataset is particularly insightful in this case because it is known for having meaningful unmixing results for $r=4,5\text{ and }6$. Results are displayed on \Cref{maxvolnmf:fig:urban_nmaxvol}. With $r=4$, we have roof, grass, a combination of asphalt and soil, and tree. With $r=5$, the distinction is being made between asphalt and dirt. With $r=6$, the distinction is being made between grass and dry grass. Typical ground-truths with $r=6$ rather suggest a distinction between two kinds of roof, instead of grass and dry grass. Here our model propose another interpretation for $r=6$ which still makes sense.

The last experiment is on the Jasper dataset, where ground-truths suggest four endmembers: tree, water, soil and road. Unmixing algorithms often struggle to correctly separate water and road on Jasper. On \Cref{maxvolnmf:fig:jasper4}, we can see that our model achieves not ideal but nonetheless decent results. The issue with our model here is that in order to improve the distinction between water and road, $\lambda$ should be increased. However, increasing $\lambda$ might not be the best option here. There are many areas where tree and soil are heavily mixed, and increasing $\lambda$ will converge to ONMF, which cannot properly unmix these areas. One way to circumvent that is to increase the rank. Results with $r=5$ are displayed on \Cref{maxvolnmf:fig:jasper5}. The additional endmember is in fact a combination of tree and soil. With this trick, water and road are properly identified without compromising the quality of the other endmembers.

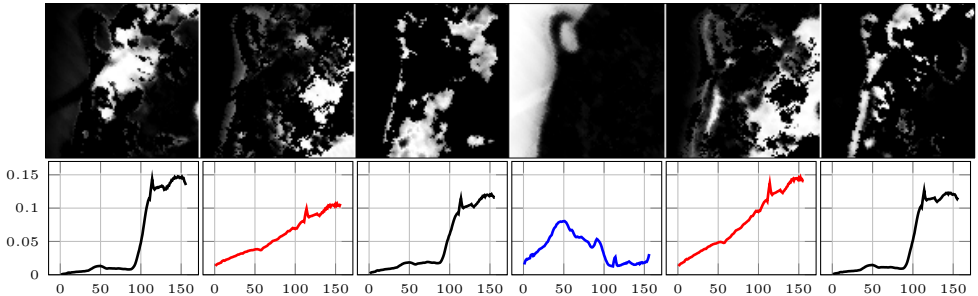
\begin{figure}[htbp!]
	\hfill\fbox{\includegraphics[width=0.95\linewidth]{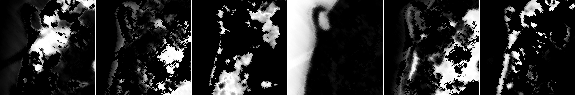}}\hfill

    \hfill\begin{tikzpicture}
        \begin{groupplot}[
            group style={
                group name=my plots,
                group size= 6 by 1,
                x descriptions at=edge bottom,
                y descriptions at=edge left,
                horizontal sep=0.06cm,
                vertical sep=0.cm,
            },
			width=0.275\linewidth,no markers,grid=major,xlabel={},ylabel={},tick label style={/pgf/number format/fixed,font=\tiny},ymin=0,ymax=0.17
        ]
        \nextgroupplot
            \addplot [line width = 1pt] table [x expr=\coordindex+1, y index=0] {maxvolnmf/figures/samson_r6_smaxvol_05_05.txt};
        \nextgroupplot
            \addplot [line width = 1pt,red] table [x expr=\coordindex+1, y index=1] {maxvolnmf/figures/samson_r6_smaxvol_05_05.txt};
        \nextgroupplot
            \addplot [line width = 1pt] table [x expr=\coordindex+1, y index=2] {maxvolnmf/figures/samson_r6_smaxvol_05_05.txt};
        \nextgroupplot
            \addplot [line width = 1pt,blue] table [x expr=\coordindex+1, y index=3] {maxvolnmf/figures/samson_r6_smaxvol_05_05.txt};
        \nextgroupplot
            \addplot [line width = 1pt,red] table [x expr=\coordindex+1, y index=4] {maxvolnmf/figures/samson_r6_smaxvol_05_05.txt};
        \nextgroupplot
            \addplot [line width = 1pt] table [x expr=\coordindex+1, y index=5] {maxvolnmf/figures/samson_r6_smaxvol_05_05.txt};
        \end{groupplot}
	\end{tikzpicture}\hfill
	\caption[Abundance maps and endmembers by Normalized MaxVol NMF on Samson with $r=6$]{Abundance maps and endmembers (tree, {\color{red}soil} and {\color{blue}water}) by Normalized MaxVol NMF with $r=6$ on Samson, with $\lambda=0.5$ and $\delta=0.5$.}
    \label{maxvolnmf:fig:samson_r6_smaxvol_05_05}
\end{figure}

\begin{figure}[htbp!]
	\hfill\fbox{\includegraphics[width=0.95\linewidth]{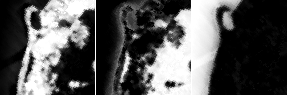}}\hfill

    \hfill\begin{tikzpicture}
        \begin{groupplot}[
            group style={
                group name=my plots,
                group size= 3 by 1,
                x descriptions at=edge bottom,
                y descriptions at=edge left,
                horizontal sep=0.07cm,
                vertical sep=0.cm,
            },
			width=0.433\linewidth,no markers,grid=major,xlabel={},ylabel={},tick label style={/pgf/number format/fixed,font=\tiny},ymin=0,ymax=0.17
        ]
        \nextgroupplot
            \addplot [line width = 1pt] table [x expr=\coordindex+1, y index=0] {maxvolnmf/figures/samson_r6_smaxvol_05_05.txt};
            \addplot [line width = 1pt,dashed] table [x expr=\coordindex+1, y index=2] {maxvolnmf/figures/samson_r6_smaxvol_05_05.txt};
            \addplot [line width = 1pt,dotted] table [x expr=\coordindex+1, y index=5] {maxvolnmf/figures/samson_r6_smaxvol_05_05.txt};
        \nextgroupplot
            \addplot [line width = 1pt,red] table [x expr=\coordindex+1, y index=1] {maxvolnmf/figures/samson_r6_smaxvol_05_05.txt};
            \addplot [line width = 1pt,red,dashed] table [x expr=\coordindex+1, y index=4] {maxvolnmf/figures/samson_r6_smaxvol_05_05.txt};
        \nextgroupplot
            \addplot [line width = 1pt,blue] table [x expr=\coordindex+1, y index=3] {maxvolnmf/figures/samson_r6_smaxvol_05_05.txt};
        \end{groupplot}
	\end{tikzpicture}\hfill
	\caption[Abundance maps grouped by endmember varieties and endmembers by Normalized MaxVol NMF with $r=6$ on Samson]{Abundance maps grouped by endmember varieties and endmembers (tree, {\color{red}soil} and {\color{blue}water}) by Normalized MaxVol NMF with $r=6$ on Samson, with $\lambda=0.5$ and $\delta=0.5$.}
    \label{maxvolnmf:fig:grouped_samson}
\end{figure}

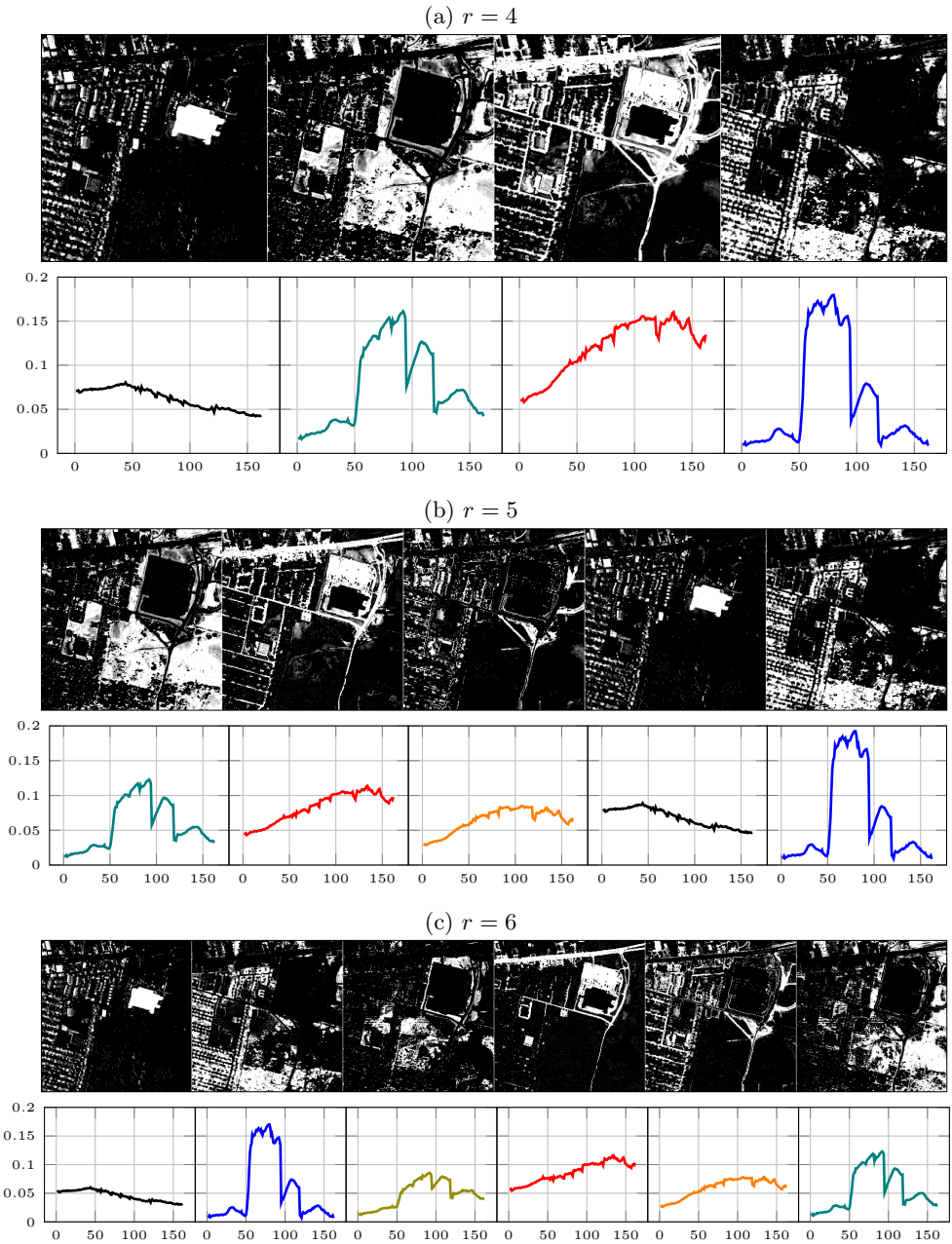
\begin{figure}[htbp!]
    \begin{subfigure}{\linewidth}
        \caption{$r=4$}
        \hfill\fbox{\includegraphics[width=0.95\linewidth]{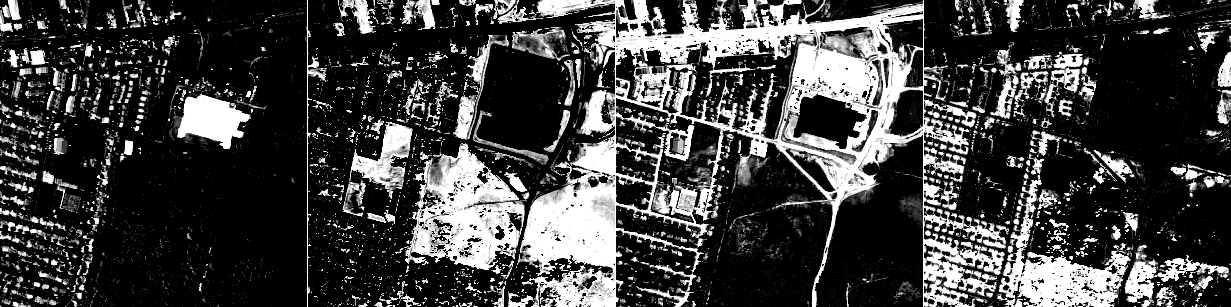}}\hfill
    
        \hfill\begin{tikzpicture}
            \begin{groupplot}[
                group style={
                    group name=my plots,
                    group size= 4 by 1,
                    x descriptions at=edge bottom,
                    y descriptions at=edge left,
                    horizontal sep=0.0cm,
                    vertical sep=0.cm,
                },
                width=0.355\linewidth,no markers,grid=major,xlabel={},ylabel={},tick label style={/pgf/number format/fixed,font=\tiny},ymin=0,ymax=0.2
            ]
            \nextgroupplot
                \addplot [line width = 1pt] table [x expr=\coordindex+1, y index=0] {maxvolnmf/figures/urban4_smaxvol_05_05.txt};
            \nextgroupplot
                \addplot [line width = 1pt,teal] table [x expr=\coordindex+1, y index=1] {maxvolnmf/figures/urban4_smaxvol_05_05.txt};
            \nextgroupplot
                \addplot [line width = 1pt,red] table [x expr=\coordindex+1, y index=2] {maxvolnmf/figures/urban4_smaxvol_05_05.txt};
            \nextgroupplot
                \addplot [line width = 1pt,blue] table [x expr=\coordindex+1, y index=3] {maxvolnmf/figures/urban4_smaxvol_05_05.txt};
            \end{groupplot}
        \end{tikzpicture}\hfill
    \end{subfigure}

    \begin{subfigure}{\linewidth}
        \caption{$r=5$}
        \hfill\fbox{\includegraphics[width=0.95\linewidth]{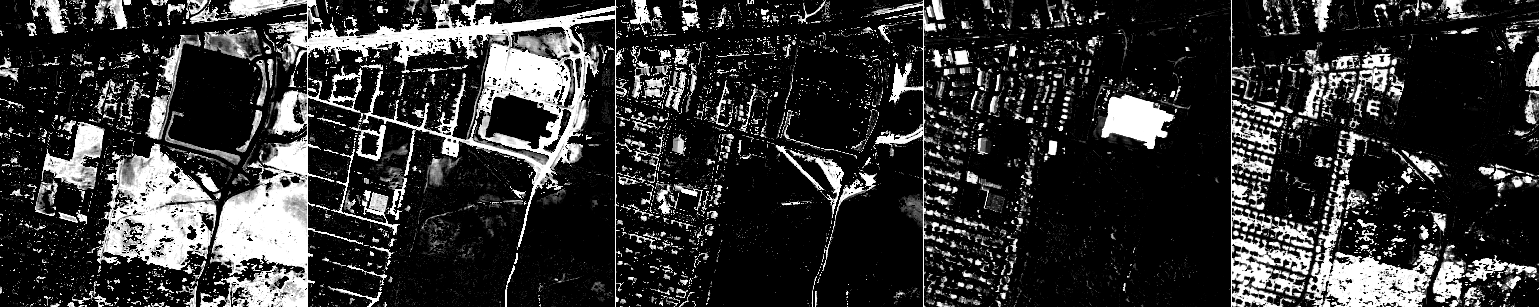}}\hfill
    
        \hfill\begin{tikzpicture}
            \begin{groupplot}[
                group style={
                    group name=my plots,
                    group size= 5 by 1,
                    x descriptions at=edge bottom,
                    y descriptions at=edge left,
                    horizontal sep=0.0cm,
                    vertical sep=0.cm,
                },
                width=0.31\linewidth,no markers,grid=major,xlabel={},ylabel={},tick label style={/pgf/number format/fixed,font=\tiny},ymin=0,ymax=0.2
            ]
            \nextgroupplot
                \addplot [line width = 1pt,teal] table [x expr=\coordindex+1, y index=0] {maxvolnmf/figures/urban5_smaxvol_05_05.txt};
            \nextgroupplot
                \addplot [line width = 1pt,red] table [x expr=\coordindex+1, y index=1] {maxvolnmf/figures/urban5_smaxvol_05_05.txt};
            \nextgroupplot
                \addplot [line width = 1pt,orange] table [x expr=\coordindex+1, y index=2] {maxvolnmf/figures/urban5_smaxvol_05_05.txt};
            \nextgroupplot
                \addplot [line width = 1pt] table [x expr=\coordindex+1, y index=3] {maxvolnmf/figures/urban5_smaxvol_05_05.txt};
            \nextgroupplot
                \addplot [line width = 1pt,blue] table [x expr=\coordindex+1, y index=4] {maxvolnmf/figures/urban5_smaxvol_05_05.txt};
            \end{groupplot}
        \end{tikzpicture}\hfill
    \end{subfigure}

    \begin{subfigure}{\linewidth}
        \caption{$r=6$}
        \hfill\fbox{\includegraphics[width=0.95\linewidth]{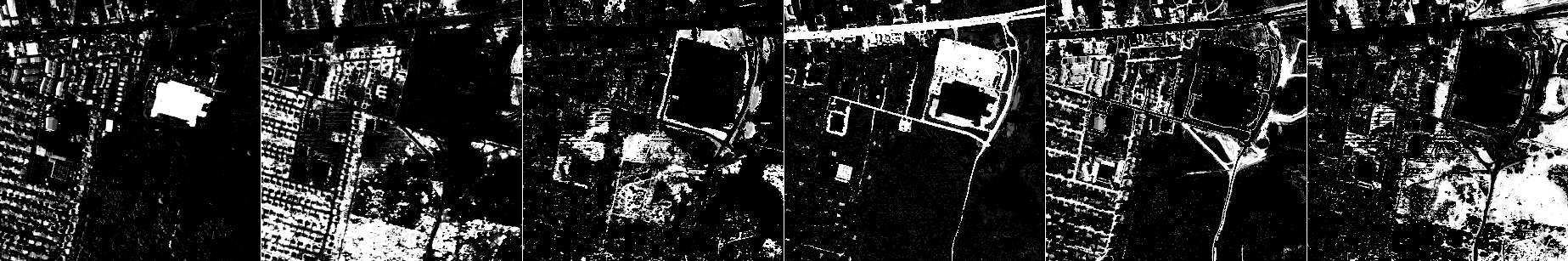}}\hfill
    
        \hfill\begin{tikzpicture}
            \begin{groupplot}[
                group style={
                    group name=my plots,
                    group size= 6 by 1,
                    x descriptions at=edge bottom,
                    y descriptions at=edge left,
                    horizontal sep=0.0cm,
                    vertical sep=0.cm,
                },
                width=0.28\linewidth,no markers,grid=major,xlabel={},ylabel={},tick label style={/pgf/number format/fixed,font=\tiny},ymin=0,ymax=0.2
            ]
            \nextgroupplot
                \addplot [line width = 1pt] table [x expr=\coordindex+1, y index=0] {maxvolnmf/figures/urban6_smaxvol_05_05.txt};
            \nextgroupplot
                \addplot [line width = 1pt,blue] table [x expr=\coordindex+1, y index=1] {maxvolnmf/figures/urban6_smaxvol_05_05.txt};
            \nextgroupplot
                \addplot [line width = 1pt,olive] table [x expr=\coordindex+1, y index=2] {maxvolnmf/figures/urban6_smaxvol_05_05.txt};
            \nextgroupplot
                \addplot [line width = 1pt,red] table [x expr=\coordindex+1, y index=3] {maxvolnmf/figures/urban6_smaxvol_05_05.txt};
            \nextgroupplot
                \addplot [line width = 1pt,orange] table [x expr=\coordindex+1, y index=4] {maxvolnmf/figures/urban6_smaxvol_05_05.txt};
            \nextgroupplot
                \addplot [line width = 1pt,teal] table [x expr=\coordindex+1, y index=5] {maxvolnmf/figures/urban6_smaxvol_05_05.txt};
            \end{groupplot}
        \end{tikzpicture}\hfill
    \end{subfigure}
	\caption[Abundance maps and endmembers by Normalized MaxVol NMF on Urban depending on $r$]{Abundance maps and endmembers (roof, {\color{blue}tree}, {\color{olive}dry grass}, {\color{red}asphalt}, {\color{orange}soil}, {\color{teal}grass}) by Normalized MaxVol NMF on Urban, with $\lambda=0.5$ and $\delta=0.5$, depending on $r$.}
    \label{maxvolnmf:fig:urban_nmaxvol}
\end{figure}



\begin{figure}[htbp!]
	\hfill\fbox{\includegraphics[width=0.95\linewidth]{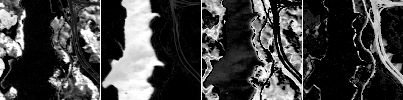}}\hfill

    \hfill\begin{tikzpicture}
        \begin{groupplot}[
            group style={
                group name=my plots,
                group size= 4 by 1,
                x descriptions at=edge bottom,
                y descriptions at=edge left,
                horizontal sep=0.0cm,
                vertical sep=0.cm,
            },
			width=0.355\linewidth,no markers,grid=major,xlabel={},ylabel={},tick label style={/pgf/number format/fixed,font=\tiny},ymin=0,ymax=0.3
        ]
        \nextgroupplot
            \addplot [line width = 1pt] table [x expr=\coordindex+1, y index=0] {maxvolnmf/figures/jasper4.txt};
        \nextgroupplot
            \addplot [line width = 1pt,blue] table [x expr=\coordindex+1, y index=1] {maxvolnmf/figures/jasper4.txt};
        \nextgroupplot
            \addplot [line width = 1pt,red] table [x expr=\coordindex+1, y index=2] {maxvolnmf/figures/jasper4.txt};
        \nextgroupplot
            \addplot [line width = 1pt,teal] table [x expr=\coordindex+1, y index=3] {maxvolnmf/figures/jasper4.txt};
        \end{groupplot}
	\end{tikzpicture}\hfill
	\caption[Abundance maps and endmembers by Normalized MaxVol NMF with $r=4$ on Jasper]{Abundance maps and endmembers (tree, {\color{blue}water}, {\color{red}soil}, {\color{teal}road}) by Normalized MaxVol NMF with $r=4$ on Jasper, with $\lambda=2$ and $\delta=1$.}
    \label{maxvolnmf:fig:jasper4}
\end{figure}
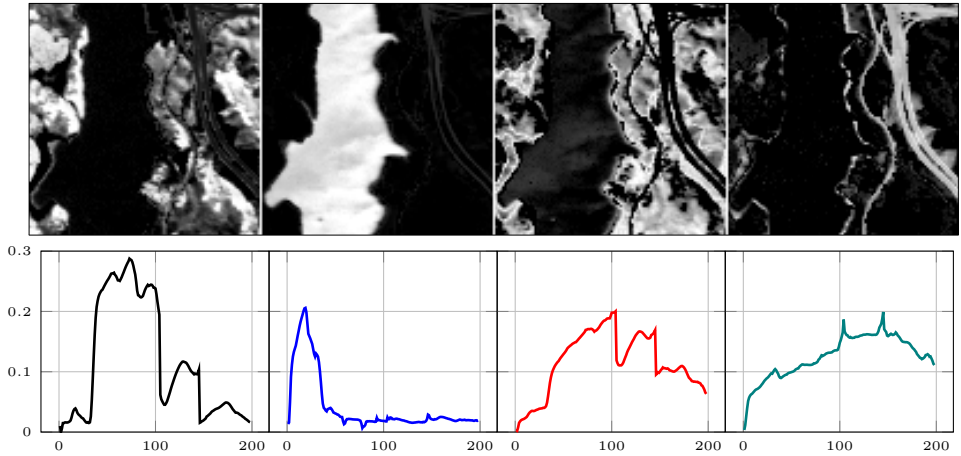

\begin{figure}[htbp!]
	\hfill\fbox{\includegraphics[width=0.95\linewidth]{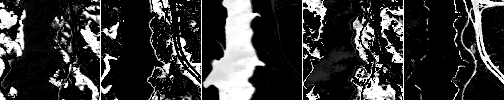}}\hfill

    \hfill\begin{tikzpicture}
        \begin{groupplot}[
            group style={
                group name=my plots,
                group size= 5 by 1,
                x descriptions at=edge bottom,
                y descriptions at=edge left,
                horizontal sep=0.0cm,
                vertical sep=0.cm,
            },
			width=0.31\linewidth,no markers,grid=major,xlabel={},ylabel={},tick label style={/pgf/number format/fixed,font=\tiny},ymin=0,ymax=0.3
        ]
        \nextgroupplot
            \addplot [line width = 1pt] table [x expr=\coordindex+1, y index=0] {maxvolnmf/figures/jasper5.txt};
        \nextgroupplot
            \addplot [line width = 1pt,red] table [x expr=\coordindex+1, y index=1] {maxvolnmf/figures/jasper5.txt};
        \nextgroupplot
            \addplot [line width = 1pt,blue] table [x expr=\coordindex+1, y index=2] {maxvolnmf/figures/jasper5.txt};
        \nextgroupplot
            \addplot [line width = 1pt,orange] table [x expr=\coordindex+1, y index=3] {maxvolnmf/figures/jasper5.txt};
        \nextgroupplot
            \addplot [line width = 1pt,teal] table [x expr=\coordindex+1, y index=4] {maxvolnmf/figures/jasper5.txt};
        \end{groupplot}
	\end{tikzpicture}\hfill
	\caption[Abundance maps and endmembers by Normalized MaxVol NMF with $r=5$ on Jasper]{Abundance maps and endmembers (tree, {\color{red}soil}, {\color{blue}water}, {\color{orange}tree+soil}, {\color{teal}road}) by Normalized MaxVol NMF with $r=5$ on Jasper, with $\lambda=0.5$ and $\delta=0.5$.}
    \label{maxvolnmf:fig:jasper5}
\end{figure}

\noindent\begin{minipage}{\textwidth}
    \begin{remark}
        On the results of our model, only the shape of the spectral signatures should be considered when comparing results, while the amplitude of the spectral signatures are to be considered with a pinch of salt due to the scaling ambiguity between $W$ and $H$. Let us remind that Normalized MaxVol NMF is not simplex structured. In fact, we could take advantage of this scaling ambiguity to improve the condition number when updating a block, but this is not the goal of this chapter.
    \end{remark}
\end{minipage}

\section{Conclusion}\label{maxvolnmf:sec:conclusion}

In this chapter, we introduced MaxVol NMF, an analogue version of MinVol NMF where the volume of $H$ is maximized instead of the volume of $W$ being minimized. Just like MinVol NMF, this new model is identifiable. We also developed two different algorithms to solve MaxVol NMF. We introduced Normalized MaxVol NMF, a variant where the volume of the row wise normalized $H$ is maximized. This model creates a continuum between NMF and ONMF and exhibits better results than MinVol NMF on hyperspectral unmixing. Its identifiability remains an open question. Similarly, a normalized version of MinVol NMF could be interesting. This could be seen as a minimum aperture NMF. One could say that this already exists through MinVol NMF with the constraint $e^\top W=e^\top$, which is partially true. For a fixed aperture, the volume of $W$ is changing depending on where the columns of $W$ are on the probability simplex. In other words, there is a little bias drawing the columns of $W$ towards $e$, which is not the case with a normalized version of MinVol NMF. Normalized MaxVol NMF and Normalized MinVol NMF could be combined to control the spectral variability when increasing the rank. Finally, the performance of Normalized MaxVol NMF should be evaluated on other kinds of data.

\chapter{Highlight of the contributions and discussions}\label{chap:conclu}

\begin{hiddenmusic}{Andrew Prahlow - Final Voyage}{https://www.youtube.com/watch?v=6zlSUvWU6z8} 
\end{hiddenmusic}In the conclusion, we first summarize the contributions of this thesis. We then discuss perspectives that could follow this work. 

\section*{Summary}\addcontentsline{toc}{section}{Summary}

The three main motivations of this thesis were the applications, algorithms and theory related to identifiable volume-based regularized matrix factorization models. Except for separable NMF in \Cref{chap:randspa}, all studied models can be seen as volume-regularized models. With BSSMF, in \Cref{chap:bssmf}, the columns of $W$ lie in a chosen hyperrectangle and the columns of $H$ lie in the probability simplex. With PMF, in \Cref{chap:polytopicmf}, the rows of $W$ and the columns of $H$ lie in respective chosen polytopes. For BSSMF and PMF, the volume regularization is in fact a hard constraint. Such constraint can always be translated to a regularization term, using an indicator function that is added to the cost function. For instance, the constraint $H\in\Delta^{r\times n}$ can be replaced by adding $\iota_{\Delta^{r\times n}}(H)$ to the cost function, where $\iota_{\Delta^{r\times n}}(H)$ outputs $0$ if every column of $H$ lies in $\Delta^r$, and $\infty$ otherwise. With MinVol NMF, in \Cref{chap:minvolnmf}, the volume of the convex hull formed by the columns of $W$ and the origin is minimized. With MaxVol NMF, in \Cref{chap:maxvolnmf}, it is the volume of the convex hull formed by the rows of $H$ and the origin that is maximized. \\

In terms of applications, each model is useful for different reasons: 
\begin{itemize}

    \item BSSMF retrieves datalike features when the data is naturally bounded. In a way, to rephrase the well known ``NMF learns parts of objects'', we can say that ``BSSMF learns meaningful objects''. We also showed how BSSMF is a better basis for recommender systems than NMF.
    
    \item Separable NMF is useful when the looked for features are assumed to be data points themselves, also called the \textit{separability assumption}. This is the case in some blind source unmixing applications when each source is observed purely at least once, like it can be for hyperspectral unmixing for instance.
    
    \item When the separability assumption does not hold anymore, either due to the noise and outliers or due to the absence of pure endmembers among the data points, MinVol NMF is often a good alternative. It has also been used for other applications where NMF already showed its capabilities, like blind source separation problems~\cite{leplat2019blind}, facial feature extraction~\cite{zhou2011minimum} or community detection~\cite{huang2019detecting}, to cite a few. Additionnaly, we showed how the MinVol criterion is promising as a regularizer for the task of matrix completion.
    
    \item MaxVol NMF creates a continuum between NMF and ONMF. It inherits from the same behaviors as MinVol NMF, but with more control on the sparsity of the decomposition. Actually, in the inexact case, MinVol NMF directly regularizes the volume of $W$, which indirectly affects the volume of $H$ and offers little control on the sparsity of the decomposition. On the contrary, MaxVol NMF directly regularizes the volume of $H$, offering more control on the sparsity of the decomposition, and indirectly affects the volume of $W$. In the context of HU, for datasets close enough to the separability assumption due to noise, MaxVol NMF seems to exhibit better results than MinVol NMF. It should be noted that MaxVol NMF is probably less powerful for datasets composed only of mixtures. MaxVol NMF where the rank is overestimated also shows interesting results to take into account spectral variability. 
\end{itemize}

Except for PMF\footnote{The main reason for not developping an algorithm for PMF is that this model is too generic. It is totally possible to use one of the many Frank-Wolfe based algorithms to derive an alternated block scheme for any PMF. However, these algorithms are not very fast. For specific polytopes, it is probably faster to compute the projection on the polytope after a gradient descent step. We empirically noticed this when the polytopes are the probability simplex, where projected gradient descents were faster than alternating with Polyhedral Coordinate Descent method with Away steps~\cite{mazumder2023cyclic} for instance.}, we developed fast algorithms for every studied model:
\begin{itemize}
    \item For BSSMF and MinVol NMF, we derived instances from an inertial block majorization minimization framework for nonsmooth nonconvex optimization, called TITAN~\cite{hien2023inertial}.
    \item For separable NMF, we developed RandSPA. It creates a continuum between SPA and VCA, which are fast greedy algorithms for column subset selection. RandSPA uses the best from both worlds if tuned accordingly, that is, the robustness of SPA and the randomness of VCA.
    \item For MaxVol NMF, the algorithm developed for MinVol NMF could not be used anymore. Hence, we developed two algorithms. One is Adgrad2, based on \cite{malitsky2020adaptive}, and the other is based on ADMM, combined with an appropriate Bregman surrogate adapted from~\cite{dragomir2021quartic}.
\end{itemize}
\pagebreak

\section*{Further research}\addcontentsline{toc}{section}{Further research}

\paragraph*{Applications} Recommender systems were the main motivation for creating BSSMF. Even if we showed that BSSMF performs better than NMF, the question remains on the competitivity of BSSMF against the state-of-the-art algorithms used for recommender systems. Of course, BSSMF would need to be customized, e.g., by adding some wisely chosen regularizers.

Normalized MaxVol NMF has only been used for HU. This model could be useful in other applications, like document clustering and recommender systems. Also, it should be noted that the maximum-volume criterion was originally thought as a regularizer for Bilinear NMF. This combination still needs to be explored. Normalized MaxVol NMF also opened the path to normalized MinVol NMF. The difference with vanilla MinVol NMF is that instead of minimizing the volume formed by the endmembers, the aperture between the endmembers is minimized. We mentioned that this behavior coupled with an overranked normalized MaxVol NMF would be able to control the spectral variability in HU. This still needs to be addressed properly, as well as other potential applications for normalized MinVol NMF.

\paragraph*{Algorithms} We showed that by taking the best run among several, RandSPA outperforms SPA and VCA. However, RandSPA could be further improved if we could learn a good $Q$ matrix. A first idea would be to use an internal loop, instead of running the algorithm several times and saving the best run. The reason is that a good RandSPA run needs $r$ successive good draws of random $Q$'s. If at each selection step we could draw several $Q$'s directly and chose the best one based on a proper criterion, this would probably improve the performance of one run of  RandSPA. The main question is then which criterion would be a good one.

\paragraph*{Theory} All the studied models were either known to be identifiable (separable NMF and MinVol NMF), or proven in this thesis to be identifiable (BSSMF, PMF, $\ell_1$-MinVol NMF, MaxVol NMF).
We also studied/mentioned some other variants of MinVol NMF and MaxVol NMF, namely 
\begin{itemize}
    \item normalized MaxVol NMF~\eqref{maxvolnmf:eq:exactmaxvol}, 
    
    \item normalized MinVol NMF, whose identifiability result should be similar to that of normalized MaxVol NMF, 
    
    \item MinVol NMF with a Frobenius penalty and without simplex structure. This model was defined for the inexact and missing data case in~\eqref{minvol:eq:newminvol}. Even when no data are missing, that is when $\P_\Omega$ is the identity~\eqref{minvol:eq:fullexactfrominvol}, better conditions than \Cref{preli:th:uniqNMFSSC} are unknown. Experiments in \Cref{minvol:sec:exp} strongly suggest that there exist milder conditions. 
    
\end{itemize} 
Due to their nonnegative nature, these models are obviously identifiable under \Cref{preli:th:uniqNMFSSC} that requires both factors to satisfy the SSC. However, it is an open question to come up with milder conditions than \Cref{preli:th:uniqNMFSSC} to retain identifiability, like \Cref{minvol:th:uniqueminvol} (that requires only one factor to be SSC).

In general, the identifiability of many CLRMFs with missing data remains unknown and quite challenging. Without success, we tried during this thesis to find reasonable conditions under which separable NMF with missing data would be identifiable. 
Of course, it would be possible to use a two-step approach: first complete the data using the low-rank assumptions (the completion is unique if sufficiently many entries are observed in random locations, see, e.g., \cite{candes2010matrix}), then apply an identifiable NMF algorithms on the completed data. However, in general, single-step approaches perform significantly better in practice.








\printbibliography{}

\end{document}